%% file: arxiv.tex
\documentclass[11pt]{article}
\usepackage[utf8]{inputenc} 
\usepackage[T1]{fontenc}  
\usepackage{bm}
\usepackage{bbm}
\usepackage{dsfont}
\usepackage{xcolor}
\usepackage{amsmath,amsfonts,amssymb,amsthm,titling,url,array}
\usepackage{graphicx}
\usepackage{caption}
\usepackage{subcaption}

\usepackage{fullpage}
\usepackage{MnSymbol}
\usepackage{float}
\usepackage{hyperref}
\usepackage{authblk}

\allowdisplaybreaks[4]

\counterwithout{figure}{section}

\newtheorem{theorem}{Theorem}[section]
\newtheorem{lemma}[theorem]{Lemma} 
\newtheorem{remark}[theorem]{Remark}
\newtheorem{definition}[theorem]{Definition}

\input{notation.tex}

\title{Implicit Balancing and Regularization:\\Generalization and Convergence Guarantees for\\Overparameterized Asymmetric Matrix Sensing}
\author[1]{Mahdi Soltanolkotabi}
\author[2]{Dominik St\"oger}
\author[1]{Changzhi Xie} 

\affil[1]{University of Southern California}
\affil[2]{KU Eichst\"att-Ingolstadt}
\begin{document}
\maketitle 
\input{abstract.tex}


\input{introduction.tex}

\input{problem_formulation.tex}
\input{main_results.tex}

\input{related_work.tex}

\input{experiments.tex}

\input{proof_sketch.tex}

\input{acknowledgements.tex}

\bibliography{Bibfiles,Bibfiles2,literature}
\bibliographystyle{unsrt}

\newpage
\appendix

\section{Proof of Lemma \ref{lemma:RIPlemma}}\label{sec:proofauxiliarylemma}

\input{auxiliary_lemmas.tex}

\section{Analysis of the Spectral Phase (Proof of Lemma \ref{lemma:spectralmain})}\label{sec:spectralphaseanalysis}

\input{spectral_phase.tex}

\section{Proofs of technical lemmas for convergence analysis (Phase II+III)}

\input{lemma_sigmingrowth.tex}

\input{lemma_noisetermgrowth.tex}
\input{lemma_anglecontrol.tex}
\input{lemma_specnormcontrol.tex}

\input{lemma_balanced_base.tex}
\input{lemma_balanced_refined1.tex}

\input{lemma_balanced_refined2.tex}

\input{lemma_localconvergence.tex}

\section{Proofs of the main lemmas for Phase 2 and Phase 3}

\input{lemma_phase2combined.tex}
\input{lemma_phase3combined.tex}

\section{Proof of the main result, Theorem \ref{theorem:main}}

\input{proof_mainresult.tex}

\end{document}

%% file: notation.tex
\newcommand{\norm}[1]{\Vert #1 \Vert}
\newcommand{\tr}{\text{trace}}
\newcommand{\fbnorm}[1]{\Vert #1 \Vert_F}
\renewcommand{\P}[1]{P_{#1}}
\newcommand{\Q}[1]{Q_{#1}}

\usepackage{enumitem}
\newcommand{\I}{\text{Id}}
\newcommand{\LA}{L_X}
\newcommand{\LAT}{L_X^T}
\newcommand{\LAP}{L_{X,\bot}}
\newcommand{\LAPT}{L_{X,\bot}^T}

\newcommand{\LsymAT}{L_{\text{sym}(X)}^T}
\newcommand{\LsymAP}{L_{\text{sym}(X),\bot}}
\newcommand{\LsymAPT}{L_{\text{sym}(X),\bot}^T}
\newcommand{\LtilA}{ \widetilde{ L_{X} }}
\newcommand{\LtilAT}{\widetilde{L_X}^T}
\newcommand{\LtilAP}{\widetilde{L_{X,\bot}}}
\newcommand{\LtilAPT}{\widetilde{L_{X,\bot}}^T}
\newcommand{\Zt}{Z_t}
\newcommand{\Ztone}{Z_{t_1}}
\newcommand{\Zttwo}{Z_{t_2}}

\newcommand{\ZtT}{Z_t^T}
\newcommand{\Zplus}{Z_{t+1}}
\newcommand{\ZplusT}{Z_{t+1}^T}
\newcommand{\symA}{\text{sym}(X)}

\newcommand{\Qt}{Q_t}
\newcommand{\QtT}{Q_t^T}

\newcommand{\Qtp}{Q_{t,\bot}}

\newcommand{\QtpT}{Q_{t,\bot}^T}
\newcommand{\Qplus}{Q_{t+1}}
\newcommand{\Qplusp}{Q_{t+1,\bot}}

\newcommand{\Ht}{H_t}

\newcommand{\tilZt}{\tilde{Z}_t}
\newcommand{\tilZtT}{\tilde{Z}_t^T}
\newcommand{\tilZplusT}{\tilde{Z}^T_{t+1}}
\newcommand{\tilLA}{\widetilde{\LA}}
\newcommand{\tilLAT}{\widetilde{\LA}^T}
\newcommand{\sglmin}{\sigma_{\min}}

\newcommand{\Deltat}{\Delta_t}

\newcommand{\Dt}{D_t}
\newcommand{\Dplus}{D_{t+1}}
\newcommand{\Dtwo}{D_{t_2}}
\newcommand{\Dthree}{D_{t_3}}

\newcommand{\Bcal}{\mathcal{B}}
\newcommand{\Acal}{\mathcal{A}}
\newcommand{\Id}{\text{Id}}
\newcommand{\nucnorm}[1]{\Vert #1 \Vert_{\ast}}
\newcommand{\twonorm}[1]{\big\Vert #1 \big\Vert_{2}}
\newcommand{\Loss}{\mathcal{L}}
\newcommand{\Losssym}{\mathcal{L}_{\text{sym}}}
\newcommand{\innerproduct}[1]{ \langle #1 \rangle }
\newcommand{\R}{\mathbb{R}}
\newcommand{\tilZplus}{\tilde{Z}_{t+1}}
\newcommand{\PZQ}{P_{\Zt \Qt}}
\newcommand{\PZQT}{P_{\Zt \Qt}^T}
\newcommand{\PZQperp}{P_{\Zt \Qt,\bot}}
\newcommand{\PZQperpT}{P_{\Zt \Qt,\bot}^T}

\newcommand{\hatK}{\hat{K}_t}

\newcommand{\tstar}{{t_{\star}}}

\newcommand{\Ztstar}{Z_{\tstar}}
\newcommand{\tilZtstar}{\tilde{Z}_{\tstar}}
\newcommand{\Etstar}{E_{\tstar}}
\newcommand{\tilEtstar}{\tilde{E}_{\tstar}}

\newcommand{\tone}{t_{1}}

%% file: abstract.tex
\begin{abstract}
Recently, there has been significant progress in understanding the convergence and generalization properties of gradient-based methods for training overparameterized learning models. However, many aspects including the role of small random initialization and how the various parameters of the model are coupled during gradient-based updates to facilitate good generalization remain largely mysterious. 
A series of recent papers have begun to study this role for non-convex formulations of symmetric Positive Semi-Definite (PSD) matrix sensing problems which involve reconstructing a low-rank PSD matrix from a few linear measurements. 
The underlying symmetry/PSDness is crucial to existing convergence and generalization guarantees for this problem. 
In this paper, we study a general overparameterized low-rank matrix sensing problem where one wishes to reconstruct an asymmetric rectangular low-rank matrix from a few linear measurements.
We prove that an overparameterized model trained via factorized gradient descent converges to the low-rank matrix generating the measurements. We show that in this setting, factorized gradient descent enjoys two implicit properties: (1) coupling of the trajectory of gradient descent where the factors are coupled in various ways throughout the gradient update trajectory and (2) an algorithmic regularization property where the iterates show a propensity towards low-rank models despite the overparameterized nature of the factorized model. These two implicit properties in turn allow us to show that the gradient descent trajectory from small random initialization moves towards solutions that are both globally optimal and generalize well.
\footnote{Accepted for presentation at the Conference on Learning Theory (COLT) 2023}
\end{abstract}

%% file: introduction.tex
\section{Introduction}

Over the past few years, there has been a significant amount of focus on understanding the optimization and generalization dynamics of \emph{overparameterized} learning problems.
Such overparameterized training, 
which involves training models with more parameters than training data,
is the contemporary learning paradigm for a variety of problems spanning deep neural networks to matrix factorization. Surprisingly, despite the existence of many global optima of the training loss with subpar generalization/prediction capability, these models avoid such \emph{overfitting} when trained via (stochastic) gradient descent \cite{zhang2016understanding}. 

Among others, there seem to be two critical components facilitating this success, putting the gradient updates on a trajectory towards parameters that are not only globally optimal but also generalize well.
The first critical component is the role of small random initialization which seems to guide the gradient trajectory towards low complexity models that tend to generalize better.
For instance in low-rank matrix reconstruction, small random initialization guides the trajectory towards low-rank solutions \cite{gunasekar2018implicit}. 
For neural networks, small random initialization plays a critical role in allowing the neural net to learn good features (a.k.a. \emph{feature learning} regime) facilitating generalization performance far superior to the features of the model learned with large initialization (whose features are essentially frozen at the random initialization a.k.a.~lazy training or NTK regime) \cite{chizat2019lazy}.
The second critical component is that the parameters of the different layers of the model are coupled in intricate ways during the gradient descent trajectory. Understanding these intricate couplings of the trajectory of gradient descent is thus crucial to understand the generalization dynamics during training.

Recently, there has been interesting progress in understanding the role of small random initialization for a variety of problems ranging from positive semi-definite low-rank matrix sensing \cite{ LiMaCOLT18,stoger2021small} to one-hidden layer neural networks in a training regime where only one of the layers plays a dominant role in terms of generalization 
\cite{damian2022neural,mousavi2022neural,bietti2022learning}.
Despite this progress, the intricate couplings of the trajectory are far less understood.
In this paper, we wish to demystify both the role of small random initialization and the intricate coupling with an emphasis on the latter.

We focus on the problem of asymmetric low-rank matrix reconstruction from a few linear measurements in the overparameterized regime where the goal is to reconstruct a matrix from linear measurements of the form 
\begin{equation*}
 y_i 
 = \innerproduct{A_i,X}
 := \tr \left( A_i^T X  \right) \quad \quad \text{for } i \in [m]:= \left\{ 1;2; \ldots; m \right\},
\end{equation*}
Here, $X \in \mathbb{R}^{n_1 \times n_2}$ is the unknown asymmetric low-rank matrix of rank $r$ that we wish to recover and $A_i$ are known measurement matrices.
To recover the unknown matrix $X$, we consider the loss
\begin{equation}\label{equ:lossdefinition}
   \Loss \left(V,W\right) := \frac{1}{2}  \sum_{i=1}^m \left( y_i - \innerproduct{A_i, VW^T} \right)^2,
\end{equation}
and train both $V \in \mathbb{R}^{n_1 \times k} $ and $W \in \mathbb{R}^{n_2 \times k}$ via gradient descent.
In this paper, we are especially interested in the overparameterized scenario, i.e., $ k (n_1 + n_2) \gg m $.
We show that in this setting, starting from small random initialization the gradient descent iterates $V_t$ and $W_t$ follow a trajectory with two implicit properties: (1) an algorithmic regularization property where the iterates show a propensity towards low-rank models despite the overparameterized nature of the factorized model.  (2) a coupling of the trajectory of the two factors via  intricate balancing properties where the factors remain approximately balanced throughout training in various ways.
In this problem, the second property is also crucial to showing the first.
Together these two implicit properties allow us to show that the gradient descent trajectory from small random initialization moves towards solutions that are both globally optimal and generalize well.

Finally, we would like to emphasize that the coupling of the gradient trajectory of the two factors in the matrix sensing problem is quite intricate, even when compared to related problems such as matrix factorization.
For matrix factorization, one coupling that has been crucial in prior analysis is that the factors remain balanced in the sense that throughout the iterations we have $V_t^TV_t\approx W_t^TW_t$, see, e.g.,~\cite{du2018algorithmic,ye2021global,jiang2022algorithmic}. A similar balancing property does hold in matrix sensing as well. However, the behavior of the imbalance matrix $V_t^TV_t-W_t^TW_t$ is quite different. This contrast is depicted in Figure \ref{fig:balancevsinitialization} where we draw the spectral norm of the imbalance matrix in both cases. This figure demonstrates that imbalancedness increases initially in the matrix sensing problem before tapering off at a constant value. This is in sharp contrast with the matrix factorization problem where the imbalancedness is essentially constant throughout training with a much smaller constant value. As a result, a more careful analysis is required to control the imbalancedness in the matrix sensing problem. Furthermore, the analysis of the matrix sensing problem requires more intricate couplings of the trajectory. In particular, we demonstrate two additional couplings of the trajectory that are crucial for achieving strong generalization and optimization guarantees with a modest number of iterations. The first one is that the imbalance matrix is even smaller in certain ``nuisance" directions. This behavior is depicted in Figure \ref{fig:balancevariant} demonstrating that imbalancedness is orders of magnitude smaller in these nuisance directions. The second form of coupling is an angular form of balancedness demonstrating that an appropriate notion of angle between the imbalance matrix and certain ``signal" directions is sufficiently small across training as depicted in Figure \ref{fig:balanceangle}. In this paper, we show all of these coupling phenomena rigorously. These intricate couplings of the trajectory are crucial to our convergence analysis allowing us to show that despite overparameterization the trajectory of gradient descent leads to solutions with good generalization with fast convergence rates. 

\begin{figure}[t]
\begin{subfigure}[b]{0.3\textwidth}
\includegraphics[width=45mm]{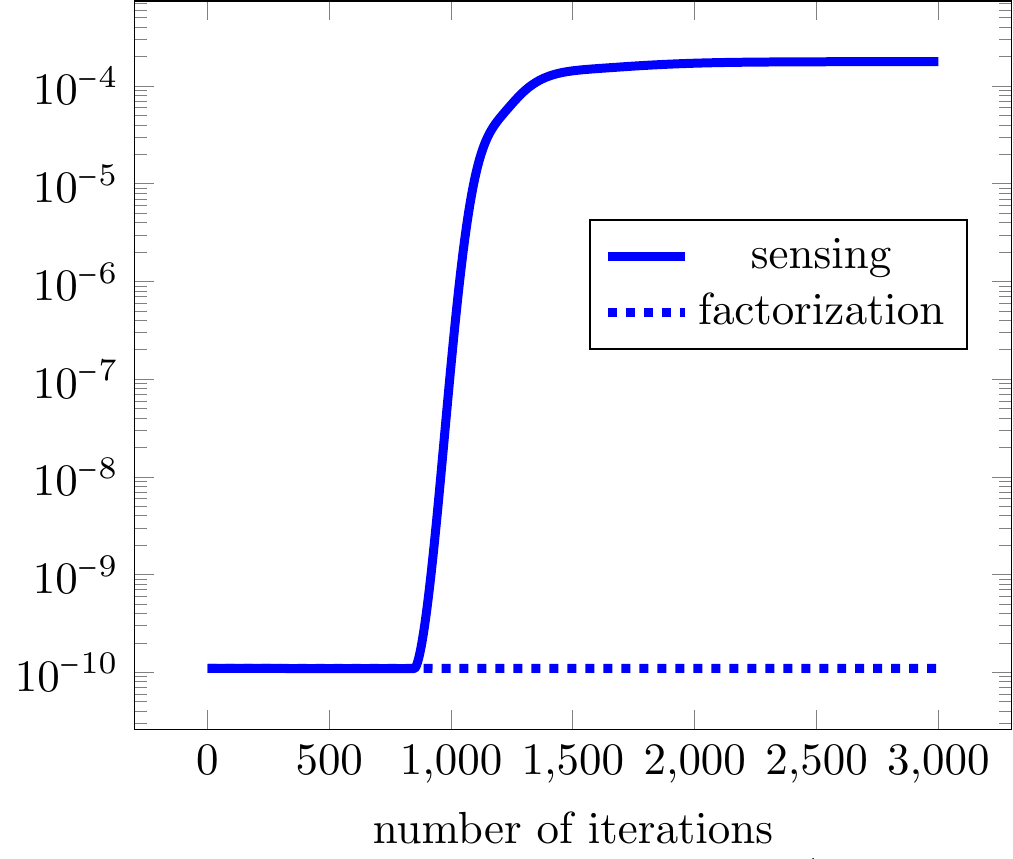}
\caption{}
\label{fig:balancevsinitialization}
\end{subfigure}
\begin{subfigure}[b]{0.3\textwidth}
\includegraphics[width=45mm]{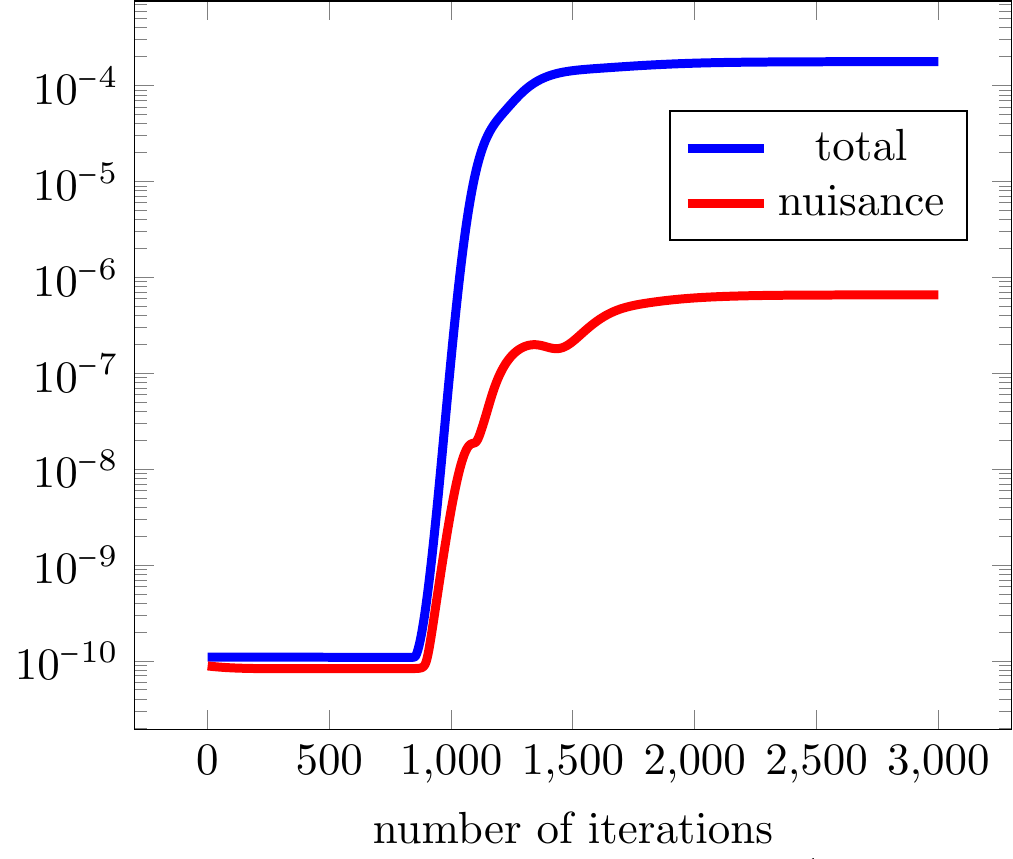}
\caption{}
\label{fig:balancevariant}
\end{subfigure}
\begin{subfigure}[b]{0.3\textwidth}
\includegraphics[width=45mm]{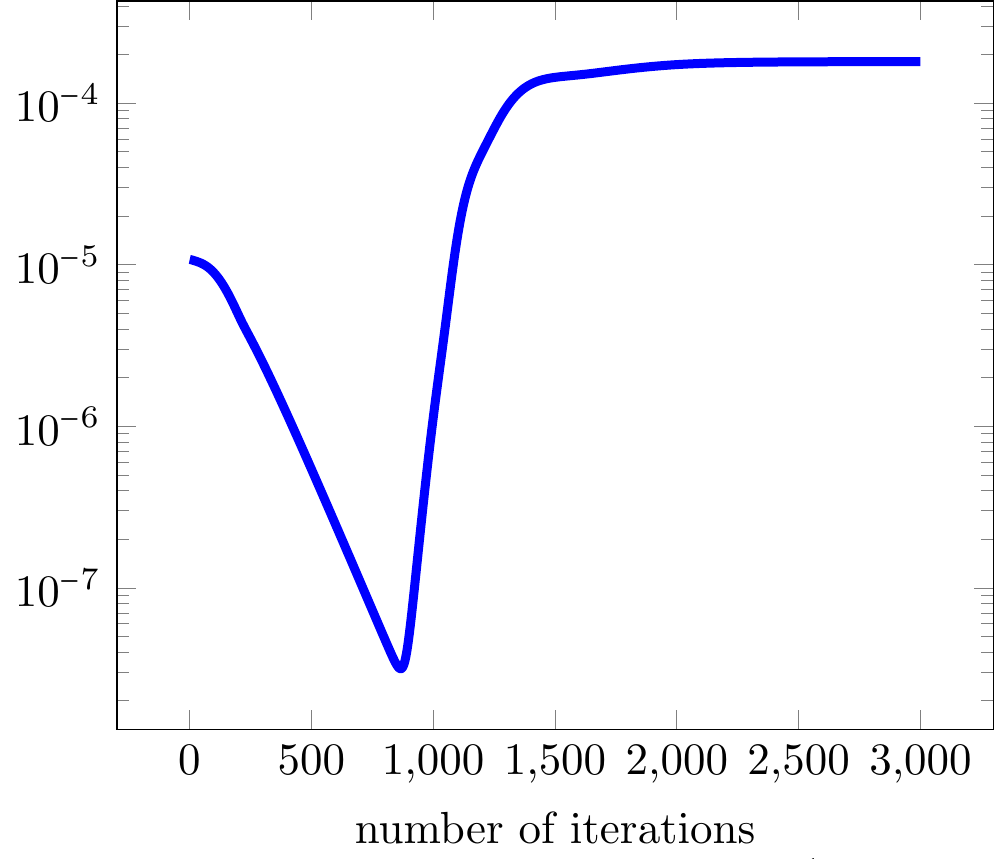}
\caption{}
\label{fig:balanceangle}
\end{subfigure}
\caption{This figure depicts various couplings between the trajectories of the two factors $V_t$ and $W_t$ throughout training. First, (a) depicts the spectral norm of the imbalance matrix $V_t^TV_t-W_t^TW_t$ ($\|V_t^TV_t-W_t^TW_t\|$) for matrix sensing (this paper) and matrix factorization. While the size of imbalancedness remains essentially constant in the matrix factorization problem and is very small, it actually increases in the matrix sensing problem before settling at a small value (albeit orders of magnitude larger than in the matrix factorization case). This shows that the two factors continue to be balanced for matrix sensing but the nature of this balancedness is much more intricate. Second, (b) shows that the size of the imbalance matrix is orders of magnitude smaller in certain ``nuisance'' directions. Such a more refined control of imbalancedness is crucial in our proofs. 
Finally, in (c) we consider a quantity regarding the 
imbalancedness of $V_t^T V_t -W_t^T W_t$ after 
some normalization using the singular values of $V_t$ and $W_t$ has been applied.
Figure c) indicates that this quantity stays small during training.
This particular coupling of the trajectory also plays a critical role in our proofs.}
\label{allfigs}
\end{figure}

%% file: problem_formulation.tex
\section{Problem Formulation}
In this paper, we focus on reconstructing a low-rank matrix $X\in \R^{n_1 \times n_2}$ of rank $r$ from  $\left[m\right]:= \left\{ 1, 2, \ldots, m \right\}$ linear measurements of the form
\begin{equation}\label{equ:measurementsequation}
    y_i = \innerproduct{A_i, X}:= \tr\left( A_i X^T \right) ,
\end{equation}
where $ \left\{ A_i\right\}_{i=1}^m \subset \mathbb{R}^{n_1 \times n_2}$ represent known measurement matrices. The linear system in \eqref{equ:measurementsequation} can also be rewritten in the more compact form $y = \Acal \left(X\right)$, where $ \mathcal{A}: \mathbb{R}^{n_1 \times n_2} \rightarrow \mathbb{R}^m$ is a linear measurement operator defined as $\left[ \mathcal{A} \left(Z\right) \right]_i = \innerproduct{A_i, Z} $ and $y\in\mathbb{R}^m$ is a vector consisting of all measurements.

To find the low-rank matrix, we consider the natural factorized loss function 
\begin{equation}\label{equ:lossdefinition}
   \Loss \left(V,W\right) := \frac{1}{2}  \sum_{i=1}^m \left( y_i - \innerproduct{A_i, VW^T} \right)^2=\frac{1}{2} \twonorm{ y - \Acal \left( VW^T\right) }^2,
\end{equation}
where $V \in \R^{n_1 \times k} $ and $ W \in \R^{n_2 \times k} $ are possibly overparameterized factors ($k\ge r$). To minimize this loss we train the factors $V$ and $W$ via gradient descent updates of the form 
\begin{align}
    V_{t+1}=&V_t - \mu \nabla_V \Loss (V_t,W_t)= V_t - \mu \left[ \left(\Acal^* \Acal \right) \left( V_t W_t^T - X \right)  \right] W_t, \label{Vupdate}\\
    W_{t+1}=&W_t - \mu \nabla_W \Loss (V_t,W_t)= W_t - \mu \left[ \left(\Acal^* \Acal \right) \left( V_t W_t^T - X \right)  \right]^T V_t\label{Wupdate}
\end{align}
starting from initial factors $V_0 \in \R^{n_1 \times k} $ and $ W_0 \in \R^{n_2 \times k} $ with step size $\mu >0$. Here, $ \Acal^*$ denotes the adjoint operator of $\Acal$.

\subsection*{Notation}

For any matrix $M \in \mathbb{R}^{d_1 \times d_2} $, we denote its Frobenius norm by $ \norm{M}_F := \sqrt{\text{trace} \left(A^T A\right) } $, its spectral norm by $ \norm{M} $, and its nuclear norm (i.e., the sum of the singular values of $M$) by $ \nucnorm{M}$. Moreover, we denote the singular value decomposition of $ M$ by $ M =  P_{M}  \Sigma_{ M}  Q_{ M}^T$ with
\begin{equation*}
\Sigma_{ M} = \text{diag} \left( \sigma_1 (M),\sigma_2 (M),\ldots,\sigma_{\text{rank} (M)} (M) \right),
\end{equation*}
where
$\sigma_1  (M)\ge \sigma_2 (M) \ge \ldots \ge \sigma_{\text{rank} (M)} (M)>0$
denote the singular values of the matrix $M$. We also use $P_{M,\bot} \in \mathbb{R}^{d_1 \times (d_2 - \text{rank} (M) )} $ to denote a matrix whose columns are orthonormal and orthogonal to the column span of $ P_M $.
The matrix $Q_{M,\bot} \in \mathbb{R}^{d_1 \times (d_2 -\text{rank} (M)  ) }$ is defined analogously. 
The set of symmetric matrices in $ \mathbb{R}^{d \times d} $ is denoted by $ \text{Sym}_{d}$.
Finally, for a symmetric matrix $ S \in \mathbb{R}^{d \times d} $ we denote its eigenvalues by
$\lambda_1 (S) \ge  \lambda_2 (S) \ge \ldots  \ge \lambda_d (S)$. 

%% file: main_results.tex
\section{Main results}
In this section, we present our main results. We begin with a few preliminary definitions.
\begin{definition}[Condition number]
We denote the condition number of the ground truth matrix $X$ by 
\begin{equation*}
    \kappa 
    :=
    \frac{\norm{X}}{\sigma_{r} \left(X\right)}.
\end{equation*}
\end{definition}
\noindent The second definition concerns the measurement operator $\mathcal{A}$.
\begin{definition}[Restricted Isometry Property \cite{RechtFazel}]
    We say that the measurement operator  $\mathcal{A}: \R^{n_1 \times n_2} \rightarrow \R^m $ has the restricted isometry property (RIP) of order $r$ with constant $\delta >0$, if for all matrices $M \in \R^{n_1 \times n_2}$ of rank at most $r$ it holds that
    \begin{equation*}
        \left(1-\delta\right) \fbnorm{M}^2 \le \twonorm{ \mathcal{A} \left(M\right) }^2 \le \left(1+\delta\right) \fbnorm{M}^2.
    \end{equation*}
\end{definition}
It is well known that if all entries of the measurement matrices $A_i$ are independent (sub-)gaussian random variables
with zero mean and variance $1/m$, 
then the operator $\mathcal{A}$ fulfills the restricted isometry property of order $r$ with constant $\delta >0$
if the number of measurements satisfies $ m \gtrsim \frac{r \left(n_1 + n_2\right) }{\delta^2} $, see \cite{candesplan,vershynin2018high}. With these notations in place, we are now ready to state the main result of this paper.
\begin{theorem}\label{theorem:main}
Let $X\in \R^{n_1 \times n_2}$ be a rectangular matrix of rank $r$
and assume that we are given measurements of the form $ y = \mathcal{A} \left(X\right) $. Furthermore,
assume that the linear measurement operator $\mathcal{A}$ satisfies the restricted isometry property of order $2r+1 $ with constant $ \delta \le \frac{c_1}{\kappa^3 \sqrt{r}} $.
To identify the matrix $X$ we run gradient descent iterates of the form $\left\{ V_t; W_t \right\}_{t \in \mathbb{N}}$ with step size $\mu$ per the updates \eqref{Vupdate} and \eqref{Wupdate} starting from the initialization factors $V_0= \alpha V \in \mathbb{R}^{n_1 \times k}$ and $W_0=\alpha W \in \mathbb{R}^{n_2 \times k} $ for some $k \ge r$.
Here, the matrices $V$ and $W$ have i.i.d. entries with distribution $\mathcal{N} \left(0,1\right)$ and $\alpha >0$ denotes the scale of initialization. Consider $0< \varepsilon <1$ and assume that the step size $\mu$ satisfies
    \begin{equation}\label{main:stepsizeassumption}
        \mu \le  \frac{c_2}{ \kappa^5 \norm{X}} \cdot \frac{1}{ \ln \left( \frac{2 \sqrt{2 \norm{X}} }{  \varepsilon \alpha \left( \sqrt{k} - \sqrt{r-1} \right) } \right) }.
    \end{equation}
Also assume that the scale of initialization $\alpha$ satisfies
    \begin{equation}\label{main:alphabound}
        \alpha \le  
    \frac{ c_3 \sqrt{ \norm{X}  } }{ k^{5}  \left( \max \left\{  n_1 + n_2;k \right\} \right)^2} \left( \frac{ \varepsilon \left( \sqrt{k} -\sqrt{r-1} \right) }{C_1 \kappa^2  \sqrt{\max \left\{  n_1 + n_2; k \right\}} } \right)^{C_2 \kappa }.
    \end{equation}
    Then, 
    after 
    \begin{equation}\label{lemma:tildetbound}
        T
        \lesssim \frac{ \ln \left( \frac{2  \sqrt{2 \norm{X}} }{ \varepsilon \alpha \left( \sqrt{k} - \sqrt{r-1} \right) } \right) }{\mu \sglmin (X)}
    \end{equation}
    iterations, with probability at least $1 -  C_3 \exp \left( -c_4 k \right) + \left(C_4 \varepsilon\right)^{k-r+1} $ it holds that
    \begin{equation}\label{equ:mainresultfinalbound}
    \frac{\norm{V_{T} W_{T}^T - X}}{\norm{X}}
    \lesssim 
    \frac{\alpha^{3/5} }{ \norm{X}^{3/10} }.
    \end{equation}
    Here, $C_1,C_2,C_3,C_4,c_1,c_2,c_3,c_4>0$ are fixed numerical constants.
\end{theorem}
A few comments are in order.

\noindent\textbf{Impact of the scale of initialization on the reconstruction error of $X$:}
    Note that $\norm{V_{T} W_{T}^T - X} $ can be interpreted as the reconstruction error.
    From inequality \eqref{equ:mainresultfinalbound} it follows 
    that this reconstruction error can be made arbitrarily small by choosing the scale of initialization $\alpha$ small enough.
    We note that polynomial dependence of the reconstruction error on $\alpha$ is also observed in our experiments, see Section \ref{sec:experiments}.
   
\noindent\textbf{Overparameterization:}
    Our result holds for any choice of $k$ (the number of columns of the factors $U_t$ and $V_t$), which determines the number of parameters in the training model. Note that in overparameterized models, i.e., $k(n_1+n_2) \gg m $, there may be infinitely many global minimizers of the loss function $\mathcal{L}$ in \eqref{equ:lossdefinition} with arbitrarily large test error. Despite that, our result guarantees that for a sufficiently small random initialization, vanilla gradient descent finds the low-rank solution.
    
\noindent\textbf{Non-overparameterized setting ($k=r$):} In this special case, our result implies that if the measurement operator fulfills the restricted isometry property,
gradient descent with small, random initialization will converge to the ground truth matrix $X$ in
polynomial time. It is known that under the RIP assumption the loss landscape is benign in the
sense that there are no local optima that are not global and all saddles have a direction of negative
curvature. However, such results do not imply that vanilla gradient descent converges quickly, i.e., in polynomial time, to a global optimum, as gradient descent may take exponential time to escape
from saddle points \cite{du2017gradient}.
To the best of our knowledge, this is the first result in the non-overparameterized asymmetric setting $k=r$ which shows the convergence of vanilla gradient descent to the ground truth from a random
initialization using only the restricted isometry property in polynomial time. 
    
    \noindent\textbf{Sample complexity:} If $\mathcal{A}$ is a Gaussian measurement operator, we need
    $m \gtrsim \kappa^{6} r^2 (n_1 +n_2)$ measurements to guarantee that the assumption on the RIP constant $\delta$ holds with high probability.
    Thus, the number of required measurements depends only on the rank of the ground truth matrix $X$, which is $r$, and not on the number of columns $k$ of $V_t$ and $W_t$, which determines the number of parameters of our training model.
    
\noindent\textbf{Step size:}
    The bound on the step size, \eqref{main:stepsizeassumption}, depends logarithmically on the scale of initialization $\alpha$.
    This assumption is necessary for us to show that the imbalance matrix $V_t^TV_t-W_t^TW_t$ remains bounded.
    However, as we will explain in Section \ref{sec:proofoutline}, to achieve this near-optimal dependence of the step size on $\alpha$ we need to conduct a fine-grained analysis of the evolution of this imbalance matrix $ V_t^T V_t -W_t^T W_t$ during training. This intricate analysis goes beyond just controlling the spectral norm of this imbalance matrix. In particular, as also mentioned in the introduction, we also need to show more intricate forms of coupling of the $V_t$ and $W_t$ trajectories.

%% file: related_work.tex
\section{Related work}

\textbf{Overparameterization in low-rank matrix recovery:} In \cite{gunasekar2017implicit}, the authors observed that in overparameterized low-rank matrix recovery models such as matrix sensing or matrix completion, gradient descent with a small, random initialization converges to a low-rank solution. 
Moreover, in this work it was conjectured, that for sufficiently small, random initialization gradient descent converges to a solution that is (almost) the nuclear norm minimizer, a common heuristic for finding the solution with minimal rank.
However, in  \cite{razin2020implicit} some examples have been constructed where the conjecture in  \cite{gunasekar2017implicit} does not hold.  
In \cite{LiMaCOLT18}, the authors show that small random initialization in symmetric matrix sensing (with a positive definite ground truth matrix) convergences to the ground truth matrix in the special case $k=n_1=n_2=n$.
A more recent paper \cite{stoger2021small} improved this result allowing for an arbitrary overparameterization parameter $k$ and also for an arbitrary scale of initialization $\alpha$.
A key insight in this work was that in the first few iterations, gradient descent exhibits a spectral bias.
A similar observation appeared in \cite{li2021towards}, which argues that gradient flow with sufficiently small initialization can be regarded as a rank-minimization heuristic.
Building upon the framework in \cite{stoger2021small}, this intuition has been made quantitative in \cite{jin2023understanding}.
The results in \cite{stoger2021small} have been generalized to the noisy case in \cite{ding2022validation}. 
In \cite{xu2023power} it was shown that preconditioned gradient descent leads to faster convergence compared to vanilla gradient descent for overparameterized (symmetric) low-rank matrix recovery problems with small random initialization.
We also note that this spectral bias has also been used in some of the algorithms and their corresponding analysis for avoiding strict saddles, e.g.,~see the interesting work in \cite{o2023line}.

Despite all of this interesting work, there has been much less understanding of the asymmetric matrix sensing problem. Very recently, the asymmetric case of the matrix sensing problem has been studied in the limit of the number of samples going to infinity (a.k.a.~population case) which is equivalent to the  matrix factorization problem.
In \cite{du2018algorithmic}, the authors show in this matrix factorization scenario that the factors $V_t$ and $W_t$ stay balanced and their product converges to the ground truth. However, no convergence rate was provided.
In \cite{ye2021global}, the same authors improved this result and showed that for matrix factorization in the non-overparameterized case, i.e., $k=r$, gradient descent from small, random initialization converges to the ground truth matrix with a polynomial amount of iterations.
Building upon \cite{ye2021global}, the paper \cite{jiang2022algorithmic} generalized these results to arbitrary rank $k$ among several other improvements.
However, the proof techniques of these works do not generalize to the finite sample scenario (i.e.~matrix sensing), which is the topic of this paper. The reason is that much more nuanced forms of coupling of the trajectories of $V_t$ and $W_t$ are required for matrix sensing as discussed in Figure \ref{allfigs} and further discussed in the proof and experimental sections.

Let us also comment on several other related works.
In \cite{ding2022flat}, it has been shown that in overparameterized low-rank matrix recovery models, flat minima, measured by the trace of the Hessian, have better generalization properties.
There is also recent work on non-convex subgradient methods in low-rank matrix recovery
when the data is grossly corrupted by noise \cite{ma2021sign,ding2021rank}.
However, in contrast to our result, in these works, the sample complexity scales with $k$ and not with $r$.
In \cite{arora2019implicit} authors focus on low-rank matrix recovery for deeper models, i.e., those which have more than two factors, and show that these models also exhibit a certain low-rank bias.

\textbf{Connection to quadratically reparameterized gradient flow in linear regression:}
It has been shown that in linear regression with quadratically reparameterized gradient flow, respectively gradient descent, with vanishingly small initialization converges to a solution, which corresponds to the $\ell_1$-minimizer, see  \cite{vaskevicius2019implicit,woodworth2020kernel,chou2021more,chou2022non}.
This can be interpreted as the commutative version of the problem studied in this paper.
A key insight in this line of research is that gradient flow on the factorized loss function is equivalent to mirror flow with an appropriately chosen Bregman divergence.
However, due to non-commutativity of the matrix multiplication this equivalence does not hold for the problem studied in this paper, see also \cite{li2022implicit}.
Finally, let us mention that in \cite{wu2021implicit} it has been shown that for low-rank matrix recovery, mirror descent equipped with a suitable Bregman divergence and starting from vanishingly random initialization converges to a solution that is vanishingly close to the nuclear norm minimizer.
However, as mentioned above, unlike in the (commutative) vector case, in the non-commutative matrix case it is unclear to which extent mirror descent connects to gradient descent on the factorized loss function.

\noindent\textbf{(Deep) Linear models and Balancing:} 
A variety of works \cite{bartlett2018gradient,arora2018optimization,arora2018convergence,bah2019learning,chou2020gradient,nguegnang2021convergence,tarmoun2021,Min2021}  studied the convergence of gradient flow and gradient descent for deep linear neural networks of the form
\begin{equation*}
 \underset{W_1, W_2, \ldots, W_N}{\min} \  \sum_{i=1}^m \big\Vert W_N \ldots W_2 W_1 x_i - y_i   \big\Vert^2.
 \end{equation*}
While this model cannot directly be compared with the one studied in this paper, the aforementioned papers also rely on the implicit coupling/balancing effect of the gradient descent algorithm. 

Finally, we note that in \cite{wang2021large} it has been observed that in a low-rank matrix factorization model of the form $ \Vert X-VW^T \Vert_F^2 $ gradient descent with a large step size implicitly balances the factors $V_t$ and $W_t$.
At first glance, this may look like a contradiction with the results presented in this paper.
However, note that \cite{wang2021large} assumes such a large step size that even the loss is not monotonically decreasing in the beginning.
Thus, their work operates in a very different regime (which is sometimes referred to as the "Edge of Stability" \cite{cohen2020gradient}).

\noindent\textbf{Non-convex optimization for low-rank matrix recovery in the non-overparameterized scenario:}
A variety of models in statistics, signal processing, and machine learning can be formulated as low-rank matrix recovery problems such as matrix completion \cite{candes_recht_MC,CandesMatrixComp2}, phase retrieval \cite{candes_strohmer,CESV12}, and blind deconvolution \cite{blind_deconvolution,ling2017blind}. Historically, one approach to solving such problems is via lifting techniques in convex relaxations such as nuclear norm minimization \cite{RechtFazel} which was the subject of intense study. We refer to \cite{davenport2016overview,fuchs2022proof} for an overview. 
However, since lifting increases the number of optimization variables, the nuclear norm minimization approach is computationally less efficient as non-convex approaches using a factorized gradient descent approach.
While the literature is too vast to give a complete overview of non-convex optimization for low-rank matrix recovery, in the following, we try to give an account of how we see our work positioned in this field.
For a more complete overview we refer to \cite{chen_overview}.
In recent years, numerous papers have analyzed gradient-descent based methods in low-rank matrix recovery, e.g., in matrix sensing \cite{tu2016low}, matrix completion \cite{xiaodong_matrixcompletion}, phase retrieval \cite{wirtinger_flow,truncated_wirtinger_flow,chen_implicit_regularization}, and blind deconvolution \cite{ling_blind_deconv,ling_demixing}. However, all of these results rely on spectral initialization.
That is, instead of using a random initialization one uses a carefully designed starting matrix as an initialization which is already close to the ground truth solution.
Moreover, note that many of these papers require adding a specific regularization term to enforce balancedness in the asymmetric scenario. 
An exception is \cite{ma2021balancingfree}, which proves convergence of vanilla gradient from spectral initialization (without adding any additional regularization).

Practitioners often use random initialization as it is model-agnostic.
Subsequently, several papers \cite{landscape_phaseretrieval,ge2016matrix_spurious,zhang2019sharp} have analyzed the loss landscape of non-convex formulations and show that the loss landscapes in these cases are benign in the sense that they do not admit spurious local minima and all saddle points have a direction of strict negative curvature.
In particular, those results imply that specialized solvers such as trust region methods, cubic regularization \cite{nesterov2006cubic, nocedal2006trust}, or noisy (stochastic) gradient-based methods \cite{jin2017escape,ge2015escaping,raginsky2017non,zhang2017hitting} can find the global optimum.
However, they do not explain why methods such as plain vanilla gradient descent can find the global optimum.

More recently, \cite{chen_global_convergence} shows that for phase retrieval, gradient descent using random initialization converges to the global optimum with a near-optimal amount of iterations. Moreover, in \cite{lee2022randomly} it has been shown that for rank-one matrix sensing, alternating least squares converges to the ground truth.
However, our understanding of why random initialization works so well in these settings is still very limited. Indeed, to the best of our knowledge, even in the non-overparameterized case where $k=r$, our work is the first one which shows convergence from a random initialization in the asymmetric scenario. 

%% file: experiments.tex
\section{Numerical experiments}\label{sec:experiments}
In this section, we conduct numerical experiments to verify our theoretical results.
In our experiments we set $X \in \mathbb{R}^{n_1 \times n_2}$ to be a random matrix with $n_1 = 100$, $n_2 = 50$ and rank $r = 5$.
Specifically, we generate random matrices $X_1\in \mathbb{R}^{n_1 \times r}$ and $X_2\in \mathbb{R}^{r \times n_2}$ whose entries are drawn from the Gaussian distribution $\mathcal{N} \left(0,1\right)$ and set $X = \frac{X_1X_2}{\norm{X_1X_2}}$.
We use $m = 2000$ random Gaussian measurements.

\subsection{Variations in imbalance with different initialization scales}
In our first experiment, we want to show how the spectral norm of the imbalance term  $\norm{V_t^TV_t-W_t^TW_t}$ evolves during training
for different choices of the scale of initialization $\alpha$.
To this aim, we randomly generate some $V_{rd}$ and $W_{rd}$ using the normal distribution and run gradient descent with $V_0 = \alpha V_{rd}, W_0 = \alpha W_{rd}$ for $\alpha = \{10^{-2}, 10^{-3}, 10^{-4}, 10^{-5}\}$.
We both consider the empirical loss \eqref{equ:lossdefinition} and the population loss 
\begin{equation}
\label{testerr}
   \Loss_{\text{population}} \left(V,W\right) = \frac{1}{2}  \big\Vert X -V W^T \big\Vert_F^2.
\end{equation}
Moreover, we set the step size $\mu = \frac{1}{100\|X\|}$. 
The results are depicted in Figure \ref{fig:balancedness_scaleinit}.

\begin{figure}[t]
    \centering
    \includegraphics[scale=0.6]{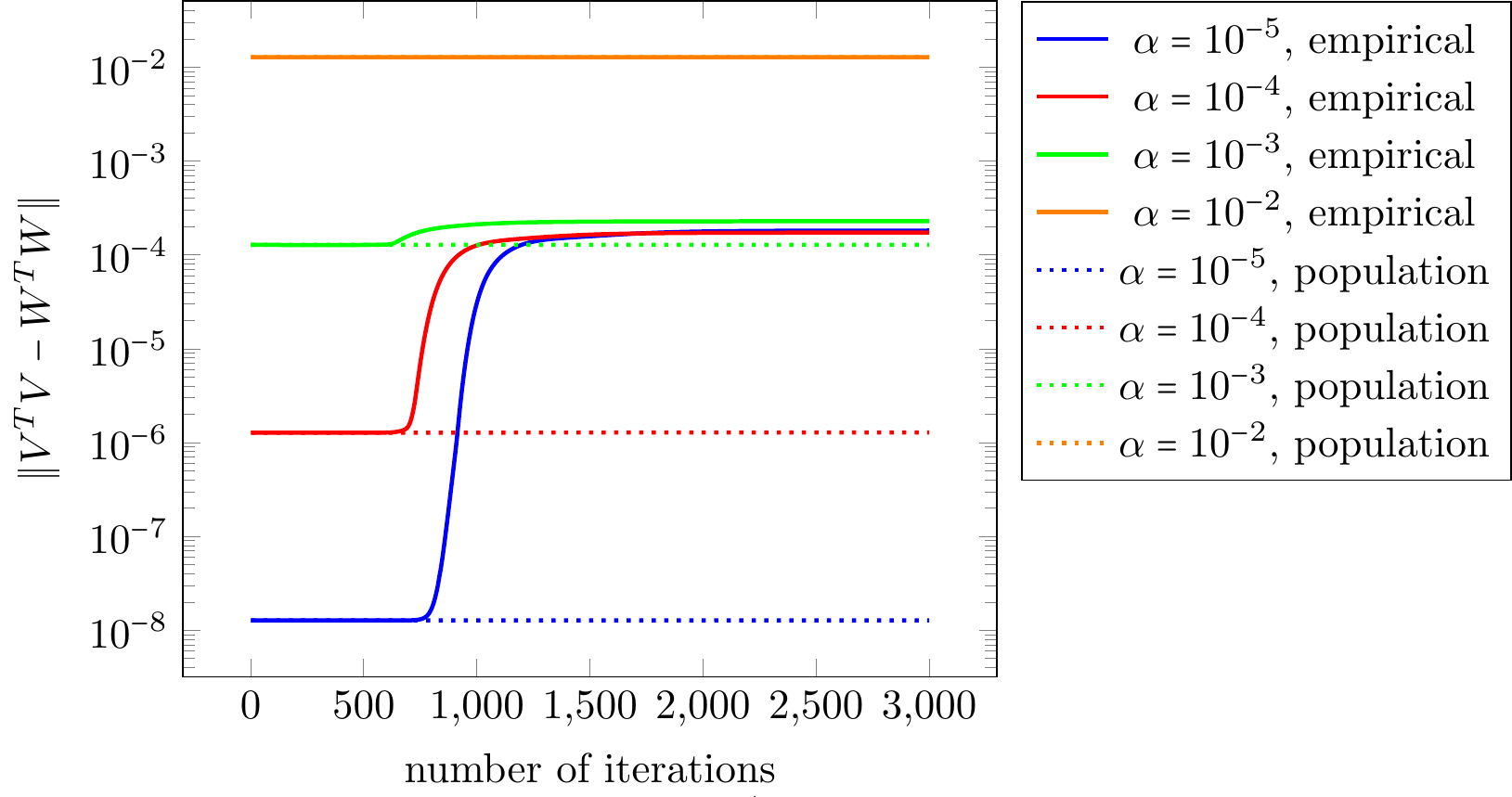}
    \caption{Evolution of the imbalance term $ \norm{V_t^T V_t - W_t^T W_t} $ with different choices of the scale of initialization $\alpha$.}
    \label{fig:balancedness_scaleinit}
\end{figure}

We observe that in the population case, the imbalance term $\norm{V_t^T V_t - W_t^T W_t} $ stays almost constant during the training.
This is in stark contrast to the empirical scenario, where we observe that the imbalance term grows until it reaches a certain threshold. We like to note that as evident in Figure \ref{fig:balancedness_scaleinit} when the scale of initialization $\alpha$ is chosen sufficiently small, this threshold does not depend on the scale of initialization.
We see that this is different for large initialization ($\alpha=10^{-2}$), but this is to be expected as in this case even in the beginning the imbalance term $\norm{V_t^T V_t - W_t^T W_t} $ is already larger than the threshold in the case of a smaller initialization. 

It is worth noting that the fact that the imbalance term evolves very differently in the population and in the empirical scenario has a huge impact on our theoretical analysis.
In contrast to \cite{ye2021global,ding2021global}, which analyze the population loss scenario, we need a much finer analysis of the imbalance term beyond controlling the spectral norm as stated earlier. For further experiments we refer to Section \ref{sec:threephase}.

\subsection{Change of test and train error during training}
\begin{figure}[t]
    \centering
    \begin{subfigure}[b]{0.45\textwidth}
    \includegraphics[scale=0.55]{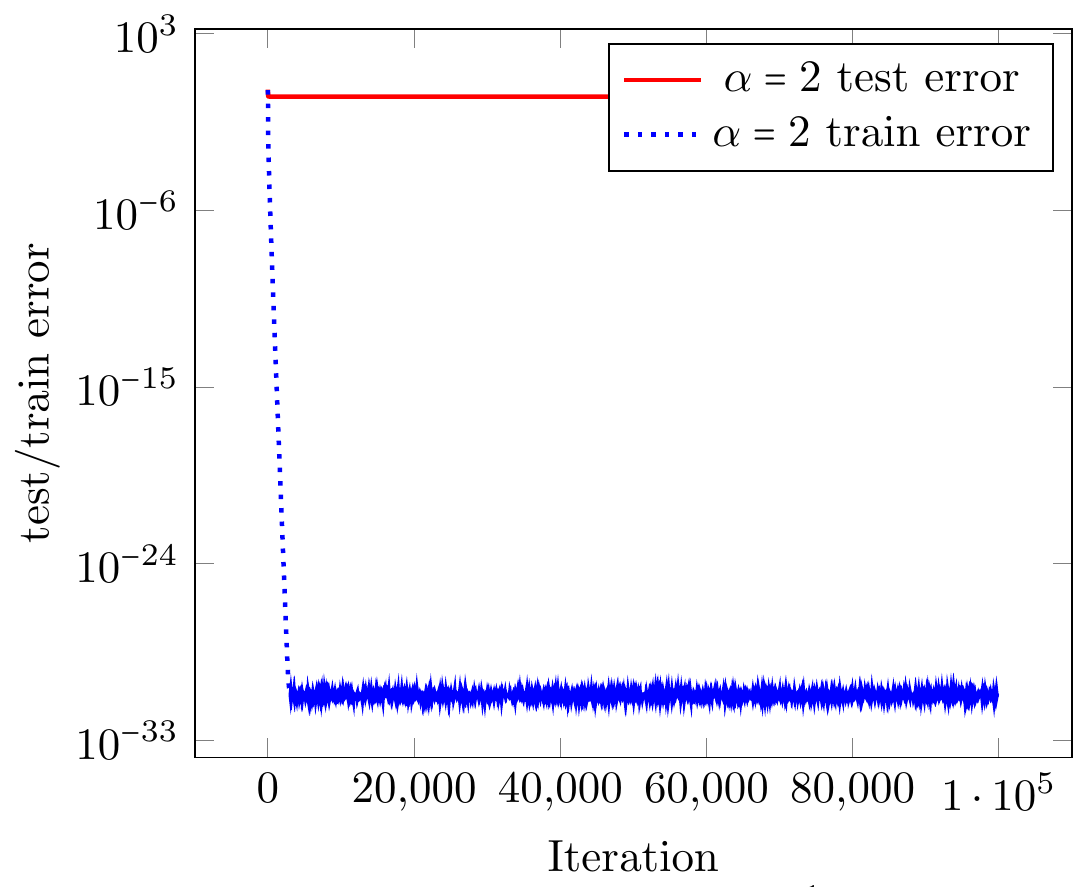}
    \caption{}
    \label{largeinit}
    \end{subfigure}
    \begin{subfigure}[b]{0.42\textwidth}
        \includegraphics[scale=0.55]{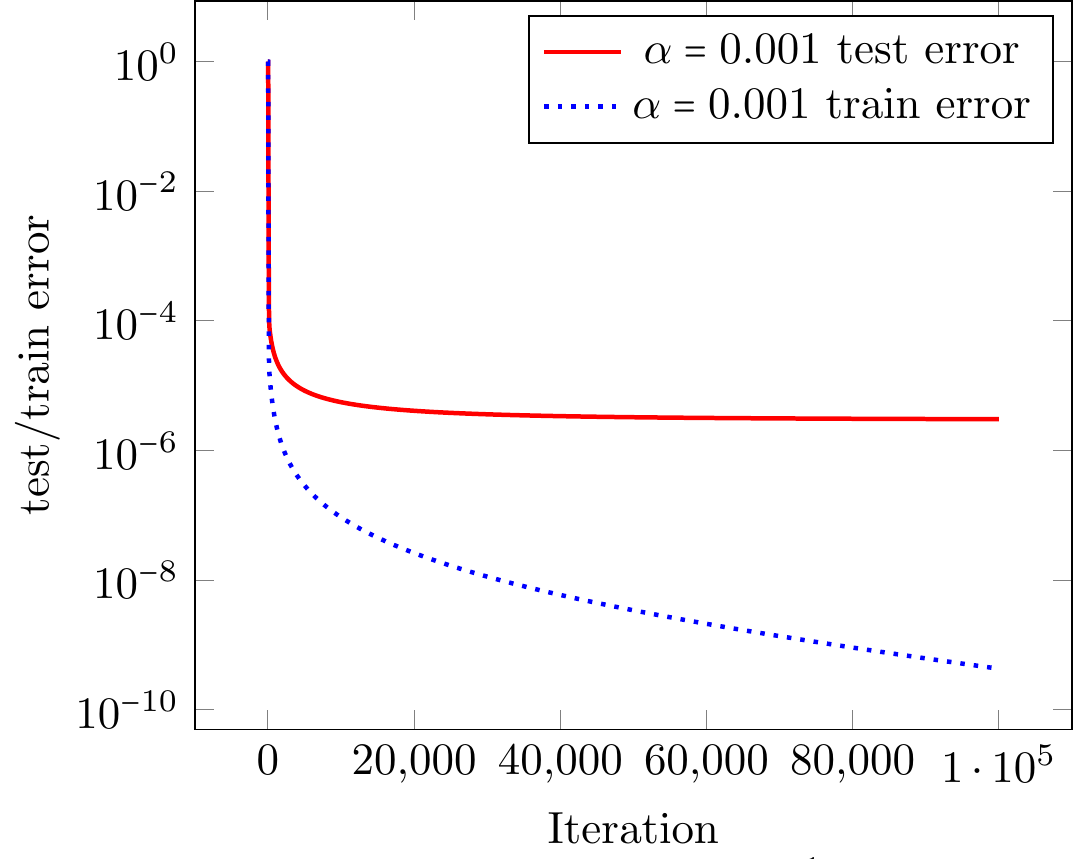}
        \caption{}
        \label{smallinit}
    \end{subfigure}
    \caption{Depiction of test error $\frac{1}{2}\|V_tW_t^T - X\|^2_F$ and train error $\mathcal{L}(V_t,W_t)$ for (a) large and (b) small $\alpha$ during training.}
    \label{fig:testtrainevolution1}
\end{figure}
In the next two experiments, we want to understand how the test error \eqref{testerr} and the train error $ \mathcal{L} \left(V_t, W_t\right) $ change during training.
For that, we fix the rank of the model to be $k=40$ and we set the step size as $ \mu = \frac{1}{4 \norm{X}} $.

In the first experiment, we compare the evolution of the train and test error for very large and for very small scale of initialization $\alpha$, see Figure \ref{fig:testtrainevolution1}.
For large initialization, depicted in Figure \ref{largeinit}, we observe that the train error converges linearly to $0$, whereas the test error stays roughly constant at a large value. This figure clearly demonstrates that in this large initialization regime the learned model does not generalize well. Choosing a large initialization corresponds to what in the literature is called \textit{lazy training} \cite{chizat2019lazy} which has been extensively studied in the context of neural networks, see e.g., \cite{oymak2020towards,du2019gradient}.
(In the context of matrix sensing with symmetric, positive definite matrices the lazy training regime has been theoretically analyzed in \cite[Theorem 4.2]{oymak2019overparameterized}.) For small initialization, depicted in Figure \ref{smallinit}, which corresponds to the regime studied in this paper, we observe that train and test error evolve very differently.
Indeed, we observe that in this regime both the train and test error decay and thus the solution found by gradient descent does generalize well.

In the next experiment, we want to understand how the relative test error $\frac{\|V_tW_t^T - X\|^2_F}{\|X\|^2_F}$ depends on the scale of initialization $\alpha$. To this aim,
 we run gradient descent until the train error is below $0.5 \times 10^{-9}$ and then we plot the test and train error for several choices of $\alpha$.
The results are depicted in Figure \ref{fig:testtrainevolution2}.
We observe that the train error depends polynomially on the scale of initialization.
This is in line with our main result, Theorem \ref{theorem:main}. Finally, let us note that the last two experiments resemble what has been observed in the symmetric matrix sensing scenario, see, e.g., \cite{stoger2021small}.

\begin{figure}[t]
    \centering
   \includegraphics[scale=0.8]{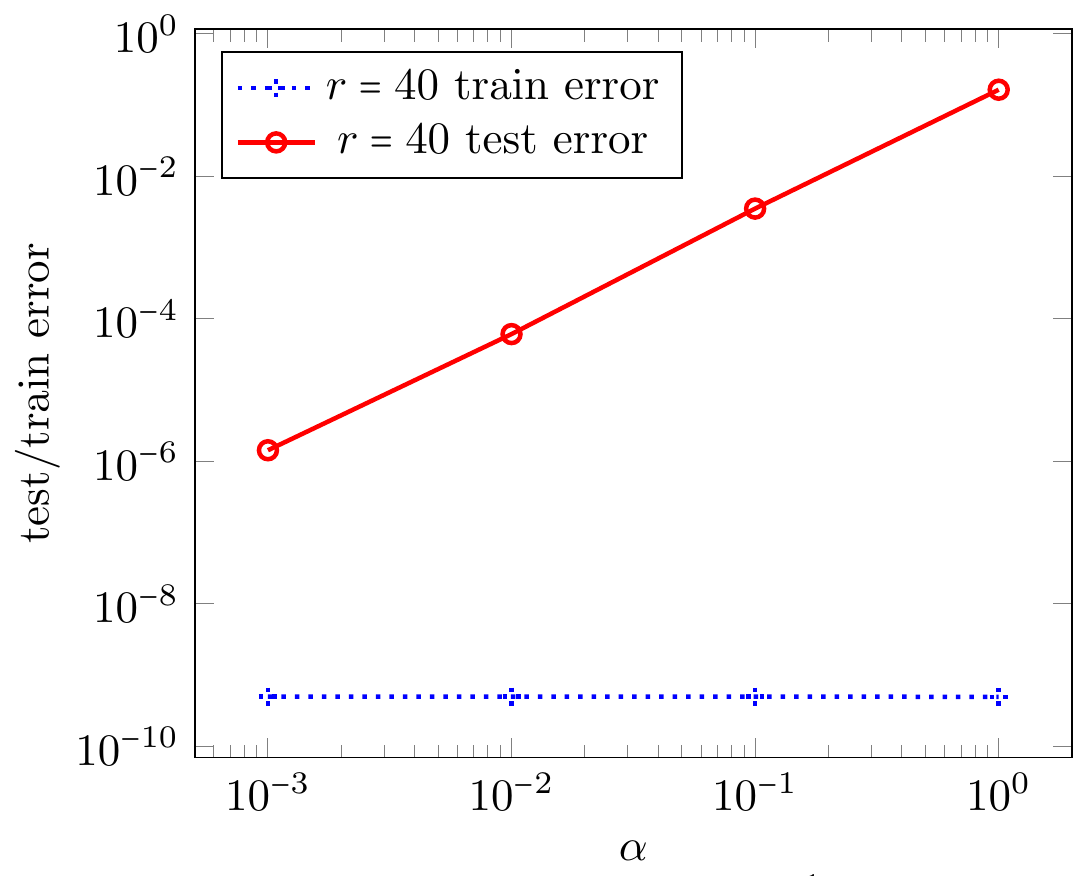}
  \caption{Relative test error $\frac{ \norm{V_tW_t^T - X}^2_F}{\norm{X}^2_F}$ at the end of training for different scales of initialization $\alpha$.}
 \label{fig:testtrainevolution2}
\end{figure}

\input{addexp.tex}

%% file: addexp.tex
\subsection{Impact of step size on balancedness}
In this section, we focus on understanding how different choices of step size impact the spectral norm of the imbalance term $\norm{V_t^T V_t - W_t^T W_t}$.
To this aim, we fix $k=10$ and $\alpha= 10^{-5}$ and we choose different step sizes $\mu$.
In Figure \ref{fig:balancedness_stepsize1}, we show how the evolution of $\norm{V_t^T V_t - W_t^T W_t} $ changes with different step sizes.
We observe that the spectral norm of the imbalance term, $\norm{V_t^T V_t - W_t^T W_t}$, behaves qualitatively similarly regardless of the choice of the step size.
The norm stays first constant and then grows rapidly, after which it stays roughly constant again.
This indicates that most of the growth of this spectral norm happens during the second phase when the signal is growing.
However, what changes with different step sizes is the threshold which $\norm{V_t^T V_t - W_t^T W_t} $ converges to at the end of training. 
Indeed, Figure \ref{fig:balancedness_stepsize1} indicates that larger step sizes lead to a larger threshold after training.
\begin{figure}[t]
    \centering
    \includegraphics[scale=0.7]{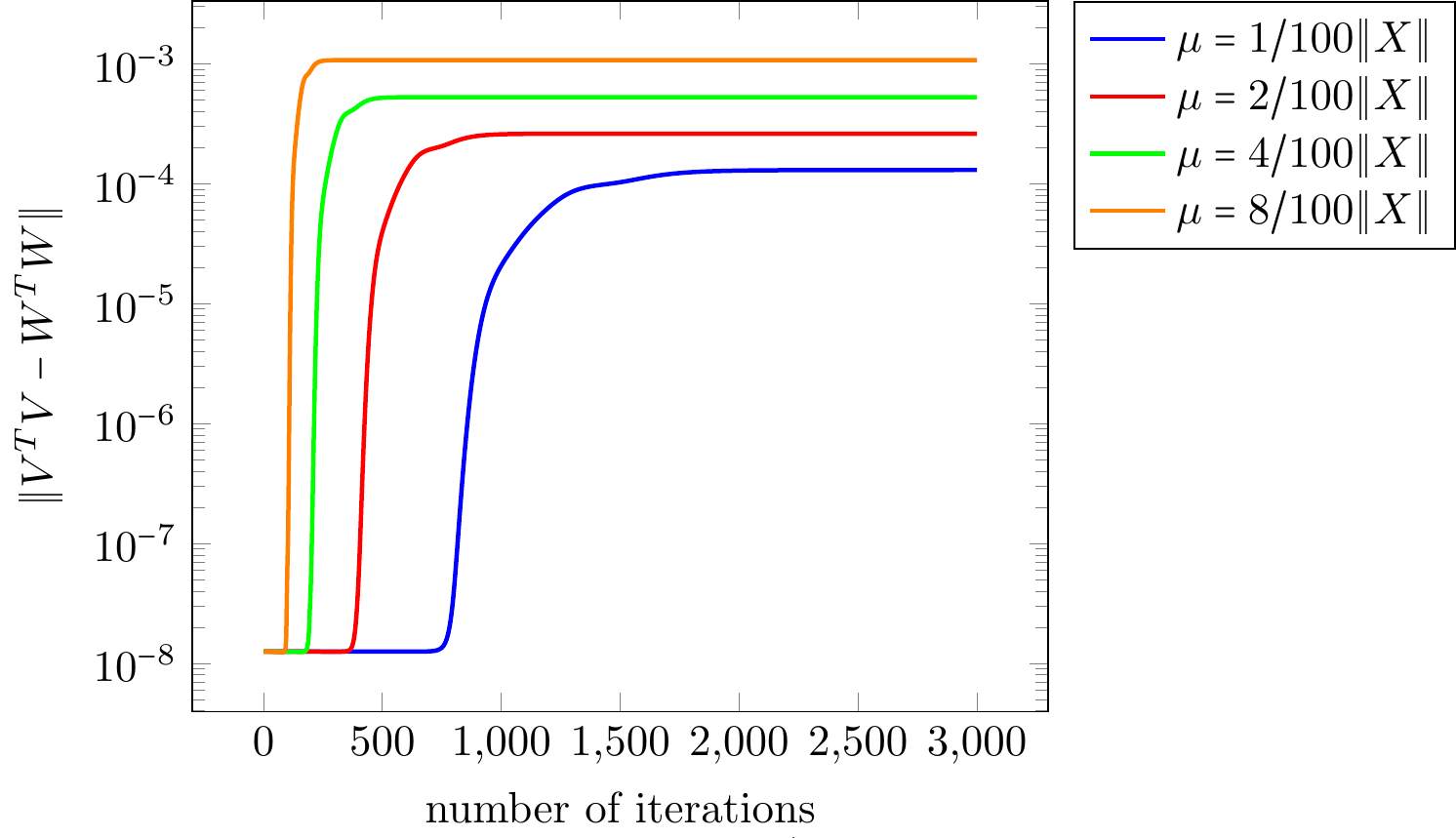}
    \caption{Evolution of the spectral norm of the imbalance term $\norm{V_t^T V_t - W_t^T W_t}$ during training with different step sizes}
     \label{fig:balancedness_stepsize1}
\end{figure}

In the next experiment, we want to examine how the value of this threshold at the end of training depends on the step size $\mu$.
For that, we repeat the experiment for $\mu \norm{X} = \{0.01, 0.02, \ldots, 0.10\}$ and show the relations between threshold and the step size.
The results are depicted in Figure \ref{fig:balancedness_stepsize2}.
We observe that at the end of training the spectral norm of the imbalance term $\norm{V_t^T V_t - W_t^T W_t} $ depends linearly on the step size $\mu$.
We note that this observation is well-aligned with our theory. In particular, we infer from Lemma \ref{lemma:balancedbase} that for each iteration $\norm{V_t^T V_t - W_t^T W_t} $ grows by an additive term which scales quadratically with the step size $\mu$.
Furthermore, Theorem \ref{theorem:main} shows that the number of iterations needed for convergence is proportional to the inverse of the step size $\mu$. Combining these two results, our theory predicts that the scaling for the threshold after convergence should be linear in the step size $\mu$.
\begin{figure}[t]
    \centering
    \includegraphics[scale=0.7]{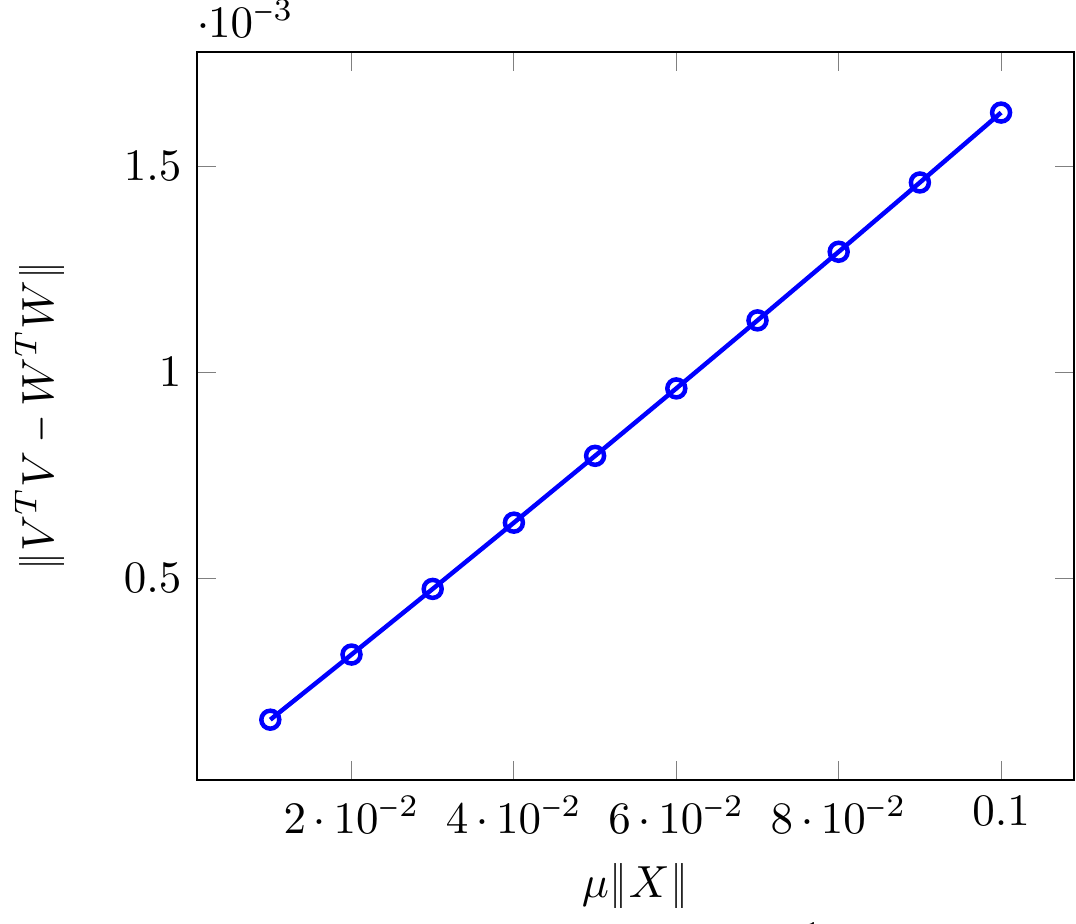}
    \caption{Spectral norm of the imbalance term $\norm{V_t^T V_t - W_t^T W_t}$ at convergence with different choices of the step size $\mu$}
    \label{fig:balancedness_stepsize2}
\end{figure}

\subsection{Evolution of the two additional couplings of the trajectory}
In this section, we focus on the behavior of  the two additional couplings of the trajectories of the factors $V_t$ and $W_t$ studied in this paper (see Section \ref{sec:threephase}). For that, we set $k=10$, the initialization scale $\alpha = 10^{-6}$ and the step size $\mu = \frac{1}{100\|X\|}$.

In the first experiment, depicted in Figure \ref{fig:BalanceVariant2}, we compare the evolution of the spectral norm of the imbalance term $\norm{V_t^T V_t - W_t^T W_t}$ and its nuisance part $\norm{(V_t^T V_t - W_t^T W_t)\Qtp} = 2\norm{\tilZt^T\Zt\Qtp}$ during training.
(The matrices $\Zt$, $\tilZt$, and $\Qtp$ will be formally defined in Section \ref{subsec:symmetrization} and Section \ref{subsec:signalnuisance}.)
We observe that the nuisance part is significantly smaller than the total imbalancedness. This phenomenon inspires us to do a tighter analysis of $\norm{\tilZt^T\Zt\Qtp}$ (see Lemma \ref{lemma:balancednessperp} for details). We note that this careful analysis is critical to our convergence analysis, allowing us to show good generalization and convergence with only a modest number of iterations.

In the next experiment, we show the evolution of the angle between the imbalance matrix and the signal direction $2\norm{\tilZtT\P{\Zt\Qt}}$ in Figure \ref{fig:BalanceAngle2}. We observe that this quantity remains small during training matching our analysis for this quantity in Lemma \ref{ref:balancednessangle}.

\begin{figure}[t]
    \centering\includegraphics[scale=0.7]{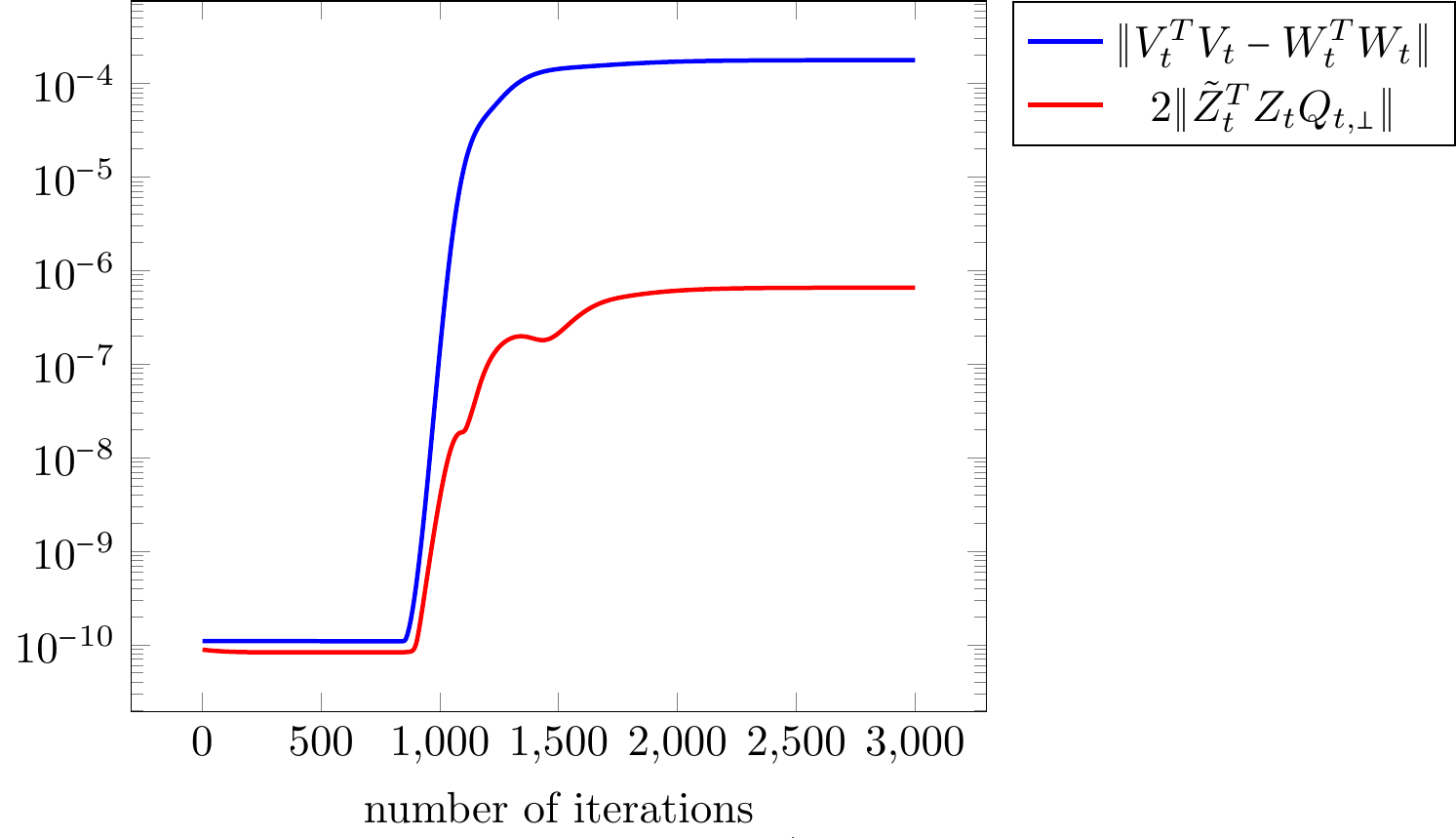}
    \caption{Evolution of the spectral norm of the imbalance term $\norm{V_t^T V_t - W_t^T W_t}$ and its nuisance part $2\|\tilde{Z}_t^T Z_t Q_{t,\bot}\|$ during training.}
    \label{fig:BalanceVariant2}
\end{figure}

\begin{figure}[t]
    \centering  \includegraphics[scale=0.7]{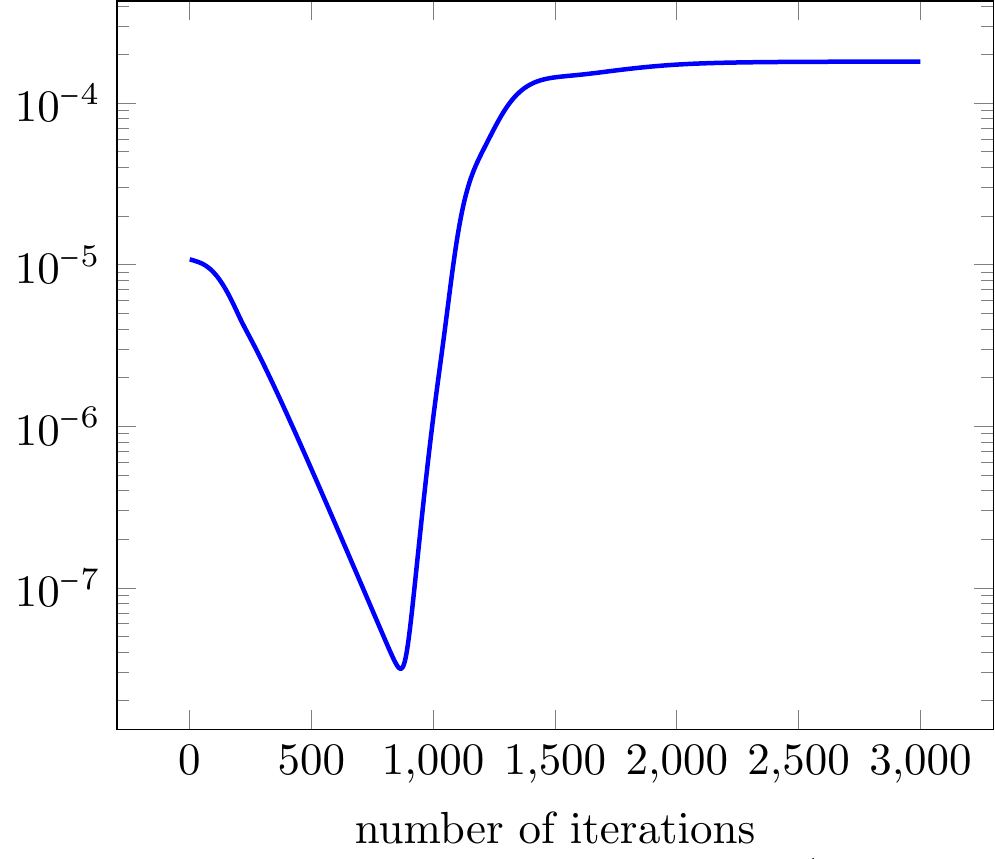}
    \caption{Evolution of the angle between the imbalance matrix and the signal direction $2\norm{\tilZtT\P{\Zt\Qt}}$ during training.}
    \label{fig:BalanceAngle2}
\end{figure}

%% file: proof_sketch.tex
\section{Proof of the main result}\label{sec:proofmainresult}

\input{short_proof_outline.tex}

In the following, we give a proof the main result, Theorem \ref{theorem:main}.
\subsection{Symmetrization}\label{subsec:symmetrization}

The first step in our proof is to show that the asymmetric model can be equivalently formulated as a symmetric model.
For that, we first define the symmetric measurement matrices $B_i \in \text{Sym}_{n_1+n_2}$ by
\begin{equation*}
    B_i:= \frac{1}{\sqrt{2}} \begin{pmatrix}
        0 & A_i \\
        A_i^T & 0 
    \end{pmatrix}
    \quad \text{for all } i \in \left[m\right]
\end{equation*}
and the associated linear measurement operator 
 $ \Bcal : \text{Sym}_{n_1 + n_2} \longrightarrow \R^m $ 
via
\begin{equation}\label{eq:prelimintern1}
   \left(  \Bcal \left( S \right) \right)_i := \innerproduct{B_i,S}
\end{equation}
for $ i \in \left[m\right] $ and for all symmetric matrices $S \in \text{Sym}_{n_1+n_2}$.
We define
\begin{equation*}
\symA :=  \begin{bmatrix}
     0 &  X\\
        X^T &  0
\end{bmatrix}
\end{equation*}
and 
\begin{equation}\label{equ:Ztdefinition}
\Zt := \frac{1}{\sqrt{2}}\begin{bmatrix}
 V_t\\
 W_t
\end{bmatrix}\quad \text{and}
\quad
\tilZt := \frac{1}{\sqrt{2}}\begin{bmatrix}
 V_t\\
- W_t
\end{bmatrix}
 \quad 
 \text{ for all } t.
\end{equation}
We observe via a straightforward calculation that 
\begin{equation*}
\frac{1}{\sqrt{2}} \Bcal \left(\symA\right)=  \Acal \left(X\right) 
\quad \text{ and } \quad 
\frac{1}{\sqrt{2}} \Bcal \left( \Zt \ZtT - \tilZt \tilZtT \right) = \Acal \left( V_t W_t^T \right)
\end{equation*}
The last two equalities imply that
\begin{equation*}
    \Loss \left( V_t, W_t  \right) = \frac{1}{2} \twonorm{\Acal \left(X- V_t W_t^T\right) }^2  =\frac{1}{4} \twonorm{\Bcal \left( \symA- \Zt \ZtT + \tilZt \tilZtT \right) }^2.
\end{equation*}
This motivates the definition of the following loss function
\begin{equation*}
\Losssym \left( Z, \tilde{Z} \right):=  \frac{1}{4} \twonorm{\Bcal \left( \symA- Z Z^T + \tilde{Z} \tilde{Z}^T  \right) }^2,
\end{equation*}
where $Z \in \R^{(n_1 + n_2) \times k} $ and $\tilde{Z} \in \R^{(n_1 + n_2) \times k}$.
We observe that
\begin{align*}
\nabla_Z \Losssym \left( Z, \tilde{Z} \right) &= - \left[ \left(   \Bcal^* \Bcal \right) \left( \symA- Z Z^T + \tilde{Z} \tilde{Z}^T  \right) \right] Z,\\
\nabla_{\tilde{Z}} \Losssym \left( Z, \tilde{Z} \right) &=  \left[ \left(   \Bcal^* \Bcal \right) \left( \symA- Z Z^T + \tilde{Z} \tilde{Z}^T  \right) \right] \tilde{Z}.
\end{align*}
Here, $\Bcal^*$ denotes the adjoint of the measurement operator $\Bcal$. 
It follows from a straightforward calculation that 
\begin{equation*}
\nabla_{Z} \Losssym \left( \Zt, \tilZt \right) =  \begin{pmatrix}
    \nabla_V \Loss \left( V_t,W_t \right) \\
    \nabla_W \Loss \left( V_t,W_t \right) 
\end{pmatrix}
\quad 
\text{and}
\quad
\nabla_{\tilde{Z}} \Losssym \left( \Zt, \tilZt \right)=  \begin{pmatrix}
    \nabla_V \Loss \left( V_t,W_t \right) \\
    -\nabla_W \Loss \left( V_t,W_t \right) 
\end{pmatrix}.
\end{equation*}
These two equations imply that
\begin{align}
\Zplus &= \Zt  + \mu \left[ \left(   \Bcal^* \Bcal \right) \left( \symA- \Zt \ZtT+ \tilZt\tilZtT  \right) \right] \Zt, \label{def:Ztdefinition} \\
 \tilZplus &= \tilZt - \mu \left[ \left(   \Bcal^* \Bcal \right) \left( \symA- \Zt \ZtT+ \tilZt\tilZtT  \right) \right] \tilZt.\label{def:tilZtdefinition}
\end{align}
This shows that a gradient descent iteration in the symmetrized loss function $ \mathcal{L}_{ \text{sym} } $ is equivalent to a gradient descent step in the original loss function $\mathcal{L}$.
In the following, we analyze the trajectories of $\Zt$ and $ \tilZt$ in the symmetric formulation.

Using this reformulation, we can use some of the proof techniques developed in \cite{stoger2021small} to analyze the gradient descent trajectory.
However, let us stress that there is a key difference compared to the model in \cite{stoger2021small}.
Namely, in \cite{stoger2021small}, it was assumed that the ground truth matrix $\symA$ is positive semidefinite. Thus, one can set $\tilZt=0$ and one only needs to optimize over $\Zt$.
However, in the scenario analyzed in this paper, the ground truth matrix $\symA$ has positive as well as negative eigenvalues. 
As already mentioned in the introduction, what makes the analysis in this paper now much more challenging is the fact that the trajectories of the two matrices $\Zt$ and $\tilZt$ are coupled with each other.
To deal with this additional difficulty we need a refined analysis, see below.

\subsection{Decomposition into signal and nuisance term}\label{subsec:signalnuisance}
A crucial ingredient of our analysis will be the decomposition of our matrices $\Zt$ and $\tilZt$ into a signal part, which will converge to the ground truth signal, and a nuisance part, which will stay small during the training.

In order to define this decomposition, recall first that the singular value decomposition of the ground truth matrix $X$ is given by $X=P_X \Sigma_X Q_X^T $. 
Then we define the following two matrices
\begin{equation}\label{equ:LXdefinition}
L_X := \frac{1}{\sqrt{2}}\begin{bmatrix}
 P_{ X}\\
 Q_{ X}
\end{bmatrix}, \quad \text{and}
\quad
\widetilde{L_X} := \frac{1}{\sqrt{2}}\begin{bmatrix}
 P_{ X}\\
- Q_{ X}
\end{bmatrix},
\end{equation}
whose columns are orthonormal.
This allows us to write the eigendecomposition of $\symA$ as
\begin{equation*}
\symA  = L_{ X} \Sigma_XL_{ X}^T - \widetilde{L_X} \Sigma_X\widetilde{L_X}^T.
\end{equation*}
Here, $ L_{ X} \Sigma_XL_{ X}^T$ represents the positive semidefinite part of $\symA$ and $ \widetilde{L_X} \Sigma_X\widetilde{L_X}^T$ represents the negative semidefinite part of $ \symA $.
Now consider the matrix $\LAT \Zt  \in \mathbb{R}^{r \times k}$.
Under the assumption that this matrix has full rank $r$ (which, as we will prove, holds true during training) we can denote its singular value decomposition by $ \LAT \Zt = P_t\Sigma_t Q_t^T$ with $Q_t\in \mathbb{R}^{k\times r}$.
Moreover, we denote by $\Qtp \in \R^{k\times (k-r)} $ a matrix with orthonormal columns, whose column span is orthogonal to the column span of $\Qt$.

Using these definitions, similarly as in \cite{stoger2021small}, we can decompose $Z_t$ into
\begin{equation*}
Z_t = 
\underset{ \text{signal part} }{\underbrace{Z_tQ_tQ_t^T}}
+ \underset{ \text{nuisance part} }{ \underbrace{Z_tQ_{t,\bot}Q_{t,\bot}^T}}.
\end{equation*}
Note that it follows immediately from the definition of $\Qtp$ that $ \LAT \Zt \Qtp =0 $.

Next, we decompose $\tilZt$ into its signal and nuisance part:
\begin{equation*}
    \tilZt = \underset{ \text{signal part} }{ \underbrace{ \tilZt \Qt \QtT}} +  \underset{ \text{nuisance part} }{\underbrace{\tilZt \Qtp \QtpT}} .  
\end{equation*}
Note that due to equations \eqref{equ:Ztdefinition} and \eqref{equ:LXdefinition} we have that $ \LAT \Zt = \LtilAT \tilZt=  P_t \Sigma_t Q_t^T$, i.e., the matrix $\LtilAT \tilZt$ and the matrix $\LAT\Zt$ are the same.
For this reason, we can also use the matrix $Q_t$ to define the decomposition of $\tilZt$ into its signal and nuisance part, i.e., 
\begin{equation*}
    \tilZt
    =
    \tilZt \Qt \QtT 
    +
    \tilZt \Qtp \QtpT.
\end{equation*}

The following lemma collects a few facts regarding the symmetry between $\Zt$ and $\tilZt$.
These are a direct consequence of equations \eqref{equ:Ztdefinition} and \eqref{equ:LXdefinition}, which is why we skip the short proof.
\begin{lemma}\label{lemma:symmetry}[Symmetry between $\Zt$ and $\tilZt$]
Assume that $L_X^T \Zt$ has rank $r$
and let $\Zt$ and $\tilZt$ as defined in this section.
Then the norms are equal:
\begin{align*}
    \norm{\Zt}&=\norm{\tilZt}, \\
    \norm{\Zt \Qt}&=\norm{\tilZt \Qt},\\
    \norm{\Zt \Qtp}&=\norm{\tilZt \Qtp}.
\end{align*}
Moreover, the following two identities hold:
\begin{align*}
    \LtilAPT \tilZt \Qt &=  \LAPT \Zt \Qt,\\
    \LtilAPT P_{\tilZt \Qt} &= \LAPT \PZQ,
\end{align*}
where we have set
\begin{equation*}
    \LtilAP
    :=
    \begin{bmatrix}
        \Id_{n_1} & 0 \\
        0 & -\Id_{n_2}
    \end{bmatrix}
    \cdot 
    \LAP.
\end{equation*}
Note that $\LtilAP \in \mathbb{R}^{(n_1+n_2) \times (n_1+n_2-r)} $ is a matrix with orthonormal columns, whose span is orthogonal to the span of $ \LtilA $.
\end{lemma}
To see why the decomposition of $\Zt$ into the signal and noise part is useful we note that 
\begin{align}
\Zplus &= \Zt  - \mu \left[ \left(   \Bcal^* \Bcal \right) \left(  \Zt \ZtT- \tilZt\tilZtT -\symA  \right) \right] \Zt \nonumber \\
&=\Zt  - \mu  \left( \Zt \ZtT- \tilZt\tilZtT -\symA \right)\Zt 
+ \mu \underset{=:\Deltat}{ \underbrace{ \left[ \left( \Id -  \Bcal^* \Bcal \right) \left(  \Zt \ZtT- \tilZt\tilZtT -\symA \right) \right] }} \Zt. \label{equ:decompositiongradient}
\end{align}
The expression $\Deltat$ can be interpreted as a perturbation term.
To control the spectral norm of the perturbation term, $\norm{\Deltat}$, we rely on the RIP.
However, since the RIP only applies for low-rank matrices, 
we decompose $\Deltat$ into a term involving the low-rank signal part, which we expect to have good control over, and a high-rank term involving the nuisance term as follows:
\begin{align*}
    \Deltat
    =&
    \left[ \left( \Id -  \Bcal^* \Bcal \right) \left(  \Zt \Qt \QtT \ZtT- \tilZt \Qt \QtT \tilZtT -\symA \right) \right]\\
    &+
    \left[ \left( \Id -  \Bcal^* \Bcal \right) \left(  \Zt \Qtp \QtpT \ZtT- \tilZt \Qtp \QtpT \tilZtT  \right) \right].
\end{align*}
The fact that we have much sharper control over the first term than over the second term is also reflected by the following lemma,
whose straightforward proof has been deferred to Appendix \ref{sec:proofauxiliarylemma}.
This also indicates that we need to deal with the signal part of $\Zt$ differently than with the nuisance part in our proof.
\begin{lemma}\label{lemma:RIPlemma}
    Assume that the linear measurement operator $\mathcal{A}: \R^{n_1 \times n_2} \rightarrow \R^m $ has the restricted isometry property (RIP) of order $2r+1$ with constant $\delta >0$.
    Then, for all iterations $t$, it holds that 
    \begin{equation}\label{ineq:RIPclaim1}
        \begin{split}
           & \norm{ \left(\Id - \mathcal{B}^* \mathcal{B} \right) \left( \Zt \Qt \QtT \ZtT -\tilZt \Qt \QtT \tilZtT   -\symA \right)  }\\
         \le  &\delta \sqrt{r} \norm{ \Zt \Qt \QtT \ZtT -\tilZt \Qt \QtT \tilZtT   -\symA }.     
        \end{split} 
    \end{equation}
    Moreover, for all $t$, we have that
    \begin{equation}\label{ineq:RIPclaim2}
        \begin{split}
                &\norm{ \left(\Id - \mathcal{B}^* \mathcal{B} \right) \left( \Zt \Qtp \QtpT  \ZtT -\tilZt\Qtp \QtpT \tilZtT  \right)  } \\
         \le &\left( k - r \right) \delta \norm{  \Zt \Qtp \QtpT \ZtT -\tilZt\Qtp \QtpT \tilZtT  }.
        \end{split}
    \end{equation}
    and
    \begin{equation}\label{ineq:RIPclaim3}
        \begin{split}
        \norm{ \left(\Id - \mathcal{B}^* \mathcal{B} \right) \left( \Zt  \ZtT -\tilZt \tilZtT  \right)  }  \le  \delta \nucnorm{  \Zt  \ZtT -\tilZt \tilZtT  }.
        \end{split}
    \end{equation}
\end{lemma}

\subsection{Three-Phase Analysis}\label{sec:threephase}
In our convergence analysis, we need to control several quantities related to $\Zt$ and $\tilZt$.
The first three quantities, which we keep track of in our analysis, have also been studied in \cite{stoger2021small}:
\begin{align*}
    \sglmin \left(\Zt \Qt\right) \tag*{(magnitude of the signal part),}\\
    \norm{ \Zt \Qtp } \tag*{(magnitude of the nuisance part),}\\
    \norm{ \LAT \PZQ } \tag*{(angle between column spaces of signal part and ground truth).}
\end{align*}
Due to the symmetry between $\Zt$ and $\tilZt$ there is no need to control the corresponding quantities for $\tilZt$, see Lemma \ref{lemma:symmetry}.

In contrast to \cite{stoger2021small}, it does not suffice to only control these three norms to analyze the asymmetric scenario.
The reason is that the dynamics of $\Zt$ and $\tilZt$ are coupled with each other as can be seen from equations \eqref{def:Ztdefinition} and \eqref{def:tilZtdefinition}.
Indeed, rewriting equation \eqref{equ:decompositiongradient} we obtain that
\begin{align*}
    \Zplus = \Zt - \mu  \left( \Zt \ZtT -\symA \right)\Zt  + \mu \Delta \Zt+\mu\tilZt \underset{ \text{imbalance matrix} }{\underbrace{ \tilZtT \Zt}}.
\end{align*}
We observe in above equation that the coupling between the trajectories $\tilZt$ and $\Zt$ is controlled by the imbalance matrix $\tilZtT \Zt $. In particular, when the norm of the imbalance matrix $\norm{\tilZtT \Zt}$ is small, the trajectories of $\tilZt$ and $\Zt$ are only weakly coupled with each other.
For this reason, in our analysis we control the norm
\begin{align*}
    \norm{ \tilZtT \Zt } 
    \tag*{(magnitude of the imbalance term)}.
\end{align*}
A direct calculation shows that
\begin{align*}
    \tilZplusT \Zplus 
    = 
    \tilZtT \Zt 
    + 
    \mu^2 \tilZtT \left( \Bcal^* \Bcal \left( \symA -\Zt \ZtT + \tilZtT \tilZt  \right) \right)^2 \Zt
    =
    \tilZtT \Zt 
    + 
    O \left( \mu^2 \right).
\end{align*}
Note that due to the choice of our initialization, we have that $\norm{\tilde{Z}_0^T Z_0} = O \left( \alpha^2 \right) $.
Thus, be choosing the step size small enough one could achieve that $\norm{ \tilZtT \Zt }$ stays at order $ O \left(\alpha^2 \right) $ during training.
Indeed, choosing the step size small enough or analyzing gradient flow is how several related works studying gradient flow deal with this issue, see, e.g., \cite{bah2019learning}. 

However, as our experiments show, see Section \ref{sec:experiments}, when using more realistic step sizes, the quantity $\norm{ \tilZtT \Zt }$ will increase during training significantly.
For this reason, it turns out that it does not suffice to only control $\norm{ \tilZtT \Zt }$ and we need a more fine-grained analysis of the coupling between $\Zt$ and $\tilZt$.
This will be achieved via controlling the following two norms
\begin{align*}
    \norm{\tilZtT \Zt \Qtp} \tag*{(imbalance term parallel to nuisance part),}\\
    \norm{\tilZtT \PZQ} \tag*{(imbalance term parallel to signal part).}
\end{align*}
The first term can be interpreted as measuring how strong the nuisance part of the signal $\Zt$ is coupled with $\tilZt$.
The second term can be seen as a projection of $\tilZt$ onto the span of the signal part of $\Zt$.
Note that this is different from $\norm{\tilZtT \Zt \Qt}$ as the singular values of $\Zt \Qt$  are not taken into account in the expression $\norm{\tilZtT \PZQ} $.

\subsubsection{Analysis of Phase 1}
Phase 1, the \textit{Spectral Phase}, is based on the observation that, due to our small, random initialization, in the first few iterations $\Zt \ZtT $ and $\tilZt \tilZtT$ are both very small compared to the (symmetrized) signal $\symA$.  
Hence, we can approximate the gradients of the loss function by
\begin{align*}
\nabla_Z \Losssym \left( \Zt, \tilde{\Zt} \right) 
= - \left[ \left(   \Bcal^* \Bcal \right) \left( \symA- \Zt \ZtT + \tilZt \tilZtT  \right) \right] \Zt  
\approx - \left[ \left(   \Bcal^* \Bcal \right) \left( \symA  \right) \right] \Zt 
\end{align*}
and by
\begin{align*}
\nabla_{\tilde{Z}} \Losssym \left( \Zt, \tilZt \right) 
=  \left[ \left(   \Bcal^* \Bcal \right) \left( \symA- \Zt \ZtT + \tilZt \tilZtT  \right) \right] \tilZt 
\approx  \left[ \left(   \Bcal^* \Bcal \right) \left( \symA  \right) \right] \tilZt.
\end{align*}
Due to equation \eqref{def:Ztdefinition} this implies that in the first few iterations we can approximate the iterates $\Zt$ by 
\begin{align*}
    \Zt 
    \approx 
    \left( \Id + \mu \left[ \left(   \Bcal^* \Bcal \right) \left( \symA  \right) \right] \right) Z_{t-1}
    \approx \left( \Id + \mu \left[ \left(   \Bcal^* \Bcal \right) \left( \symA  \right) \right] \right)^{t} Z_{0}.
\end{align*}
Analogously, due to equation \eqref{def:tilZtdefinition} we can approximate the iterates $\tilZt$ by 
\begin{align*}
    \tilZt
    \approx 
    \left( \Id - \mu \left[ \left(   \Bcal^* \Bcal \right) \left( \symA  \right) \right] \right)^t \tilde{Z}_0.
\end{align*}
This observation allows us to prove that after the first few iterations the signal parts of $\Zt$ and $\tilZt$ are well aligned with the eigenvectors of the ground truth signal $X$.
This is made precise in Lemma \ref{lemma:spectralmain} below.
The proof of this lemma is similar to the analysis of the spectral phase in \cite{stoger2021small}, which is why its complete proof has been deferred to Appendix \ref{sec:spectralphaseanalysis}.
We also refer to \cite{stoger2021small} for a more detailed discussion of the intuition behind the spectral phase.
\begin{lemma}\label{lemma:spectralmain}
Assume that $V_0 = \alpha V \in \R^{n_1 \times k} $ and $W_0= \alpha W \in \R^{n_2 \times k}$ for some fixed parameter $\alpha >0$, where the matrices $V$ and $W$ have i.i.d. entries with distribution $ \mathcal{N} \left(0,1\right) $.
Let the iterates $\left\{ \Zt; \tilZt \right\}_{t\in \mathbb{N}} $ be defined as in equations \eqref{def:Ztdefinition} and \eqref{def:tilZtdefinition}.
Moreover, assume that 
\begin{equation}\label{spectral:alphaassumption}
    \alpha \le 
    \sqrt{\frac{\norm{X}  }{ C_1 \max \left\{  n_1 + n_2;k \right\} }} \left( \frac{c \varepsilon \left( \sqrt{k} -\sqrt{r-1} \right) }{6\kappa^2  \sqrt{\max \left\{  n_1 + n_2; k \right\}} } \right)^{34\kappa }
\end{equation}
for some $0 < \varepsilon <1 $ and for some constant $0< c \le 1/32$.
Moreover, assume that
\begin{equation*}
    \norm{X- \left( \mathcal{A}^* \mathcal{A} \right) \left(X\right) } \le \frac{c}{\kappa^2} \sglmin \left( X \right)
\end{equation*}
and that $ \mu \le \frac{\tilde{c}}{\sglmin (X)} $ for a sufficiently small absolute constant $\tilde{c}>0$.
Then, with probability at least $1 -  C_2 \exp \left( - c_1 k \right) + \left( C_3\varepsilon \right)^{k-r+1}  $, after 
\begin{equation}\label{tstarbound}
    \tstar 
    := \left\lceil \frac{\ln\left(\frac{c\sglmin \left(L_{F_1}^T Z_0\right)}{2\kappa^{2} \norm{Z_0} }\right)}{\ln \left(1 - \frac{\mu\sglmin(X)}{8} \right)} \right\rceil
    \le \frac{17 \ln\left(\frac{6\kappa^{2} \sqrt{ \max \left\{ n_1+n_2; k \right\} } }{c\varepsilon \left( \sqrt{k} - \sqrt{r-1} \right) }\right)}{ \mu \sglmin \left(X\right) }
\end{equation}
iterations, the following inequalities hold:
\begin{align}
\norm{L_{X,\bot}^{T} P_{Z_{t_\star}Q_{t_\star}}}&\le \frac{28c}{\kappa^2}, \label{spectral:final1} \\
\sigma_{\min}(Z_{t_\star}Q_{t_\star}) 
&\ge \left(\frac{2\kappa^{2}  }{c}\right)^{8\kappa} \frac{\alpha \varepsilon \left( \sqrt{k} -\sqrt{r-1} \right)}{4} ,\label{spectral:final2}\\
\norm{Z_{t_\star}Q_{t_\star,\bot}} 
&\le \min \left\{ 2 \sglmin \left( \Ztstar Q_{\tstar} \right); \  \left( \frac{6\kappa^2  \sqrt{\max \left\{  n_1+n_2; k \right\} }}{ c \varepsilon \left( \sqrt{k} - \sqrt{r -1} \right)} \right)^{ 16 \kappa} \cdot \frac{12c \alpha \sqrt{\max \left\{ n_1 + n_2; k \right\} }}{ \kappa^2 } \right\} ,\label{spectral:final3}\\
\norm{Z_{\tstar}} &\le 2 \sqrt{\norm{X}},\label{spectral:final4}\\
\norm{\tilde{Z}_{\tstar}^T Z_{\tstar}  } 
&\le  \frac{ c \norm{X}}{\kappa^4}, \label{spectral:finalbalanc1} \\ 
\norm{\tilde{Z}_{\tstar}^T Z_{\tstar} Q_{\tstar, \bot}  }
&\le  \frac{c \sqrt{\norm{X}} }{\kappa^3} \norm{\Ztstar Q_{\tstar,\bot} },\label{spectral:finalbalanc2} \\ 
\norm{\tilde{Z}_{\tstar}^T P_{ Z_{\tstar} Q_{\tstar} }  }
&\le \frac{ c \sqrt{\norm{X}} }{\kappa^3}  . \label{spectral:finalbalanc3}
\end{align}
Here, $C_1, C_2, C_3, c_1 >0$ are two absolute constants.
\end{lemma}

\subsubsection{Analysis of Phase 2}
In Phase 2, the \textit{saddle avoidance phase}, we show that $\sglmin \left( \LAT \Zt \right) $ grows exponentially until we have that $ \sglmin \left( \LAT \Zt \right) \ge  \sqrt{\frac{\sglmin (X)}{8}} $.
Moreover, we show that in Phase 2, that the spectral norm of the nuisance term, $\norm{ \Zt \Qtp}$, is growing much slower than $ \sglmin \left(  \LAT \Zt \right)$.
This is captured by the following two lemmas, whose proofs have been deferred to Appendix \ref{sec:sigmingrowth} and Appendix \ref{sec:noisetermgrowth}.
\begin{lemma}\label{lemma:sigmingrowth}
Assume that $\mu \le \frac{c}{\norm{X}\kappa}$, $\norm{\Zt} \le 2\sqrt{\norm{X}}$, $\norm{\LAPT\P{\Zt\Qt}} \le \frac{c}{\kappa}$, $ \norm{\Deltat} \le c \sglmin(X) $, and that $ \LAT \Zt \Qt$ is invertible. Then it holds that
\begin{equation}\label{lemma1result}
\sglmin(\LAT\Zplus) \ge \sglmin(\LAT\Zplus\Qt) \ge \sglmin(\LAT\Zt)\left(1 + \frac{1}{4}\mu\sglmin(X) - \mu\sglmin^2(\LAT\Zt)\right).
\end{equation}
Here, $c>0$ is an absolute constant chosen small enough.
\end{lemma}

\begin{lemma}\label{lemma:noisetermgrowth}
Let $0 \le \varepsilon \le 1$.
Assume that $\mu \le \frac{c\varepsilon}{\norm{X}\kappa}$,  $\norm{\Zt} \le 2\sqrt{\norm{X}}$, $ \norm{\Deltat} \le c\varepsilon \sglmin (X) $, and $\norm{\LAPT\P{\Zt\Qt}} \le \frac{ c \varepsilon }{ \kappa }$.
Moreover, assume that $\LAT\Zplus\Qt$ and $\LAT\Zt\Qt$ have full rank.
Then it holds that 
\begin{equation*}
\norm{\Zplus\Qplusp} \le \left(1 - \frac{\mu}{2}\norm{\Zt\Qtp}^2 + \mu \varepsilon  \sglmin(X) \right) \norm{\Zt\Qtp} + 2\mu\sqrt{\norm{X}}\norm{\tilZtT\Zt \Qtp}.
\end{equation*}
Here, $c>0$ is an absolute constant chosen small enough.
\end{lemma}
We observe that in order to apply Lemma \ref{lemma:sigmingrowth} and Lemma \ref{lemma:noisetermgrowth}
we need to control several key quantities, which are described in the beginning of Section \ref{sec:threephase}.

The following lemma controls the angle between the positive eigenvectors ground truth signal $\symA$ and the signal part of $\Zt$.
The proof of Lemma \ref{lemma:anglecontrol} has been deferred to Appendix \ref{sec:anglecontrol}.
\begin{lemma}\label{lemma:anglecontrol}
Assume that $\mu \le \frac{c}{\norm{X}\kappa}$, $\norm{\LAPT\P{\Zt\Qt}} \le c\kappa^{-1}$, $ \norm{\Zt} \le 2\sqrt{\norm{X}}$, and $\norm{\Deltat} \le c \sglmin(X) $.
Moreover, assume that $\norm{\Zt\Qtp}\le \min \left\{ c \kappa^{-1/2} \sqrt{\sglmin(X)}; 2\sglmin(\Zt\Qt) \right\}$, $\norm{\tilZtT\Zt \Qtp} \le \sqrt{\norm{X}}\sglmin\left(\Zt\Qt\right)$, and that $\Zt \Qt$ has rank $r$.
Then it holds that
\begin{align*}
&\norm{\LAPT\P{\Zplus Q_{t+1} }} \le \\
&\left(1 - \frac{1}{4}\mu\sglmin \left(X\right) \right)\norm{\LAPT\P{\Zt\Qt}} 
+  2 \mu  \sqrt{\norm{X}} \norm{ \tilZtT \P{\Zt\Qt} }
+C\mu\frac{\sqrt{\norm{X}}\norm{\tilZtT\Zt \Qtp}}{\sglmin\left(\Zt\Qt\right)} + C \mu \norm{\Deltat} +  C\mu^2\norm{X}^2.
\end{align*}
Here, $c>0$ is an absolute constant chosen small enough and $C>0$ is an absolute constant chosen large enough.
\end{lemma}
The next lemma, whose proof has been deferred to Appendix \ref{sec:normcontrolled}, shows the spectral norm of $\Zt$ will never be too large compared to $ \sqrt{ \norm{X} }$.
\begin{lemma}\label{lemma:normcontrolled}
Suppose that $\norm{\Zt} \le 2 \sqrt{\norm{X}}$, $ \norm{\Deltat} \le \frac{1}{100} \norm{X} $, $\norm{\Zt \Qtp}\le \frac{ \sqrt{\norm{X}}}{100}$, $\norm{\LAPT P_{\Zt \Qt}} \le \frac{1}{100} $, and $ \mu \le \frac{1}{100 \norm{X}} $.
Then it holds that
\begin{equation*}
    \norm{\Zplus} \le 2 \sqrt{\norm{X}}.
\end{equation*}
\end{lemma}
The next three lemmas allow us to control the quantities related to the imbalance term, $ \norm{\tilZtT \Zt} $, $ \norm{ \tilZt \Zt \Qtp } $, and $ \norm{ \tilZt \PZQ } $.
In particular, using these lemmas we can prove that the trajectories of $\tilZt$ and $\Zt$ are sufficiently decoupled during training. The first lemma controls the growth of $\norm{\tilZtT \Zt} $.
\begin{lemma}\label{lemma:balancedbase}
Assume that $ \norm{\Zt} \le 2 \sqrt{\norm{X}} $
and $\norm{\Deltat} \le  \norm{X} $.
Then it holds that
\begin{equation*}
    \norm{\tilZplusT \Zplus} \le \norm{\tilZtT \Zt} + 400 \mu^2 \norm{X}^3.
\end{equation*}
\end{lemma}
The proof of Lemma \ref{lemma:balancedbase} has been deferred to Appendix \ref{sec:balancedbase}.
Note that in particular, Lemma \ref{lemma:balancedbase} shows that the growth is quadratic in the step size $\mu$.

The next lemma controls the growth of the balancedness term parallel to the nuisance part, $\tilZtT \Zt \Qtp $.
Note that it shows that its growth is upper bound by the size of the nuisance part $\norm{\Zt \Qtp}$.
Using this lemma we can show that  $\tilZtT \Zt \Qtp $ stays small relative to $\sqrt{\norm{X}} \norm{\Zt \Qtp} $.
\begin{lemma}\label{lemma:balancednessperp}
Assume that $\LAT \Zplus \Qt $ has full rank.
Moreover, assume that $\max \left\{ \norm{ \Zt }; \norm{\Zplus } \right\} \le 2 \sqrt{\norm{X}}$, $ \norm{\LAPT \PZQ} \le c $,  $\mu \le \frac{c}{\norm{X} \kappa}$, $\norm{\Zt} \le  2\sqrt{\norm{X}} $, and $ \norm{\Deltat} \le c \sglmin (X) $, where $c>0$ is an absolute constant chosen small enough.
Set 
\begin{equation*}
\beta := \norm{\LAPT \PZQ}  \norm{X}  + \norm{ \Zt \Qtp  }^2 + \norm{\Deltat}.
\end{equation*}
Then it holds that
 \begin{align*}
    &\norm{\tilZplus^T \Zplus \Qplusp} \le  \norm{ \tilZtT \Zt\Qtp } \\
    &+ C\mu \left(   \left( \norm{\LAPT\P{\Zt\Qt}} + \mu\norm{X} \right)\beta  +  \mu \norm{X}^2 \right)  \sqrt{\norm{ X }} \norm{ \Zt\Qtp }+ 8 \mu \beta \norm{\Zt \Qtp}^2,
\end{align*} 
where $C>0$ is an absolute constant chosen large enough.
\end{lemma}
The proof of Lemma \ref{lemma:balancednessperp} has been deferred to Appendix \ref{sec:balancednessperp}.
Next, Lemma \ref{ref:balancednessangle}, whose proof has been deferred to Appendix \ref{sec:balancednessangle}, allows us to prove that $\norm{ \tilZplusT \PZQ } $ stays bounded.
\begin{lemma}\label{ref:balancednessangle}
    Assume that $ \norm{\Zt} \le 2 \sqrt{\norm{X}} $, $ \norm{ \Zt \Qtp} \le \min \left\{  2 \sglmin \left( \Zt \Qt \right); c \sqrt{\sglmin (X)}  \right\} $, $\norm{\Deltat} \le c \sglmin \left( X\right)  $,  $ \mu \le \frac{c}{\norm{X} \kappa}  $, $\norm{\tilZtT\Zt \Qtp} \le \frac{c}{\kappa} \sglmin \left(\Zt \Qt \right) \sqrt{\norm{X}}$, and $\norm{\LAPT \PZQ } \le \frac{c}{\kappa} $, where $c>0$ is an absolute constant chosen sufficiently small.
    Then it holds that
    \begin{equation*}
        \begin{split}
        &\norm{ \tilZplus^T P_{\Zplus \Qplus}  } \\
        \le & \left( 1-  \frac{\mu}{4} \sglmin \left( X \right) \right) \norm{\tilZt^T \PZQ } 
        +4 \mu \norm{X} \norm{  \Zt \Qtp  }
        + 2\mu \norm{ \tilZt^T \Zt  } \sqrt{\norm{X}}+ \frac{\mu  \norm{\tilZt^T \Zt\Qtp}  \sglmin (X) }{ \sglmin \left( \Zt\Qt \right)  } \\
        &  + 800 \mu^2  \norm{X}^{5/2}.
        \end{split}
    \end{equation*}
\end{lemma}
Now we have all ingredients in place to prove Lemma \ref{lemma:phase2combined}, which is the main lemma for the second training phase.
The proof of Lemma \ref{lemma:phase2combined} has been deferred to Appendix \ref{sec:phase2proof}.
\begin{lemma}\label{lemma:phase2combined}
    Assume that the measurement operator $\mathcal{A}$ satisfies the rank-$(2r+1)$ restricted isometry property with constant $\delta \le  \frac{\hat{c}_1}{\kappa^{3} \sqrt{r} }$
    and assume that the step size satisfies $ \mu \le \frac{\hat{c}_2}{\kappa^{4} \norm{X}} $.
    Furthermore, assume that $ \sglmin \left( \LAT Z_{t_1} \right) \le \sqrt{ \frac{\sglmin (X)}{8} } $.
    Let $\left\{ \Zt \right\}_{t \in \mathbb{N}}$ and $\left\{ \tilZt \right\}_{t \in \mathbb{N}}$ be the iterates as defined in \eqref{def:Ztdefinition} and \eqref{def:tilZtdefinition}.
    Assume that after $t_1$ iterations we have that
    \begin{align}
        \norm{\Ztone Q_{t_1,\bot}} & \le \min \left\{  2 \sglmin (\Ztone Q_{t_1}) ;  \frac{ \hat{c}_3 \sqrt{ \sglmin \left(X\right) }}{ k \kappa^{35/8} } \right\}, \label{ineq:phase2assump1} \\
        \norm{\LAPT P_{Z_{t_1} Q_{t_1}} } & \le \frac{\hat{c}_4}{\kappa^2},  \label{ineq:phase2assump2} \\
        \norm{\Ztone} &\le 2 \sqrt{\norm{X}},  \label{ineq:phase2assump3} \\
        \norm{ \widetilde{Z}_{t_1}^T Z_{t_1} Q_{t_1,\bot} }  &\le \frac{\hat{c}_4 \hat{c}_5 \sqrt{ \norm{X} } }{\kappa^{3} }  \norm{\Ztone Q_{t_1,\bot}},    \label{ineq:phase2assump_balanc2}\\
        \norm{\widetilde{Z}_{t_1}^T P_{Z_{t_1} Q_{t_1 } }  }  &\le  \frac{ \hat{c}_4  \hat{c}_6 \sqrt{\norm{X}}}{\kappa^{3}}.   \label{ineq:phase2assump_balanc3}
    \end{align}
    Moreover, assume that 
    \begin{align} \label{ineq:phase2assump_balanc1}
        \norm{ \widetilde{Z}_{t_1}^T Z_{t_1} } + 400 \mu^2 \Bigg\lceil  \frac{\ln \left( \frac{ \sqrt{ \sglmin \left(X\right) } }{\sqrt{8} \sglmin \left( \LAT \Ztone \right) }   \right)}{\ln \left( 1 +\frac{\mu}{8}  \sglmin \left(X\right)   \right)} \Bigg\rceil \norm{X}^3
         \le \frac{\hat{c}_7}{\kappa^{4}} \norm{X}.
    \end{align}
    Then there is a natural number $t_2 \ge t_1$ and a constant $\gamma \in \mathbb{R}$ with
    \begin{align}
        t_2-t_1 &\le \Bigg\lceil \frac{\ln \left( \frac{ \sqrt{ \sglmin \left(X\right) } }{\sqrt{8} \sglmin \left( \LAT \Ztone \right) }   \right)}{\ln \left( 1 +\frac{\mu}{8}  \sglmin \left(X\right)   \right)} \Bigg\rceil, \label{ineq:t2bound} \\
        \norm{\Ztone Q_{\tone, \bot}} \le \gamma & \le \left( \frac{1}{128} \right)^{1/10}  \left( \sglmin \left(X\right) \right)^{1/10}   \norm{ \Ztone Q_{t_1,\bot}}^{4/5}  \label{ineq:gammabound} 
    \end{align}
    such that after $t_2$ iterations it holds that
    \begin{align}
        \sglmin \left( \LAT \Zttwo \right) &\ge  \sqrt{\frac{\sglmin (X)}{8}}, \label{ineq:phase2final1} \\
        \norm{\Zttwo Q_{t_2,\bot}} & 
        \le \gamma ,   \label{ineq:phase2final2}  \\
        \norm{\LAPT P_{Z_{t_2} Q_{t_2} } } & \le \frac{\hat{c}_4}{\kappa^2}, \label{ineq:phase2final3}  \\
        \norm{\Zttwo} &\le 2 \sqrt{\norm{X}},  \label{ineq:phase2final4}  \\
        \norm{ \widetilde{Z}_{t_2}^T Z_{t_2}  } &\le \norm{ \widetilde{Z}_{t_1}^T Z_{t_1} } + 400\mu^2 \left( t_2-t_1 \right) \norm{X}^3,  \label{ineq:phase2final_balanc1}\\ 
        \norm{ \widetilde{Z}_{t_2}^T Z_{t_2} Q_{t_2,\bot} }  & 
        \le \frac{ \hat{c}_4 \hat{c}_5 \sqrt{\norm{X}} \gamma }{\kappa^3},   \label{ineq:phase2final_balanc2}\\
        \norm{\widetilde{Z}_{t_2}^T P_{Z_{t_2} Q_{t_2 } }  }  &\le  \frac{\hat{c}_4 \hat{c}_6 \sqrt{\norm{X}}}{\kappa^3}.   \label{ineq:phase2final_balanc3}
    \end{align}
Here, $\hat{c}_1, \hat{c}_2, \hat{c}_3, \hat{c}_4, \hat{c}_5, \hat{c}_6, \hat{c}_7>0$ are absolute constants chosen small enough.
\end{lemma}
\begin{remark}\label{choiceofconstants}
In fact, in our proof we will show a bit more than the lemma suggests. 
Namely, we will prove that inequalities \eqref{ineq:phase2final2}--\eqref{ineq:phase2final_balanc3} hold for all $t$ such that $t_1 \le t \le t_2$.

As it turns out in our proof, one needs to be a bit careful on how to choose the constant.
In fact, we will need to choose the constants $\hat{c}_1, \hat{c}_2, \hat{c}_3, \hat{c}_4, \hat{c}_5, \hat{c}_6, \hat{c}_7>0$ such that the following relationships are satisfied:
$\hat{c}_1, \hat{c}_2, \hat{c}_3^{4/5} \ll \hat{c}_4 \hat{c}_5 $, $\hat{c}_4 \ll \hat{c}_5 \ll \hat{c}_6 \ll 1$.
Here, $a \ll b$ means that the constants $a$ and $b$ must be chosen such that they satisfy $a \le cb$ where $c$ is an absolute constant chosen sufficiently small.
\end{remark}
\subsubsection{Analysis of Phase 3}
Phase 3, the \textit{refinement phase}, begins after we have that $ \sglmin \left( \LAT \Zt \right) \ge \sqrt{\frac{\sglmin (X)}{8}} $.
In the following, we prove that in Phase 3 the learned signal $\Zt \ZtT - \tilZt \tilZtT $ converges to $\symA$ with respect to the spectral norm.
In other words, we show that 
\begin{equation}\label{term:intern2}
\norm{ \symA - \Zt \ZtT + \tilZt \tilZtT }
\end{equation}
becomes small.
A key difficulty in our analysis is now that the column spans of the matrices $\Zt$ and $\tilZt$ are not orthogonal to each other.
For this reason, individually estimating $ \norm{\LA \Sigma_X \LAT - \Zt\ZtT} $ and $ \norm{  \widetilde{L_X} \Sigma_X\widetilde{L_X}^T - \tilZt \tilZtT} $ leads to suboptimal bounds for \eqref{term:intern2}.
However, we can deal with this by using inequality \eqref{term:intern3} in Lemma \ref{SpecLossboundedlemma} below.  
\begin{lemma}\label{SpecLossboundedlemma}
Assume that $\mu \le \frac{ c}{ \kappa \norm{X} }$, $\norm{\LAPT\P{\Zt\Qt}} \le \frac{c }{\kappa }$, $\norm{\Zt}\le 2\sqrt{\norm{X}}$, and 
\begin{equation*}
\norm{\Deltat} \le \frac{c}{\kappa} \norm{\symA - \Zt\ZtT + \tilZt\tilZtT}.
\end{equation*}
Moreover, assume that $\sglmin(\Zt\Qt) \ge \sqrt{\frac{\sglmin(X)}{8}}$ and $\norm{\Zt\Qtp}\le c \sqrt{\sglmin(X)}$.
Then it holds that
\begin{equation*}
\norm{\LAPT \left( \symA - \Zt\ZtT + \tilZt \tilZtT \right) } \le 5\norm{\LAT \left( \symA - \Zt \ZtT + \tilZt \tilZtT \right) } + 4\norm{\Zt\Qtp}^2
\end{equation*}
and
\begin{equation}\label{term:intern3}
\norm{ \symA - \Zt\ZtT + \tilZt \tilZtT } 
\le 6\norm{\LAT \left( \symA - \Zt\ZtT + \tilZt \tilZtT \right) } 
+ 4\norm{\Zt\Qtp}^2.
\end{equation}
Here, $c>0$ is a sufficiently small absolute constant.
\end{lemma}
The proof of Lemma \ref{SpecLossboundedlemma} been deferred to Appendix \ref{sec:localconvergence}.
Due to inequality \eqref{term:intern3} to show that term \eqref{term:intern2} becomes small, it suffices now to show that 
\begin{equation*}
    \norm{\LAT \left( \symA - \Zt\ZtT + \tilZt \tilZtT \right) }
\end{equation*}
becomes sufficiently small.
For this task, we use the following lemma, whose proof has been deferred to Appendix \ref{sec:localconvergence}.
\begin{lemma}\label{lemma:localconvergence}
Under the assumptions of Lemma \ref{SpecLossboundedlemma} it holds that
\begin{align*}
    &\norm{\LAT(\symA - \Zplus\ZplusT + \tilZplus \tilZplusT)} \\
    \le &\left(1 - \frac{\mu}{128}\sglmin(X) \right)\norm{\LAT \left( \symA - \Zt\ZtT + \tilZt\tilZtT \right)} +  \frac{\mu}{20}\sglmin(X)\norm{\Zt\Qtp}^2.
\end{align*}
\end{lemma}
It is worth noting that both Lemma \ref{SpecLossboundedlemma} and Lemma \ref{lemma:localconvergence} do not not require any assumption on the imbalance matrix $\tilZtT \Zt $.

In \cite{stoger2021small} two similar lemmas are used to prove linear convergence in Phase 3.
However, the proof of Lemma \ref{SpecLossboundedlemma} in the paper at hand is significantly more involved due to the appearance of the additional $\tilZt$ term.

With this lemma in place and using the technical Lemmas \ref{lemma:anglecontrol}, \ref{lemma:normcontrolled}, \ref{lemma:balancedbase}, \ref{lemma:balancednessperp}, and \ref{ref:balancednessangle} from Phase 2, we are able to prove the following central lemma, which describes the training behaviour in the third phase.
Its proof has been deferred to Appendix \ref{sec:phase3proof}.
\begin{lemma}\label{lemma:phase3combined}
    Assume that $\mathcal{A}$ satisfies the rank-$(2r+1)$ restricted isometry property with constant $\delta < \frac{\hat{c}_1}{\kappa^3 \sqrt{r}} $.
    Furthermore, assume that the step size satisfies $ \mu \le \frac{\hat{c}_2}{\kappa^{4} \norm{X}} $.
    Let $\left\{ \Zt \right\}_{t \in \mathbb{N}}$ and $\left\{ \tilZt \right\}_{t \in \mathbb{N}}$ be the iterates defined in equations \eqref{def:Ztdefinition} and \eqref{def:tilZtdefinition}.
    Assume that there is a natural number $t_2$ and a positive real number $\gamma$ with 
    \begin{equation}\label{phase3:gammabound}
        \gamma 
        \le c_3 \min \left\{  \frac{ \sqrt{ \sglmin (X)}}{\kappa^{9/2}}; \frac{ \sqrt{\norm{X} } }{k^{4/3}} \right\} 
    \end{equation}
    such that
    \begin{align}
        \sglmin \left( \LAT Z_{t_2} \right)  &\ge  \sqrt{\frac{\sglmin (X)}{8}}, \label{ineq:phase3assump1}\\
        \norm{\Zttwo Q_{t_2,\bot}} 
        & \le \gamma, \label{ineq:phase3assump2}\\
        \norm{\LAPT P_{Z_{t_2}} Q_{t_2} } & \le \frac{\hat{c}_4}{\kappa^2},\label{ineq:phase3assump3}  \\
        \norm{\Zttwo} &\le 2 \sqrt{\norm{X}}, \label{ineq:phase3assump4} \\
        \norm{\tilde{Z}^T_{t_2} Z_{t_2} Q_{t_2, \bot} } 
        &\le \frac{\hat{c}_4 \hat{c}_5 \sqrt{\norm{X}}}{\kappa^3} \cdot \gamma \label{ineq:phase3assumpbalanc2}\\
        \norm{ \tilde{Z}^T_{t_2} P_{Z_{t_2} Q_{t_2}}} &\le \frac{\hat{c}_4 \hat{c}_6 \sqrt{\norm{X}}}{\kappa^3}. \label{ineq:phase3assumpbalanc3} 
    \end{align}
    Moreover, assume that 
    \begin{equation}\label{ineq:phase3assump_balanc1}
        \norm{ \widetilde{Z}_{t_2}^T Z_{t_2} } + 120000\mu \frac{ \ln \left( \frac{9  \sqrt{\norm{X}} }{200 k \gamma} \right) \norm{X}^3 }{ \sglmin (X)}  \le \frac{\hat{c}_7}{\kappa^{4}} \norm{X}.
    \end{equation}
    Then there is a natural number $t_3  \ge t_2$ with
    \begin{equation*}
        t_3-t_2 \le \frac{300 \ln \left( \frac{9  \sqrt{\norm{X}} }{200 k \gamma} \right) }{\mu \sglmin (X)}
    \end{equation*}
    such that after $t_3$ iterations it holds that
    \begin{equation}
        \norm{\symA - Z_{t_3} Z^T_{t_3}+  \tilde{Z}_{t_3} \tilde{Z}^T_{t_3} } \lesssim k \gamma  \sqrt{\norm{X} }. \label{phase3:finalerror}
    \end{equation}
The constants $\hat{c}_1,\hat{c}_2,\hat{c}_3,\hat{c}_4,\hat{c}_5,\hat{c}_6,\hat{c}_7>0$ are the same constants as those appearing in Lemma \ref{lemma:phase2combined}.
\end{lemma}

%% file: short_proof_outline.tex
\subsection{Outline of the proof}\label{sec:proofoutline}

In the following, we will give an outline of our proof.
As we will, see our proof consists of the following three steps.
\begin{enumerate}
    \item \textbf{Symmetrization:}
    First, we will show how the asymmetric matrix sensing problem in equation  can be equivalently reformulated as a matrix sensing problem with symmetric matrices.
    With this reformulation, we will be able to use some of the tools developed in \cite{stoger2021small} in the next two steps.
    However, while on a first glance this reformulated problem might resemble the one in \cite{stoger2021small}, there is a key difference.
    Namely, in \cite{stoger2021small}, it was assumed that both the ground truth and the learned matrices are positive semidefinite.
    Instead, in the scenario in this paper the ground truth matrix will have both positive and negative eigenvalues and we train one positive definite and one negative definite matrix.
    The dynamics of this two matrices are coupled with each other, which will lead to important changes in the proof as we will point out below.  
    \item \textbf{Decomposition of the learned matrices into signal and nuisance part:}
    In the second step, we will discuss how to decompose both training matrices into a signal and nuisance part.
    As in \cite{stoger2021small} the idea is that in the third step we then can show the signal matrices, which have rank $r$, will converge (approximately) to the positive definite part, respectively the negative definite part, of the ground truth matrix,
    whereas the nuisance terms will stay small.
    \item \textbf{Three-Phase Analysis:}
    To analyse the dynamics of gradient descent we will utilize the three-phase-analysis introduced in \cite{stoger2021small}.
    In the first phase, the \textit{alignment/spectral phase},
    we will show how, similar to a spectral initialization, the subspaces spanned by the signal parts of the learned get gradually more aligned with the subspace spanned by the ground truth matrix.
    In the second phase, the \textit{saddle avoidance phase}, we will show how the singular values of the signal parts are growing until they reach a certain basin of attraction of the ground truth matrices.
    In the third phase, the \textit{local convergence phase}, we will prove that the signal parts of the learned matrices converge linearly to the ground truth matrix.

    As already pointed out in the description of the first step,
    the key difficulty will be to deal with the fact that the dynamics of the two learned matrices are coupled with each other. 
    In Section \ref{sec:threephase} we will describe in detail how we deal with this.
\end{enumerate}

%% file: acknowledgements.tex
 \section*{Acknowledgements}
 MS is supported by the Packard Fellowship in Science and Engineering, a Sloan Fellowship in Mathematics, an NSF-CAREER under award \#1846369, DARPA Learning with Less Labels (LwLL) and FastNICS programs, and NSF-CIF awards \#1813877 and \#2008443.
 Furthermore, the authors want to thank Rene Vidal for fruitful discussions about matrix factorization.

%% file: auxiliary_lemmas.tex
\begin{proof}[Proof of Lemma \ref{lemma:RIPlemma}]
First we compute that 
\begin{equation}\label{ineq:RIPaux6} 
    \begin{split}
    \Zt \Qt \QtT \ZtT -\tilZt \Qt \QtT \tilZtT 
    &=   \frac{1}{2} \begin{bmatrix}
        V_t \Qt \\
        W_t \Qt
       \end{bmatrix} 
       \begin{bmatrix}
        V_t \Qt \\
        W_t \Qt
       \end{bmatrix}^T
       - \frac{1}{2} \begin{bmatrix}
        V_t \Qt \\
        -W_t \Qt
       \end{bmatrix} 
       \begin{bmatrix}
        V_t \Qt \\
        -W_t \Qt
       \end{bmatrix}^T \\
    &= 
    \begin{bmatrix}
    0 & V_t \Qt \QtT W_t^T \\
    W_t \Qt \QtT V_t^T & 0.
    \end{bmatrix}.
    \end{split}
\end{equation}
It follows that 
\begin{align*}
    \innerproduct{ B_i, \Zt \Qt \QtT \ZtT -\tilZt \Qt \QtT \tilZtT  }
    = \sqrt{2} \innerproduct{ A_i,V_t \Qt \QtT W_t^T  }.
\end{align*}
Analogously, we can compute that
\begin{align*}
\innerproduct{ B_i, \symA } = \sqrt{2} \innerproduct{ A_i,  X }.
\end{align*}
We obtain that
\begin{align*}
     &\left( \mathcal{B}^* \mathcal{B} \right) \left( \Zt \Qt \QtT \ZtT -\tilZt \Qt \QtT \tilZtT   -\symA \right)\\
    =& \sum_{i=1}^m B_i \innerproduct{ B_i, \Zt \Qt \QtT \ZtT -\tilZt \Qt \QtT \tilZtT   -\symA } \\
    =& \sqrt{2}  \sum_{i=1}^m B_i \innerproduct{ A_i,V_t \Qt \QtT W_t^T -X  }\\
    =&  \sum_{i=1}^m  \begin{bmatrix}
        0 & A_i \\
        A_i^T & 0 
    \end{bmatrix} \innerproduct{ A_i,V_t \Qt \QtT W_t^T  -X }  \\
    =&
     \begin{bmatrix}
        0 & \sum_{i=1}^m  A_i  \innerproduct{ A_i,V_t \Qt \QtT W_t^T  -X } \\
        \sum_{i=1}^m  A_i^T  \innerproduct{ A_i,V_t \Qt \QtT W_t^T  -X } & 0 
    \end{bmatrix}
    \\
    =& 
    \begin{bmatrix}
        0 & \left( \mathcal{A}^* \mathcal{A} \right) \left(  V_t \Qt \QtT W_t^T  - X \right) \\
        \left[ \left( \mathcal{A}^* \mathcal{A} \right) \left(  V_t \Qt \QtT W_t^T  - X \right) \right]^T & 0 
    \end{bmatrix}.
\end{align*}
Combining this last equality with equation \eqref{ineq:RIPaux6} and the definition of $\symA$ we obtain that
\begin{align*}
    &\left(\Id - \mathcal{B}^* \mathcal{B} \right) \left( \Zt \Qtp \QtpT \ZtT -\tilZt\Qtp \QtpT \tilZtT - \symA  \right)\\
    =     
    &\begin{bmatrix}
        0 & \left( \Id - \mathcal{A}^* \mathcal{A} \right) \left(   V_t \Qt \QtT W_t^T  -X \right)    \\
        \left[ \left( \Id - \mathcal{A}^* \mathcal{A} \right) \left(  V_t \Qt \QtT W_t^T  - X \right) \right]^T  &  0
    \end{bmatrix}.
\end{align*}
This implies that 
\begin{equation}\label{ineq:RIPaux1}
    \begin{split}
    &\norm{ \left(\Id - \mathcal{B}^* \mathcal{B} \right) \left( \Zt \Qt \QtT \ZtT -\tilZt \Qt \QtT \tilZtT   -\symA \right)  }\\
    =
    &\norm{  \left( \Id - \mathcal{A}^* \mathcal{A} \right) \left(   V_t \Qt \QtT W_t^T  -X\right)  }.
    \end{split}
\end{equation}
Note that $V_t \Qt \QtT W_t^T  - X $ has rank at most $2r$.
Hence, one can use the Restricted Isometry Property to show that (see, e.g., \cite[Lemma 7.3]{stoger2021small} or \cite{candesplan}) 
\begin{equation}\label{ineq:RIPaux2}
    \norm{  \left( \Id - \mathcal{A}^* \mathcal{A} \right) \left(   V_t \Qt \QtT W_t^T  -X \right)  }
    \le \delta \sqrt{r} \norm{ V_t \Qt \QtT W_t^T  - X  }.
\end{equation}
To estimate the right-hand side further note that it follows from equation \eqref{ineq:RIPaux6}  and the definition of $\symA$ that
\begin{equation*}
\begin{bmatrix}
    0 & V_t \Qt \QtT W_t^T  -X \\
    W_t \Qt \QtT V_t^T  -X^T & 0 
\end{bmatrix}
=\Zt \Qt \QtT \ZtT -\tilZt \Qt \QtT \tilZtT   -\symA,
\end{equation*}
which implies that
\begin{equation}\label{ineq:RIPaux5}
    \norm{ V_t \Qt \QtT W_t^T  - X  } = \norm{ \Zt \Qt \QtT \ZtT -\tilZt \Qt \QtT \tilZtT   -\symA }.
\end{equation}
By combining inequalities \eqref{ineq:RIPaux1}, \eqref{ineq:RIPaux2}, and \eqref{ineq:RIPaux5} we obtain inequality \eqref{ineq:RIPclaim1}.

It remains to prove inequality \eqref{ineq:RIPclaim2}. 
Using an analogous computation as in the proof of inequality \eqref{ineq:RIPaux1} we can show that 
\begin{equation}\label{ineq:RIPaux3}
    \begin{split}
    \norm{ \left(\Id - \mathcal{B}^* \mathcal{B} \right) \left( \Zt \Qtp \QtpT \ZtT -\tilZt \Qtp \QtpT \tilZtT   \right)  }
    =
    &\norm{  \left( \Id - \mathcal{A}^* \mathcal{A} \right) \left(   V_t \Qtp \QtpT W_t^T  \right)  }.    
    \end{split}
\end{equation} 
Next, consider the singular value decomposition $V_t \Qtp \QtpT W_t^T= \sum_{i=1}^{k-r} \sigma_i v_i w_i^T $.
It follows that
\begin{align}
    \norm{  \left( \Id - \mathcal{A}^* \mathcal{A} \right) \left(   V_t \Qtp \QtpT W_t^T  \right)  }
    &\le \sum_{i=1}^{k-r} \sigma_i \norm{  \left( \Id - \mathcal{A}^* \mathcal{A} \right) \left(  v_i w_i^T \right)  } \nonumber\\
    &\stackrel{(a)}{\le}  \delta  \sum_{i=1}^{k-r} \sigma_i \nonumber \\
    &= \delta \nucnorm{ V_t \Qtp \QtpT W_t^T } \nonumber \\
    &\le \delta \left( k-r \right) \norm{ V_t \Qtp \QtpT W_t^T } \nonumber \\
    &\stackrel{(b)}{=} \delta \left( k-r \right) \norm{ \Zt \Qtp \QtpT \ZtT -\tilZt \Qtp \QtpT \tilZtT  } \label{ineq:RIPaux4},
\end{align}
where inequality $(a)$ follows from the Restricted Isometry Property, see, e.g., \cite[Proof of Lemma 7.3]{stoger2021small}. 
Equality $(b)$ follows from the equality 
\begin{equation*}
\begin{pmatrix}
    0 & V_t \Qtp \QtpT W_t^T \\
    W_t \Qtp \QtpT V_t^T   & 0
\end{pmatrix}
    = \Zt \Qtp \QtpT \ZtT -\tilZt \Qtp \QtpT \tilZtT.    
\end{equation*}
Combining inequalities \eqref{ineq:RIPaux3} and \eqref{ineq:RIPaux4} we obtain inequality \eqref{ineq:RIPclaim2}.

Inequality \eqref{ineq:RIPclaim3} can be proven similarly as inequality \eqref{ineq:RIPclaim2}, which is why we omit the details. 
This completes the proof of Lemma \ref{lemma:RIPlemma}.
\end{proof}

%% file: spectral_phase.tex
The goal of this section is to prove Lemma \ref{lemma:spectralmain}.
For that, we will closely trace the proof for the spectral phase in \cite{stoger2021small}.
First, we need to introduce several definitions.
We define
\begin{equation*}
    F:= \left( \Bcal^* \Bcal \right) (\symA).
\end{equation*}
Moreover, for all natural numbers $t$ we define
\begin{align}
\Zt' &:= \left(\I + \mu F \right)^tZ_0, \label{def:Ztprime}\\
G_t &:= \left(\I + \mu F \right)^t, \nonumber\\
E_t &:= \Zt - \Zt'.\nonumber
\end{align}
Denote by $Z_{t}:=\sum_{i=1}^{k} \sigma_{i} u_{i} v_{i}^{T}$ the singular value decomposition of $Z_{t}$.
We define $M_{t}:=\sum_{i=1}^{r} \sigma_{i} u_{i} v_{i}^{T}$ and $N_{t}:=\sum_{i=r+1}^{k} \sigma_{i} u_{i} v_{i}^{T}$.

The first lemma shows how close the iterates $\Zt$ stay to power method iterates $\Zt'$ with respect to the spectral norm.
\begin{lemma}\label{lemma:spectral1}
Assume that $\norm{Z_0}^2\le  \frac{\norm{F}}{16}$ and that $\mathcal{A}$ satisfies the restricted isometry property of order $2r+1 $ with constant $\delta <1 $. Then, for all integers $t$ such that 
\begin{equation*}
0 \le t \le \frac{\ln\left(\frac{\norm{F}}{16  \min \left\{ k;n_1+n_2\right\} \norm{Z_0}^2}\right)}{3\ln \left( 1+\mu  \norm{F} \right)}
\end{equation*}
it holds that
\begin{equation}\label{ineq:spectral_aux1}
\norm{E_t} \le \frac{16}{ \norm{F} } \min \left\{  k; n_1 + n_2  \right\}  (1 + \mu \norm{F})^{3t} \norm{Z_0}^3 \le \norm{Z_0}.
\end{equation}
\end{lemma}
\begin{proof}[Proof of Lemma \ref{lemma:spectral1}]
We define
\begin{equation*}
 \hat{E}_{i}:= \mu \left[  \Bcal^* \Bcal (Z_{i-1}Z_{i-1}^{T} - \tilde{Z}_{i-1}\tilde{Z}_{i-1}^T)\right] Z_{i-1}.
\end{equation*}
To prove the lemma, we will first establish the following auxiliary equation for any natural number $t \ge 1$. 
\begin{equation}\label{ineq:claim}
Z'_{t}-Z_{t} = \sum_{i=1}^{t} \left( \Id + \mu F \right)^{t-i}\hat{E}_{i}.
\end{equation}
\noindent \textbf{Proof of the equation \eqref{ineq:claim}:}
We will prove this equation via induction.
For $t=1$ we note that
\begin{align*}
    Z_1 = & Z_0 - \mu \left[  \left(  \Bcal^* \Bcal \right) (Z_0Z_0^T - \tilde{Z}_0\tilde{Z}_0^T - \symA) \right] Z_0 \\
    = & (\Id + \mu F)Z_0 - \mu  \left[  \left(\Bcal^* \Bcal \right) (Z_0Z_0^T - \tilde{Z}_0\tilde{Z}_0^T) \right] Z_0 \\
    = & Z_1' - \hat{E}_1,
\end{align*}
which proves the claim for $t = 1$.
Now assume that equation \eqref{ineq:claim} holds for a natural number $t\ge 2$.
We obtain that
\begin{align*}
    \Zplus = & \Zt - \mu  \left[  \left(\Bcal^* \Bcal \right) \left(\Zt\ZtT - \tilZt\tilZtT - \symA \right) \right] \Zt \\
    = & \left(\Id + \mu F \right) \Zt - \mu \left[ \left(\Bcal^* \Bcal \right)  \left(\Zt\ZtT - \tilZt\tilZtT \right) \right] \Zt \\
    \stackrel{(a)}{=} & \left( \left( \Id + \mu F \right) Z'_t - \sum_{i=1}^{t} \left( \Id + \mu F \right)^{t-i+1}\hat{E}_{i}\right) - \hat{E}_{t+1} \\
    \stackrel{(b)}{=} & Z'_{t+1} - \sum_{i=1}^{t} \left(\Id + \mu F \right)^{t-i+1}\hat{E}_{i} - \hat{E}_{t+1}\\ 
    = & Z'_{t+1} - \sum_{i=1}^{t+1} \left( \Id + \mu F \right)^{t+1-i}\hat{E}_{i},
\end{align*}
where in equality $(a)$ we have used the induction hypothesis and in equality $(b)$ we used the definition of $Z'_{t+1} $.
This shows the induction step for $t + 1$ and, hence, equation \eqref{ineq:claim} is shown for any natural number $t\ge 1$.\\

\noindent  In order to bound $\norm{E_t} = \|Z'_{t}-Z_{t}\|$ we will again proceed by induction. First, note that for the induction step $t=0$ inequality \eqref{ineq:spectral_aux1} holds true since $Z_0=Z'_0$.
Now let $t\ge 1$ be a natural number.
We observe that for any natural number $ t \ge 1  $ we have that
\begin{align*}
    \norm{ \hat{E}_{i} } \le & \mu  \norm{ \left( \Bcal^* \Bcal \right) \left(Z_{i-1}Z_{i-1}^{T} - \tilde{Z}_{i-1}\tilde{Z}_{i-1}^T \right) }  \norm{Z_{i-1}} \\
    \stackrel{(a)}{\le} & \left(1+\delta\right) \mu \nucnorm{  (Z_{i-1}Z_{i-1}^{T} - \tilde{Z}_{i-1}\tilde{Z}_{i-1}^T) }  \norm{Z_{i-1} } \\
    \le & 2 \mu \nucnorm{  (Z_{i-1}Z_{i-1}^{T} - \tilde{Z}_{i-1}\tilde{Z}_{i-1}^T) } \norm{Z_{i-1}} \\
    \le & 2 \mu \left(  \nucnorm{  \left(Z_{i-1}Z_{i-1}^{T}} + \nucnorm{ \tilde{Z}_{i-1}\tilde{Z}_{i-1}^T \right) }\right) \norm{Z_{i-1}} \\ 
    \stackrel{(b)}{=} & 4 \mu  \nucnorm{  Z_{i-1}Z_{i-1}^{T}}  \norm{Z_{i-1}} \\
    = & 4 \mu  \fbnorm{  Z_{i-1}}^2  \norm{Z_{i-1}} \\
    \le & 4 \mu \min \left\{ k;n_1+n_2\right\} \norm{Z_{i-1}}^3 \\
    \le & 4\mu  \min \left\{ k;n_1+n_2\right\}   \left( \norm{Z'_{i-1}} + \norm{ E_{i-1}} \right)^3 \\
    \le & 16\mu \min \left\{ k;n_1+n_2\right\}  \left( \norm{Z'_{i-1}}^3 + \norm {E_{i-1} }^3 \right) \\
    \stackrel{(c)}{\le} & 16\mu \min \left\{ k;n_1+n_2\right\}  \left( \norm{ \left( \Id + \mu F \right)^{3i-3} } \norm{Z_0}^3 + \norm{ E_{i-1} }^3 \right) \\
    \le & 16\mu \min \left\{ k;n_1+n_2\right\}  \left(  \left( 1 + \mu \norm{F} \right)^{3i-3}  \norm{Z_0}^3 + \norm{ E_{i-1} }^3 \right) \\
    \stackrel{(d)}{\le} & 32\mu \min \left\{ k;n_1+n_2\right\}  \left(1 + \mu \norm{F} \right)^{3i-3} \norm{Z_0}^3,
\end{align*}
where in inequality $(a)$ we used the Restricted Isometry Property (see inequality \eqref{ineq:RIPclaim3} in Lemma \ref{lemma:RIPlemma}).
Equality $(b)$ follows from equations \eqref{def:Ztdefinition} and \eqref{def:tilZtdefinition}.
Inequality $(c)$ follows from the definition of $Z'_{i-1}$ and inequality $(d)$ from the induction hypothesis that $ \norm{ E_{i-1} } \le \norm{Z_0}$ for $i \le t$.
Using equation \eqref{ineq:claim} we obtain that
\begin{align*}
    \norm{E_t} = &  \norm{Z_{t} - Z'_{t}} \\
    \le &\sum_{i=1}^{t} \norm{ \left( \Id + \mu F \right)^{t-i} } \norm{ \hat{E}_{i} } \\
    \le &32\mu \min \left\{ k;n_1+n_2\right\} \sum_{i=1}^{t}(1 + \mu \norm{F} )^{t+2i-3} \norm{Z_0}^3 \\
    = &32 \mu \min \left\{ k;n_1+n_2\right\} (1 + \mu \norm{F})^{t-1}\frac{(1 + \mu \norm{F} )^{2t} - 1}{(1 + \mu\norm{F})^{2}-1} \norm{Z_0}^3 \\
    \le &32 \mu \min \left\{ k;n_1+n_2\right\} (1 + \mu \norm{F})^{t-1}\frac{(1 + \mu \norm{F} )^{2t+1}}{2\mu \norm{F}} \norm{Z_0}^3 \\
    = &\frac{16}{\norm{F}}  \min \left\{ k;n_1+n_2\right\} \left(1 + \mu\norm{F} \right)^{3t}  \norm{Z_0}^3.
\end{align*}
By the assumption $t \le \frac{\ln\left(\frac{ \norm{F}}{16  \min \left\{ k;n_1+n_2\right\} \norm{Z_0}^2}\right)}{3\ln \left(1+\mu \norm{F}  \right)}$ we have that
\begin{equation*}
\frac{16}{\norm{F}}  \min \left\{ k;n_1+n_2\right\} \left(1 + \mu \norm{F}\right)^{3t}  \norm{Z_0}^3 \le \frac{16}{ \norm{F} }\cdot\frac{\norm{F}}{16  \norm{Z_0}^2 } \norm{Z_0}^3 = \norm{Z_0}.
\end{equation*}
Combining the above two inequalities we have shown the induction step for $t\ge 1$, which finishes the proof.
\end{proof}

Let $Z_{t}=\sum_{i=1}^{k} \sigma_{i} u_{i} v_{i}^{T}$ be the singular value decomposition of $Z_{t}$.
Define $M_{t}:=\sum_{i=1}^{r} \sigma_{i} u_{i} v_{i}^{T}$ and $N_{t}:=\sum_{i=r+1}^{k} \sigma_{i} u_{i} v_{i}^{T}$.
Moreover, recall that by construction 
\begin{align*}
F= \left(\Bcal^* \Bcal \right) \left( \symA \right) = \begin{pmatrix}
    0 &  \left(\mathcal{A}^* \mathcal{A} \right)  (X) \\
   \left[  \left( \mathcal{A}^* \mathcal{A}\right)  (X) \right]^T & 0
\end{pmatrix}
\end{align*}
has the same number of positive and negative eigenvalues. 
Denote by $F_1 $ the subspace spanned by the eigenvectors corresponding to the $r$ largest eigenvalues of $F$.
By $L_{F_1}$ we denote an orthonormal matrix whose column span is equal to $F_1$.

To prove Lemma \ref{lemma:spectralmain}, we also need the following two technical lemmas.
\begin{lemma}\label{lemma:spectral2}
    Assume that
    \begin{equation}\label{ineq:specbound4}
    \lambda_{r+1}(G_{t}) \norm{Z_0}+ \norm{E_{t}}<\lambda_{r}(G_{t}) \sglmin \left( L_{F_1}^T Z_0 \right).
    \end{equation}
    Then it holds that
    \begin{align}
    \sigma_{r}(Z_{t}) &\ge \lambda_{r}(G_{t})\sigma_{\min}(L_{F_1}^T Z_0) - \norm{E_t}, \label{ineq:specbound1}\\
    \sigma_{r+1}(Z_{t}) &\le \lambda_{r+1} \left(G_{t} \right ) \norm{Z_0} + \norm{E_t}, \label{ineq:specbound2}\\
    \norm{ L_{F_1,\bot}^T P_{M_t}}  &\le \frac{\lambda_{r+1}(G_t) \norm{Z_0}+ \norm{E_t}}{ \lambda_{r}\left( G_t \right)\sglmin \left( L_{F_1}^T Z_0 \right) - \lambda_{r+1}(G_t) \norm{Z_0}- \norm{E_t}}. \label{ineq:specbound3}
    \end{align}
    \end{lemma}
\begin{lemma}\label{lemma:spec3}
    Assume that $ \norm{ \LAPT P_{M_{t}}} \le \frac{1}{8}$. Then the following inequalities hold:
    \begin{align}
        \sigma_{r}(Z_tQ_t)&\ge \frac{1}{2}\sigma_{r} \left(Z_t \right),\label{SpecZtQt}\\
        \norm{ \LAPT P_{Z_tQ_t}} &\le 7 \norm{\LAPT P_{M_{t}}},\label{Specangle}\\
        \norm{Z_tQ_{t,\bot}}&\le 2\sigma_{r+1} \left(Z_t\right).\label{SpecZtQtp}
    \end{align}
\end{lemma}
These two lemmas are analogous to Lemma 8.3 and Lemma 8.4 in \cite{stoger2021small} and can be proven with exactly the same arguments, which is why we skip the details. 

The next lemma shows that after a certain amount of iterations, the signal part $ \Zt \Qt$ is sufficiently well-aligned with the ground truth signal and,
moreover, that the singular values of the signal part and the spectral norm of the nuisance part are sufficiently separated.
\begin{lemma}\label{lemma:spec4}
Assume that 
\begin{equation}\label{spectral:closeness}
    \norm{  \mathcal{A}^* \mathcal{A} \left( X \right) -X   } \le \frac{ c \sglmin (X)}{\kappa^2}
\end{equation}
and
\begin{equation}\label{definition:tstar}
t_\star := 
\left\lceil     \frac{\ln\left(\frac{c\sglmin \left(L_{F_1}^T Z_0\right)}{2\kappa^{2} \norm{Z_0} }\right)}{\ln \left(1 - \frac{\mu\sglmin(X)}{8} \right)} \right\rceil
\le \frac{\ln\left(\frac{\norm{F}}{16  \min \left\{ k;n_1+n_2\right\} \norm{Z_0}^2}\right)}{3\ln \left( 1+\mu  \norm{F} \right)}
\end{equation}
for a positive constant $ c \le \frac{1}{32}$.
Moreover, assume that the step size $\mu$ satisfies $ \mu \le \frac{\tilde{c}}{\sglmin (X)} $, where $\tilde{c}>0$ is a sufficiently small absolute constant.
Then it holds that
\begin{align}
\norm{L_{X,\bot}^{T} P_{Z_{t_\star}Q_{t_\star}}}&\le \frac{28c}{\kappa^2},\label{ineq:spectral2}\\
\sigma_{\min}(Z_{t_\star}Q_{t_\star}) 
&\ge \left(\frac{2\kappa^{2} \norm{Z_0} }{c\sglmin \left(L_{F_1}^T Z_0\right)}\right)^{2\kappa} \frac{\sglmin \left( L_{F_1}^T Z_0 \right)}{4} ,\label{ineq:spectral3}\\
\norm{Z_{t_\star}Q_{t_\star,\bot}} 
&\le \min \left\{ 2 \sglmin \left( \Ztstar Q_{\tstar} \right); \  \left( \frac{2\kappa^2 \norm{Z_0}}{c \sglmin \left( L_{F_1}^T Z_0 \right)} \right)^{ 16 \kappa} \cdot \frac{4c \sglmin \left( L_{F_1}^T Z_0 \right)}{ \kappa^2 } \right\} .\label{ineq:spectral4}
\end{align}
\end{lemma}

\begin{proof}[Proof of Lemma \ref{lemma:spec4}]
Before proving inequalities \eqref{ineq:spectral2}, \eqref{ineq:spectral3}, and \eqref{ineq:spectral4}, we first prove the following auxiliary inequality
\begin{equation}
\gamma:= \frac{\lambda_{r+1}(G_{t_\star}) \norm{Z_0}+\norm{E_{t_\star}}}{\lambda_{r}(G_{t_\star}) \sglmin\left( L_{F_1}^T Z_0  \right)} \le \frac{c}{\kappa^2}. \label{ineq:spectral1} 
\end{equation}
For that, we recall that
\begin{equation*}
 F
 = \left(\Bcal^* \Bcal \right) \left( \symA \right)
 = \begin{pmatrix}
    0 &   \left( \mathcal{A}^* \mathcal{A} \right)  (X)  \\
   \left[  \left(\mathcal{A}^* \mathcal{A}  (X) \right) \right]^T  & 0
\end{pmatrix}.   
\end{equation*}
It follows by Weyl's inequality and assumption \eqref{spectral:closeness} that 
\begin{equation}\label{ineq:Weyl1}
    \lambda_{r} \left( F \right)
    =\sigma_{r} \left(\left(\mathcal{A}^* \mathcal{A}  (X) \right)  \right)
    \ge \sigma_{\min} \left( X \right) - \norm{\left(\mathcal{A}^* \mathcal{A}  (X) \right)-X }
    \ge \frac{\sigma_{\min} \left( X \right)}{2}.
\end{equation}
Again using Weyl's inequality, assumption \eqref{spectral:closeness},  and, in addition, $ \sigma_{r+1 } (X) =0  $ we can derive that
\begin{equation}\label{ineq:Weyl2}
    \lambda_{r+1} \left( F \right)  
    =\sigma_{r+1} \left(\left(\mathcal{A}^* \mathcal{A}  (X) \right)  \right)
    \le \norm{\left(\mathcal{A}^* \mathcal{A}  (X) \right)-X }
    \le \frac{\sigma_{\min} \left( X \right)}{4}.
\end{equation}
Next, we note that for any $1\le i \le n_1+n_2$ it holds that
\begin{align*}
    \lambda_{i} \left( G_{t_\star} \right) 
    = \lambda_{i} \left( \left( 1+ \mu F  \right)^{t_\star}  \right)
    = \left(  \lambda_{i} \left( 1 + \mu F \right)  \right)^{t_\star}
    = \left(  1+ \mu \lambda_{i} \left( F \right)  \right)^{t_\star},
\end{align*}
since we have assumed that our step size satisfies $\mu \le \frac{\tilde{c}}{\sglmin (X)} \le \frac{1}{2 \norm{F}}$ (where the second inequality is also due to \eqref{spectral:closeness}).
Combining this observation with inequalities \eqref{ineq:Weyl1} and \eqref{ineq:Weyl2} it follows that
\begin{align}
    \lambda_{r} \left( G_ {t_\star}\right) & \ge  \left(  1+ \mu \frac{\sglmin (X)}{2}  \right)^{t_\star},\label{ineq:intern401}\\
    \lambda_{r+1} \left( G_{t_\star}\right) &\le \left(  1+ \mu \frac{\sglmin (X)}{4}  \right)^{t_\star}.  \label{ineq:intern402}
\end{align}
Thus, we obtain that
\begin{align*}
\gamma 
\stackrel{(a)}{\le} & \frac{\left(  1+ \mu \frac{\sglmin (X)}{4}  \right)^{t_\star} \norm{Z_0} + \norm{E_{t_\star} }}{\left(  1+ \mu \frac{\sglmin (X)}{2}  \right)^{t_{\star}} \sglmin \left( L_{F_1}^T Z_0  \right)}\\
\stackrel{(b)}{\le} & \frac{\left( \left(  1+ \mu \frac{\sglmin (X)}{4}  \right)^{t_{\star}} +1 \right) \norm{Z_0} }{\left(  1+ \mu \frac{\sglmin (X)}{2}  \right)^{t_\star} \sglmin \left( L_{F_1}^T Z_0  \right)}\\
\le & \left( \frac{1+ \frac{\mu \sglmin (X)}{4}}{1+ \frac{\mu \sglmin (X)}{2}}  \right)^{t_{\star}}  \frac{ 2\norm{Z_0} }{\sglmin \left( L_{F_1}^T Z_0  \right) }\\ 
= & \left( 1 - \mu \frac{ \frac{ \sglmin (X)}{4}}{1+ \frac{\mu \sglmin (X)}{2}}  \right)^{t_{\star}}  \frac{ 2\norm{Z_0} }{ \sglmin \left( L_{F_1}^T Z_0  \right) }\\ 
\stackrel{(c)}{\le} & \left( 1 - \frac{\mu \sglmin (X)}{8}  \right)^{t_{\star}}  \frac{ 2\norm{Z_0} }{\sglmin \left( L_{F_1}^T Z_0  \right) },\\ 
\end{align*}
where in inequality $(a)$ we used the definition of $\gamma$ and inequalities \eqref{ineq:intern401} and \eqref{ineq:intern402}.
Inequality $(b)$ follows from $\norm{E_{t_{\star}}} \le \norm{Z_0} $, which is a consequence of Lemma \ref{lemma:spectral1} and the definition of $t_{\star}$, see assumption \eqref{definition:tstar}.
Inequality $(c)$ is due to $\mu \le \frac{2}{\sglmin (X)}$, which follows from our assumption on the step size $\mu$.
It follows from the definition of $\tstar$, see \eqref{definition:tstar}, that
\begin{align}\label{ineq:gammabound1}
    \gamma
    \le \exp \left( t_{\star} \ln \left( 1 - \mu \frac{ \sglmin (X)  }{8}  \right) \right)  \frac{ 2\norm{Z_0} }{ \sglmin \left( L_{F_1}^T Z_0  \right) }
    \le \frac{c}{\kappa^2}.
\end{align}
This shows inequality \eqref{ineq:spectral1} and now we are in a position to prove inequalities \eqref{ineq:spectral2}, \eqref{ineq:spectral3}, and \eqref{ineq:spectral4}.\\

\noindent\textbf{Proof of inequality \eqref{ineq:spectral2}:}
First, we note that it follows from the Davis-Kahan $\sin \Theta$-Theorem (see \cite{kahanbound}) and assumption \eqref{spectral:closeness} that
\begin{equation}\label{spectral:DavisKahanapplication}
    \norm{ \LAPT L_{F_1} }
    =\norm{ L_{F_1,\bot}^T \LA } 
    \le \frac{  \norm{F-\symA}}{\lambda_{r} (\symA)-\norm{F-\symA}} 
    = \frac{\norm{   \left(\mathcal{A}^* \mathcal{A}  (X) \right) - X }}{\sigma_{\min} (X)-\norm{   \left(\mathcal{A}^* \mathcal{A}  (X) \right) -X}} 
    \le \frac{2c}{\kappa^2}.
\end{equation}
Next, we observe that
\begin{align}\label{ineq:intern369}
    \frac{ \lambda_{r+1} \left(G_{t_\star} \right) \norm{Z_0}+ \norm{E_{t_\star} } }{\lambda_{r}(G_{t_\star}) \sglmin\left( L_{F_1}^T Z_0 \right)- \lambda_{r+1} \left(G_{t_\star} \right) \norm{Z_0}- \norm{E_{t_\star}}}
    \le \frac{2 \left( \lambda_{r+1} \left(G_{t_\star} \right) \norm{Z_0}+ \norm{E_{t_\star}}\right)  }{\lambda_{r}(G_{t_\star}) \sglmin\left( L_{F_1}^T Z_0 \right) }
    \le \frac{2c}{\kappa^2},
\end{align}
where we have used in both inequalities that $ \gamma \le \frac{c}{\kappa^2} \le \frac{1}{2} $, see inequality \eqref{ineq:gammabound1}.
We also observe that
\begin{align}
   \norm{ \LAPT  P_{M_{t_{\star}}}} 
   &\le \norm{ \LAPT L_{F_1} L_{F_1}^T P_{M_{t_{\star}}}}+\norm{L_{X, \bot}^TL_{F_1,\bot} L_{F_1,\bot}^T P_{M_{t_{\star}}}} \nonumber \\
   &\le \norm{ \LAPT L_{F_1} }+\norm{ L_{F_1,\bot}^T P_{M_{t_{\star}}}}\nonumber\\
   &\stackrel{(a)}{\le} \frac{2c}{\kappa^2} +  \frac{\lambda_{r+1} \left(G_{\tstar} \right) \norm{Z_0}+ \norm{E_{\tstar}}}{ \lambda_{r}\left( G_{\tstar} \right)\sglmin \left( L_{F_1}^T Z_0 \right) - \lambda_{r+1}\left(G_{\tstar} \right) \norm{Z_0}- \norm{E_{\tstar}}}\nonumber\\
   &\stackrel{(b)}{\le} \frac{4c}{\kappa^2} 
   \stackrel{(c)}{\le}  \frac{1}{8}, \label{ineq:intern357}
\end{align}
where in inequality $(a)$ we used \eqref{spectral:DavisKahanapplication} and  \eqref{ineq:specbound3} in Lemma \ref{lemma:spectral2}.
(Lemma \ref{lemma:spectral2} is applicable since assumption \eqref{ineq:specbound4} is fulfilled due to \eqref{ineq:spectral1}.)
Inequality $(b)$ follows from inequality \eqref{ineq:spectral1} and the definition of $\gamma$.
Inequality $(c)$ is due to our assumption $c\le \frac{1}{32}$.
Now we can prove inequality \eqref{ineq:spectral2} by observing that
\begin{align}
\norm{\LAPT P_{Z_{t_\star}Q_{t_\star}}}
&\stackrel{(a)}{\le} 7 \norm{ \LAPT P_{M_{t_\star}} } \nonumber \\
&\le 7 \norm{ \LAPT \left( L_{F_1} L_{F_1}^T + L_{F_1, \bot} L_{F_1,\bot}^T \right) P_{M_{t_\star}} }\nonumber\\
&\le 7 \left(   \norm{ \LAPT L_{F_1} } + \norm{ L_{F_1,\bot}^T  P_{M_{t_\star}} } \right)\nonumber\\
& \stackrel{(b)}{\le} 7\norm{ \LAPT  L_{F_1} } + \frac{7 \left( \lambda_{r+1} \left(G_{t_\star} \right) \norm{Z_0}+ \norm{E_{t_\star} } \right)}{\lambda_{r}(G_{t_\star}) \sglmin\left( L_{F_1}^T Z_0 \right)- \lambda_{r+1} \left(G_{t_\star} \right) \norm{Z_0}- \norm{E_{t_\star}}}, \label{ineq:spectral765}
\end{align}
where for inequality $(a)$ we used inequality \eqref{Specangle} in Lemma \ref{lemma:spec3}, where this lemma is applicable since we have that $ \norm{\LAPT P_{Z_{t_\star}Q_{t_\star}} } \le \frac{1}{8}$ due to inequality \eqref{ineq:intern357}.
Inequality $(b)$ follows from Lemma \ref{lemma:spectral2}.
(The assumption in this lemma is fulfilled since we have that $\gamma \le \frac{c}{\kappa^2}<1$.)
Inserting inequality \eqref{spectral:DavisKahanapplication} in \eqref{ineq:spectral765}, we obtain that
\begin{equation*}
\norm{\LAPT P_{Z_{t_\star}Q_{t_\star}}}\le  7\norm{ \LAPT  L_{F_1} } + \frac{14 c }{\kappa^2} \stackrel{\eqref{spectral:DavisKahanapplication}}{\le} \frac{28c}{\kappa^2}. 
\end{equation*}
This proves inequality \eqref{ineq:spectral2}.\\

\noindent\textbf{Proof of inequality \eqref{ineq:spectral3}:}
Recall from the proof of inequality \eqref{ineq:spectral2} that both Lemma \ref{lemma:spectral2} and Lemma \ref{lemma:spec3} are applicable.
Then we note that
\begin{align}
\sglmin \left(Z_{t_\star}Q_{t_\star}\right) 
&\stackrel{(a)}{\ge} \frac{1}{2}\sigma_{r}(Z_{t_\star}) \nonumber\\
&\stackrel{(b)}{\ge} \frac{1}{2} \left(\lambda_{r} \left(G_{t_\star} \right)\sigma_{\min} \left(L_{F_1}^{T} Z_0 \right)-\norm{E_{t_\star}} \right)\nonumber \\
& \stackrel{(c)}{\ge}\frac{1}{4}\lambda_{r}(G_{t_\star})\sigma_{\min}(L_{F_1}^{T} Z_0), \label{spectral:intern62}
\end{align}
where in inequality $(a)$ we used inequality \eqref{SpecZtQt} in Lemma \ref{lemma:spec3}.
Inequality $(b)$ follows from inequality \eqref{ineq:specbound1} in Lemma \ref{lemma:spectral2}.
Inequality $(c)$ is a consequence of $\gamma \le 1/2$.
In order to proceed, we note that 
\begin{align}
    \lambda_{r} \left( G_\tstar \right)
    &\stackrel{(a)}{\ge} \left( 1 + \frac{\mu \sglmin (X)}{2} \right)^\tstar \nonumber \\
    & = \exp \left( \tstar \ln \left(  1 + \frac{\mu \sglmin (X)}{2}\right) \right)\nonumber\\
    &\stackrel{(b)}{\ge} \exp \left( \frac{\ln\left(\frac{c\sglmin \left(L_{F_1}^T Z_0\right)}{2\kappa^{2} \norm{Z_0} }\right)}{\ln \left(1 - \frac{\mu\sglmin(X)}{8} \right)} \ln \left( 1 + \frac{\mu \sglmin (X)}{2} \right) \right)\nonumber\\
    &\stackrel{(c)}{\ge} \exp \left( 2 \ln\left(\frac{2\kappa^{2} \norm{Z_0} }{c\sglmin \left(L_{F_1}^T Z_0\right)}\right)  \sglmin (X)  \right)\nonumber\\
    &= \left(\frac{2\kappa^{2} \norm{Z_0}}{c\sglmin \left(L_{F_1}^T Z_0\right) }\right)^{2\kappa}.\label{spectral:intern61}
\end{align}
For inequality $(a)$ we used inequality \eqref{ineq:intern401} and for inequality $(b)$ we used the definition of $\tstar$.
Inequality $(c)$ follows from the elementary inequality $ \frac{x}{1-x} \le \ln (1+x) $ and from the assumption $ \mu \le \frac{\tilde{c}}{\sglmin (X)} $.
By combining inequalities \eqref{spectral:intern62} and \eqref{spectral:intern61} we obtain inequality \eqref{ineq:spectral3}.\\

\noindent\textbf{Proof of inequality \eqref{ineq:spectral4}:}
Again, recall from the proof of inequality \eqref{ineq:spectral2} that both Lemma \ref{lemma:spectral2} and Lemma \ref{lemma:spec3} are applicable.
We observe that
\begin{align}
\norm{Z_{t_\star}Q_{t_\star,\bot}}
&\stackrel{(a)}{\le} 2\sigma_{r+1}(Z_{\tstar})\nonumber\\
&\stackrel{(b)}{\le} 2 \left(\lambda_{r+1}(G_{t_\star})    \norm{Z_0}+ \norm{E_{t_\star}} \right)\nonumber \\
&\stackrel{(c)}{\le} \frac{2c \lambda_{r}(G_{\tstar})\sigma_{\min}( L_{F_1}^T Z_0)}{\kappa^2}. \label{spectral:intern51}
\end{align}
Inequality $(a)$ follows from \eqref{SpecZtQtp}, see Lemma \ref{lemma:spec3}, and inequality $(b)$ follows from \eqref{ineq:specbound2}, see Lemma \ref{lemma:spectral2}.
Inequality $(c)$ follows from $\gamma \le \frac{c}{\kappa^2}$, see inequality \eqref{ineq:spectral1}.
By combining inequalities \eqref{spectral:intern62} and \eqref{spectral:intern51} and using that $ c \le 1/32 $ we observe that 
\begin{equation}\label{spectral:intern81}
 \norm{Z_{t_\star}Q_{t_\star,\bot}} \le 2 \sglmin \left( \Ztstar Q_{\tstar} \right). 
\end{equation}
Note that by Weyl's inequality
\begin{equation*}
    \lambda_{r} \left( F \right)
    = \sigma_{r} \left( \mathcal{A}^* \mathcal{A} \left( X \right) \right)
    = \sigma_{\min} \left( X  \right) + \norm{  \left( \mathcal{A}^* \mathcal{A} \right) \left(X\right) -X  }
    \le \frac{5}{4} \sigma_{\min} \left( X\right).
\end{equation*}
Thus, we can compute 
\begin{align*}
    \lambda_{r} \left(G_{\tstar}\right)
    &= \left( 1+ \mu \lambda_{r} \left( F \right) \right)^\tstar \\
    &\le \left( 1+ \frac{5 \mu \sigma_{\min} \left(X\right)  }{4}  \right)^\tstar\\
    &=  \exp \left( \tstar \ln \left( 1+ \frac{5 \mu \sigma_{\min} \left(X\right) }{4} \right) \right)\\
    &\stackrel{(a)}{\le} \exp \left( \left( \frac{  \ln \left( \frac{c \sglmin \left( L_{F_1}^T Z_0 \right)}{2\kappa^2 \norm{Z_0}} \right)}{\ln \left( 1- \frac{\mu \sglmin \left( X \right)}{8} \right)} +1 \right) \ln \left( 1+ \frac{5 \mu \sigma_{\min} \left(X\right) }{4} \right) \right)\\
    &\stackrel{(b)}{\le} 2 \left( \frac{2\kappa^2 \norm{Z_0}}{c \sglmin \left( L_{F_1}^T Z_0 \right)} \right)^{16 \kappa}.
\end{align*}
For inequality $(a)$ we used the definition of $\tstar$ and for inequality $(b)$ we used the elementary inequality $ \frac{x}{1-x} \le \ln \left(1+x\right) \le x $ and the assumption that $ \mu \le \frac{ \tilde{c} }{\sglmin (X)} $.
By combining this inequality chain with inequality \eqref{spectral:intern51} we obtain that 
\begin{equation}\label{spectral:intern82}
\norm{ \Ztstar Q_{t_\star,\bot}} \le \left( \frac{2\kappa^2 \norm{Z_0}}{c \sglmin \left( L_{F_1}^T Z_0 \right)} \right)^{16 \kappa} \frac{4c \sglmin \left( L_{F_1}^T Z_0 \right)}{\kappa^2}.
\end{equation}
By combining \eqref{spectral:intern81} and \eqref{spectral:intern82} we obtain inequality \eqref{ineq:spectral4}. 
This finishes the proof of Lemma \ref{lemma:spec4}.
\end{proof}
We also need to check that after the spectral phase also the conditions related to the imbalance term are fulfilled.
\begin{lemma}\label{lemma:spec5}
Let $0< c\le 1/32$. Assume that
\begin{equation*}
    \norm{X - \left(\mathcal{A}^* \mathcal{A}\right) \left(X\right) } \le \frac{c \sglmin \left(X\right)}{\kappa^2}
\end{equation*}
and that the measurement operator $\mathcal{A}$ satisfies the restricted isometry property of order $2r+1$ with constant $\delta <1$.
Moreover, assume that $ \mu \le \frac{\tilde{c}}{\sglmin (X)} $, where $\tilde{c}>0$ is a sufficiently small absolute constant.
If 
\begin{equation*}
    \tstar := \left\lceil\frac{\ln\left(\frac{c\sglmin \left(L_{F_1}^{T} Z_0\right)}{2\kappa^{2}\norm{Z_0}}\right)}{\ln\left(1 - \frac{\mu\sglmin(X)}{8}\right)}\right\rceil 
    \le \frac{\ln\left(\frac{\norm{F}}{16  \min \left\{ k;n_1+n_2\right\} \norm{Z_0}^2}\right)}{3\ln \left( 1+\mu  \norm{F} \right)},
\end{equation*}
then it holds that 
\begin{align}
    \norm{Z_{\tstar}} &\le 4 \left(\frac{2\kappa^{2} \norm{Z_0}}{c\sglmin \left(L_{F_1}^{T}Z_0 \right)}\right)^{32 \kappa} \norm{Z_0},\label{ineq:spectral_balanc4}\\
    \norm{\tilde{Z}_{\tstar}^T Z_{\tstar}  }&\le 6 \left(\frac{2\kappa^{2} \norm{Z_0}}{c\sglmin(L_{F_1}^{T}Z_0)}\right)^{32 \kappa}  \norm{Z_0}^2, \label{ineq:spectral_balanc1} \\ 
    \norm{\tilde{Z}_{\tstar}^T Z_{\tstar} Q_{\tstar, \bot}  }&\le  4 \left(\frac{2\kappa^{2} \norm{Z_0}}{c\sglmin \left(L_{F_1}^{T}Z_0 \right)}\right)^{32 \kappa}  \norm{Z_0}  \norm{\Ztstar Q_{\tstar, \bot}} ,\label{ineq:spectral_balanc2} \\ 
    \norm{\tilde{Z}_{\tstar}^T P_{ Z_{\tstar} Q_{\tstar} }  }&\le \frac{228 c}{\kappa^2}  \left( \frac{2\kappa^2 \norm{Z_0}}{ c \sglmin \left( L_{F_1}^T Z_0 \right) } \right)^{32 \kappa} \norm{Z_0} . \label{ineq:spectral_balanc3}
\end{align}
\end{lemma}

\begin{proof}
\noindent\textbf{Proof of inequality \eqref{ineq:spectral_balanc4}:}
For that, we first observe that
\begin{align}
\norm{Z_\tstar}
&\le \norm{Z_\tstar'} + \norm{E_\tstar}\nonumber\\
&\le \left( 1 + \mu \norm{F} \right)^\tstar \norm{Z_0} + \norm{E_\tstar}\nonumber\\
&\stackrel{(a)}{\le} \left( 1 + \mu \norm{F} \right)^\tstar \norm{Z_0} + \norm{Z_0}\nonumber\\
&\le 2 \left( 1 + \mu \norm{F} \right)^\tstar \norm{Z_0},\label{ineq:spectralintern1}
\end{align}
where inequality $(a)$ follows from $ \norm{E_\tstar}\le \norm{Z_0}$, see Lemma \ref{lemma:spectral1}.
Moreover, we note that
\begin{align}
    \left( 1 + \mu \norm{F} \right)^\tstar 
    &= \exp \left(  \tstar \ln \left( 1 + \mu \norm{F} \right)  \right) \nonumber \\
    &\stackrel{(a)}{\le} \exp \left(  \left( \frac{\ln\left(\frac{c\sglmin \left(L_{F_1}^{T}Z_0 \right)}{2\kappa^{2} \norm{Z_0}}\right)}{\ln \left(1 - \frac{\mu\sglmin(X)}{8}\right)} +1 \right)  \ln  \left( 1 + \mu \norm{F} \right) \right)\nonumber\\
    &\stackrel{(b)}{\le} \exp \left(   \ln\left(\frac{2\kappa^{2} \norm{Z_0}}{c\sglmin(L_{F_1}^{T}Z_0)}\right) \frac{16  \norm{F}}{ \sglmin (X) } + \mu \norm{F}  \right) \nonumber\\
    &\stackrel{(c)}{\le}   \left(\frac{2\kappa^{2} \norm{Z_0}}{c\sglmin(L_{F_1}^{T}Z_0)}\right)^{32 \kappa}  \exp \left( \mu \norm{F} \right) \nonumber \\ 
    &\stackrel{(d)}{\le}  2 \left(\frac{2\kappa^{2} \norm{Z_0}}{c\sglmin(L_{F_1}^{T}Z_0)}\right)^{32 \kappa}. \label{spectral:intern12}
\end{align}
Here, inequality $(a)$ follows from the definition of $\tstar$ and inequality $(b)$ follows from the elementary inequality $ \frac{x}{1-x} \le \ln (1+x) \le x  $ as well as the assumption $ \mu \le \frac{\tilde{c}}{\sglmin \left(X\right)} $.
In inequalities $(c)$ and $(d)$ we have used that $\norm{F} \le 2 \norm{X} $, which follows from
\begin{equation*}
    \norm{F}
    = \norm{\left( \mathcal{A}^* \mathcal{A} \right) \left(X\right) }
     \le \norm{X} +  \norm{ X - \left(  \mathcal{A}^* \mathcal{A} \right) \left(X\right) } \le 2 \norm{X}.
\end{equation*}
Moreover, in inequality $(d)$ we have used our assumption on the step size $\mu$.
Combining inequality \eqref{spectral:intern12} with inequality \eqref{ineq:spectralintern1}, we obtain inequality \eqref{ineq:spectral_balanc4}.\\

\noindent\textbf{Proof of inequality \eqref{ineq:spectral_balanc1}:}
In order to show inequality \eqref{ineq:spectral_balanc1}, we first introduce the following notation for all natural numbers $t\ge 1$
\begin{align*}
    \tilZt' &:= \left( \Id -\mu F \right)^t \tilde{Z}_0, \\
    \tilde{E}_t &:= \tilZt - \tilZt'.
\end{align*}
Before proving inequality \eqref{ineq:spectral_balanc1}, we will first show that 
\begin{equation}\label{ineq:spectralsymmetry}
\tilde{E}_t=
\underset{=:D}{\underbrace{
\begin{pmatrix}
    \Id & 0 \\
    0 & -\Id 
\end{pmatrix}}}
E_t.
\end{equation}
To show this, we note that $\tilde{Z}_0'= \tilde{Z}_0 = DZ_0 = Z_0'$, see \eqref{def:Ztprime}. 
Then, it follows by induction that
\begin{align*}
  \tilZt'
=\left( \Id - \mu F \right) \tilde{Z}_{t-1}'
=\left( \Id - \mu F \right)  \begin{pmatrix}
    \Id & 0 \\
    0 & -\Id 
\end{pmatrix} Z_{t-1}'
= 
\begin{pmatrix}
    \Id & 0 \\
    0 & -\Id 
\end{pmatrix}
\left(\Id + \mu F\right) Z_{t-1}' = 
\begin{pmatrix}
    \Id & 0 \\
    0 & -\Id 
\end{pmatrix}
\Zt'.
\end{align*}
Since we have that $\Zt = D \tilZt $ (see \eqref{equ:Ztdefinition}) equation \eqref{ineq:spectralsymmetry} follows from the definition of $E_t$ and $\tilde{E}_t'$. 
Next, we compute that 
\begin{equation*}
\tilZtstar'^T Z_{\tstar}' 
= \tilde{Z}_0^T \left(  \Id -\mu F  \right)^\tstar \left(  \Id + \mu F  \right)^\tstar Z_0 
=  \tilde{Z}_0^T \left(  \Id - \mu^2 F^2  \right)^\tstar Z_0.
\end{equation*}
This implies that 
\begin{align}
\norm{ \tilZtstar'^T Z_{\tstar}' } &\le  \norm{ \tilde{Z}_0 } \norm{\Id - \mu^2 F^2 }^\tstar \norm{ Z_0 } \nonumber \\
&=\norm{\Id - \mu^2 F^2 }^\tstar \norm{ Z_0 }^2\nonumber \\
&\le  \norm{ Z_0 }^2, \label{spectral:intern11}
\end{align}
where in the last line we used that $F^2$ is a positive semidefinite matrix, $\norm{F} \le 2 $, and our assumption on the step size $\mu$. 
Next, we note that 
\begin{equation*}
    \tilZtstar^T Z_{\tstar} 
    = \left( \tilZtstar' + \tilEtstar \right)^T \left( \Ztstar' + \Etstar \right)
    =  \tilZtstar'^T \Ztstar' + \tilEtstar^T \Ztstar'   + \tilZtstar'^T \Etstar   +\tilEtstar^T \Etstar.
\end{equation*}
This implies that 
\begin{align}
\norm{  \tilZtstar^T \Ztstar} 
&\stackrel{\eqref{spectral:intern11}}{\le}   \norm{ Z_0 }^2 + \norm{ \tilEtstar } \norm{Z_{\tstar}'} + \norm{\Etstar} \norm{\tilZtstar}  + \norm{\tilEtstar} \norm{\Etstar}    \nonumber \\ 
&\stackrel{(a)}{=}  \norm{ Z_0 }^2 +  2 \norm{Z_{\tstar}'} \norm{\Etstar}  +\norm{\Etstar}^2.\nonumber \\ 
&\stackrel{(b)}{\le}   \norm{ Z_0 }^2 +  2 \norm{Z_{\tstar}'} \norm{Z_0}  +\norm{Z_0}^2\nonumber\\  
&\le 2  \norm{ Z_0 }^2 +  2 \norm{Z_{\tstar}'} \norm{Z_0} \nonumber\\  
&\le 2  \norm{ Z_0 }^2 +  2  \left( 1 + \mu \norm{F} \right)^\tstar  \norm{Z_0} \nonumber\\  
&\stackrel{(c)}{\le} 2  \norm{ Z_0 }^2 + 4\left(\frac{2\kappa^{2} \norm{Z_0}}{c\sglmin(L_{F_1}^{T}Z_0)}\right)^{32 \kappa} \norm{Z_0}^2. \label{spectral:intern6}
\end{align}
Equation $(a)$ is due to $\norm{\tilEtstar}= \norm{\Etstar}$ and $ \norm{\tilZtstar} = \norm{\Ztstar}$, which is a consequence of the symmetry between $\tilZtstar$ and $\Ztstar$, see Lemma \ref{lemma:symmetry} and equation \eqref{ineq:spectralsymmetry}.
In inequality $(b)$ we used that $\norm{\Etstar}\le \norm{Z_0}$, which is a consequence of Lemma \ref{lemma:spectral1}.
Inequality $(c)$ can be obtained by using \eqref{spectral:intern12}.
It follows that
\begin{align*}
 \norm{  \tilZtstar^T \Ztstar} 
& \le 6 \left(\frac{2\kappa^{2} \norm{Z_0}}{c\sglmin(L_{F_1}^{T}Z_0)}\right)^{32 \kappa}  \norm{Z_0}^2.
\end{align*}
This proves inequality \eqref{ineq:spectral_balanc1}.\\

\noindent\textbf{Proof of inequality \eqref{ineq:spectral_balanc2}:}
We observe that 
\begin{equation*}
    \norm{\tilde{Z}_{\tstar}^T Z_{\tstar} Q_{\tstar, \bot}  }
    \le \norm{\tilZtstar} \norm{\Ztstar Q_{\tstar, \bot}}
    \stackrel{(a)}{\le} 4 \left(\frac{2\kappa^{2} \norm{Z_0}}{c\sglmin \left(L_{F_1}^{T}Z_0 \right)}\right)^{32 \kappa} \norm{Z_0}\norm{\Ztstar Q_{\tstar, \bot}},
\end{equation*}
where in $(a)$ we used inequality \eqref{ineq:spectral_balanc4} and the fact that $ \norm{\tilZtstar} = \norm{\Ztstar} $, which is a consequence of Lemma \ref{lemma:symmetry}.\\
 
\noindent\textbf{Proof of inequality \eqref{ineq:spectral_balanc3}:}
We compute that 
\begin{align}
\norm{\tilde{Z}_{\tstar}^T P_{ Z_{\tstar} Q_{\tstar} }  } 
&\le \norm{ Q_{\tstar}^T \tilde{Z}_{\tstar}^T P_{ Z_{\tstar} Q_{\tstar} }  } + \norm{ Q_{\tstar,\bot}^T \tilde{Z}_{\tstar}^T P_{ Z_{\tstar} Q_{\tstar} }  } \nonumber  \\ 
&\le \norm{ P^T_{ \tilde{Z}_{\tstar} Q_{\tstar} }  P_{ Z_{\tstar} Q_{\tstar} }  } \norm{ \tilde{Z}_{\tstar} } + \norm{ \tilde{Z}_{\tstar} Q_{\tstar,\bot}  } \nonumber \\   
&= \norm{ P^T_{ \tilde{Z}_{\tstar} Q_{\tstar} }  P_{ Z_{\tstar} Q_{\tstar} }  } \norm{ Z_{\tstar} } + \norm{ Z_{\tstar} Q_{\tstar,\bot}  }, \label{spectral:intern1}
\end{align}
where the last line follows from $\norm{\Ztstar}= \norm{\tilZtstar}$ and $ \norm{ Z_{\tstar} Q_{\tstar,\bot}  } = \norm{ \tilZtstar Q_{\tstar,\bot}  }$, see Lemma \ref{lemma:symmetry}.
Next, we compute that
\begin{align}
    \norm{ P^T_{ \tilde{Z}_{\tstar} Q_{\tstar} }  P_{ Z_{\tstar} Q_{\tstar} }  } 
    &\le \norm{ P^T_{ \tilde{Z}_{\tstar} Q_{\tstar} } \LA \LAT  P_{ Z_{\tstar} Q_{\tstar} }  } + \norm{ P^T_{ \tilde{Z}_{\tstar} Q_{\tstar} } \LAP \LAPT   P_{ Z_{\tstar} Q_{\tstar} }  } \nonumber \\   
    &\le \norm{ P^T_{ \tilde{Z}_{\tstar} Q_{\tstar} } \LA } \norm{ \LAT  P_{ Z_{\tstar} Q_{\tstar} }  } + \norm{ P^T_{ \tilde{Z}_{\tstar} Q_{\tstar} } \LAP } \norm{ \LAPT   P_{ Z_{\tstar} Q_{\tstar} }  } \nonumber \\   
    &\le \norm{ P^T_{ \tilde{Z}_{\tstar} Q_{\tstar} } \LA } +  \norm{ \LAPT   P_{ Z_{\tstar} Q_{\tstar} }  } \nonumber \\   
    &\le \norm{ P^T_{ \tilde{Z}_{\tstar} Q_{\tstar} } \widetilde{L_{X,\bot}} } +  \norm{ \LAPT   P_{ Z_{\tstar} Q_{\tstar} }  } \nonumber \\   
    &= 2 \norm{ \LAPT   P_{ Z_{\tstar} Q_{\tstar} }  }, \label{spectral:intern2}  
\end{align}
where in the last equality we have used the fact that $\LAPT   P_{ Z_{\tstar} Q_{\tstar} } = \widetilde{L_{X,\bot}}^T  P_{ \tilde{Z}_{\tstar} Q_{\tstar} }  $, see Lemma \ref{lemma:symmetry}.
By combining inequalities \eqref{spectral:intern1} and \eqref{spectral:intern2} we obtain that 
\begin{equation*}
    \norm{\tilde{Z}_{\tstar}^T P_{ Z_{\tstar} Q_{\tstar} }  }\le 2  \norm{Z_{\tstar}} \norm{\LAPT P_{ Z_{\tstar} Q_{\tstar}   } } + \norm{ Z_{\tstar} Q_{\tstar,\bot}  } . 
\end{equation*}
Thus, using inequalities \eqref{ineq:spectral2} and \eqref{ineq:spectral4} from Lemma \ref{lemma:spec4} and inequality \eqref{ineq:spectral_balanc4} we obtain that
\begin{align*}
    \norm{\tilde{Z}_{\tstar}^T P_{ Z_{\tstar} Q_{\tstar} }  }
&\le \frac{224 c}{\kappa^2}  \left( \frac{2\kappa^2 \norm{Z_0}}{ c \sglmin \left( L_{F_1}^T Z_0 \right) } \right)^{32 \kappa} \norm{Z_0}
+ \left( \frac{2\kappa^2 \norm{Z_0}}{c \sglmin \left( L_{F_1}^T Z_0 \right)} \right)^{16 \kappa} \cdot \frac{4c \sglmin \left( L_{F_1}^T Z_0 \right)}{ \kappa^2 }\\      
&\le \frac{228 c}{\kappa^2}  \left( \frac{2\kappa^2 \norm{Z_0}}{ c \sglmin \left( L_{F_1}^T Z_0 \right) } \right)^{32 \kappa} \norm{Z_0},
\end{align*}
where the second inequality we used to $  \sglmin \left( L_{F_1}^T Z_0 \right)\le \norm{Z_0}$.
This completes the proof.
\end{proof}

The next lemma tells us how small one needs to choose the initialization $Z_0$ such that inequality \eqref{ineq:spectralinternclaim} below holds. 
The right-hand side of inequality \eqref{ineq:spectralinternclaim} can be interpreted as the maximal number of iterations for which the gradient descent iterates $\Zt$ can be approximated by the power method iterates $\Zt'$,
whereas the number on the left-hand side gives an upper bound on the number of iterations which are needed to align the subspace of the learned signal $\Zt$ with the subspace of the true signal.
\begin{lemma}\label{lemma:tstarcheck}
 Assume that 
 \begin{equation}\label{ineq:Z0smallness}
     \norm{Z_0} \le \sqrt{\frac{\norm{X} }{ 24 \min \left\{  n_1 + n_2; k \right\} }} \left( \frac{c \sglmin \left( L_{F_1}^T Z_0 \right)}{2\kappa^2 \norm{Z_0}} \right)^{ 18 \kappa }.
 \end{equation}   
 Moreover, assume that $ \frac{2}{3} \norm{X} \le \norm{F} \le \frac{4}{3} \norm{X} $ and  $ \mu \le \frac{\tilde{c}}{\sglmin \left(X\right)} $ for a sufficiently small absolute constant $\tilde{c}>0$.
 Then it holds for any constant $0<c\le 1$ that 
 \begin{equation}\label{ineq:spectralinternclaim}
    t_{\star}:=
     \left\lceil \frac{\ln\left(\frac{c\sglmin \left(L_{F_1}^T Z_0\right)}{2\kappa^{2} \norm{Z_0} }\right)}{\ln \left(1 - \frac{\mu\sglmin(A)}{8} \right)} \right\rceil
     \le \frac{ \ln \left( \frac{\norm{F}}{16 \min \left\{ k; n_1 + n_2 \right\}\norm{Z_0}^2}  \right)  }{3 \ln \left(1 + \mu \norm{F}\right)}.
 \end{equation}
\end{lemma}

\begin{proof}[Proof of Lemma \ref{lemma:tstarcheck}]
    We observe that to prove the claim it suffices to show that
    \begin{align*}
        3 \ln \left( 1+ \mu \norm{F} \right) \left(  \frac{\ln\left(\frac{c\sglmin \left(L_{F_1}^T Z_0\right)}{2\kappa^{2} \norm{Z_0} }\right)}{\ln \left(1 - \frac{\mu\sglmin(X)}{8} \right)}   + 1\right) 
        \le  \ln \left( \frac{\norm{F}}{16 \min \left\{  n_1 + n_2; k \right\}\norm{Z_0}^2}  \right). 
    \end{align*}
Next, we note that 
\begin{align*}
  3 \ln \left( 1+ \mu \norm{F} \right) \left(  \frac{\ln\left(\frac{c\sglmin \left(L_{F_1}^T Z_0\right)}{2\kappa^{2} \norm{Z_0} }\right)}{\ln \left(1 - \frac{\mu\sglmin(X)}{8} \right)}   + 1\right) 
  &\stackrel{(a)}{\le} 3 \mu \norm{F}  \left( \frac{8 \ln \left( \frac{2 \kappa^2 \norm{Z_0}}{c \sglmin \left( L_{F_1}^T Z_0 \right)} \right)}{  \mu \sglmin (X)   } +1  \right)      \\
  &\stackrel{(b)}{\le} 4 \mu \norm{X}  \left( \frac{8 \ln \left( \frac{2 \kappa^2 \norm{Z_0}}{c \sglmin \left( L_{F_1}^T Z_0 \right)} \right)}{  \mu \sglmin (X)   }  +1 \right)      \\
  &= 32 \kappa \ln \left( \frac{2\kappa^2 \norm{Z_0}}{c \sglmin \left( L_{F_1}^T Z_0 \right)} \right) + 4 \mu \norm{X}\\
  &\stackrel{(c)}{\le} 36 \kappa \ln \left( \frac{2\kappa^2 \norm{Z_0}}{c \sglmin \left( L_{F_1}^T Z_0 \right)} \right),
\end{align*}
where in inequality $(a)$ we used the elementary inequality $\ln \left( 1+x\right) \le x$ and the assumption that $\mu\le \frac{\tilde{c}}{ \sglmin \left(X\right)}$.
Inequality $(b)$ follows from the assumption $ \norm{F} \le \frac{4}{3} \norm{X}$.
For inequality $(c)$ we used the assumption that $\mu \le \frac{\tilde{c}}{ \sglmin (A)}$ with a sufficiently small absolute constant $\tilde{c} >0$.
Thus, we observe that \eqref{ineq:spectralinternclaim} holds, if we have that
\begin{equation*}
    \left( \frac{2\kappa^2 \norm{Z_0}}{c \sglmin \left( L_{F_1}^T Z_0 \right)} \right)^{36 \kappa }  
    \le \frac{\norm{F}}{16 \min \left\{  n_1 + n_2; k \right\}\norm{Z_0}^2}.
\end{equation*}
By rearranging terms and using the assumption $ \frac{2}{3} \norm{X} \le \norm{F} $, we see that this is implied by assumption \eqref{ineq:Z0smallness}.
\end{proof}
So far, our results hold for any deterministic initialization $Z_0$.
The next lemma utilizes the fact that $Z_0$ is a random matrix.
\begin{lemma}\label{lemma:probability}
Assume that $V_0 = \alpha V \in \R^{n_1 \times k} $ and $W_0= \alpha W \in \R^{n_2 \times k}$ for some fixed parameter $\alpha >0$, where the matrices $V$ and $W$ have i.i.d. entries with distribution $ \mathcal{N} \left(0,1\right) $.
Then, with probability at least $1-  C_1 \exp \left( - c_1 \max \left\{ n_1 +n_2; k \right\} \right)  $, it holds that
\begin{equation}\label{ineq:prob1}
   \frac{ \alpha \sqrt{ \max \left\{  n_1 +n_2; k \right\} } }{2}  
   \le  \norm{Z_0} 
   \le 3 \alpha \sqrt{ \max \left\{ n_1+n_2; k \right\} } .
\end{equation}
Moreover, for any $\varepsilon>0$, with probability at least $1- \left( C_2 \varepsilon  \right)^{k -r+1} -  \exp \left( - c_2 k \right) $ we have that
\begin{equation}\label{ineq:prob2}
    \sglmin \left( L_{F_1}^T Z_0 \right) \ge \alpha \varepsilon \left( \sqrt{k} - \sqrt{r-1} \right).
\end{equation}
Here, $C_1, C_2, c_1, c_2>0$ are some fixed numerical constants.
\end{lemma}
\begin{proof}[Proof of Lemma \ref{lemma:probability}]
Recall that
\begin{equation*}
    Z_0=
    \begin{pmatrix}
        V_0 \\
        W_0
    \end{pmatrix}
    = \alpha \begin{pmatrix}
        V\\
        W
    \end{pmatrix}
    = \alpha Z 
     \in \R^{ \left( n_1 + n_2 \right) \times k }.
\end{equation*} 
Since $Z$ has i.i.d.~entries with distribution $\mathcal{N} \left(0,1\right)$, 
it is well-known (see, e.g., \cite[Corollary 7.3.3]{vershynin2018high}) that with probability at least $1- 2 \exp \left( -\frac{\max \left\{ n_1+n_2; k \right\} }{C} \right)  $ we must have that 
\begin{equation*}
    \norm{Z}
     \le 3\sqrt{ \max \left\{ n_1+n_2; k \right\} }  . 
\end{equation*}
This implies the second inequality in \eqref{ineq:prob1}.
To prove the first inequality in \eqref{ineq:prob1} it suffices to note that when $k \le n_1 +n_2$ then it holds that
\begin{equation*}
    \norm{Z}
    \ge \twonorm{Ze_1}
    \ge \frac{\sqrt{n_1+n_2}}{2},
\end{equation*}
where the second inequality holds with probability at least  $ 1 - O \left( \exp \left( -\frac{n_1 + n_2}{C}\right) \right) $.
Analogously, if $k \ge n_1 +n_2$, we have that
\begin{equation*}
    \norm{Z}
    \ge \twonorm{Z^Te_1}
    \ge \frac{\sqrt{k}}{2},
\end{equation*}
with probability at least $ 1 - O \left( \exp \left( -\frac{k}{C} \right) \right) $.
Since $Z_0= \alpha Z$, by choosing the fixed numerical constants $C_1, c_1 >0$ appropriately this shows the first inequality in \eqref{ineq:prob1}.

In order to prove inequality \eqref{ineq:prob2}, we note that $ \sglmin \left( L_{F_1}^T Z_0 \right) = \alpha \sglmin \left( L_{F_1}^T Z \right) $.
Note that $L_{F_1} \in \R^{ (n_1 + n_2) \times r }$ is a fixed matrix (conditional on the measurement matrices $ \left\{ A_i \right\}_{i=1}^m$), which describes a subspace of dimension $r$.
In particular, due to the rotation invariance of the Gaussian distribution, the matrix $ L_{F_1}^T Z $ is again a random matrix with i.i.d. entries with distribution $ \mathcal{N} \left(0,1\right)$ (conditional on $\left\{ A_i \right\}_{i=1}^m$).
Thus it follows from \cite[Theorem 1.1]{rudelson_vershynin} that for every $\varepsilon>0$  with probability at least $1 - \left( C_2 \varepsilon \right)^{k-r+1} - \exp \left( -c_2k  \right) $ it holds that 
\begin{equation*}
\sglmin \left( L_{F_1}^T Z_0 \right)  = \alpha \sglmin \left( L_{F_1}^T Z \right) \ge \alpha \varepsilon \left( \sqrt{k} - \sqrt{r} \right).
\end{equation*}
 Choosing the fixed numerical constants $C_2,c_2>0$ appropriately this implies the second claim in Lemma \ref{lemma:probability}.
\end{proof}
Now we have all ingredients in place to prove the main result for the spectral phase.
\begin{proof}[Proof of Lemma \ref{lemma:spectralmain}]
First, we note that due to Lemma \ref{lemma:probability} we have with probability at least $1 - C_2 \exp \left( - c_1 k \right)  + \left( C_3 \varepsilon \right)^{k-r+1}  $ that 
\begin{align}
   \frac{ \alpha \sqrt{ \max \left\{  n_1 +n_2; k \right\} } }{2}  
   \le  \norm{Z_0} 
   &\le 3\alpha  \sqrt{ \max \left\{ n_1 + n_2; k \right\} },  \label{ineq:spectralintern13}\\
    \sglmin \left( L_{F_1}^T Z_0 \right)
    & \ge \alpha \varepsilon \left( \sqrt{k} - \sqrt{r-1} \right),\label{ineq:spectralintern14}
\end{align}
where $C_2,C_3, c_1>0 $ are some fixed absolute constants.
We first check that the condition
 \begin{equation*}
    t_{\star}=
     \left\lceil \frac{\ln\left(\frac{c\sglmin \left(L_{F_1}^T Z_0\right)}{2\kappa^{2} \norm{Z_0} }\right)}{\ln \left(1 - \frac{\mu\sglmin(X)}{8} \right)} \right\rceil
     \le \frac{ \ln \left( \frac{\norm{F}}{16 \min \left\{  n_1 + n_2; k \right\}\norm{Z_0}^2}  \right)  }{3 \ln \left(1 + \mu \norm{F}\right)}
 \end{equation*}
 is satisfied. Due to Lemma \ref{lemma:tstarcheck} it suffices to note that 
 \begin{align*}
     \norm{Z_0} 
     &\stackrel{(a)}{\le} 3\alpha \sqrt{\max \left\{  n_1 +n_2; k \right\}}\\
     &\stackrel{(b)}{\le} 3 \sqrt{\frac{\norm{X} }{ C_1 \min \left\{  n_1 + n_2; k \right\} }} \left( \frac{c \varepsilon \left( \sqrt{k} -\sqrt{r-1} \right) }{6\kappa^2 \sqrt{ \max \left\{ k; n_1 + n_2 \right\} }} \right)^{ 18 \kappa }\\
     &\stackrel{(c)}{\le} \sqrt{\frac{\norm{X} }{ 24 \min \left\{  n_1 + n_2; k \right\} }} \left( \frac{c \sglmin \left( L_{F_1}^T Z_0 \right)}{2\kappa^2 \norm{Z_0}} \right)^{18 \kappa },
 \end{align*}   
 where inequality $(a)$ follows from \eqref{ineq:spectralintern13} and inequality $(b)$ follows from our assumption on $\alpha$.
 Inequality $(c)$ follows from \eqref{ineq:spectralintern13} and \eqref{ineq:spectralintern14} (and from choosing the constant $C_2 >0$ sufficiently large). 
Thus, we can apply Lemmas \ref{lemma:spec4} and \ref{lemma:spec5} and by combining these lemmas with inequalities \eqref{ineq:spectralintern13}, \eqref{ineq:spectralintern14} and assumption \eqref{spectral:alphaassumption} we obtain that
 \begin{align*}
 \norm{L_{A,\bot}^{T} P_{Z_{t_\star}Q_{t_\star}}}
 &\le \frac{28c}{\kappa^2},\\
 \sigma_{\min}(Z_{t_\star}Q_{t_\star}) 
 &\ge \left(\frac{2\kappa^{2} \norm{Z_0} }{c\sglmin \left(L_{F_1}^T Z_0\right)}\right)^{ 2 \kappa} \frac{\sglmin \left( L_{F_1}^T Z_0 \right)}{4} 
 \ge \left(\frac{2\kappa^{2}  }{c} \right)^{ 2 \kappa} \frac{ \alpha \varepsilon \left( \sqrt{k} -\sqrt{r-1} \right) }{4}, \\
 \norm{Z_{t_\star}Q_{t_\star,\bot}} 
 &\le \min \left\{ 2 \sglmin \left( \Ztstar Q_{\tstar} \right); \  \left( \frac{2\kappa^2 \norm{Z_0}}{c \sglmin \left( L_{F_1}^T Z_0 \right)} \right)^{ 16 \kappa} \cdot \frac{4c \sglmin \left( L_{F_1}^T Z_0 \right)}{ \kappa^2 } \right\} \\
 &\le \min \left\{ 2 \sglmin \left( \Ztstar Q_{\tstar} \right); \  \left( \frac{6\kappa^2  \sqrt{\max \left\{  n_1+n_2; k \right\} }}{ c \varepsilon \left( \sqrt{k} - \sqrt{r -1} \right)} \right)^{ 16 \kappa} \cdot \frac{12 c \alpha  \sqrt{\max \left\{  n_1 +n_2; k \right\}}}{ \kappa^2 } \right\},\\
 \norm{Z_{\tstar}} 
 &\le 4 \left(\frac{2\kappa^{2} \norm{Z_0}}{c\sglmin \left(L_{F_1}^{T}Z_0 \right)}\right)^{32 \kappa} \norm{Z_0}
 \le 12 \alpha \left(\frac{6\kappa^2  \sqrt{\max \left\{  n_1+n_2; k \right\} }}{ c\varepsilon \left( \sqrt{k} - \sqrt{r -1} \right)}\right)^{32 \kappa}  \sqrt{\max \left\{  n_1 +n_2;k \right\}}\\ 
 &\le 2 \sqrt{\norm{X}},\\
 \norm{\tilde{Z}_{\tstar}^T Z_{\tstar}  }
 &\le 6 \left(\frac{2\kappa^{2} \norm{Z_0}}{c\sglmin(L_{F_1}^{T}Z_0)}\right)^{32 \kappa}  \norm{Z_0}^2 
 \le 54 \alpha^2 \left(\frac{6\kappa^2  \sqrt{\max \left\{  n_1+n_2; k \right\} }}{ c\varepsilon \left( \sqrt{k} - \sqrt{r -1} \right)}\right)^{32 \kappa}  \max \left\{  n_1 +n_2;k \right\}   \\ 
 &\le \frac{ c \norm{X}}{  \kappa^4},\\
 \norm{\tilde{Z}_{\tstar}^T Z_{\tstar} Q_{\tstar, \bot}  }
 &\le  4 \left(\frac{2\kappa^{2} \norm{Z_0}}{c\sglmin \left(L_{F_1}^{T}Z_0 \right)}\right)^{32 \kappa}  \norm{Z_0}  \norm{\Ztstar Q_{\tstar, \bot}}  \\ 
 &\le  12 \alpha \left(\frac{6\kappa^2  \sqrt{\max \left\{  n_1+n_2; k \right\} }}{ c \varepsilon \left( \sqrt{k} - \sqrt{r -1} \right)}\right)^{32 \kappa}  \sqrt{\max \left\{  n_1 +n_2 ; k \right\}}  \norm{\Ztstar Q_{\tstar, \bot}}\\
 &\le \frac{ c \sqrt{\norm{X}}}{ \kappa^3 } \norm{\Ztstar Q_{\tstar, \bot} },\\
 \norm{\tilde{Z}_{\tstar}^T P_{ Z_{\tstar} Q_{\tstar} }  }
 &\le \frac{228 c}{\kappa^2}  \left( \frac{2\kappa^2 \norm{Z_0}}{ c \sglmin \left( L_{F_1}^T Z_0 \right) } \right)^{32 \kappa} \norm{Z_0} \\ 
 &\le \frac{674 \alpha c}{\kappa^2}  \left(  \frac{6\kappa^2  \sqrt{\max \left\{  n_1+n_2;k \right\} }}{ c \varepsilon \left( \sqrt{k} - \sqrt{r -1} \right)}\right)^{32 \kappa}  \sqrt{\max \left\{  n_1 +n_2; k \right\}} \\
 &\le \frac{ c \sqrt{\norm{X}} }{\kappa^3}.
 \end{align*}
This shows the inequalities \eqref{spectral:final1}-\eqref{spectral:finalbalanc3}.
To finish the proof we note that
 \begin{align*}
    t_{\star}=
     \left\lceil \frac{\ln\left(\frac{c\sglmin \left(L_{F_1}^T Z_0\right)}{2\kappa^{2} \norm{Z_0} }\right)}{\ln \left(1 - \frac{\mu\sglmin(X)}{8} \right)} \right\rceil
     &\le \frac{\ln\left(\frac{c\sglmin \left(L_{F_1}^T Z_0\right)}{2\kappa^{2} \norm{Z_0} }\right)}{\ln \left(1 - \frac{\mu\sglmin(X)}{8} \right)} +1 \\ 
     &\le \frac{16 \ln\left(\frac{2\kappa^{2} \norm{Z_0} }{c\sglmin \left(L_{F_1}^T Z_0\right)}\right)}{ \mu \sglmin \left(X\right) } +1 \\ 
     &\le \frac{17 \ln\left(\frac{6\kappa^{2} \sqrt{ \max \left\{ n_1+n_2; k \right\} } }{c\varepsilon \left( \sqrt{k} - \sqrt{r-1} \right) }\right)}{ \mu \sglmin \left(X\right) }.
 \end{align*}
 This shows inequality \eqref{tstarbound}.
 Thus, the proof is complete.
\end{proof}

%% file: lemma_sigmingrowth.tex
\subsection{Proof of Lemma \ref{lemma:sigmingrowth}: Controlling $\sglmin(\LAT\Zt)$}\label{sec:sigmingrowth}

\begin{proof}[Proof of Lemma \ref{lemma:sigmingrowth}]
First, we note that it follows from $ \symA = \LA \Sigma_X \LAT - \tilLA \Sigma_X \tilLAT $ that $\LAT \symA = \Sigma_X \LAT $.
We note that
\begin{align}\label{lemma1eq}
    \LAT\Zplus\Qt = & \LAT \left( \I - \mu \left(\Zt\ZtT - \tilZt\tilZtT - \symA \right)  -\mu \Deltat \right) \Zt\Qt \nonumber\\
    = &  \left(\I + \mu\Sigma_X \right) \LAT\Zt\Qt - \mu\LAT \left(\Zt\ZtT - \tilZt\tilZtT \right) \Zt\Qt  -\LAT \Deltat \Zt \Qt \nonumber\\
    = & \left(\I + \mu\Sigma_X \right) \LAT\Zt\Qt - \mu\LAT \left(\Zt\ZtT - \tilZt\tilZtT \right)\left(\LA\LAT + \LAP\LAPT\right)\Zt\Qt -\LAT \Deltat \Zt \Qt \nonumber\\
    = & \left(\I + \mu\Sigma_X \right) \LAT\Zt\Qt - \mu\LAT\Zt\ZtT\LA\LAT \Zt\Qt + \mu\LAT\tilZt\tilZtT\LA\LAT\Zt\Qt  \nonumber\\
    & - \mu\LAT \left(\Zt\ZtT - \tilZt\tilZtT \right) \LAP\LAPT\Zt\Qt -\LAT \Deltat \Zt \Qt \nonumber\\
    = & \left(\I + \mu\Sigma_X + \mu\LAT\tilZt\tilZtT\LA \right)\LAT\Zt\Qt \left(\I - \mu\QtT\ZtT\LA\LAT\Zt\Qt\right) \nonumber\\
    & + \mu^2\underbrace{ \left(\Sigma_X + \LAT\tilZt\tilZtT\LA \right)\LAT\Zt\ZtT\LA\LAT\Zt\Qt}_{=:U_1} \nonumber\\
    & - \mu\underbrace{\LAT \left(\Zt\ZtT - \tilZt\tilZtT \right)\LAP\LAPT\Zt\Qt}_{=:U_2} \nonumber\\
    &- \mu \underbrace{\LAT \Deltat \Zt \Qt}_{=: U_3}.
\end{align}

To further simplify \eqref{lemma1eq}, we want to transform $U_i$ into the form $D_i\LAT\Zt\Qt(\I - \mu\QtT\ZtT\LA\LAT\Zt\Qt)$ for $i=1,2,3$, where $D_i$ are matrices to be determined.
Note that using the assumption that $ \LAT \Zt \Qt$ is invertible we can compute that
\begin{align}\label{lemma1LAPTZtQteq}
    \Zt\Qt = &\Zt\Qt(\LAT\Zt\Qt)^{-1}\LAT\Zt\Qt \nonumber\\
    = & \P{\Zt\Qt}\P{\Zt\Qt}^T\Zt\Qt \left(\LAT\P{\Zt\Qt}\P{\Zt\Qt}^T\Zt\Qt\right)^{-1}\LAT\Zt\Qt \nonumber\\
    = & \P{\Zt\Qt}\P{\Zt\Qt}^T\Zt\Qt\left(\P{\Zt\Qt}^T\Zt\Qt\right)^{-1}\left(\LAT\P{\Zt\Qt}\right)^{-1}\LAT\Zt\Qt \nonumber\\
    = & \P{\Zt\Qt}\left(\LAT\P{\Zt\Qt}\right)^{-1}\LAT\Zt\Qt.
\end{align}
Also note that for any matrix $K$ we have the identity
\begin{equation*}
    \left(\I - \mu KK^T \right)  K 
    = K\left(\I - \mu K^T K\right).
\end{equation*}
In particular, when $\I - \mu KK^T$ is invertible, it follows that
\begin{equation}\label{lemma1Xeq}
    K = \left(\I - \mu KK^T \right)^{-1}K \left(\I - \mu K^TK \right).
\end{equation}
Moreover, using the assumptions $\mu\le\frac{c}{\kappa\norm{X}}$ and $\norm{\Zt}\le 2\sqrt{\norm{X}}$, we have that
\begin{equation}\label{lemma1muineq}
    \mu\le \frac{c}{\norm{X}\kappa} \le \frac{1}{8\norm{X}} \le \frac{1}{2\norm{\Zt}^2} \le \frac{1}{2\norm{\LAT\Zt}^2}.
\end{equation}
In particular, this implies that $\I - \mu  \LAT\Zt\Qt \QtT \ZtT \LA  $ is invertible. 
Hence, we can set $K := \LAT\Zt\Qt$ and equation \eqref{lemma1Xeq} yields 
\begin{equation}\label{lemma2Xeq}
    \LAT\Zt\Qt = \left(\I - \mu \LAT\Zt\Qt \QtT \ZtT \LA  \right)^{-1} \LAT\Zt\Qt \left(\I - \mu  \QtT \ZtT \LAT \LAT\Zt\Qt  \right).
\end{equation}
Using \eqref{lemma1LAPTZtQteq} and \eqref{lemma2Xeq} we can rewrite $U_1,U_2,U_3$ in the following way:
\begin{align}\label{U1eq}
    U_1 = & \left(\Sigma_X + \LAT\tilZt\tilZtT\LA \right)\LAT\Zt\ZtT\LA\LAT\Zt\Qt \nonumber \\
    = &\underbrace{\left(\Sigma_X + \LAT\tilZt\tilZtT\LA\right)\LAT\Zt\ZtT\LA\left(\I - \mu\LAT\Zt\Qt\QtT\ZtT\LA\right)^{-1}}_{=:D_1}\LAT\Zt\Qt \left(\I - \mu\QtT\ZtT\LA\LAT\Zt\Qt \right),
\end{align}
\begin{align}\label{U2eq}
    U_2 
   = & \LAT\left(\Zt\ZtT - \tilZt\tilZtT\right)\LAP\LAPT \Zt\Qt \nonumber\\
    = &\LAT\left(\Zt\ZtT - \tilZt\tilZtT\right)\LAP \LAPT\P{\Zt\Qt}\left(\LAT\P{\Zt\Qt}\right)^{-1}\LAT\Zt\Qt \nonumber\\
    = &\underbrace{\LAT\left(\Zt\ZtT - \tilZt\tilZtT\right)\LAP \LAPT\P{\Zt\Qt}\left(\LAT\P{\Zt\Qt}\right)^{-1}(\I - \mu\LAT  \Zt\Qt\QtT\ZtT\LA)^{-1}}_{=:D_2} \nonumber\\
    & \cdot\LAT\Zt\Qt\left(\I - \mu\QtT\ZtT\LA\LAT\Zt\Qt\right),
\end{align}
\begin{align}\label{U3eq}
    U_3 &= \LAT \Deltat \Zt \Qt \nonumber \\
    &= \LAT \Deltat \P{\Zt\Qt}(\LAT\P{\Zt\Qt})^{-1}\LAT\Zt\Qt \nonumber \\
    &=  \underbrace{ \LAT \Deltat \P{\Zt\Qt}(\LAT\P{\Zt\Qt})^{-1}  (\I - \mu \LAT\Zt\Qt \QtT \ZtT \LA  )^{-1} }_{=:D_3} \LAT\Zt\Qt(\I - \mu  \QtT \ZtT \LAT \LAT\Zt\Qt  ). 
\end{align}
Combining \eqref{lemma1eq}, \eqref{U1eq}, \eqref{U2eq}, and \eqref{U3eq} we conclude that
\[\LAT\Zplus\Qt = \left(\I + \mu\Sigma_X + \mu^2 D_1 - \mu D_2 - \mu D_3\right)\LAT\Zt\Qt \left(\I - \mu\QtT\ZtT\LA\LAT\Zt\Qt \right).\]
It follows that
\begin{align}\label{mahdieq}
    &\sglmin \left(\LAT\Zplus\Qt \right) \nonumber\\
    \ge &\sglmin \left(\I + \mu\Sigma_X + \mu^2 D_1 - \mu D_2 -\mu D_3 \right) \sglmin \left(\LAT\Zt\Qt \left(\I - \mu\QtT\ZtT\LA\LAT\Zt\Qt \right) \right) \nonumber\\
    \stackrel{(a)}{=} &\sglmin\left(\I + \mu\Sigma_X + \mu^2 D_1 - \mu D_2 - \mu D_3   \right) \sglmin \left(\LAT\Zt\Qt \right) \left(1 - \mu\sglmin^2 \left(\LAT\Zt\Qt \right) \right) \nonumber\\
    = &\sglmin\left(\I + \mu\Sigma_X + \mu^2 D_1 - \mu D_2 - \mu D_3   \right) \sglmin\left(\LAT\Zt\right) \left(1 - \mu\sglmin^2(\LAT\Zt) \right) \nonumber\\
    \stackrel{(b)}{\ge} &\left( \sglmin \left( \I + \mu\Sigma_X \right) - \mu^2\norm{D_1} - \mu\norm{D_2} - \mu \norm{D_3} \right) \sglmin\left(\LAT\Zt\right) \left(1 - \mu\sglmin^2\left(\LAT\Zt \right) \right) \nonumber\\
    \stackrel{(c)}{=} &\left(1 + \mu\sglmin(X) - \mu^2\norm{D_1} - \mu\norm{D_2} - \mu \norm{D_3}  \right)\sglmin\left(\LAT\Zt\right) \left(1 - \mu\sglmin^2 \left(\LAT\Zt\right) \right).
\end{align}
Equation (a) can be derived from the singular value decomposition of $\LAT\Zt\Qt$ and \eqref{lemma1muineq}. In inequality (b) we use Weyl's inequality. Equation (c) holds due to the fact that $\Sigma_X$ is a positive definite matrix.
In order to prove the desired inequality \eqref{lemma1result}, we need to bound $\norm{D_1}$, $\norm{D_2}$, and $ \norm{D_3} $.
To this aim, first note that using \eqref{lemma1muineq} and the assumption $\norm{\LAPT\P{\Zt\Qt}}\le \frac{c}{\kappa}\le 1/2$ we have that
\begin{align}\label{ineq:intern1}
    \norm{(\I - \mu\LAT\Zt\Qt\QtT\ZtT\LA)^{-1}} = &\frac{1}{\sglmin(\I - \mu\LAT\Zt\Qt\QtT\ZtT\LA)}=\frac{1}{1 - \mu\norm{\LAT\Zt\Qt\QtT\ZtT\LA}} \le 2, 
\end{align}
and
\begin{align}\label{ineq:intern2}
    \norm{(\LAT\P{\Zt\Qt})^{-1}}= &\frac{1}{\sglmin(\LAT\P{\Zt\Qt})} = \frac{1}{\sqrt{1 - \norm{\LAPT\P{\Zt\Qt}}^2}} \le 2.
\end{align}
Combining the latter two inequalities with the assumptions $\norm{\LAPT\P{\Zt\Qt}}\le c\kappa^{-1}$ and $\norm{\Zt} \le 2\sqrt{\norm{X}}$ and the fact that $\norm{\Zt} = \norm{\tilZt} $, we have
\begin{align*}
    \norm{D_1} 
    \le &\norm{\Sigma_X + \LAT\tilZt\tilZtT\LA} \norm{\LAT\Zt}\norm{\ZtT\LA}\norm{\I-\mu\LAT  \Zt\Qt\QtT\ZtT\LA)^{-1}} \\
    \stackrel{\eqref{ineq:intern1}}{\le} &2(\norm{X} + \norm{\LAT\tilZt}^2)\norm{\LAT\Zt}^2 \\
    \le &2(\norm{X} + \norm{\tilZt}^2)\norm{\Zt}^2 \\
    \le &40\norm{X}^2,
\end{align*}
\begin{align*}
    \norm{D_2} \le & \norm{\LAT \left(\Zt\ZtT - \tilZt\tilZtT \right) \LAP }\norm{\LAPT\P{\Zt\Qt}}\norm{(\LAT\P{\Zt\Qt})^{-1}}\norm{\left(\I - \mu\LAT\Zt\Qt\QtT\ZtT\LA \right)^{-1}} \\
    \stackrel{\eqref{ineq:intern1},\eqref{ineq:intern2}}{\le} & 4c\kappa^{-1}\norm{\LAT \left( \Zt\ZtT - \tilZt\tilZtT \right) \LAP} \\
    \le & 4c\kappa^{-1}\norm{\Zt\ZtT - \tilZt\tilZtT} \\
    \le & 32c\sglmin(X),
\end{align*}
and
\begin{align*}
    \norm{D_3} \le & \norm{ \LAT \Deltat \P{\Zt\Qt}(\LAT\P{\Zt\Qt})^{-1}  \left(\I - \mu \LAT\Zt\Qt \QtT \ZtT \LA  \right)^{-1}}\\
    \le & \norm{ \Deltat  } \norm{ \left(\LAT\P{\Zt\Qt} \right)^{-1} } \norm{ \left(\I - \mu \LAT\Zt\Qt \QtT \ZtT \LA  \right)^{-1} } \\
    \stackrel{\eqref{ineq:intern1},\eqref{ineq:intern2}}{\le}
    & 4 \norm{ \Deltat  } \\
    \le & 4 c \sglmin \left(X \right),
\end{align*}
where in the last line we used the assumption $ \norm{ \Deltat  } \le c \sglmin (X)  $.
Plugging these three bounds into \eqref{mahdieq} we have the following chain of inequalities which complete the proof of the lemma.
\begin{align*}
    &\sglmin \left(\LAT\Zplus\Qt \right) \nonumber\\
    \ge & \left(1 + \mu\sglmin \left(X \right) - \mu^2\norm{D_1} - \mu\norm{D_2} - \mu \norm{D_3} \right)\sglmin \left(\LAT\Zt\right) \left(1 - \mu\sglmin^2 \left(\LAT\Zt \right) \right) \nonumber\\
    \ge &\left(1 + \mu\sglmin(X) - 40\mu^2\norm{X}^2- 36c\mu\sglmin(X) \right) \sglmin(\LAT\Zt) \left(1 - \mu\sglmin^2 \left(\LAT\Zt \right) \right) \\
    \stackrel{(a)}{\ge} & \left(1 + \frac{1}{2}\mu\sglmin \left(X \right) \right) \sglmin(\LAT\Zt) \left(1 - \mu\sglmin^2(\LAT\Zt)\right) \\
    = &\left(1 + \frac{1}{2}\mu \left(1 - \mu\sglmin^2 \left( \LAT\Zt \right)\right) \sglmin(X) - \mu\sglmin^2 \left(\LAT\Zt \right) \right)\sglmin \left( \LAT\Zt \right) \\
   \stackrel{(b)}{\ge} & \left(1 + \frac{1}{4}\mu\sglmin(X) - \mu\sglmin^2 \left(\LAT\Zt \right) \right) \sglmin \left(\LAT\Zt \right).
\end{align*}
Inequality (a) follows from the assumption $\mu \le \frac{c}{\norm{X}\kappa}$ and by choosing the absolute constant $c>0$ small enough. 
In inequality (b) we used inequality \eqref{lemma1muineq}.
\end{proof}

%% file: lemma_noisetermgrowth.tex
\subsection{Proof of Lemma \ref{lemma:noisetermgrowth}: Controlling $\norm{\Zt\Qtp}$}\label{sec:noisetermgrowth}

\begin{proof}[Proof of Lemma \ref{lemma:noisetermgrowth}]
We first derive a formula for $\QtT\Qplusp$.
By the definition of $\Qplusp$ we have $\LAT\Zplus\Qplusp = 0$.
Combined with $\Qt\QtT + \Qtp\QtpT = \I$, we obtain that
\begin{equation*}
    \LAT\Zplus\Qt\QtT\Qplusp = -\LAT\Zplus\Qtp\QtpT\Qplusp,
\end{equation*}
which implies
\begin{equation}\label{QtTQpluspeq}
\QtT\Qplusp = - \left(\LAT\Zplus\Qt \right)^{-1}\LAT\Zplus\Qtp\QtpT\Qplusp.
\end{equation}
Due to $\LAT\Zplus\Qplusp = 0$ we have
\begin{equation*}
\Zplus\Qplusp = \LA\LAT\Zplus\Qplusp + \LAP\LAPT\Zplus\Qplusp = \LAP\LAPT\Zplus\Qplusp,
\end{equation*}
which implies $\norm{\Zplus\Qplusp} = \norm{\LAPT\Zplus\Qplusp}$.
Using equation \eqref{QtTQpluspeq} we obtain that
\begin{align*}
    \LAPT\Zplus\Qplusp = & \LAPT\Zplus\Qt\QtT\Qplusp + \LAPT\Zplus\Qtp\QtpT\Qplusp \\
    = & -\LAPT\Zplus\Qt(\LAT\Zplus\Qt)^{-1}\LAT\Zplus\Qtp\QtpT\Qplusp + \LAPT\Zplus\Qtp\QtpT\Qplusp.
\end{align*}
Note that
\begin{align*}
    \LAPT\Zplus\Qt(\LAT\Zplus\Qt)^{-1}= & \LAPT\P{\Zplus\Qt}\P{\Zplus\Qt}^T\Zplus\Qt\left(\LAT\P{\Zplus\Qt}\P{\Zplus\Qt}^T\Zplus\Qt\right)^{-1} \\
    = & \LAPT\P{\Zplus\Qt}\P{\Zplus\Qt}^T\Zplus\Qt\left(\P{\Zplus\Qt}^T\Zplus\Qt\right)^{-1}\left(\LAT\P{\Zplus\Qt}\right)^{-1} \\
    = & \LAPT\P{\Zplus\Qt}\left(\LAT\P{\Zplus\Qt}\right)^{-1}.
\end{align*}
Thus, we obtain that 
\begin{equation*}
\LAPT\Zplus\Qplusp = -\LAPT\P{\Zplus\Qt}(\LAT\P{\Zplus\Qt})^{-1}\LAT\Zplus\Qtp\QtpT\Qplusp + \LAPT\Zplus\Qtp\QtpT\Qplusp.
\end{equation*}
By using the triangle inequality and the submultiplicativity of the spectral norm we obtain that
\begin{align}\label{LAPTZplusQplusple}
    \norm{\Zplus\Qplusp} = & \norm{\LAPT\Zplus\Qplusp} \nonumber\\
    \le & \norm{\LAPT\P{\Zplus\Qt}(\LAT\P{\Zplus\Qt})^{-1}\LAT\Zplus\Qtp\QtpT\Qplusp} + \norm{\LAPT\Zplus\Qtp\QtpT\Qplusp} \nonumber\\
    \le & \norm{\LAPT\P{\Zplus\Qt}} \norm{\left(\LAT\P{\Zplus\Qt}\right)^{-1}} \norm{\LAT\Zplus\Qtp} + \norm{\LAPT\Zplus\Qtp} \nonumber\\
    = & \frac{\norm{\LAPT\P{\Zplus\Qt}}}{\sqrt{1 - \norm{\LAPT\P{\Zplus\Qt}}^2}} \norm{\LAT\Zplus\Qtp} + \norm{\LAPT\Zplus\Qtp}.
\end{align}
In order to proceed bounding $\norm{\Zplus\Qplusp}$ from above, we first derive an upper bound for $\norm{\LAPT\P{\Zplus\Qt}}$. Note that
\begin{equation*}
\Zplus\Qt = \left(\I - \mu\left(\Zt\ZtT - \tilZt\tilZtT - \symA + \Deltat\right) \right)\Zt\Qt
\end{equation*}
For simplicity of notation, define 
\begin{equation*}
M_t := \Zt\ZtT - \tilZt\tilZtT - \symA +\Deltat,
\end{equation*}
\begin{equation*}
\hat{H}_t := \left(\I - \mu M_t \right)\P{\Zt\Qt}.
\end{equation*}
First, we note that 
    \begin{align}
     \norm{M_t}
    \le & \norm{X}+ \norm{\Zt \ZtT} + \norm{\tilZt \tilZtT} + \norm{\Deltat} \nonumber \\
    = & \norm{X}+2 \norm{\Zt}^2 + \norm{\Deltat}\nonumber  \\
    \le & 10 \norm{X}, \label{ineq:boundforM}
\end{align}
where in the last line we used the assumptions $ \norm{\Zt} \le 2 \sqrt{\norm{X}} $, $\norm{\Deltat} \le c \sglmin \left( X\right) $, and that $\norm{\Zt}=\norm{\tilZt}$ by symmetry.
Since $\Sigma_{\Zt\Qt}\Q{\Zt\Qt}^T$ has full rank by assumption, the matrix $H$ has the same column space as $\Zplus\Qt$. 
In particular, we have that
\begin{align*}
    \norm{\LAPT\P{\Zplus\Qt}} = &\norm{\LAPT\P{ \hat{H}_t }} \\
    \le &\norm{\LAPT \hat{H}_t }\norm{(\Sigma_{ \hat{H}_t } \Q{\hat{H}_t }^T)^{-1}} \\
    = &\frac{\norm{\LAPT \hat{H}_t }}{\sglmin \left(\hat{H}_t \right)}.
\end{align*}
Next, we observe that 
\begin{align*}
    \sglmin \left(\hat{H}_t \right) \ge & \sglmin \left( \Id - \mu M_t \right)\sglmin(\P{\Zt\Qt}) \\
    \ge & \left(1 - \mu \norm{M_t}\right)\sglmin(\P{\Zt\Qt}) \ge \frac{1}{2}.
\end{align*}
In the last inequality we used the assumption $\mu \le \frac{c}{\norm{X}\kappa}$ and inequality \eqref{ineq:boundforM}. 
Combining the above two inequalities we obtain that
\begin{align}
    \norm{\LAPT\P{\Zplus\Qt}} \le &2\norm{\LAPT \hat{H}_t } \nonumber\\
    = &2\norm{\LAPT(\I - \mu M_t )\P{\Zt\Qt}} \nonumber\\
    \le &2\norm{\LAPT\P{\Zt\Qt}} + 2\mu\norm{ M_t } \nonumber\\
    \le &2\norm{\LAPT\P{\Zt\Qt}} + 20\mu\norm{X} \label{ineq:aux1212}  \\
    \le &\frac{24c \varepsilon}{\kappa}.\label{LAPTPZplusQtle}
\end{align}
In the last two inequalities we used the assumptions $\mu \le \frac{c \varepsilon }{\norm{X}\kappa}$, $\norm{\LAPT\P{\Zt\Qt}} \le \frac{c \varepsilon }{\kappa}$ and inequality \eqref{ineq:boundforM}. In order to bound the term $\norm{\LAT\Zplus\Qtp}$, which appears in inequality \eqref{LAPTZplusQplusple}, we note that due to $\LAT\Zt\Qtp = 0$ we have that 
\begin{align}\label{LATZplusQtple}
    \norm{\LAT\Zplus\Qtp} = & \norm{\LAT \left( \I - \mu M_t \right)\Zt\Qtp} \nonumber\\
    = & \mu\norm{\LAT M_t\Zt\Qtp} \nonumber\\
    \le & \mu\norm{M_t}\norm{\Zt\Qtp} \nonumber\\
    \le & 10\mu\norm{X}\norm{\Zt\Qtp}.
\end{align}
In the last line we used inequality \eqref{ineq:boundforM}.
It remains to bound $\norm{\LAPT\Zplus\Qtp}$ in inequality \eqref{LAPTZplusQplusple}.
For simplicity of notation we denote $\hatK := \LAPT\Zt\Qtp$. First, we compute that
\begin{align}\label{LAPTZplusQtpeq}
    &\LAPT\Zplus\Qtp \nonumber\\
    = & \LAPT(\Zt - \mu(\Zt\ZtT - \tilZt\tilZtT - \symA + \Deltat)\Zt)\Qtp \nonumber\\
    = & \hatK - \mu\LAPT\Zt\ZtT\LAP \hatK + \mu\LAPT\tilZt\tilZtT\Zt\Qtp - \mu\LAPT\tilLA\Sigma_X\tilLAT\LAP \hatK - \mu \LAPT \Deltat \Zt \Qtp  \nonumber\\
    = & \hatK - \mu\LAPT\Zt(\Qt\QtT + \Qtp\QtpT)\ZtT\LAP \hatK + \mu\LAPT\tilZt\tilZtT\Zt\Qtp - \mu\LAPT\tilLA\Sigma_X\tilLAT\LAP \hatK \nonumber \\
     &- \mu \LAPT \Deltat \Zt \Qtp \nonumber\\
     = & \hatK - \mu\LAPT\Zt \Qt\QtT\ZtT\LAP \hatK - \mu \hatK \hatK^T \hatK + \mu\LAPT\tilZt\tilZtT\Zt\Qtp - \mu\LAPT\tilLA\Sigma_X\tilLAT\LAP \hatK \nonumber \\
     &- \mu \LAPT \Deltat \LAP \hatK  \nonumber\\
    = & \left(\I - \mu\LAPT\tilLA\Sigma_X\tilLAT\LAP - \mu\LAPT\Zt\Qt\QtT\ZtT\LAP - \mu \LAPT \Deltat \LAP \right) \hatK \left(\I - \mu \hatK^T \hatK \right) \nonumber\\
    & - \mu^2 \left( \LAPT\tilLA\Sigma_X\tilLAT\LAP + \LAPT\Zt\Qt\QtT\ZtT\LAP +  \LAPT \Deltat \LAP \right)  \hatK \hatK^T \hatK    \\
    &+ \mu\LAPT\tilZt\tilZtT\Zt\Qtp. \nonumber  
\end{align}
To proceed, we note that 
\begin{align*}
    & \norm{ \left( \I - \mu\LAPT\tilLA\Sigma_X\tilLAT\LAP - \mu\LAPT\Zt\Qt\QtT\ZtT\LAP - \mu \LAPT \Deltat \LAP  \right) \hatK \left(\I - \mu \hatK^T \hatK \right)} \\
    \le & \norm{\I - \mu\LAPT\tilLA\Sigma_X\tilLAT\LAP - \mu\LAPT\Zt\Qt\QtT\ZtT\LAP-  \mu \LAPT \Deltat \LAP }\norm{ \hatK \left(\I - \mu \hatK^T \hatK \right)} \\
    \le & \left(  \norm{ \I - \mu\LAPT\tilLA\Sigma_X\tilLAT\LAP - \mu\LAPT\Zt\Qt\QtT\ZtT\LAP } + \mu \norm{  \LAPT \Deltat \LAP }  \right) \norm{\hatK \left(\I - \mu \hatK^T\hatK \right)} \\
    \stackrel{(a)}{\le} & \left( 1+ \mu c \varepsilon \sglmin (X)  \right)  \norm{\hatK \left(\I - \mu \hatK^T\hatK \right)}  \\
    \stackrel{(b)}{=} &  \left( 1+ \mu c \varepsilon \sglmin (X)  \right) \left(  1 - \mu\norm{\hatK}^2 \right)\norm{\hatK} \\
    \le &  \left( 1+ \mu c \varepsilon \sglmin (X) - \mu\norm{\hatK}^2 \right)\norm{\hatK},
\end{align*}
where in inequality (a) we used the assumptions $ \norm{\Deltat} \le c \varepsilon \sglmin \left( X \right)  $, $ \mu \le \frac{c \varepsilon}{\kappa} $, and $ \norm{\Zt} \le 2 \sqrt[]{\norm{X}} $.
Equation (b) can be derived by the singular value decomposition of $\hat{K}$, the fact $\mu \le \frac{c \varepsilon }{\norm{X}\kappa} \le \frac{1}{\norm{\Zt}^2} \le \frac{1}{3\norm{\hatK}^2}$, and that the function $x \left(1- \mu x^2 \right)$ is monotonically increasing for $x\in [0, \frac{1}{\sqrt{3\mu}}]$. 
Furthermore, we note that
\begin{align*}
    & \mu^2\norm{(\LAPT\tilLA\Sigma_X\tilLAT\LAP + \LAPT\Zt\Qt\QtT\ZtT\LAP + \LAPT \Deltat \LAP  )\hatK \hatK^T \hatK} \\
    \le & \mu^2\norm{\LAPT\tilLA\Sigma_X\tilLAT\LAP + \LAPT\Zt\Qt\QtT\ZtT\LAP+\LAPT \Deltat \LAP}\norm{\hatK \hatK^T \hatK} \\
    \le & \mu^2 \left( \norm{ \LAPT\tilLA\Sigma_X\tilLAT\LAP } + \norm{ \LAPT\Zt\Qt\QtT\ZtT\LAP } + \norm{ \LAPT \Deltat \LAP }  \right)  \norm{\hatK}^3 \\ 
    \le & \mu^2 \left( \norm{X}+ \norm{\Zt}^2 + \norm{\Deltat}\right) \norm{\hatK}^3 \\
    \stackrel{(c)}{\le} & 6\mu^2\norm{X}\norm{\hatK}^3 \\
     \stackrel{(d)}{\le} & \frac{\mu}{2}\norm{\hatK}^3.
\end{align*}
In inequality (c) we use the assumptions $\norm{\Zt} \le 2\sqrt{\norm{X}}$ and $ \norm{\Deltat} \le c \varepsilon \sglmin (X) $. 
Inequality (d) follows from the assumption $\mu \le \frac{c \varepsilon}{\norm{X}\kappa}$.
Moreover, we observe that
\begin{align*}
    \norm{\LAPT\tilZt\tilZtT\Zt\Qtp} \le \norm{\LAPT\tilZt}\norm{\tilZtT\Zt \Qtp}  \le 2\sqrt{\norm{X}}\norm{\tilZtT\Zt \Qtp},
\end{align*}
where in the last inequality we used the assumption $ \norm{\Zt} \le 2 \sqrt{\norm{X} } $ as well as the fact that by symmetry $ \norm{\Zt} = \norm{\tilZt}  $.
Putting the above three inequalities into \eqref{LAPTZplusQtpeq} we obtain that
\begin{align}\label{LAPTZplusQtple}
    \norm{\LAPT\Zplus\Qtp} \le &  \left( 1+ \mu c \varepsilon \sglmin (X) - \mu\norm{\hatK}^2 \right)\norm{\hatK} + \frac{\mu}{2} \norm{\hatK}^3 +  2 \mu \sqrt{\norm{X}}\norm{\tilZtT\Zt \Qtp} \nonumber  \\
    = & \left( 1 - \frac{\mu}{2}\norm{\Zt\Qtp}^2 + \mu c \varepsilon \sglmin (X) \right) \norm{\Zt\Qtp} + 2\mu\sqrt{\norm{X}}\norm{ \tilZtT\Zt \Qtp }.
\end{align}
Putting inequalities \eqref{LAPTPZplusQtle}, \eqref{LATZplusQtple} and \eqref{LAPTZplusQtple} into \eqref{LAPTZplusQplusple} we obtain that
\begin{align*}
    \norm{\Zplus\Qplusp} \le & \frac{\norm{\LAPT\P{\Zplus\Qt}}}{\sqrt{1 - \norm{\LAPT\P{\Zplus\Qt}}^2}}\norm{\LAT\Zplus\Qtp} + \norm{\LAPT\Zplus\Qtp} \\
    \le & \frac{48c\varepsilon}{\kappa}\cdot 10\mu\norm{X}\norm{\Zt\Qtp}+ \left( 1 - \frac{\mu}{2}\norm{\Zt\Qtp}^2 + \mu c \varepsilon \sglmin (X) \right) \norm{\Zt\Qtp} + 2\mu\sqrt{\norm{X}}\norm{ \tilZtT\Zt \Qtp  } \\
    \le & \left(1 - \frac{\mu}{2}\norm{\Zt\Qtp}^2 +\mu \varepsilon\sglmin(X) \right)\norm{\Zt\Qtp} + 2\mu\sqrt{\norm{X}}\norm{ \tilZtT\Zt \Qtp },
\end{align*}
where the last line follows by choosing the absolute constant $c>0$ small enough.
This finishes the proof.
\end{proof}

%% file: lemma_anglecontrol.tex
\subsection{Proof of Lemma \ref{lemma:anglecontrol}: Controlling $\norm{ \LAPT \P{\Zt\Qt} }$}\label{sec:anglecontrol}
We define the inverse of the square root of a symmetric positive definite matrix $D = \P{D}\Sigma_{D}\P{D}^T$ by $D^{-1/2} = \P{D}\Sigma_{D}^{-1/2}\P{D}^T$, where $\left(\Sigma_{D}^{-1/2}\right)_{ii} = \frac{1}{\sqrt{D_{ii}}}$. 
In the following, we will need the following technical lemma, which gives a bound on the first-order Taylor approximation of the matrix inverse square root.
\begin{lemma}\label{MatrixTaylor}
Let $D$ be a symmetric matrix such that $\norm{D} \leq 1/2$. Then it holds that
\[\norm{(\I + D)^{-1/2} - \I + \frac{1}{2}D} \leq 3\norm{D}^{2}.\]
\end{lemma}
For the straightforward proof of Lemma \ref{MatrixTaylor} we refer to  \cite[Lemma B.2]{stoger2021small}. 
To prove Lemma \ref{lemma:anglecontrol} we will need the following technical lemma.

\begin{lemma}\label{lemma:aux1}
Assume that $\mu \le \frac{c}{\norm{X}\kappa}$, $\norm{\LAPT\P{\Zt\Qt}} \le c\kappa^{-1}$, $\norm{\Zt\Qtp} \le 2\sglmin(\Zt\Qt)$, $\norm{\Zt}\le 2\sqrt{\norm{X}}$, and $\norm{\Deltat} \le c \sglmin (X) $.
Then it holds that
\begin{equation}\label{ineq:aux1}
\norm{\QtpT\Qplus} \le \mu\left(\norm{\Zt\Qt}\norm{\Zt\Qtp}+ 40 \mu \norm{X} \norm{\Zt\Qt}^2 \right)\norm{\LAPT\P{\Zt\Qt}} + 4\mu\frac{\sqrt{\norm{X}}\norm{\tilZtT\Zt \Qtp}}{\sglmin\left(\Zt\Qt\right)} + 4\mu \norm{\Deltat}.
\end{equation}
Moreover, we have that
\begin{equation*}
\sglmin(\QtT\Qplus) \ge \frac{1}{2}.
\end{equation*}
Here, $c>0$ is an absolute constant chosen small enough.
\end{lemma}
\begin{proof}
We recall that
\begin{equation*}
    \Zplus = \left( \I -\mu M_t \right) \Zt,
\end{equation*}
where
\begin{equation*}
    M_t = \Zt\ZtT - \tilZt\tilZtT - \symA +\Deltat.
\end{equation*}
As shown in inequality \eqref{ineq:boundforM} we have
\begin{equation}\label{ineq:boundforM1}
    \norm{M_t} \le 10 \norm{X}.
\end{equation}
We observe that
\begin{align*}
    \QtpT\Qplus = &\QtpT\ZplusT\LA(\LAT\Zplus\ZplusT\LA)^{-1/2}.
\end{align*}
Due to $\LAT\Zt\Qtp = 0$ we have
\begin{align*}
    \LAT\Zplus\Qtp = &\LAT\left(\I - \mu\left(\Zt\ZtT - \tilZt\tilZtT - \symA + \Deltat\right)\right)\Zt\Qtp \nonumber\\
    = &-\mu\LAT \left(\Zt\ZtT - \tilZt\tilZtT - \symA + \Deltat \right)\Zt\Qtp \nonumber\\
    = &-\mu\LAT\left(\Zt\ZtT - \tilZt\tilZtT\right)\Zt\Qtp + \mu\LAT \left(\LA\Sigma_X\LAT - \tilLA \Sigma_X\tilLAT \right)\Zt\Qtp -\mu \LAT \Deltat \Zt \Qtp \nonumber\\
    = &-\mu\LAT \left(\Zt\ZtT - \tilZt\tilZtT \right) \Zt\Qtp  -  \mu \LAT \Deltat \Zt \Qtp \nonumber\\
    = &-\mu\LAT\Zt\Qt\QtT\ZtT\LAP\LAPT\Zt\Qtp + \mu\LAT\tilZt\tilZtT\Zt\Qtp -\mu \LAT\Deltat \Zt \Qtp.
\end{align*}
Combining the above two equalities we obtain that
\begin{align} \label{ineq:LemmaC2le}
    &\norm{\QtpT\Qplus} \nonumber\\
    = &\norm{(\LAT\Zplus\ZplusT\LA)^{-1/2}(-\mu\LAT\Zt\Qt\QtT\ZtT\LAP\LAPT\Zt\Qtp + \mu\LAT\tilZt\tilZtT\Zt\Qtp -\mu \LAT\Deltat \Zt \Qtp )} \nonumber\\
    \le &\mu\norm{(\LAT\Zplus\ZplusT\LA)^{-1/2}\LAT\Zt\Qt\QtT\ZtT\LAP\LAPT\Zt\Qtp} \nonumber\\ 
    & + \mu\norm{(\LAT\Zplus\ZplusT\LA)^{-1/2}\LAT\tilZt\tilZtT\Zt\Qtp} + \mu \norm{ (\LAT\Zplus\ZplusT\LA)^{-1/2} \LAT\Deltat \Zt \Qtp    } \nonumber\\
    \le &\mu\norm{(\LAT\Zplus\ZplusT\LA)^{-1/2}\LAT\Zt\Qt}\norm{\QtT\ZtT\LAP}\norm{\LAPT\Zt\Qtp} \nonumber\\ 
    & + \mu\norm{(\LAT\Zplus\ZplusT\LA)^{-1/2}}\norm{\tilZt} \norm{ \tilZtT \Zt \Qtp}  + \mu \norm{ (\LAT\Zplus\ZplusT\LA)^{-1/2}} \norm{\Deltat} \norm{ \Zt \Qtp } \nonumber\\
    \le &\mu\norm{(\LAT\Zplus\ZplusT\LA)^{-1/2}\LAT\Zt\Qt}\norm{\LAPT\P{\Zt\Qt}}\norm{\Zt\Qt}\norm{\Zt\Qtp} \nonumber\\ 
    & + \mu\frac{ \norm{\tilZt} \norm{ \tilZtT \Zt \Qtp}  }{\sglmin(\LAT\Zplus)} + \mu \frac{ \norm{\Deltat} \norm{ \Zt \Qtp }}{\sglmin(\LAT\Zplus)}.
\end{align}
To deal with the first summand in the penultimate line in the inequality chain above, we note that
\begin{align} \label{ineq:LemmaC2firstsummand}
    & \norm{\left(\LAT\Zplus\ZplusT\LA \right)^{-1/2}\LAT\Zt\Qt} \nonumber\\
    = &\norm{ \left(\LAT\Zplus\ZplusT\LA \right)^{-1/2}\LAT \left(\Zplus+\mu M_t \Zt \right)\Qt} \nonumber\\
    \le &\norm{ \left(\LAT\Zplus\ZplusT\LA \right)^{-1/2}\LAT\Zplus\Qt} + \mu\norm{ \left(\LAT\Zplus\ZplusT\LA \right)^{-1/2}\LAT M_t \Zt\Qt} \nonumber\\
    \le &\norm{\left( \LAT\Zplus\ZplusT\LA \right)^{-1/2}\LAT\Zplus} + \mu\norm{ \left(\LAT\Zplus\ZplusT\LA \right)^{-1/2}}\norm{ M_t }\norm{\Zt\Qt} \nonumber\\
    = & 1 + \mu\frac{\norm{ M_t }\norm{\Zt\Qt}}{\sglmin(\LAT\Zplus)}.
\end{align}
To deal with $\sglmin\left(\LAT\Zplus\right)$ appearing in the denominator, we compute that
\begin{align} \label{ineq:LemmaC2sglminLATZplus}
    \sglmin\left(\LAT\Zplus\right) \ge &\sglmin\left(\LAT\Zplus\Qt\right) \nonumber\\
    = &\sglmin\left(\LAT\left( \I -\mu M_t \right) \Zt\Qt\right) \nonumber\\
    = &\sglmin\left(\LAT\left( \I -\mu M_t \right) \P{\Zt\Qt}\P{\Zt\Qt}^T \Zt\Qt\right) \nonumber\\
    \ge &\sglmin\left(\LAT\left( \I -\mu M_t \right) \P{\Zt\Qt}\right)\sglmin\left(\P{\Zt\Qt}^T \Zt\Qt\right) \nonumber\\
    = &\sglmin\left(\LAT\P{\Zt\Qt} - \mu \LAT M_t \P{\Zt\Qt}\right)\sglmin\left(\Zt\Qt\right) \nonumber\\
    \ge &\left(\sglmin\left(\LAT\P{\Zt\Qt}\right) - \mu \norm{\LAT M_t\P{\Zt\Qt}}\right)\sglmin\left(\Zt\Qt\right) \nonumber\\
    \stackrel{(a)}{\ge} &\left(\sqrt{1 - \norm{\LAPT\P{\Zt\Qt}}^2} - 10\mu\norm{X}\right)\sglmin\left(\Zt\Qt\right) \nonumber\\
    \stackrel{(b)}{\ge} &\frac{1}{2}\sglmin\left(\Zt\Qt\right).
\end{align}
In inequality (a) we used the inequality \eqref{ineq:boundforM1}.
Inequality (b) follows from the assumptions $\mu \le \frac{c}{\norm{X}\kappa}$ and $\norm{\LAPT\P{\Zt\Qt}} \le c\kappa^{-1}$.

Inserting \eqref{ineq:LemmaC2firstsummand} and \eqref{ineq:LemmaC2sglminLATZplus} into \eqref{ineq:LemmaC2le} we obtain that
\begin{align}
    &\norm{\QtpT\Qplus} \nonumber \\
    \le &\mu\left(1 + \mu\frac{\norm{M_t}\norm{\Zt\Qt}}{\sglmin\left(\LAT\Zplus\right)}\right)\norm{\LAPT\P{\Zt\Qt}}\norm{\Zt\Qt}\norm{\Zt\Qtp} + \mu\frac{ \norm{\tilZt} \norm{ \tilZtT \Zt \Qtp}  + \norm{\Deltat} \norm{ \Zt \Qtp } }{\sglmin\left(\LAT\Zplus\right)} \nonumber \\
    \le &\mu\left(1 + 2\mu\frac{\norm{M_t}\norm{\Zt\Qt}}{\sglmin\left(\Zt\Qt\right)}\right)\norm{\LAPT\P{\Zt\Qt}}\norm{\Zt\Qt}\norm{\Zt\Qtp} + 2\mu\frac{ \norm{\tilZt} \norm{ \tilZtT \Zt \Qtp}   + \norm{\Deltat} \norm{ \Zt \Qtp } }{\sglmin\left(\Zt\Qt\right)}\nonumber \\
    \stackrel{(a)}{\le} &\mu\norm{\LAPT\P{\Zt\Qt}}(\norm{\Zt\Qt}\norm{\Zt\Qtp} + 4\mu\norm{M_t}\norm{\Zt\Qt}^2) + 4 \mu \frac{\norm{\tilZt}\norm{\tilZtT\Zt \Qtp}}{\sglmin\left(\Zt\Qt\right)}  +4\mu \norm{\Deltat} \nonumber\\
    \stackrel{(b)}{\le} &\mu \left( \norm{\Zt\Qt}\norm{\Zt\Qtp}+ 40 \mu \norm{X} \norm{\Zt\Qt}^2  \right) \norm{\LAPT\P{\Zt\Qt}} + 4\mu \frac{\sqrt{\norm{X}}\norm{\tilZtT\Zt \Qtp}}{\sglmin\left(\Zt\Qt\right)}+4\mu \norm{\Deltat} \label{ineq:lemmashown1} \\
    \stackrel{(c)}{\le} &\mu \left( \norm{\Zt\Qt}\norm{\Zt\Qtp}+ 40 \mu \norm{X} \norm{\Zt\Qt}^2  \right) + 4\mu \frac{\sqrt{\norm{X}}\norm{\tilZtT}\norm{\Zt \Qtp}}{\sglmin\left(\Zt\Qt\right)}+4\mu \norm{\Deltat} \nonumber \\
    \stackrel{(d)}{\le} &\mu \left( 4\norm{X}+ 160 \mu \norm{X}^2  \right) + 8\mu \frac{\norm{X}\norm{\Zt \Qtp}}{\sglmin\left(\Zt\Qt\right)}+4\mu \norm{\Deltat} \nonumber \\
    \stackrel{(e)}{\le} &\frac{1}{2}. \nonumber
\end{align}
In inequality $(a)$ we used the assumption $\norm{\Zt\Qtp} \le 2\sglmin(\Zt\Qt)$. 
Inequalities $(b)$ follows from the assumption that $\norm{\Zt} \le 2\sqrt{\norm{X}}$ and inequality \eqref{ineq:boundforM1}.
Inequality $(c)$ follows from $ \norm{\LAPT\P{\Zt\Qt}}\le 1$ and the submultiplicativity of the spectral norm. 
Inequalities $(d)$ follows from the assumption that $\norm{\Zt} \le 2\sqrt{\norm{X}}$ and the fact that $\norm{\tilZt} =\norm{\Zt} $.
Inequality $(e)$ follows then from the assumptions  $\norm{\Deltat} \le c \sglmin (A)  $, $\norm{\Zt \Qtp} \le 2 \sglmin \left(\Zt \Qt \right)$, and $\mu \le \frac{c}{\kappa \norm{X}}$.
Note that inequality \eqref{ineq:lemmashown1} implies inequality \eqref{ineq:aux1}.
From the last line in the above inequality chain it follows that
\[\sglmin(\QtT\Qplus) = \sqrt{1 - \norm{\QtpT\Qplus}^2} \ge \frac{1}{2}.\]
This finishes the proof.
\end{proof}
Now we have all ingredients in place to give a proof of Lemma \ref{lemma:anglecontrol}.
\begin{proof}[Proof of Lemma \ref{lemma:anglecontrol}]
First, we recall that
\begin{equation*}
    \Zplus = \left( \I -\mu M_t \right) \Zt,
\end{equation*}
where
\begin{equation*}
    M_t = \Zt\ZtT - \tilZt\tilZtT - \symA +\Deltat.
\end{equation*}
As shown in inequality \eqref{ineq:boundforM} we have
\begin{equation}\label{ineq:boundforM2}
    \norm{M_t} \le 10 \norm{X}.
\end{equation} It follows that
\begin{align*}
    \Zplus\Qplus = & (\I - \mu M_t )\Zt\Qplus \\
    = & (\I-\mu M_t )\Zt\Qt\QtT\Qplus + (\I - \mu M)\Zt\Qtp\QtpT\Qplus \\
    = & (\I-\mu M_t )\P{\Zt\Qt}\P{\Zt\Qt}^T\Zt\Qt\QtT\Qplus + \left(\I - \mu M_t \right) \Zt\Qtp\QtpT\Qplus.
\end{align*}
Note that $\P{\Zt\Qt}^T\Zt\Qt\QtT\Qplus$ is invertible since $\P{\Zt\Qt}^T\Zt\Qt$ is invertible by assumption and $\QtT\Qplus$ is invertible by Lemma \ref{lemma:aux1}.
It follows that
\begin{align*}
    &(\I - \mu M_t)\Zt\Qtp\QtpT\Qplus \\
    = &(\I - \mu M_t)\Zt\Qtp\QtpT\Qplus(\P{\Zt\Qt}^T\Zt\Qt\QtT\Qplus)^{-1}\P{\Zt\Qt}^T\Zt\Qt\QtT\Qplus \\
    = &(\I - \mu M_t)\underbrace{\Zt\Qtp\QtpT\Qplus(\P{\Zt\Qt}^T\Zt\Qt\QtT\Qplus)^{-1}\P{\Zt\Qt}^T}_{=:K}\P{\Zt\Qt}\P{\Zt\Qt}^T\Zt\Qt\QtT\Qplus \\
    = &(\I - \mu M_t)K\P{\Zt\Qt}\P{\Zt\Qt}^T\Zt\Qt\QtT\Qplus.
\end{align*}
Hence, we obtain that
\[\Zplus\Qplus = (\I - \mu M_t)(\I + K)\P{\Zt\Qt}\P{\Zt\Qt}^T\Zt\Qt\QtT\Qplus.\]
Since $\P{\Zt\Qt}^T\Zt\Qt\QtT\Qplus$ is invertible, the span of the left singular vectors of 
\begin{equation*}
H:=\left(\I - \mu M_t \right) \left(\I + K \right)\P{\Zt\Qt}
\end{equation*}
is the same as the span of the left singular vectors of $\Zplus\Qplus$, which yields that
\[\norm{\LAPT\P{\Zplus\Qplus}} = \norm{\LAPT\P{H}} = \norm{\LAPT\P{H}\Q{H}^T} = \norm{\LAPT H(H^TH)^{-1/2}}.\]
For simplicity of notations, we define 
\begin{equation*}
B:= \left(\I - \mu M_t \right) \left(\I + K \right) - \I= K-\mu M_t -\mu M_t K.
\end{equation*}
It follows that
\begin{align*}
    (H^TH)^{-1/2} 
    = & \left(\P{\Zt\Qt}^T\left(\I + B\right)^T(\I + B)\P{\Zt\Qt}\right)^{-1/2} \\
    = & \left(\P{\Zt\Qt}^T\left(\I + B^T + B + B^TB\right)\P{\Zt\Qt}\right)^{-1/2} \\
    = & \left(\I + \underbrace{ \P{\Zt\Qt}^T\left(B^T + B\right)\P{\Zt\Qt} + \P{\Zt\Qt}^TB^TB\P{\Zt\Qt}}_{=:D} \right)^{-1/2}.
\end{align*}
Next, we want to apply Lemma \ref{MatrixTaylor}.
For that we need to check that $ \norm{D} \le 1/2 $ holds.
In the following we are going to show $ \norm{B} \le 1/6$, which implies $ \norm{D} \le 1/2 $.
For that, we note first that
\begin{align}
    \norm{K} = &\norm{\Zt\Qtp\QtpT\Qplus(\P{\Zt\Qt}^T\Zt\Qt\QtT\Qplus)^{-1}\P{\Zt\Qt}^T}  \nonumber\\
    = &\norm{\Zt\Qtp\QtpT\Qplus(\QtT\Qplus)^{-1}(\P{\Zt\Qt}^T\Zt\Qt)^{-1}\P{\Zt\Qt}^T}\nonumber \\
    \le &\norm{\Zt\Qtp}\norm{\QtpT\Qplus}\norm{(\QtT\Qplus)^{-1}}\norm{(\P{\Zt\Qt}^T\Zt\Qt)^{-1}}\norm{\P{\Zt\Qt}^T} \nonumber \\
    = &\frac{\norm{\Zt\Qtp}\norm{\QtpT\Qplus}}{\sglmin(\QtT\Qplus)\sglmin(\P{\Zt\Qt}^T\Zt\Qt)}\nonumber \\
    = &\frac{\norm{\Zt\Qtp}\norm{\QtpT\Qplus}}{\sglmin(\QtT\Qplus)\sglmin(\Zt\Qt)} \nonumber \\
    \stackrel{(a)}{\le} &4\norm{\QtpT\Qplus}  \nonumber\\
    \stackrel{(b)}{\le} & 4 \mu\left(\norm{\Zt\Qt}\norm{\Zt\Qtp}+ 40 \mu \norm{X} \norm{\Zt\Qt}^2 \right)\norm{\LAPT\P{\Zt\Qt}} + 16 \mu\frac{\sqrt{\norm{X}}\norm{\tilZtT\Zt \Qtp}}{\sglmin\left(\Zt\Qt\right)} + 16 \mu \norm{\Deltat} \nonumber \\
    \stackrel{(c)}{\le} & 700c \mu \sglmin \left(X\right) \norm{\LAPT\P{\Zt\Qt}} + 16\mu\frac{\sqrt{\norm{X}}\norm{\tilZtT\Zt \Qtp}}{\sglmin\left(\Zt\Qt  \right)} +16 \mu \norm{\Deltat} \stackrel{(d)}{\le} \frac{1}{12}. \label{ineq:aux7}
\end{align}
Inequality $(a)$ follows from Lemma \ref{lemma:aux1} and from the assumption $\norm{\Zt \Qtp} \le 2 \sglmin \left(\Zt \Qt \right)  $.
Inequality $(b)$ follows from using Lemma \ref{lemma:aux1} again.
In inequality $(c)$ we use the assumptions $\norm{\Zt} \le 2\sqrt{\norm{X}} $, $\mu \le \frac{c}{\norm{X} \kappa} $, and $ \norm{ \Zt \Qtp} \le c \kappa^{-1/2} \sqrt{ \sglmin \left(X\right) }$.
Inequality $(d)$ follows from the assumptions $\norm{\Zt\Qtp} \le c\kappa^{-1/2}\sqrt{\sglmin(X)}$, $\norm{\tilZtT\Zt \Qtp} \le \sqrt{\norm{X}}\sglmin\left(\Zt\Qt\right)$, $ \norm{\Deltat} \le c \sglmin(X) $, and $\norm{\Zt} \le 2\sqrt{\norm{X}}$.
Next, we note that
\begin{align}
    \norm{B} \le &\mu\norm{M_t} + \norm{K} + \mu\norm{M_tK} \nonumber \\
    \le &\mu\norm{M_t} + \norm{K} + \mu\norm{M_t}\norm{K} \nonumber \\
    \stackrel{(a)}{\le} & \frac{3}{2} \mu\norm{ M_t } + \norm{K}  \label{ineq:aux5} \\
    \le  & \frac{3}{2} \mu \left( \norm{ \Zt\ZtT - \tilZt\tilZtT-\symA }  + \norm{\Deltat}  \right) + \norm{K}        \nonumber  \\
     \stackrel{(b)}{\le} & \frac{3}{2} \mu \left( 9 \norm{X} + \norm{\Deltat}  \right) + \norm{K}      \nonumber    \\
    \le & 1/6, \nonumber
\end{align}
where in inequality $(a)$ we used that $ \norm{K} \le 1/12 \le 1/2 $.
In inequality $(b)$ we used the assumption $ \norm{\Zt} \le 2 \sqrt{\norm{X}} $.
Hence, we have shown $\norm{B} \le 1/6 $, which implies $\norm{D} \le 1/2$.
Thus, we can apply Lemma \ref{MatrixTaylor} and it follows that
\begin{align*}
    & \left(\I + \P{\Zt\Qt}^T(B^T + B)\P{\Zt\Qt} + \P{\Zt\Qt}^TB^TB\P{\Zt\Qt} \right)^{-1/2} \\
    = & \I - \frac{1}{2} \left(\P{\Zt\Qt}^T(B^T + B)\P{\Zt\Qt} + \P{\Zt\Qt}^TB^TB\P{\Zt\Qt} \right) + C,
\end{align*}
where $C$ is a matrix such that
\begin{align}\label{ineq:intern3}
\norm{C} \le 3\norm{\P{\Zt\Qt}^T \left(B^T + B\right)\P{\Zt\Qt} + \P{\Zt\Qt}^TB^TB\P{\Zt\Qt}}^2.
\end{align}
By a direct computation we obtain that
\begin{align*}
    &\LAPT H \left(H^TH \right)^{-1} \\
    = &\LAPT \left(\I + B \right)\P{\Zt\Qt} \left(\I-\frac{1}{2} \left(\P{\Zt\Qt}^T \left(B^T + B \right) \P{\Zt\Qt} + \P{\Zt\Qt}^TB^TB\P{\Zt\Qt} \right) + C \right) \\
    = &\LAPT \left(\I + B -\frac{1}{2}\P{\Zt\Qt}\P{\Zt\Qt}^T \left(B^T + B \right) \right) \P{\Zt\Qt} - \frac{1}{2}\LAPT B\P{\Zt\Qt}\P{\Zt\Qt}^T \left(B^T + B \right)\P{\Zt\Qt} \\ 
    &-\LAPT \left(\I + B \right)\P{\Zt\Qt} \left( \frac{1}{2}\P{\Zt\Qt}^TB^TB\P{\Zt\Qt} - C \right) \\
    = &\LAPT \left(\I -\mu M_t + \mu\frac{1}{2}\P{\Zt\Qt}\P{\Zt\Qt}^T \left(M_t^T + M_t \right) \right) \P{\Zt\Qt} \\
    & + \LAPT \left(K - \frac{1}{2}\P{\Zt\Qt}\P{\Zt\Qt}^T \left( K^T + K \right) \right) \P{\Zt\Qt} \\
    & -\mu\LAPT \left(M_t K - \frac{1}{2}\P{\Zt\Qt}\P{\Zt\Qt}^T \left(K^TM_t^T + M_t K \right) \right)\P{\Zt\Qt} \\
    & -\frac{1}{2}\LAPT B\P{\Zt\Qt}\P{\Zt\Qt}^T \left( B^T + B \right) \P{\Zt\Qt} \\
    & -\LAPT \left(\I + B\right)\P{\Zt\Qt} \left(\frac{1}{2}\P{\Zt\Qt}^TB^TB\P{\Zt\Qt} - C\right).
\end{align*}
It follows that
\begin{align}
    &\norm{\LAPT H \left(H^TH \right)^{-1}} \nonumber \\
    \le & \norm{\LAPT \left(\I - \mu M_t + \mu\frac{1}{2}\P{\Zt\Qt}\P{\Zt\Qt}^T \left( M_t^T + M_t \right) \right)\P{\Zt\Qt}} \nonumber \\ 
    & + 2\norm{K} + 2\mu\norm{M_t K} + \frac{1}{2}\norm{B}\norm{B^T + B} + \left(1 + \norm{B} \right) \left(\norm{B}^2 + \norm{C} \right) \nonumber \\
    \le & \norm{\LAPT \left(\I - \mu M_t + \mu\frac{1}{2}\P{\Zt\Qt}\P{\Zt\Qt}^T \left(M_t^T + M_t \right) \right)\P{\Zt\Qt}} \nonumber \\ 
    & + 2\norm{K} + 2\mu\norm{ M_t }\norm{K} + 2\norm{B}^2 + \norm{B}^3 + \left(1 + \norm{B} \right)\norm{C}. \label{ineq:aux4}
\end{align}
In order to proceed, we first note that
\begin{align*}
    &\LAPT \left(\I - \mu M_t + \mu\frac{1}{2}\P{\Zt\Qt}\P{\Zt\Qt}^T \left( M_t^T + M_t \right) \right) \P{\Zt\Qt} \\
    =&\LAPT \left(\I - \mu M_t + \mu  \P{\Zt\Qt}\P{\Zt\Qt}^T M_t \right)\P{\Zt\Qt} \\
    = &\LAPT \left(\I - \mu \left(\I - \P{\Zt\Qt}\P{\Zt\Qt}^T \right) \left(\Zt\ZtT - \tilZt\tilZtT - \symA + \Deltat \right) \right)\P{\Zt\Qt} \\
    = &\LAPT\P{\Zt\Qt} - \mu\LAPT \left(\I - \P{\Zt\Qt}\P{\Zt\Qt}^T \right)\left(\Zt\ZtT - \tilZt\tilZtT\right)\P{\Zt\Qt} \\
    & + \mu\LAPT \left(\I - \P{\Zt\Qt}\P{\Zt\Qt}^T\right) \symA\P{\Zt\Qt} + \mu \LAPT \left(\I - \P{\Zt\Qt}\P{\Zt\Qt}^T \right) \Deltat \P{\Zt\Qt}     \\
    = &\LAPT\P{\Zt\Qt} - \mu\LAPT\P{\Zt\Qt,\bot}\P{\Zt\Qt,\bot}^T(\Zt\Qtp\QtpT\ZtT - \tilZt\tilZtT)\P{\Zt\Qt} \\
    & - \mu\LAPT\tilLA\Sigma_X\tilLAT\LAP\LAPT\P{\Zt\Qt} - \mu\LAPT\P{\Zt\Qt}\P{\Zt\Qt}^T\symA\P{\Zt\Qt} \\
    & +  \mu \LAPT  \P{\Zt\Qt, \perp}\P{\Zt\Qt, \perp}^T \Deltat \P{\Zt\Qt} \\
    =& \underbrace{ \left( \I - \mu\LAPT\tilLA\Sigma_X\tilLAT\LAP \right) \LAPT\P{\Zt\Qt} \left(\I - \mu\P{\Zt\Qt}^T\symA\P{\Zt\Qt} \right) }_{=:(i)} \\
    & - \mu^2 \underbrace{ \LAPT\tilLA\Sigma_X\tilLAT\LAP\LAPT\P{\Zt\Qt}\P{\Zt\Qt}^T\symA\P{\Zt\Qt}}_{=:(ii)} \\
    & - \mu \underbrace{\LAPT\P{\Zt\Qt,\bot}\P{\Zt\Qt,\bot}^T\Zt\Qtp\QtpT\ZtT\P{\Zt\Qt}}_{=:(iii)} \\
    & + \mu \underbrace{ \LAPT\P{\Zt\Qt,\bot}\P{\Zt\Qt,\bot}^T\tilZt\tilZtT\P{\Zt\Qt} }_{=:(iv)} + \mu \underbrace{ \LAPT  \P{\Zt\Qt, \perp}\P{\Zt\Qt, \perp}^T \Deltat \P{\Zt\Qt}}_{=:(v)}.
\end{align*}
In the next step, we estimate the spectral norm of the individual summands.
We first note that
\begin{align*}
   \norm{(i)}= &\norm{ \left(\I - \mu\LAPT\tilLA\Sigma_X\tilLAT\LAP \right)\LAPT\P{\Zt\Qt} \left(\I - \mu\P{\Zt\Qt}^T\symA\P{\Zt\Qt} \right)} \\
    \stackrel{(a)}{\le} &\norm{\LAPT\P{\Zt\Qt} \left(\I - \mu\P{\Zt\Qt}^T\symA\P{\Zt\Qt} \right)} \\
    \le &\norm{\LAPT\P{\Zt\Qt}}\norm{\I - \mu\P{\Zt\Qt}^T\symA\P{\Zt\Qt}} \\
    = &\norm{\LAPT\P{\Zt\Qt}}\norm{\I - \mu\P{\Zt\Qt}^T \left(\LA\Sigma_X\LAT - \tilLA\Sigma_X\tilLAT \right) \P{\Zt\Qt}} \\
    \le &\norm{\LAPT\P{\Zt\Qt}} \left(1 - \mu\lambda_{\text{min}} \left(\P{\Zt\Qt}^T\LA\Sigma_X\LAT\P{\Zt\Qt} \right) + \mu\norm{\P{\Zt\Qt}^T\tilLA\Sigma_X\tilLAT\P{\Zt\Qt}} \right) \\
    \le &\norm{\LAPT\P{\Zt\Qt}} \left(1 - \mu\sglmin^2(\LAT \P{\Zt\Qt})\sglmin(X) + \mu\norm{\tilLAT\P{\Zt\Qt}}^2\norm{X} \right) \\
    = &\norm{\LAPT\P{\Zt\Qt}}\left(1 - \mu(1 - \norm{\LAPT\P{\Zt\Qt}}^2)\sglmin(X) + \mu\norm{\LAPT\P{\Zt\Qt}}^2\norm{X} \right) \\
    \stackrel{(b)}{\le} &\norm{\LAPT\P{\Zt\Qt}} \left(1 - \frac{1}{2}\mu\sglmin(X) \right).
\end{align*}
In inequality $(a)$ we used the assumption $\mu \le \frac{c}{\norm{X}\kappa}$ and in inequality $(b)$ we used the assumption that $\norm{\LAPT\P{\Zt\Qt}}\le \frac{c}{\kappa}$.
Next, we note that
\begin{align*}
    \norm{(ii)} =  \norm{\LAPT\tilLA\Sigma_X\tilLAT\LAP\LAPT\P{\Zt\Qt}\P{\Zt\Qt}^T\symA\P{\Zt\Qt}} \le  \norm{X}^2.
\end{align*}
Moreover, we have that
\begin{align*}
    \norm{(iii)}=& 
    \norm{\LAPT\P{\Zt\Qt,\bot}\P{\Zt\Qt,\bot}^T\Zt\Qtp\QtpT\ZtT\P{\Zt\Qt}}  \\
    =&\norm{\LAPT\P{\Zt\Qt,\bot}\P{\Zt\Qt,\bot}^T\Zt\Qtp\QtpT\ZtT\LAP\LAPT\P{\Zt\Qt}}  \\
    \le &\norm{\Zt\Qtp}^2\norm{\LAPT\P{\Zt\Qt}} \\
    \le &c^2\sglmin(X)\norm{\LAPT\P{\Zt\Qt}},
\end{align*}
where in the last line we use the assumption $\norm{\Zt\Qtp}\le c\sqrt{\sglmin(X)}$.
Next, we note that
\begin{align*}
    \norm{(iv)}= &\norm{\LAPT\P{\Zt\Qt,\bot}\P{\Zt\Qt,\bot}^T\tilZt\tilZtT\P{\Zt\Qt}}  \\
    \le & \norm{\tilZt} \norm{ \tilZtT \P{\Zt\Qt} }\\
    \le & 2 \sqrt{\norm{X}} \norm{ \tilZtT \P{\Zt\Qt} },
\end{align*}
where in the last line we used that $  \norm{\tilZt} = \norm{\Zt} $ due to symmetry and that by assumption $ \norm{\Zt} \le 2\sqrt{\norm{X}} $.
Moreover, we note that
\begin{align*}
    \norm{(v)}&= \norm{ \LAPT  \P{\Zt\Qt, \perp}\P{\Zt\Qt, \perp}^T \Deltat \P{\Zt\Qt} } \le \norm{ \Deltat }. 
\end{align*}
By adding up all summands, we obtain that
\begin{equation}\label{ineq:aux6}
    \begin{split}
        &\norm{\LAPT(\I - \mu M_t + \mu\frac{1}{2}\P{\Zt\Qt}\P{\Zt\Qt}^T(M_t^T + M_t))\P{\Zt\Qt}} \\
        \le &\left(1 - \left(\frac{1}{2} - c^2\right) \mu\sglmin(X) \right)\norm{\LAPT\P{\Zt\Qt}} +  \mu^2\norm{X}^2 
        + 2 \mu  \sqrt{\norm{X}} \norm{ \tilZtT \P{\Zt\Qt} }
        + \mu \norm{\Deltat}. 
    \end{split}
\end{equation}

Furthermore, we observe that
\begin{align*}
    \norm{C} \stackrel{(a)}{\le} &3\norm{\P{\Zt\Qt}^T \left(B^T + B \right) \P{\Zt\Qt} + \P{\Zt\Qt}^TB^TB\P{\Zt\Qt}}^2 \\
    \le &3 \left(2\norm{B} + \norm{B}^2 \right)^2 \\
    \stackrel{(b)}{\le} &27\norm{B}^2.
\end{align*}
where $(a)$ is inequality \eqref{ineq:intern3}. Inequality $(b)$ follows from $\norm{B}\le 1/6 \le 1$, which we have shown before. 
As a result, we obtain
\begin{align*}
&\norm{\LAPT H(H^TH)^{-1}} \\
    \stackrel{\eqref{ineq:aux4}}{\le} & \norm{\LAPT \left(\I - \mu M_t+ \mu\frac{1}{2}\P{\Zt\Qt}\P{\Zt\Qt}^T \left(M_t^T + M_t \right) \right)\P{\Zt\Qt}} \\ 
    & + 2\norm{K} + 2\mu\norm{M_t}\norm{K} + 2\norm{B}^2 + \norm{B}^3 + \left( 1 + \norm{B} \right) \norm{C} \\
    \stackrel{(a)}{\le} & \norm{\LAPT \left(\I - \mu M_t + \mu\frac{1}{2}\P{\Zt\Qt}\P{\Zt\Qt}^T \left(M_t^T + M_t \right) \right)\P{\Zt\Qt}} + 3\norm{K} + 57\norm{B}^2 \\
    \stackrel{\eqref{ineq:aux5}}{\le} & \norm{\LAPT \left(\I - \mu M_t + \mu\frac{1}{2}\P{\Zt\Qt}\P{\Zt\Qt}^T \left(M_t^T + M_t \right) \right)\P{\Zt\Qt}} + 3\norm{K} + 57 \left( \frac{3\mu}{2} \norm{M_t} + \norm{K}\right)^2 \\
    \stackrel{(b)}{\le} & \norm{\LAPT \left(\I - \mu M_t + \mu\frac{1}{2}\P{\Zt\Qt}\P{\Zt\Qt}^T \left(M_t^T + M_t \right) \right) \P{\Zt\Qt}} + 200\norm{K} + 500 \mu^2\norm{M_t}^2 \\
    \stackrel{\eqref{ineq:aux6}, \eqref{ineq:aux7}}{\le} & \left(1 - \left(\frac{1}{2} - c^2\right) \mu\sglmin(X) \right)\norm{\LAPT\P{\Zt\Qt}} +  \mu^2\norm{X}^2 +  2 \mu  \sqrt{\norm{X}} \norm{ \tilZtT \P{\Zt\Qt} }+   \mu \norm{\Deltat}  \\
    & + 140000c\mu\sglmin(X)\norm{\LAPT\P{\Zt\Qt}} + 3200\mu\frac{\sqrt{\norm{X}}\norm{\tilZtT\Zt \Qtp}}{\sglmin\left(\Zt\Qt\right)} +3200 \mu \norm{\Deltat}\\
    &+ \tilde{C} \mu^2\norm{X}^2 + \tilde{C} \mu^2 \norm{\Deltat} \\
    \le& \left(1 - \frac{1}{4}\mu\sglmin \left(X\right) \right)\norm{\LAPT\P{\Zt\Qt}} 
    +  2 \mu  \sqrt{\norm{X}} \norm{ \tilZtT \P{\Zt\Qt} }
    +C\mu\frac{\sqrt{\norm{X}}\norm{\tilZtT\Zt \Qtp}}{\sglmin\left(\Zt\Qt\right)}\\
    & + C \mu \norm{\Deltat} +  C\mu^2\norm{X}^2.
\end{align*}
In inequality $(a)$ we used that $\norm{B}\le 1/6\le 1$, $ \norm{C} \le 27 \norm{B}^2 $, and  $2\mu \norm{M_t} \le 1 $. 
The latter follows directly from the definition of $M_t$ combined with our assumptions.
In inequality $(b)$ we used the elementary inequality $ab \le \frac{a^2 + b^2}{2} $ and that $\norm{K} \le 1/12$.
The last two inequalities follow from the assumption on our step size $\mu$, by choosing the absolute constants $ \tilde{C}, C>0$ large enough, and by choosing the absolute constant $c>0$ small enough.
\end{proof}

%% file: lemma_specnormcontrol.tex
\subsection{Proof of Lemma \ref{lemma:normcontrolled}: Controlling $\norm{\Zt}$}\label{sec:normcontrolled}

\begin{proof}[Proof of Lemma \ref{lemma:normcontrolled}]
    We note that
     \begin{align}
    \Zplus &= \Zt -\mu \left( \Zt \ZtT - \tilZt \tilZtT - \symA \right) \Zt + \mu \Deltat \Zt \nonumber \\
    & =\left( \Id - \mu  \Zt \ZtT\right)\Zt + \mu \tilZt \tilZtT \Zt - \mu  \symA \Zt + \mu \Deltat \Zt. \label{ineq:aux52}
     \end{align}
    To deal with this expression, we compute that
     \begin{align*}
        &\tilZt \tilZtT \Zt \\
        =&   \tilZt \Qt \QtT \tilZtT \Zt \Qt \QtT +   \tilZt \Qt \QtT \tilZtT \Zt \Qtp \QtpT +  \tilZt \Qtp \QtpT \tilZtT \Zt \\
        =&   \tilZt \Qt \QtT \tilZtT P_{\tilZt Q_t} P_{\tilZt Q_t}^T P_{\Zt Q_t} P_{\Zt Q_t}^T  \Zt \Qt \QtT +   \tilZt \Qt \QtT \tilZtT \Zt \Qtp \QtpT +  \tilZt \Qtp \QtpT \tilZtT \Zt .
     \end{align*}
     We obtain that
     \begin{align*}
         \norm{\tilZt \tilZtT \Zt} \le& \norm{  P_{\tilZt Q_t}^T P_{\Zt Q_t} } \norm{\tilZt \Qt}^2  \norm{\Zt \Qt}  + \norm{\tilZt \Qt}^2  \norm{\Zt \Qtp}  + \norm{\tilZt \Qtp}^2 \norm{\Zt} \\
         \stackrel{(a)}{\le} & \left( \norm{  P_{\tilZt Q_t}^T P_{\Zt Q_t} } \norm{ \Zt}^2 + \norm{\Zt}  \norm{\Zt \Qtp}  +  \norm{\Zt \Qtp}^2   \right) \norm{\Zt}\\
         \stackrel{(b)}{\le} & \left( 4 \norm{  P_{\tilZt Q_t}^T P_{\Zt Q_t} } \norm{X} + \frac{\norm{X}}{30}   \right) \norm{\Zt},
     \end{align*}   
    where in inequality $(a)$ we have used that due to symmetry it holds that $\norm{\Zt \Qt} = \norm{\tilZt \Qt}  $ and $ \norm{\Zt \Qtp} = \norm{\tilZt \Qtp} $. 
    Inequality $(b)$ follows from the assumptions $\norm{\Zt} \le 2 \sqrt{\norm{X}} $ and $\norm{\Zt \Qtp}\le \frac{ \sqrt{\norm{X}}}{100} $. 
    Next, we note that
     \begin{align*}
        \norm{  P_{\tilZt Q_t}^T P_{\Zt Q_t} } &\le  \norm{  P_{\tilZt Q_t}^T \LAP \LAPT P_{\Zt Q_t} } +  \norm{  P_{\tilZt Q_t}^T \LA \LAT P_{\Zt Q_t} }  \\
        &\le \norm{ \LAPT P_{\Zt Q_t} } + \norm{  P_{\tilZt Q_t}^T \LA  }\\
        &= \norm{ \LAPT P_{\Zt Q_t} } + \norm{  P_{\Zt \Qt}^T \tilLA  }\\
        &\le  \norm{  \LAPT P_{\Zt Q_t}  } + \norm{ P_{\Zt \Qt}^T \LAP }\\
        &\le \frac{1}{50}  ,
    \end{align*}
    where we have used the assumption $\norm{\LAPT L_{\Zt \Qt}} \le \frac{1}{100}  $.
     Hence, we have shown that 
     \begin{equation}\label{ineq:aux51}
        \norm{\tilZt \tilZtT \Zt} 
        \le \left( \frac{4}{50}+\frac{1}{30} \right)  \norm{X} \norm{\Zt}
        \le \frac{\norm{X} \norm{\Zt}}{5}.
     \end{equation}
    From \eqref{ineq:aux52} and the triangle inequality it follows that
     \begin{align}
     \norm{ \Zplus } 
     &\le  \norm{ \left( \Id - \mu  \Zt \ZtT\right)\Zt } + \mu \norm{ \tilZt \tilZtT \Zt  } + \mu \norm{\symA} \norm{\Zt}  + \mu  \norm{ \Deltat} \norm{\Zt} \nonumber  \\
     &\stackrel{(a)}{=}  \norm{ \left( \Id - \mu  \Zt \ZtT\right)\Zt } +   \mu \norm{ \tilZt \tilZtT \Zt  }+   \mu \norm{X} \norm{\Zt} + \mu \norm{ \Deltat} \norm{\Zt} \nonumber \\ 
     &\stackrel{(b)}{\le}  \norm{ \left( \Id - \mu  \Zt \ZtT\right)\Zt } +   \mu  \frac{ \norm{X} \norm{\Zt}}{5} +   \mu \norm{X} \norm{\Zt} + \mu \norm{ \Deltat} \norm{\Zt} \nonumber \\ 
     &\stackrel{(c)}{\le}  \norm{ \left( \Id - \mu  \Zt \ZtT\right)\Zt } +  2 \mu \norm{X} \norm{\Zt} \nonumber \\
     & \stackrel{(d)}{=} \left( 1- \mu \norm{\Zt}^2\right) \norm{\Zt} +  2 \mu \norm{X} \norm{\Zt} \nonumber\\
    &= \left( 1- \mu \norm{\Zt}^2 + 2 \mu \norm{X} \right) \norm{\Zt}. \label{ineq:aux24}
     \end{align}
    In equality $(a)$ we used the fact that $\norm{\symA}=\norm{X}$, which follows from the definition of $\symA$.
    Inequality $(b)$ follows from \eqref{ineq:aux51}. 
    In inequality $(c)$ we used the assumption $ \norm{\Deltat} \le \frac{\norm{X}}{100} $.
    Equality $(d)$ follows from the singular value decomposition of $ \left( \Id - \mu  \Zt \ZtT\right)\Zt $ and $\Zt$, the fact that the function $x \mapsto (1-\mu x^2)x $ is increasing in the interval $ x\in \left( 0,\frac{1}{\sqrt{3 \mu} } \right) $, as well as the assumptions $ \mu \le \frac{c}{\norm{X}} $ and $\norm{\Zt} \le 2\sqrt{\norm{X}} $.
    
    In order to deduce the claim from inequality \eqref{ineq:aux24}, we will distinguish two cases.
    First, we assume that $ \norm{\Zt} < \frac{3}{2} \sqrt{\norm{X}} $. 
    Then the claim $\norm{\Zplus} \le 2\sqrt{\norm{X}} $ follows immediately from inequality \eqref{ineq:aux24} combined with the assumption $ \mu \le \frac{1}{100 \norm{X}}  $. 
    For the second case, we assume that $ \frac{3}{2} \sqrt{\norm{X}} \le \norm{\Zt} \le 2 \sqrt{ \norm{X} } $. 
    Then we can verify that \eqref{ineq:aux24} implies that $ \norm{\Zplus} \le \norm{\Zt} $.
    Since we assumed that $ \norm{\Zt} \le 2 \sqrt{ \norm{X} } $, this implies in particular that $ \norm{\Zplus} \le 2 \sqrt{ \norm{X} } $. 
    This finishes the proof.
\end{proof}

%% file: lemma_balanced_base.tex
\subsection{Proof of Lemma \ref{lemma:balancedbase}: Controlling $\norm{\tilZt^T \Zt} $}\label{sec:balancedbase}

\begin{proof}[Proof of Lemma \ref{lemma:balancedbase}]
First, we recall that
\begin{align*}
    \Zplus   &= \left( \I -\mu M_t) \right) \Zt, \\
    \tilZplus&=\left( \I +\mu M_t \right) \tilZt,
\end{align*}
where
\begin{equation*}
    M_t= \Zt\ZtT - \tilZt\tilZtT - \symA +\Deltat.
\end{equation*}
We calculate that 
\begin{align*}
    \tilZplusT \Zplus
    =  \tilZplusT \left( \I +\mu M_t \right)\left( \I -\mu M_t \right) \Zt 
    =  \tilZtT \Zt - \mu^2  \tilZtT M_t^2 \Zt. 
\end{align*}
It follows that 
\begin{align}
  \norm{\tilZplusT \Zplus }
  \le & \norm{ \tilZtT \Zt} + \mu^2 \norm{M_t}^2 \norm{\tilZt} \norm{\Zt} \nonumber \\ 
  \le &  \norm{ \tilZtT \Zt} + 4 \mu^2 \norm{M_t}^2 \norm{X}, \label{ineq:balancebase1}
\end{align}
where in the last line we used that by symmetry $ \norm{\tilZt} = \norm{\Zt}$ and that by assumption $ \norm{\Zt} \le 2 \sqrt{\norm{X}} $.
In order to proceed we note that 
\begin{align*}
 \norm{M_t} \le \norm{\Zt}^2 + \norm{\tilZt}^2 + \norm{X} + \norm{\Deltat} \le 10 \norm{X},   
\end{align*}
where in the last line we used the assumptions $ \norm{\Zt} \le 2 \sqrt{\norm{X}}$ and $\norm{\Deltat} \le \norm{X}$ and that by symmetry $ \norm{\Zt} = \norm{\tilZt} $.
Inserting this inequality into \eqref{ineq:balancebase1} proves the claim.
\end{proof}

%% file: lemma_balanced_refined1.tex
\subsection{Proof of Lemma \ref{lemma:balancednessperp}: Controlling $\norm{ \tilZtT \Zt\Qtp }$}\label{sec:balancednessperp}

\begin{proof}[Proof of Lemma \ref{lemma:balancednessperp}]

We first note that 
\begin{equation*}
    \tilZplus^T \Zplus \Qplusp =  \tilZplus^T \LAP \LAPT \Zplus \Qplusp. 
\end{equation*}
Analogously as in the proof of Lemma \ref{lemma:noisetermgrowth} by using the assumption that $\LAT \Zplus \Qt $ has full rank we can derive that 
\begin{equation*}
    \LAPT\Zplus\Qplusp = -\LAPT\P{\Zplus\Qt}(\LAT\P{\Zplus\Qt})^{-1}\LAT\Zplus\Qtp\QtpT\Qplusp + \LAPT\Zplus\Qtp\QtpT\Qplusp.    
\end{equation*}
Hence, we obtain by the triangle inequality that
\begin{align*}
    &\norm{\tilZplus^T \Zplus \Qplusp} \\
 \le & \norm{\tilZplus^T \LAP \LAPT\P{\Zplus\Qt}(\LAT\P{\Zplus\Qt})^{-1}\LAT\Zplus\Qtp\QtpT\Qplusp  }
 + \norm{\tilZplus^T \LAP \LAPT\Zplus\Qtp\QtpT\Qplusp  } \\
 \le & \underbrace{ \norm{\tilZplus^T \LAP \LAPT\P{\Zplus\Qt}(\LAT\P{\Zplus\Qt})^{-1}\LAT\Zplus\Qtp }}_{=:(I)} 
 + \underbrace{\norm{\tilZplus^T \LAP \LAPT\Zplus\Qtp }}_{=:(II)}.
\end{align*}
We estimate the two summands individually.\\

\noindent \textbf{Bounding (I):} 
First, we recall that
\begin{align*}
    \Zplus   &= \left( \I -\mu M_t \right) \Zt, \\
    \tilZplus&=\left( \I +\mu M_t \right) \tilZt,
\end{align*}
where
\begin{equation*}
    M_t= \Zt\ZtT - \tilZt\tilZtT - \symA +\Deltat.
\end{equation*}
As shown in inequality \eqref{ineq:boundforM} we have
\begin{equation}\label{ineq:boundforM3}
\norm{M_t} \le 10 \norm{X}.
\end{equation}
Then we compute that
\begin{align*}
    \LAT\Zplus\Qtp =& \LAT\left(\Id - \mu M_t\right) \Zt \Qtp \\
    =& \LAT\left(\Id - \mu M_t\right) \LAP \LAPT \Zt \Qtp \\
    =& - \mu\LAT M_t \LAP \LAPT \Zt \Qtp.
\end{align*}
By the submultiplicativity of the spectral norm it follows that 
\begin{align}
    (I) = & \norm{\tilZplus^T \LAP \LAPT\P{\Zplus\Qt}(\LAT\P{\Zplus\Qt})^{-1}\LAT\Zplus\Qtp } \nonumber \\
    = & \norm{- \mu\tilZplus^T \LAP \LAPT\P{\Zplus\Qt}(\LAT\P{\Zplus\Qt})^{-1}\LAT M_t \LAP \LAPT \Zt \Qtp } \nonumber \\
    \le & \mu   \norm{\tilZplus}  \frac{ \norm{ \LAPT\P{\Zplus\Qt} }}{ \sglmin(\LAT\P{\Zplus\Qt}) } \norm{ \LAT M_t\LAP } \norm{ \Zt \Qtp}. \label{ineq:secbalancing1}
\end{align}
Analogously as in Lemma \ref{lemma:anglecontrol} (see inequality \eqref{ineq:aux1212}) we can show that 
\begin{align}
    \norm{\LAPT\P{\Zplus\Qt}} 
    &\le  2\norm{\LAPT\P{\Zt\Qt}} + 20\mu\norm{X}. \label{ineq:balancingaux1} 
\end{align}
Note that due to our assumption on $ \mu $, $\norm{\Deltat}$, and $ \norm{\LAPT\P{\Zt\Qt}}  $ this also implies that $  \norm{\LAPT\P{\Zplus\Qt}}  \le \frac{1}{2}$, which implies that $\sglmin(\LAT\P{\Zplus\Qt})  \ge 1/2 $. 
Thus, from \eqref{ineq:secbalancing1} it follows that
\begin{align*}
    (I) 
    \le& 2 \mu \norm{\tilZplus}  \norm{ \LAPT\P{\Zplus\Qt} } \norm{ \LAT M_t\LAP } \norm{ \Zt \Qtp}\\
    \le& 4 \mu \sqrt{\norm{X}}  \norm{ \LAPT\P{\Zplus\Qt} } \norm{ \LAT M_t\LAP } \norm{ \Zt \Qtp}\\
    \le& 8 \mu  \left( \norm{\LAPT\P{\Zt\Qt}} + 10\mu\norm{X} \right) \norm{ \LAT M_t\LAP } \sqrt{\norm{X}} \norm{ \Zt \Qtp}.
\end{align*}
In the second inequality we used the assumption $ \norm{\Zplus} \le 2 \sqrt{\norm{X}} $ and that by symmetry it holds that $ \norm{\Zplus} = \norm{\tilZplus} $. In the third inequality we used inequality \eqref{ineq:balancingaux1}.\\
\noindent \textbf{Bounding (II):}
First, we observe that
\begin{align*}
    &\tilZplus^T \LAP \LAPT\Zplus\Qtp \\
    =& \tilZtT \left( \I +\mu M_t  \right) \LAP \LAPT  \left( \I -\mu M_t \right) \Zt\Qtp \\
    =&  \tilZtT \LAP \LAPT \Zt\Qtp  + \mu \tilZtT \left(  M_t  \LAP \LAPT -  \LAP \LAPT M_t \right) \Zt\Qtp - \mu^2 \tilZtT M_t \LAP \LAPT  M_t \Zt\Qtp \\
    =&  \tilZtT \Zt\Qtp  + \mu \tilZtT \left(  M_t  -  \LAP \LAPT M_t \right) \LAP \LAPT\Zt\Qtp - \mu^2 \tilZtT M_t \LAP \LAPT  M_t \Zt\Qtp \\
    =&  \tilZtT\Zt\Qtp  + \mu \tilZtT \LA \LAT M_t \LAP \LAPT  \Zt\Qtp - \mu^2 \tilZtT M_t \LAP \LAPT  M_t \Zt\Qtp.
\end{align*}
It follows from the triangle inequality that 
\begin{align*}
    (II) \le &  \norm{\tilZtT \Zt\Qtp} + \mu \norm{\tilZtT \LA \LAT M_t \LAP \LAPT  \Zt\Qtp} + \mu^2 \norm{\tilZtT M_t \LAP \LAPT  M_t \Zt\Qtp} \\
    \le &  \norm{\tilZtT \Zt\Qtp} + \mu \norm{ \LAT \tilZt} \norm{  \LAT M_t \LAP} \norm{ \Zt\Qtp } + \mu^2  \norm{\tilZt} \norm{ M_t}^2  \norm{ \Zt\Qtp }.
\end{align*}
Next, we note that
\begin{align*}
    \norm{ \LAT \tilZt} = & \norm{ \tilLAT \Zt} \\
    \le & \norm{ \LAPT \Zt} \\
    \le & \norm{ \LAPT \Zt \Qt} + \norm{ \LAPT \Zt \Qtp} \\
    \le & \norm{ \LAPT \PZQ } \norm{ \Zt } + \norm{\Zt \Qtp}.
\end{align*}
It follows that 
\begin{align*}
    (II)\le & \norm{\tilZtT \Zt\Qtp} + \mu \left(\norm{ \LAPT \PZQ } \norm{ \Zt } + \norm{\Zt \Qtp}\right)\norm{  \LAT M_t \LAP} \norm{ \Zt\Qtp } \\
    &+ \mu^2  \norm{\tilZt} \norm{M_t}^2  \norm{ \Zt\Qtp } \\
    \le & \norm{\tilZtT \Zt\Qtp} + \mu \left(2\norm{ \LAPT \PZQ } \sqrt{\norm{X}} + \norm{\Zt \Qtp}\right)\norm{  \LAT M_t \LAP} \norm{ \Zt\Qtp } \\
    &+ 200\mu^2 \norm{X}^{\frac{5}{2}}  \norm{ \Zt\Qtp },
\end{align*}
where in the second inequality we used the inequality \eqref{ineq:boundforM3}, the assumption $ \norm{\Zt} \le 2 \sqrt{\norm{X}}$ and that by symmetry it holds that $ \norm{\Zt} = \norm{\tilZt} $.\\

\noindent \textbf{Combining the estimates:}
By combining the previous two steps we obtain that 
\begin{align}
    & \norm{\tilZplus^T \Zplus \Qplusp} \nonumber \\
    \le & 8 \mu  \left( \norm{\LAPT\P{\Zt\Qt}} + 10\mu\norm{X} \right) \norm{ \LAT M_t\LAP } \sqrt{\norm{X}} \norm{ \Zt \Qtp} +  \norm{ \tilZtT \Zt\Qtp } \nonumber \\
    & + \mu \left(2\norm{ \LAPT \PZQ } \sqrt{\norm{X}} + \norm{\Zt \Qtp}\right)\norm{  \LAT M_t \LAP} \norm{ \Zt\Qtp } + 200\mu^2 \norm{X}^{\frac{5}{2}} \norm{ \LAPT  M_t}^2  \norm{ \Zt\Qtp } \nonumber\\
    =&  \norm{ \tilZtT \Zt\Qtp }  + 2\mu \left(  \left( 5\norm{\LAPT\P{\Zt\Qt}} + 40\mu\norm{X} \right) \norm{ \LAT M_t\LAP }  + 100\mu \norm{X}^2 \right)  \sqrt{\norm{ X }} \norm{ \Zt\Qtp }\nonumber\\
    &+\mu \norm{ \LAT M_t \LAP } \norm{\Zt \Qtp}^2. \label{ineq:secbalancing4}
\end{align}
Next, we estimate $ \norm{\LAT M_t\LAP} $. For that, we first calculate 
\begin{align*}
    \LAT M_t\LAP 
    &= \LAT \left(   \Zt\ZtT - \tilZt\tilZtT - \symA +\Deltat  \right)  \LAP \\
    &=  \LAT   \Zt\ZtT  \LAP -  \LAT  \tilZt\tilZtT  \LAP  +  \LAT \Deltat  \LAP \\
    &=  \LAT   \Zt \Qt \QtT \ZtT  \LAP  -  \LAT  \tilZt \Qt \QtT  \tilZtT  \LAP  -  \LAT  \tilZt \Qtp \QtpT  \tilZtT  \LAP  +  \LAT \Deltat  \LAP. 
\end{align*}
It follows that 
\begin{align*}
    \norm{ \LAT M_t\LAP } \le& \norm{ \Zt \Qt } \norm{ \LAPT   \Zt \Qt } + \norm{ \LAT  \tilZt \Qt } \norm{ \tilZt \Qt } + \norm{ \tilZt \Qtp  }^2 + \norm{\Deltat}\\
    = & \norm{ \Zt \Qt } \norm{ \LAPT   \Zt \Qt } + \norm{ \tilLAT  \Zt \Qt } \norm{ \tilZt \Qt } + \norm{ \tilZt \Qtp  }^2 + \norm{\Deltat}\\
    \le& \norm{     \Zt \Qt } \norm{ \LAPT   \Zt \Qt } + \norm{ \LAPT  \Zt \Qt } \norm{   \tilZt \Qt  } + \norm{ \tilZt \Qtp  }^2 + \norm{\Deltat}\\
    \stackrel{(a)}{=}& 2 \norm{     \Zt \Qt } \norm{ \LAPT   \Zt \Qt } + \norm{ \Zt \Qtp  }^2 + \norm{\Deltat} \\
    \le& 2 \norm{\LAPT \PZQ}  \norm{     \Zt \Qt }^2  + \norm{ \Zt \Qtp  }^2 + \norm{\Deltat},
\end{align*}
where in equality $(a)$ we used the symmetry between $\Zt$ and $\tilZt$, see Lemma \ref{lemma:symmetry}.
Using the assumption $ \norm{\Zt} \le 2 \sqrt{\norm{X}} $, we obtain that
\begin{align}
    \norm{ \LAT M_t\LAP }
    &\le  8 \norm{\LAPT \PZQ}  \norm{X}  + \norm{ \Zt \Qtp  }^2 + \norm{\Deltat}\le 8 \beta,\label{ineq:secbalancing3} 
\end{align}
where we have set $ \beta := \norm{\LAPT \PZQ}  \norm{X}  + \norm{ \Zt \Qtp  }^2 + \norm{\Deltat}$.
Inserting \eqref{ineq:secbalancing3} into \eqref{ineq:secbalancing4} we obtain that
\begin{align*}
    &\norm{\tilZplus^T \Zplus \Qplusp} \\   
    \le &\norm{ \tilZtT \Zt\Qtp }  + 2\mu \left(  8 \left( 5\norm{\LAPT\P{\Zt\Qt}} + 40\mu\norm{X} \right) \beta + 100 \mu \norm{X}^2 \right)  \sqrt{\norm{ X }} \norm{ \Zt\Qtp } + 8 \mu \beta \norm{\Zt \Qtp}^2\\
    \le &\norm{ \tilZtT \Zt\Qtp }  + C\mu \left(   \left( \norm{\LAPT\P{\Zt\Qt}} + \mu\norm{X} \right) \beta +  \mu \norm{X}^2 \right)  \sqrt{\norm{ X }} \norm{ \Zt\Qtp }+ C \mu \beta \norm{\Zt \Qtp}^2,
\end{align*}
where the last line follows by choosing the constant $C>0$ large enough.
This proves the claim.
\end{proof}

%% file: lemma_balanced_refined2.tex
\subsection{Proof of Lemma \ref{ref:balancednessangle}: Controlling $\norm{\tilZt^T \PZQ } $}\label{sec:balancednessangle}

\begin{proof}[Proof of Lemma \ref{ref:balancednessangle}]
    Recall that 
    \begin{align*}
        \Zplus &= \left(\Id - \mu M_t \right) \Zt, \\
        \tilZplus &=  \left(\Id + \mu M_t \right) \tilZt,
    \end{align*}
    where 
    \begin{equation*}
        M_t=  \Zt \ZtT - \tilZt \tilZtT -\symA + \Deltat.
    \end{equation*}
    Analogously, as in the proof of Lemma \ref{lemma:anglecontrol} we set 
    \begin{equation*}
        H:=\left( \I - \mu M_t \right) \left(\I + K\right) \P{\Zt\Qt},
    \end{equation*}
    where
    \begin{equation*}
        K:=\Zt\Qtp\QtpT\Qplus(\P{\Zt\Qt}^T\Zt\Qt\QtT\Qplus)^{-1}\P{\Zt\Qt}^T.
    \end{equation*}
    In the proof of Lemma \ref{lemma:anglecontrol} we have seen that $P_{\Zplus \Qplus}$ has the same column span as $H$. 
    It follows that
    \begin{equation*}
        \norm{\tilZplusT P_{\Zplus \Qplus}}= \norm{ \tilZplusT H \left(H^TH\right)^{-1/2}}.    
    \end{equation*} 
    We calculate that 
    \begin{align*}
        \tilZplusT P_{\Zplus \Qplus}
        =& \tilZt^T \left( \Id + \mu M_t \right) \left( \Id - \mu M_t \right)  \left(\I + K\right) \PZQ \left(H^TH\right)^{-1/2} \\
        =& \underbrace{ \tilZt^T \PZQ \left(H^TH\right)^{-1/2}}_{=:(I)} 
        + \underbrace{ \tilZt^T K \PZQ \left(H^TH\right)^{-1/2}}_{=:(II)}
         -\mu^2  \underbrace{ \tilZt M_t^2 \left( \Id +K \right) \PZQ  \left(H^TH\right)^{-1/2}}_{=:(III)}.
    \end{align*}
    We are going to estimate the spectral norm of the summands $(I)$, $(II)$, and $(III)$ individually.
    Before that, we will first derive bounds for $\norm{M_t} $, $\norm{ \QtpT\Qplus  }$, $\norm{K} $, and $ \sglmin \left(H\right) $. 
    First, we note that 
    \begin{align}
     \norm{M_t}
    \le & \norm{X}+ \norm{\Zt \ZtT} + \norm{\tilZt \tilZtT} + \norm{\Deltat} \nonumber \\
    = & \norm{X}+2 \norm{\Zt}^2 + \norm{\Deltat}\nonumber  \\
    \le & 10 \norm{X}, \label{ineq:balanceintern1}
    \end{align}
    where in the last line we used the assumptions $ \norm{\Zt} \le 2 \sqrt{\norm{X}} $ and $\norm{\Deltat} \le c \sglmin \left( X\right) $. Next, we derive an upper bound for $\norm{ \QtpT\Qplus  }$. It follows from Lemma \ref{lemma:aux1} that
    \begin{align}
        \norm{ \QtpT\Qplus  } \le&  \mu\left(\norm{\Zt\Qt}\norm{\Zt\Qtp}+ 40 \mu \norm{X} \norm{\Zt\Qt}^2 \right)\norm{\LAPT\P{\Zt\Qt}} + 4\mu\frac{\sqrt{\norm{X}}\norm{\tilZtT\Zt \Qtp}}{\sglmin\left(\Zt\Qt\right)} + 4 \mu \norm{\Deltat} \nonumber \\
        \stackrel{(a)}{\le}&  \mu\left( 2 \sqrt{\norm{X}} \norm{\Zt\Qtp}+ 160 \mu \norm{X}^2  \right) \norm{\LAPT \PZQ } + 4 \mu\frac{\sqrt{\norm{X}}\norm{\tilZtT\Zt \Qtp}}{\sglmin\left(\Zt\Qt\right)} + 4 \mu \norm{\Deltat} \nonumber\\
        \stackrel{(b)}{\le}& C_3 \mu c \sglmin \left(X\right), \label{ineq:balanceintern2additional}
    \end{align}
    where in inequality $(a)$ we used the assumption $\norm{\Zt} \le 2 \sqrt{\norm{X}}$
    and in inequality $(b)$ we used the assumptions $ \norm{ \Zt \Qtp} \le c \sqrt{\sglmin (X)} $, $\norm{\Deltat} \le c \sglmin \left( X\right)  $, $\norm{\tilZtT\Zt \Qtp} \le \frac{c}{\kappa} \sglmin \left(\Zt \Qt \right) \sqrt{\norm{X}} $, $ \mu \le \frac{c}{\norm{X} \kappa}  $, and $\norm{\LAPT \PZQ } \le \frac{c}{\kappa} $,
    where $C_3>0$ is an absolute constant chosen large enough.
    Then we derive an upper bound for $\norm{K} $. For that, we compute that
    \begin{align}
    \norm{K} 
    &\le \frac{ \norm{ \Zt\Qtp } \norm{ \QtpT \Qplus }  }{ \sglmin \left( \P{\Zt\Qt}^T\Zt\Qt \right) \sglmin \left( \QtT\Qplus  \right) } \nonumber\\
    &\stackrel{(a)}{\le} \frac{ 2 \norm{ \QtpT\Qplus }  }{ \sglmin \left( \QtT\Qplus  \right) } \nonumber\\
    & \stackrel{(b)}{\le} 4 \norm{ \QtpT\Qplus  } \nonumber\\
    & \stackrel{(c)}{\le} C_1 \mu c \sglmin \left(X\right)\stackrel{(d)}{\le} \frac{1}{8}. \label{ineq:balanceintern2}
    \end{align}
    In inequality $(a)$ we used the assumption $ \norm{ \Zt \Qtp} \le 2  \sglmin \left( \Zt \Qt \right) $. In inequality $(b)$ we have used $  \sglmin \left( \QtT\Qplus  \right) \ge 1/2 $, which follows from Lemma \ref{lemma:aux1}.
     In inequality $(c)$ we used inequality \eqref{ineq:balanceintern2additional}, where $C_1: = 4C_3$. Inequality $(d)$ follows from our assumption on the step size $\mu$ and by choosing the absolute constant $c>0$ small enough. In order to control $\sglmin \left( H \right) $ we observe that
    \begin{align}
       \sglmin \left( H \right) & \ge \left( 1 - \mu \norm{M_t} \right) \left( 1- \norm{K}  \right)  \norm{\PZQ} \nonumber  \\
        &= \left( 1 - \mu \norm{M_t} \right) \left( 1- \norm{K}  \right) \nonumber \\
        &\ge 1/2, \label{ineq:balanceintern3}
    \end{align}
    where in the last inequality we used inequalities \eqref{ineq:balanceintern1}, \eqref{ineq:balanceintern2}, and our assumption on the step size $\mu$.\\ 

    \noindent Now we are in a position to derive upper bounds for the spectral norms of the terms $(I)$, $(II)$, and $(III)$.\\

    \noindent\textbf{Estimation of (I):}
    Our goal is to derive an upper bound for the spectral norm of 
    \begin{equation*}
    (I)=\tilZt^T \PZQ \left(H^TH\right)^{-1/2}.
    \end{equation*}
    For that, we compute first
    \begin{align*}
        H^T H
        =& \P{\Zt\Qt}^T \left(\I + K^T\right) \left( \I - \mu M_t \right)^2 \left(\I + K\right) \P{\Zt\Qt},
    \end{align*}
    which can be rewritten as 
    \begin{align*}
        H^T H =& \P{\Zt\Qt}^T \left(\I - 2\mu M_t +K +K^T + F \right) \P{\Zt\Qt} \\
        =& \I - 2  \mu \P{\Zt\Qt}^T M_t  \PZQ +  \P{\Zt\Qt}^T \left( K + K^T +F \right) \PZQ, 
    \end{align*}
    where
    \begin{align}
        F:=& \left(\I + K^T\right) \left( \I - \mu M_t \right)^2 \left(\I + K\right) -(\I - 2\mu M_t +K +K^T)\nonumber \\
        =& \mu^2 M_t^2 - 2 \mu K^T M_t - 2 \mu M_t K +K^T K+  \mu^2 K^T M_t^2 - 2\mu K^TM_tK + \mu^2 M_t^2K + \mu^2 K^TM_t^2 K. \label{equ:definitionF}
    \end{align}
    Since $\norm{K} \le 1/8 \le 1 $ due to inequality \eqref{ineq:balanceintern2} we obtain that 
    \begin{align}
        \norm{F} \le& \mu^2 \norm{M_t}^2 + 4 \mu \norm{M_t} \norm{K} + \norm{K}^2 + 2 \mu^2 \norm{K} \norm{M_t}^2 + 2\mu \norm{M_t} \norm{K}^2  + \mu^2 \norm{M_t}^2 \norm{K}^2 \nonumber \\
        \le & \mu^2 \norm{M_t}^2 + 4 \mu \norm{M_t} \norm{K} + \norm{K}^2 + 2 \mu^2 \norm{M_t}^2 + 2\mu \norm{M_t} \norm{K}  + \mu^2 \norm{M_t}^2\\
        \stackrel{(a)}{\le} & 16 \mu^2 \norm{M_t}^2 + 13 \norm{K}^2 \nonumber \\
        \stackrel{(b)}{\le} & 16 \mu^2 \norm{M_t}^2 + 13 \norm{K} \nonumber  \\
        \stackrel{\eqref{ineq:balanceintern1},\eqref{ineq:balanceintern2} }{\le} & 1600 \mu^2 \norm{X}^2 + 13  C_1 \mu \frac{ c \sglmin \left(X\right) }{\kappa} \nonumber \\
        \stackrel{(c)}{\le} & C_2 \mu c \sglmin \left(X\right) \label{ineq:balancingintern16} \\
        \stackrel{(d)}{\le } & 1, \label{ineq:balancingintern17}
    \end{align}
    where inequality $(a)$ follows from the elementary inequality $ab \le \frac{a^2}{2} + \frac{b^2}{2} $.
    In inequality $(b)$ we used again that $\norm{K} \le 1$, which is due to inequality \eqref{ineq:balanceintern2}. Inequality $(c)$ follows from the assumption $ \mu \le \frac{c}{ \norm{X} \kappa} $ and by choosing the constant $C_2>0$ large enough.
    Inequality $(d)$ follows from our assumption on the step size $\mu$ and by choosing the absolute constant $c>0$ small enough.
    Since $ \norm{K}\le \frac{1}{8}$, $\norm{M_t}\le 10 \norm{X}$, by our assumption $ \mu \le \frac{c}{\kappa \norm{X}}$ we can apply Lemma \ref{MatrixTaylor} and obtain that 
    \begin{align*}
        \left( H^T H \right)^{-1/2} =\Id + \mu \PZQ^T M_t  \PZQ - \frac{1}{2} \P{\Zt\Qt}^T \left( K + K^T +F \right) \PZQ + G, 
    \end{align*}
    where $G$ is a symmetric matrix, which satisfies 
    \begin{equation}\label{ineq:balancingGdefinition}
        \norm{G} \le 3 \norm{ -2  \mu \PZQ^T M_t  \P{\Zt\Qt} +  \P{\Zt\Qt}^T \left( K + K^T +F \right) \PZQ }^2.
    \end{equation}
    It follows that 
    \begin{align*}
        \tilZt^T \PZQ \left( H^T H\right)^{-1/2}
       =& \tilZt^T \PZQ  \left( \I + \mu   \P{\Zt\Qt}^T M_t  \PZQ \right) \\
       & - \frac{1}{2} \tilZt^T \PZQ  \PZQ^T \left( K + K^T +F \right) \PZQ + \tilZt^T \PZQ G.
    \end{align*}
    In particular, we obtain that 
    \begin{align}\label{ineq:internaandb}
        \norm{  \tilZt^T \PZQ \left( H^T H\right)^{-1/2} }
        \le \underbrace{ \norm{  \tilZt^T \PZQ  \left( \I + \mu   \P{\Zt\Qt}^T M_t  \PZQ \right) }}_{=:(\square)}
        + \frac{1}{2} \underbrace{ \norm{ \tilZt^T \PZQ } \left( 2 \norm{K} + \norm{F}  +\norm{G}   \right)}_{=:( \square \square )}  .
    \end{align}
    We are going to estimate the two summands individually.
    In order to estimate $( \square )$ we first compute that 
    \begin{align*}
        &\tilZt^T \PZQ  \left( \I + \mu   \P{\Zt\Qt}^T M_t  \PZQ \right)\\
        =&\tilZt^T \PZQ  \left( \I - \mu   \P{\Zt\Qt}^T \left( \symA- \Zt \ZtT + \tilZt \tilZtT - \Deltat \right)  \PZQ \right) \\
        =&\tilZt^T \PZQ  \left( \I - \mu   \P{\Zt\Qt}^T \left( \symA + \tilZt \tilZtT - \Deltat \right)  \PZQ \right)+ \mu \tilZt^T \PZQ \PZQT \Zt \ZtT  \PZQ\\
        =&\tilZt^T \PZQ  \left( \I - \mu   \P{\Zt\Qt}^T \left( \symA + \tilZt \tilZtT - \Deltat \right)  \PZQ \right)+ \mu \tilZt^T \Zt \ZtT  \PZQ-\mu \tilZt^T \PZQperp \PZQperpT \Zt \ZtT  \PZQ.
    \end{align*}
    It follows that 
    \begin{align}
        (\square) =& \norm{\tilZt^T \PZQ  \left( \I + \mu   \P{\Zt\Qt}^T M_t  \PZQ \right)} \nonumber \\
        \le& \norm{ \tilZt^T \PZQ   }  \norm{  \I - \mu   \P{\Zt\Qt}^T \left( \symA + \tilZt \tilZtT - \Deltat  \right) \PZQ } + \mu \norm{ \tilZt^T \Zt  } \norm{\Zt} + \mu \norm{\tilZt} \norm{ \PZQperpT \Zt  } \norm{\Zt} \nonumber \\
        \stackrel{(a)}{\le}& \norm{ \tilZt^T \PZQ   }  \norm{  \I - \mu   \P{\Zt\Qt}^T \left( \symA + \tilZt \tilZtT - \Deltat  \right) \PZQ } + 2\mu \norm{ \tilZt^T \Zt  }\sqrt{\norm{X}} + 4 \mu \norm{X} \norm{ \PZQperpT \Zt  } \nonumber \\
        \stackrel{(b)}{\le}& \norm{ \tilZt^T \PZQ   }  \norm{  \I - \mu   \P{\Zt\Qt}^T \left( \symA + \tilZt \tilZtT - \Deltat  \right) \PZQ } + 2\mu \norm{ \tilZt^T \Zt  }\sqrt{\norm{X}} + 4 \mu \norm{X} \norm{  \Zt \Qtp  }, \label{ineq:internbalancing14}
    \end{align}
    where in inequality (a) we used the assumption $\norm{\tilZt} = \norm{\Zt} \le 2\sqrt{\norm{X}} $ and in inequality (b) we used the fact that $ \norm{ \PZQperpT \Zt  } =  \norm{ \PZQperpT \Zt \Qtp  } \le \norm{ \Zt \Qtp  } $. Furthermore, we have that 
    \begin{align}
    &\norm{ \I - \mu   \P{\Zt\Qt}^T \left( X + \tilZt \tilZtT - \Deltat  \right) \PZQ } \nonumber \\
    =& \norm{ \I - \mu   \P{\Zt\Qt}^T \left( \LA \Sigma_X \LAT - \tilLA \Sigma_X \tilLAT + \tilZt \tilZtT - \Deltat  \right) \PZQ } \nonumber\\
    \le &  \norm{  \I - \mu   \P{\Zt\Qt}^T \LA \Sigma_X \LAT \PZQ - \mu \PZQT \tilZt \tilZtT \PZQ  }\nonumber\\
    &+ \mu  \norm{  \PZQT \tilLA \Sigma_X \tilLAT \PZQ  }
    + \mu  \norm{  \PZQT \Deltat \PZQ  }\nonumber\\
    \le & \norm{  \I - \mu   \P{\Zt\Qt}^T \LA \Sigma_X \LAT \PZQ  }
    + \mu  \norm{  \PZQT \tilLA \Sigma_X \tilLAT \PZQ  } + \mu  \norm{  \PZQT \Deltat \PZQ }\nonumber\\
    \le &  1- \mu \sglmin \left(    \P{\Zt\Qt}^T \LA \Sigma_X \LAT \PZQ \right)
    + \mu  \norm{  \PZQT \tilLA \Sigma_X \tilLAT \PZQ  } + \mu  \norm{   \Deltat }. \label{ineq:internbalancing13}
    \end{align}
    We observe that 
    \begin{align}\label{ineq:internbalancing11}
    \sglmin \left(    \P{\Zt\Qt}^T \LA \Sigma_X \LAT \PZQ \right)    
    \ge \sglmin \left(  \P{\Zt\Qt}^T \LA \right)^2 \sglmin \left( \Sigma_X \right)
    \ge \frac{3}{4} \sglmin \left( X \right),
    \end{align}
    where we have used the assumption $ \norm{\LAPT \PZQ} \le \frac{c}{\kappa} $.
    Since $ \norm{ \tilLAT \PZQ  } \le    \norm{\LAPT \PZQ}  $, the same assumption also implies 
    \begin{align}\label{ineq:internbalancing12}
        \norm{  \PZQT \tilLA \Sigma_X \tilLAT \PZQ  } \le \norm{  \PZQT \tilLA } \norm{ \Sigma_X } \le \frac{1}{8} \sglmin \left(X\right) .
    \end{align}
    Inserting inequalities \eqref{ineq:internbalancing11} and \eqref{ineq:internbalancing12} into \eqref{ineq:internbalancing13} and using the assumption $ \norm{ \Deltat} \le c \sglmin \left(X\right) $ it follows that 
    \begin{align}
    \norm{ \I - \mu   \P{\Zt\Qt}^T \left( \symA  + \tilZt \tilZtT - \Deltat  \right) \PZQ } 
    \le 1- \frac{\mu}{2} \sglmin \left( X \right). 
    \end{align}
    Inserting the above inequality into \eqref{ineq:internbalancing14}  we obtain that 
    \begin{align}
        (\square)\le \left( 1-  \frac{\mu}{2} \sglmin \left( X \right) \right) \norm{\tilZt^T \PZQ } 
        + 2\mu \norm{ \tilZt^T \Zt  } \sqrt{\norm{X}}+ 4 \mu \norm{X} \norm{  \Zt \Qtp  }. \label{ineq:balancing:upperbounda}
    \end{align}
    Next, we are going to estimate term $(\square \square)$ in \eqref{ineq:internaandb}.
    First, we note that it follows from inequality \eqref{ineq:balancingGdefinition} that 
    \begin{align*}
        \norm{G} 
        \le &3 \norm{ -2  \mu \PZQ^T M_t  \P{\Zt\Qt} +  \P{\Zt\Qt}^T \left( K + K^T +F \right) \PZQ }^2\\
        \le & 3 \left( 2\mu \norm{M_t} + 2 \norm{K} +\norm{F} \right)^2 \\
        \le & 24 \mu^2 \norm{M_t}^2 + 6 \left(  2 \norm{K} +\norm{F}  \right)^2,
    \end{align*}
    where in the last line we used the elementary inequality $ (a+b)^2 \le 2a^2 + 2b^2 $.
    It follows that 
    \begin{align}
        (\square \square) = & \norm{ \tilZt^T \PZQ } \left( 2 \norm{K} + \norm{F}  +\norm{G}   \right) \nonumber \\
        \le &  \norm{ \tilZt^T \PZQ } \left( 2 \norm{K} + \norm{F}  + 24 \mu^2 \norm{M}^2 + 6 \left(  2 \norm{K} +\norm{F}  \right)^2   \right) \nonumber \\
        \stackrel{\eqref{ineq:balanceintern1}}{\le} &  \norm{ \tilZt^T \PZQ } \left( 2 \norm{K} + \norm{F}  + 2400 \mu^2 \norm{X}^2 + 6 \left(  2 \norm{K} +\norm{F}  \right)^2   \right) \nonumber \\
        \le &  \norm{ \tilZt^T \PZQ } \left( 2 \norm{K} + \norm{F}  + 2400 \mu c \sglmin \left( X \right) + 6 \left(  2 \norm{K} +\norm{F}  \right)^2   \right), \label{ineq:internbalancing15} 
    \end{align}
    where in the last line we used the assumption $ \mu \le \frac{c}{\norm{X} \kappa} $.
    Inserting our bounds for $\norm{K}$ and $\norm{F}$ into \eqref{ineq:internbalancing15} and obtain that
    \begin{align}
        (\square \square)
        \stackrel{(i)}{\le} & \norm{ \tilZt^T \PZQ } \left( 2 \norm{K} + \norm{F}  + 2400 \mu c \sglmin \left( X \right) + 18 \left(  2 \norm{K} +\norm{F}  \right)   \right) \nonumber \\
        \stackrel{(ii)}{\le} & C_3\mu c \sglmin (X)\norm{ \tilZt^T \PZQ }. \label{ineq:balancing:upperboundb}
    \end{align}
    In inequality $(i)$ we used that $\norm{K} \le 1 $ and $\norm{F}\le 1$ which we have shown above.
    In inequality $(ii)$ we used inequalities \eqref{ineq:balanceintern2} and \eqref{ineq:balancingintern16}. 
    $C_3>0$ is an absolute constant chosen large enough.
    Inserting the upper bounds for $(\square)$ and $(\square\square)$ (inequalities \eqref{ineq:balancing:upperbounda} and \eqref{ineq:balancing:upperboundb}) into \eqref{ineq:internaandb}  we obtain that  
    \begin{align*}
        &\norm{(I)} \le (\square) + \frac{1}{2}(\square\square) \\
        \le & \left( 1-  \frac{\mu}{2} \sglmin \left( X \right) \right) \norm{\tilZt^T \PZQ } 
        + 2\mu \norm{ \tilZt^T \Zt  } \sqrt{\norm{X}}  +  4 \mu \norm{X} \norm{  \Zt \Qtp  }+  \frac{1}{2}C_3\mu c \sglmin (X)\norm{ \tilZt^T \PZQ }  \\
        \le & \left( 1-  \frac{\mu}{4} \sglmin \left( X \right) \right) \norm{\tilZt^T \PZQ } 
        + 2\mu \norm{ \tilZt^T \Zt  } \sqrt{\norm{X}} +4 \mu \norm{X} \norm{  \Zt \Qtp  } ,
    \end{align*}
    where the last line follows since the absolute constant $c>0$ has been chosen small enough.\\

    \noindent\textbf{Estimation of (II):}
    By inserting the definition of $K$ we obtain that
    \begin{align*}
        \tilZt^T K \PZQ \left(H^TH\right)^{-1/2} 
        &= \tilZt^T \Zt\Qtp\QtpT\Qplus(\P{\Zt\Qt}^T\Zt\Qt\QtT\Qplus)^{-1}\P{\Zt\Qt}^T \PZQ \left(H^TH\right)^{-1/2}\\
        &= \tilZt^T \Zt\Qtp\QtpT\Qplus \left(  \QtT\Qplus \right)^{-1} \left(\P{\Zt\Qt}^T\Zt\Qt \right)^{-1} \left(H^TH\right)^{-1/2}. 
    \end{align*}
    It follows that 
    \begin{equation}\label{ineq:balancingintern18}
    (II) = \norm{\tilZt^T K \PZQ \left(H^TH\right)^{-1/2} }
    \le  \frac{\norm{\tilZt^T \Zt\Qtp} \norm{ \QtpT\Qplus }}{\sglmin \left(  \QtT\Qplus  \right) \sglmin \left( \Zt\Qt \right) \sglmin \left( H\right)}.    
    \end{equation}
    In particular, using inequality \eqref{ineq:balanceintern2additional}, \eqref{ineq:balanceintern3} and $  \sglmin \left( \QtT\Qplus  \right) \ge 1/2 $, which follows from Lemma \ref{lemma:aux1}, we obtain that 
    \begin{align*}
        (II) \le & 4C_3 \frac{\mu c  \norm{\tilZt^T \Zt\Qtp}  \sglmin (X) }{ \sglmin \left( \Zt\Qt \right)  } \\
        \le & \frac{\mu \norm{\tilZt^T \Zt\Qtp}  \sglmin (X) }{ \sglmin \left( \Zt\Qt \right)  },
    \end{align*}
    where the last line follows since the absolute constant $c>0$ has been chosen small enough.

    \noindent\textbf{Estimation of (III):}
    We note that
    \begin{align*}
        \norm{(III)}
        &=\norm{ \tilZt M_t^2 \left( \Id + K \right) \PZQ  \left(H^TH\right)^{-1/2}  }\\
        & \le \frac{ \norm{\tilZt} \norm{M_t}^2 \left( 1+ \norm{K} \right)}{\sglmin \left( H \right)} \\
        &\stackrel{(a)}{\le}  400 \norm{\tilZt} \norm{X}^2 
        \stackrel{(b)}{\le}  800 \norm{X}^{5/2},
    \end{align*}
    where in $(a)$ we used inequalities \eqref{ineq:balanceintern1},  \eqref{ineq:balanceintern2}, and \eqref{ineq:balanceintern3} and in inequality $(b)$ we used the assumption $  \norm{\Zt} \le 2 \sqrt{\norm{X}} $ and that by symmetry $ \norm{\Zt}= \norm{\tilZt} $, see Lemma \ref{lemma:symmetry}.\\

    By combining the upper bounds for the spectral norms of $(I)$, $(II)$, and $(III)$ we conclude that
    \begin{align*}
        \norm{ \tilZplusT P_{\Zplus \Qplus}} \le & \norm{(I)} + \norm{(II)} + \norm{(III)}\\
        \le & \left( 1-  \frac{\mu}{4} \sglmin \left( X \right) \right) \norm{\tilZt^T \PZQ } 
        + 2\mu \norm{ \tilZt^T \Zt  } \sqrt{\norm{X}}+ 4 \mu \norm{X} \norm{  \Zt \Qtp  }\\
        &+  \frac{\mu \norm{\tilZt^T \Zt\Qtp}  \sglmin (X) }{ \sglmin \left( \Zt\Qt \right)  }   + 800 \mu^2  \norm{X}^{5/2}. 
    \end{align*}
\end{proof}

%% file: lemma_localconvergence.tex
\subsection{Proofs of Lemma \ref{SpecLossboundedlemma} and  Lemma \ref{lemma:localconvergence}: Local linear convergence}\label{sec:localconvergence}

To simplify notation we set
\begin{equation*}
\Dt := \symA - \Zt\ZtT + \tilZt\tilZtT.
\end{equation*}
Hence, we can write
\begin{align*}
\Zplus &= \Zt + \mu \left(\Dt +\Deltat \right) \Zt,\\
\tilZplus &= \tilZt - \mu \left(\Dt +\Deltat \right) \tilZt.
\end{align*}
We will first prove Lemma \ref{SpecLossboundedlemma}.

\begin{proof}[Proof of Lemma \ref{SpecLossboundedlemma}]
For the proof of Lemma \ref{SpecLossboundedlemma} we need to introduce some additional notation.
Namely, by $ L_{\text{sym} (X)} \in \mathbb{R}^{(n_1+n_2) \times 2k}$ we denote a matrix with orthonormal columns, whose span is equal to the column span of $ \symA $.
By $ \LsymAP \in \mathbb{R}^{ (n_1 + n_2 ) \times ( n_1 + n_2 -2k) } $ we denote a matrix with orthonormal columns, whose span is orthogonal to the span of $ L_{\text{sym} (X)} $.
In particular, we have that 
\begin{equation*}
\LAP \LAPT = \tilLA \tilLAT + \LsymAP \LsymAPT.
\end{equation*}
Thus, by the triangle inequality and submultiplicativity of the spectral norm it holds that
\begin{align} \label{LAPTDtle}
    \norm{\LAPT\Dt} \le & \norm{\tilLAT\Dt} + \norm{\LsymAPT\Dt} \nonumber \\
    \le & \norm{\tilLAT\Dt} + \norm{\LsymAPT\Dt\LA} + \norm{\LsymAPT\Dt\tilLA} + \norm{\LsymAPT\Dt\LsymAP} \nonumber \\
    \le & \norm{\tilLAT\Dt} + \norm{\Dt\LA} + \norm{\Dt\tilLA} + \norm{\LsymAPT\Dt\LsymAP} \nonumber \\
    = & 3\norm{\LAT\Dt} + \norm{\LsymAPT\Dt\LsymAP},
\end{align}
where the last equation hold because $D_t$ is a symmetric matrix and because of the symmetry between $\Zt$ and $\tilZt$.
To bound the second term in line \eqref{LAPTDtle}, we note that
\begin{align} \label{LsymAPTDtLsymAPle}
    &\norm{\LsymAPT\Dt\LsymAP} \nonumber \\
    = & \norm{\LsymAPT\left(\tilZt\tilZtT - \Zt\ZtT\right)\LsymAP} \nonumber \\
    \le & \norm{\LsymAPT\left(\tilZt\Qt\QtT\tilZtT - \Zt\Qt\QtT\ZtT\right)\LsymAP} \nonumber \\
    & + \norm{\LsymAPT\tilZt\Qtp\QtpT\tilZtT\LsymAP} + \norm{\LsymAPT\Zt\Qtp\QtpT\ZtT\LsymAP} \nonumber \\
    \le & \norm{\LsymAPT\left(\tilZt\Qt\QtT\tilZtT - \Zt\Qt\QtT\ZtT\right)\LsymAP} + \norm{\tilZt\Qtp}^2 + \norm{\Zt\Qtp}^2 \nonumber \\
    = & \norm{\LsymAPT \left(\tilZt\Qt\QtT\tilZtT - \Zt\Qt\QtT\ZtT\right)\LsymAP} + 2\norm{\Zt\Qtp}^2,
\end{align}
where the last line follows from the symmetry between $\Zt$ and $\tilZt$.
To simplify notation, we define 
\begin{equation*}
\Ht := \tilZt\Qt\QtT\tilZtT - \Zt\Qt\QtT\ZtT.
\end{equation*}
Moreover, we note that
\begin{align}
    \norm{\LsymAPT\P{\Ht}} \le & \norm{\LsymAPT\P{\Zt \Qt}}+ \norm{\LsymAPT\P{\tilZt \Qt}} \nonumber \\
    \le & \norm{\LAPT \P{\Zt \Qt}}+ \norm{ \LtilAPT \P{\tilZt \Qt}}  \nonumber \\
    = & 2 \norm{ \LAPT \P{\Zt \Qt} }, \label{ineq:auxconvergencelemma}
\end{align}
where the last line follows from the symmetry between $\Zt$ and $\tilZt$.
In particular, this inequality, the assumption $ \norm{ \LAPT P_{\Zt \Qt} } \le \frac{c}{\kappa} $, and the fact\footnote[1]{In fact, for all symmetric and positive semidefinite matrices $A$ and $B$ of rank $k$ it holds that $A-B$ has rank $2k$ whenever the intersection of the range (of rank $k$) of $A$ and $B$ only contains the null element.} that $H_t $ has rank $2k$ imply that $\LsymAT\P{\Ht}$ is invertible.
This observation allows us then to bound the first term in inequality \eqref{LsymAPTDtLsymAPle} by computing that
\begin{align*}
    \LsymAPT\Ht\LsymAP = & \LsymAPT\P{\Ht}\P{\Ht}^T\Ht\LsymAP \\
    = & \LsymAPT\P{\Ht}(\LsymAT\P{\Ht})^{-1}\LsymAT\P{\Ht}\P{\Ht}^T\Ht\LsymAP \\
    = & \LsymAPT\P{\Ht}(\LsymAT\P{\Ht})^{-1}\LsymAT\Ht\LsymAP.
\end{align*}
It follows that
\begin{align*}
    \norm{\LsymAPT\Ht\LsymAP} \le & \frac{\norm{\LsymAPT\P{\Ht}}}{\sglmin(\LsymAT\P{\Ht})}\norm{\LsymAT\Ht\LsymAP} \\
    = & \frac{\norm{\LsymAPT\P{\Ht}}}{\sqrt{1 - \norm{\LsymAPT\P{\Ht}}^2}}\norm{\LsymAT\Ht\LsymAP}.
\end{align*}
Note that
\begin{align*}
    \norm{\LsymAT \Ht\LsymAP} \le & \norm{\LsymAT(\tilZt\tilZtT - \Zt\ZtT)\LsymAP} \\
    + & \norm{\LsymAT\tilZt\Qtp\QtpT\tilZtT\LsymAP} + \norm{\LsymAT\Zt\Qtp\QtpT\ZtT\LsymAP} \\
    \le & \norm{\LsymAT(\tilZt\tilZtT - \Zt\ZtT)\LsymAP} + \norm{\tilZt\Qtp}^2 + \norm{\Zt\Qtp}^2 \\
    \stackrel{(a)}{=} & \norm{\LsymAT\Dt\LsymAP} + 2\norm{\Zt\Qtp}^2 \\
    \le & \norm{\LAT\Dt} + \norm{\tilLAT\Dt} + 2\norm{\Zt\Qtp}^2 \\
    \stackrel{(b)}{\le} & 2\norm{\LAT\Dt} + 2\norm{\Zt\Qtp}^2,
\end{align*}
where both equality $(a)$ and inequality $(b)$ follow from the symmetry between $\Zt$ and $\tilZt$.
Combining the above two inequality chains with inequality \eqref{ineq:auxconvergencelemma} we obtain that
\begin{align} \label{LsymAPTHtLsymAPle}
    \norm{\LsymAPT\Ht\LsymAP} 
    &\le  \frac{ 4 \norm{\LAPT\P{\Zt\Qt}} }{\sqrt{1 -2 \norm{\LAPT\P{\Zt\Qt}}}} \left(  \norm{\LAT\Dt} + \norm{\Zt\Qtp}^2 \right) \nonumber \\
    &\le 2\norm{\LAT\Dt} + 2\norm{\Zt\Qtp}^2,
\end{align}
where in the second inequality we used the assumption that $\norm{\LAPT\P{\Zt\Qt}} \le \frac{c}{\kappa}$. 
Combining inequalities \eqref{LAPTDtle}, \eqref{LsymAPTDtLsymAPle}, and \eqref{LsymAPTHtLsymAPle} we conclude that
\[\norm{\LAPT\Dt} \le 5\norm{\LAT\Dt} + 4\norm{\Zt\Qtp}^2,\]
which shows the first inequality in the statement of Lemma \ref{SpecLossboundedlemma}. To prove the second inequality in the statement of Lemma \ref{SpecLossboundedlemma}, it suffices to note that
\begin{align*}
    \norm{\Dt} \le & \norm{\LAT\Dt} + \norm{\LAPT\Dt}
    \le  6\norm{\LAT\Dt} + 4\norm{\LAPT\Dt}.
\end{align*}
This finishes the proof of Lemma \ref{SpecLossboundedlemma}.
\end{proof}
With Lemma \ref{SpecLossboundedlemma} in place we can also prove Lemma \ref{lemma:localconvergence}.
\begin{proof}[Proof of Lemma \ref{lemma:localconvergence}]
We compute that
\begin{align*}
    \Dplus = &\symA - \Zplus\ZplusT + \tilZplus\tilZplusT \\
    = &\symA - \left(\Zt + \mu \left(\Dt +\Deltat \right) \Zt\right) \left(\Zt + \mu \left(\Dt +\Deltat \right) \Zt\right)^T\\
    & + \left(\tilZt - \mu\left(\Dt +\Deltat \right) \tilZt\right) \left(\tilZt - \mu \left(\Dt +\Deltat \right) \tilZt\right)^T \\
    = &\symA - \Zt\ZtT - \mu \left(\Dt +\Deltat \right) \Zt\ZtT - \mu\Zt\ZtT \left(\Dt +\Deltat \right) - \mu^2 \left( \Dt +\Deltat \right) \Zt\ZtT \left(\Dt +\Deltat \right) \\
    & + \tilZt\tilZtT - \mu\left(\Dt +\Deltat \right)\tilZt\tilZtT - \mu\tilZt\tilZtT \left(\Dt +\Deltat \right) + \mu^2\left(\Dt +\Deltat \right) \tilZt\tilZtT \left(\Dt +\Deltat \right) \\
    = &\Dt - \mu\Dt\left( \Zt\ZtT + \tilZt\tilZtT \right) - \mu\left(\Zt\ZtT + \tilZt\tilZtT \right)\Dt - \mu\Deltat \left( \Zt\ZtT + \tilZt\tilZtT \right) - \mu \left(\Zt\ZtT + \tilZt\tilZtT \right) \Deltat \\
    & - \mu^2 \left(\Dt +\Deltat \right) \left(\Zt\ZtT - \tilZt\tilZtT \right) \left(\Dt +\Deltat \right) \\
    = &\left(\I - \mu\Zt\ZtT - \mu\tilZt\tilZtT\right) \Dt \left( \I - \mu\Zt\ZtT - \mu\tilZt\tilZtT \right) - \mu \left( \Deltat\left(\Zt\ZtT + \tilZt\tilZtT \right) + \left( \Zt\ZtT + \tilZt\tilZtT \right) \Deltat\right) \\
    & - \mu^2 \left(\Zt\ZtT + \tilZt\tilZtT\right) \Dt \left(\Zt\ZtT + \tilZt\tilZtT \right) - \mu^2 \left(\Dt+\Deltat\right) \left(\Zt\ZtT - \tilZt\tilZtT \right) \left(\Dt + \Deltat\right).
\end{align*}
By the triangle inequality it follows that 
\begin{align*}
\norm{ \LAT  \Dplus } \le & \underbrace{ \norm{ \LAT \left(\I - \mu\Zt\ZtT - \mu\tilZt\tilZtT\right) \Dt \left(\I - \mu\Zt\ZtT - \mu\tilZt\tilZtT \right)} }_{=:(I)}\\
& + \mu \underbrace{ \norm{ \LAT\left( \Deltat(\Zt\ZtT + \tilZt\tilZtT)+(\Zt\ZtT + \tilZt\tilZtT)\Deltat\right) }}_{=:(II)} \\
& + \mu^2 \underbrace{ \norm{ \LAT \left(\Zt\ZtT + \tilZt\tilZtT \right)\Dt \left(\Zt\ZtT + \tilZt\tilZtT \right)  }}_{=:(III)}\\
& + \mu^2 \underbrace{\norm{ \LAT \left( \left(\Dt+\Deltat\right) (\Zt\ZtT - \tilZt\tilZtT) \left(\Dt + \Deltat\right) \right)   }}_{=:(IV)}.
\end{align*}
We bound each summand individually.

\noindent \textbf{Estimation of $(I)$:}
We first compute that
\begin{align} \label{LATIeq}
   & \LAT \left(\I - \mu\Zt\ZtT - \mu\tilZt\tilZtT\right) \Dt \left(\I - \mu\Zt\ZtT - \mu\tilZt\tilZtT \right) \nonumber \\
    = & \LAT(\I - \mu\Zt\ZtT - \mu\tilZt\tilZtT)\LA\LAT\Dt(\I - \mu\Zt\ZtT - \mu\tilZt\tilZtT) \nonumber \\
    & + \LAT(\I - \mu\Zt\ZtT - \mu\tilZt\tilZtT)\LAP\LAPT\Dt(\I - \mu\Zt\ZtT - \mu\tilZt\tilZtT).
\end{align}
To bound the spectral norm of the first summand in this expression, we note that
\begin{align*}
    &\norm{\LAT(\I - \mu\Zt\ZtT - \mu\tilZt\tilZtT)\LA\LAT\Dt(\I - \mu\Zt\ZtT - \mu\tilZt\tilZtT)} \\
    \le &\norm{\LAT(\I - \mu\Zt\ZtT - \mu\tilZt\tilZtT)\LA}\norm{\LAT\Dt}\norm{\I - \mu\Zt\ZtT - \mu\tilZt\tilZtT} \\
    \stackrel{(a)}{\le} &\norm{\LAT(\I - \mu\Zt\ZtT - \mu\tilZt\tilZtT)\LA}\norm{\LAT\Dt} \\
    = &\norm{\I - \mu\LAT\Zt\ZtT\LA - \mu\LAT\tilZt\tilZtT\LA}\norm{\LAT\Dt} \\
    \stackrel{(b)}{\le} &\norm{\I - \mu\LAT\Zt\ZtT\LA}\norm{\LAT\Dt} \\
    \stackrel{(c)}{=} &\left(1 - \mu\sglmin^2\left(\LAT\Zt\Qt\right)\right)\norm{\LAT\Dt}.
\end{align*}
The inequalities $(a)$ and $(b)$ and equality $(c)$ are a consequence of the assumptions $\mu \le \frac{c}{\kappa \norm{X}}$, $\norm{\Zt} \le 2 \sqrt{\norm{X}} $, and the fact that by symmetry $ \norm{\Zt} = \norm{\tilZt} $.
Note that
\begin{align*}
    \sglmin^2 \left(\LAT\Zt\Qt\right) = & \sglmin^2\left(\LAT\P{\Zt\Qt}\P{\Zt\Qt}^T\Zt\Qt\right) \\
    \ge & \sglmin^2\left(\LAT\P{\Zt\Qt}\right)\sglmin^2\left(\Zt\Qt\right) \\
    \stackrel{(a)}{\ge} & \frac{1}{2}\sglmin^2\left(\Zt\Qt\right) \\
    \stackrel{(b)}{\ge} & \frac{1}{16}\sglmin \left(X \right).
\end{align*}
Inequality $(a)$ follows from the assumption $\norm{\LAPT\P{\Zt\Qt}} \le \frac{c}{\kappa} \le \frac{1}{2}$.
In inequality $(b)$ we use the assumption  $\sglmin(\Zt\Qt) \ge \sqrt{\frac{\sglmin(X)}{8}}$.
Combining the above two inequalities, we obtain that
\begin{align} \label{LALATDtle}
    & \norm{\LAT \left(\I - \mu\Zt\ZtT - \mu\tilZt\tilZtT \right) \LA\LAT\Dt \left(\I - \mu\Zt\ZtT - \mu\tilZt\tilZtT \right) } \nonumber \\
    \le & \left(1 - \frac{\mu}{16}\sglmin \left(X \right) \right)\norm{\LAT\Dt}.
\end{align}
To bound the spectral norm of the second summand in \eqref{LATIeq}, we note that
\begin{align*}
    & \norm{\LAT(\I - \mu\Zt\ZtT - \mu\tilZt\tilZtT)\LAP\LAPT \Dt(\I - \mu\Zt\ZtT - \mu\tilZt\tilZtT)} \\
    \le & \norm{\LAT(\I - \mu\Zt\ZtT - \mu\tilZt\tilZtT)\LAP}\norm{\LAPT\Dt}\norm{\I - \mu\Zt\ZtT - \mu\tilZt\tilZtT} \\
    \le & \norm{\LAT(\I - \mu\Zt\ZtT - \mu\tilZt\tilZtT)\LAP}\norm{\LAPT\Dt} \\
    = & \mu\norm{\LAT\Zt\ZtT\LAP + \LAT\tilZt\tilZtT\LAP}\norm{\LAPT\Dt} \\
    \le & \mu \left( \norm{\LAT\Zt\ZtT\LAP} + \norm{\LAT\tilZt\tilZtT\LAP} \right)\norm{\LAPT\Dt},
\end{align*}
where inequality $(a)$ follows from the assumptions $ \mu \le \frac{c}{ \kappa \norm{X}}$ and $\norm{\Zt} \le 2 \sqrt{\norm{X}} $ and the fact that $ \norm{\Zt} = \norm{\tilZt} $.
We observe that
\begin{align*}
    \norm{\LAT\Zt\ZtT\LAP} = & \norm{\LAT\Zt\Qt\QtT\ZtT\LAP} \\
    \le & \norm{\Zt\Qt}\norm{\QtT\ZtT\LAP} \\
    \le & \norm{\LAPT\P{\Zt\Qt}}\norm{\Zt\Qt}^2,
\end{align*}
and
\begin{align*}
    \norm{\LAT\tilZt\tilZtT\LAP} \le & \norm{\LAT\tilZt\Qt\QtT\tilZtT\LAP} + \norm{\LAT\tilZt\Qtp\QtpT\tilZtT\LAP} \\
    \le & \norm{\LAT\tilZt\Qt}\norm{\QtT\tilZtT} + \norm{\tilZt\Qtp\QtpT\tilZtT} \\
    \stackrel{(a)}{=} & \norm{\LAT\tilZt\Qt} \norm{\Zt\Qt} + \norm{\Zt\Qtp}^2 \\
    \stackrel{(b)}{\le} & \norm{ \LtilAPT  \tilZt\Qt}\norm{\Zt\Qt} + \norm{\Zt\Qtp}^2 \\
    \stackrel{(c)}{=} & \norm{\LAPT\Zt\Qt}\norm{\Zt\Qt} + \norm{\Zt\Qtp}^2 \\
    \le & \norm{\LAPT\P{\Zt\Qt}}\norm{\Zt\Qt}^2 + \norm{\Zt\Qtp}^2.
\end{align*}
In equation $(a)$ we used that by symmetry it holds that $\norm{\Zt \Qt}=\norm{\tilZt \Qt}$ and $ \norm{\Zt \Qtp} = \norm{\tilZt \Qtp} $.
(For the definition of $\LtilAP$ we refer to Lemma \ref{lemma:symmetry}.)
Inequality $(b)$ holds since the column span of $L_X$ is contained in the column span of $\LtilAPT $. 
Equality $(c)$ holds since we have that $\LtilAPT  \tilZt\Qt =\LAPT\Zt\Qt $, see Lemma \ref{lemma:symmetry}.
Combining the above three inequalities, we obtain that
\begin{align} \label{LAPLAPTDtle}
    & \norm{\LAT \left(\I - \mu\Zt\ZtT - \mu\tilZt\tilZtT \right) \LAP\LAPT\Dt \left(\I - \mu\Zt\ZtT - \mu\tilZt\tilZtT \right)} \nonumber \\
    \le & 2\mu \left(\norm{\LAPT\P{\Zt\Qt}}\norm{\Zt\Qt}^2 + \norm{\Zt\Qtp}^2 \right) \norm{\LAPT\Dt} \nonumber \\
    \stackrel{(a)}{\le} & 2\mu \left(\norm{\LAPT\P{\Zt\Qt}}\norm{\Zt\Qt}^2 + \norm{\Zt\Qtp}^2 \right) \left(5\norm{\LAT\Dt} + 4\norm{\Zt\Qtp}^2 \right) \nonumber \\
    \stackrel{(b)}{\le} & \frac{\mu}{200}\sglmin(X) \left(5\norm{\LAT\Dt} + 4\norm{\Zt\Qtp}^2 \right).
\end{align}
In inequality (a) we use Lemma \ref{SpecLossboundedlemma}.
Inequality (b) follows from the assumptions $\norm{\LAPT\P{\Zt\Qt}} \le c\kappa^{-1}$, $\norm{\Zt}\le 2\sqrt{\norm{X}}$, $\norm{\Zt\Qtp}\le c \sqrt{\sglmin(X)}$ and by choosing the absolute constant $c>0$ small enough.
Combining \eqref{LATIeq}, \eqref{LALATDtle}, and \eqref{LAPLAPTDtle} we obtain that 
\begin{align*}
    &\norm{\LAT  \left(\I - \mu\Zt\ZtT - \mu\tilZt\tilZtT\right) \Dt \left(\I - \mu\Zt\ZtT - \mu\tilZt\tilZtT \right)} \\
    \le &\left(1 - \frac{\mu}{16}\sglmin(X) \right) \norm{\LAT\Dt} +  \frac{\mu}{200}\sglmin(X) \left(5\norm{\LAT\Dt} + 4\norm{\Zt\Qtp}^2 \right) \\
    \le & \left(1 - \frac{\mu}{32}\sglmin(X)\right)\norm{\LAT\Dt} +  \frac{\mu}{50}\sglmin(X)\norm{\Zt\Qtp}^2.
\end{align*}
\noindent \textbf{Estimation of $(II)$:}
We note that
\begin{align*}
 \norm{   \Deltat(\Zt\ZtT + \tilZt\tilZtT)+(\Zt\ZtT + \tilZt\tilZtT)\Deltat} \le & 2 \norm{\Zt\ZtT + \tilZt\tilZtT} \norm{\Deltat}\\
 \stackrel{(a)}{\le} & 16 \norm{X} \norm{\Deltat} \\
 \stackrel{(b)}{\le}  & \frac{1}{600} \sglmin \left(X\right) \norm{\Dt} \\
 \stackrel{(c)}{\le} & \frac{ \sglmin \left(X\right)}{600} \left( 6\norm{\LAT\Dt} + 4\norm{\Zt\Qtp}^2\right)\\
 = &  \frac{ \sglmin \left(X\right)}{100} \norm{\LAT\Dt} + \frac{ \sglmin \left(X\right)}{150} \norm{\Zt\Qtp}^2,
\end{align*}
where in inequality $(a)$ we used the assumption  $\norm{\Zt}\le 2\sqrt{\norm{X}}$ and the fact that $ \norm{\Zt} = \norm{\tilZt}$.
In inequality $(b)$ we used the assumption $\norm{\Deltat} \le \frac{c}{\kappa} \norm{\Dt} $.
Inequality $(c)$ follows from Lemma \ref{SpecLossboundedlemma}.\\

\noindent \textbf{Estimation of $(III)$:}
Using the assumption $\norm{\Zt}\le 2\sqrt{\norm{X}}$ and the fact that $ \norm{\Zt} = \norm{\tilZt} $ we obtain that
\begin{align*}
     \norm{(\Zt\ZtT + \tilZt\tilZtT)\Dt(\Zt\ZtT + \tilZt\tilZtT)} \le & \norm{\Zt\ZtT + \tilZt\tilZtT}\norm{\Dt}\norm{\Zt\ZtT + \tilZt\tilZtT} 
    \le  64\norm{X}^2\norm{\Dt}.
\end{align*}
\noindent \textbf{Estimation of $(IV)$:}
We obtain that
\begin{align*}
    \norm{ \left( \Dt + \Deltat \right) (\Zt\ZtT - \tilZt\tilZtT) \left( \Dt + \Deltat \right) }   \le & \norm{\Dt + \Deltat}\norm{\Zt\ZtT - \tilZt\tilZtT}\norm{\Dt + \Deltat} \\
    \stackrel{(a)}{\le} & 4 \norm{\Dt} \norm{\Zt\ZtT - \tilZt\tilZtT} \norm{\Dt} \\
    = & 4 \norm{\symA - \Zt\ZtT + \tilZt\tilZtT}\norm{\Zt\ZtT - \tilZt\tilZtT}\norm{\Dt} \\
    \stackrel{(b)}{\le} & 320\norm{X}^2\norm{\Dt},
\end{align*}
where in inequality $(a)$ we used the triangle inequality and the assumption $ \norm{\Deltat} \le \frac{c}{\kappa} \norm{\Dt} $.
Inequality $(b)$ follows from the assumption $\norm{\Zt}\le 2\sqrt{\norm{X}}$ and from $ \norm{\Zt} = \norm{\tilZt} $.\\

\noindent \textbf{Conclusion:}
By combining the estimates for $(I)$, $(II)$, $(III)$, and $(IV)$ we obtain that
\begin{align*}
    & \norm{\LAT\Dplus} \\
    \le & \left(1 - \frac{\mu}{32}\sglmin(X)\right)\norm{\LAT\Dt} +  \frac{\mu}{50}\sglmin(X)\norm{\Zt\Qtp}^2 + \mu \sglmin (X) \left( \frac{1}{100} \norm{\LAT\Dt} + \frac{1}{150} \norm{\Zt\Qtp}^2\right) \\
    & +  64\mu^2\norm{X}^2\norm{\Dt} + 320\mu^2\norm{X}^2\norm{\Dt}\\
    \le & \left(1 - \frac{\mu}{64}\sglmin(X)\right)\norm{\LAT\Dt} +  \frac{\mu}{25}\sglmin(X)\norm{\Zt\Qtp}^2 +  64\mu^2\norm{X}^2\norm{\Dt} + 320\mu^2\norm{X}^2\norm{\Dt}\\
    \stackrel{(a)}{\le} & \left(1 - \frac{\mu}{64}\sglmin(X) \right)\norm{\LAT\Dt} +  \frac{\mu}{25}\sglmin(X)\norm{\Zt\Qtp}^2 + 384 \mu c \sglmin(X) \left(6\norm{\LAT\Dt} + 4\norm{\Zt\Qtp}^2 \right) \\
    \stackrel{(b)}{\le} & \left(1 - \frac{\mu}{128}\sglmin(X) \right)\norm{\LAT\Dt} +  \frac{\mu}{20}\sglmin(X)\norm{\Zt\Qtp}^2.
\end{align*}
In inequality $(a)$ we used Lemma \ref{SpecLossboundedlemma} and the assumption $\mu \le c\kappa^{-1} \norm{X}^{-1}$.
Inequality $(b)$ holds since the absolute constant $c>0$ is chosen small enough.
This finishes the proof of Lemma \ref{lemma:localconvergence}.
\end{proof}

%% file: lemma_phase2combined.tex
\subsection{Proof of main lemma for Phase 2 (Lemma \ref{lemma:phase2combined})}\label{sec:phase2proof}

\begin{proof}[Proof of Lemma \ref{lemma:phase2combined}]
At the beginning of the proof we would like to recall that as described in Remark \ref{choiceofconstants} the constants 
$\hat{c}_1,\hat{c}_2,\hat{c}_3,\hat{c}_4,\hat{c}_5,\hat{c}_6,\hat{c}_7 $
are chosen such that 
$\hat{c}_1, \hat{c}_2, \hat{c}_3^{4/5} \ll \hat{c}_4 \hat{c}_5 $, $\hat{c}_4 \ll \hat{c}_5 \ll \hat{c}_6 \ll 1$ holds.
Set
\begin{equation*}
    t_2:= \min \left\{ t \in \mathbb{N} : \ \sglmin \left( \LAT \Zt \right) \ge  \sqrt{\frac{\sglmin (X)}{8}} \text{ and } t \ge t_1  \right\}.
\end{equation*}
We show by induction that it holds that for $t_1 \le t \le t_2$ that
\begin{align}
    \sglmin \left( \LAT \Zt \right) &\ge  \left(1 + \frac{1}{8}\mu\sglmin(X)  \right)^{t-t_1} \sglmin \left( \LAT \Ztone \right),  \label{ineq:phase2hyp1} \\
    \norm{\Zt \Qtp} & \le \left( 1+ \frac{\mu}{1500} \sglmin(X)  \right)^{t-t_1} \norm{\Ztone Q_{t_1,\bot}},  \label{ineq:phase2hyp2} \\
    \norm{\LAPT P_{\Zt \Qt }} & \le \frac{\hat{c}_4}{\kappa^2},    \label{ineq:phase2hyp3} \\
    \norm{\Zt} &\le 2 \sqrt{\norm{X}}, \label{ineq:phase2hyp4}\\
    \norm{\tilZtT \Zt } &\le  \norm{ \widetilde{Z}_{t_1}^T Z_{t_1} } + 400\mu^2 \left( t-t_1 \right) \norm{X}^3    \le  \frac{\hat{c}_7 \norm{X}}{\kappa^{4}},  \label{ineq:phase2hyp_balanc1}\\
    \norm{ \tilZtT \Zt \Qtp } &\le \frac{\hat{c}_4 \hat{c}_5 \sqrt{ \norm{X}} }{\kappa^{3} } \left( 1+ \frac{\mu}{1500} \sglmin(X)  \right)^{t-t_1} \norm{\Ztone Q_{t_1,\bot}},  \label{ineq:phase2hyp_balanc2}\\
    \norm{ \tilZtT \PZQ } &\le \frac{\hat{c}_4 \hat{c}_6 \sqrt{\norm{X}}}{\kappa^3}.  \label{ineq:phase2hyp_balanc3}
\end{align}        
Before establishing these inequalities, we note that from \eqref{ineq:phase2hyp1} and the definition of $t_2$ the upper bound on $t_2-t_1$ given by inequality \eqref{ineq:t2bound} directly follows. 

We are going to prove these inequalities by induction. For that, we note first that the base case $t=t_1$ follows directly from the assumptions in this lemma.\\
To show the induction step $ t\rightarrow t+1$ for $t<t_2$ we first note that 
\begin{align}
\norm{\Deltat} =& \norm{(\Bcal^* \Bcal - \I)  \left( \symA - \Zt \ZtT +  \tilZt \tilZtT \right) } \nonumber \\
\le&\norm{(\Bcal^* \Bcal - \I)  \left( \symA  \right) }+\norm{(\Bcal^* \Bcal - \I)  \left( \Zt \Qt \QtT \ZtT \right) }+\norm{(\Bcal^* \Bcal - \I)  \left( \Zt \Qtp \QtpT \ZtT \right) } \nonumber \\
&+\norm{(\Bcal^* \Bcal - \I)  \left(  \tilZt \Qt \QtT \tilZtT \right) }+\norm{(\Bcal^* \Bcal - \I)  \left(  \tilZt \Qtp \QtpT \tilZtT \right) } \nonumber \\
\stackrel{(a)}{\le} & \delta \sqrt{r} \left(  \norm{\symA} + \norm{\Zt \Qt}^2 + \norm{\tilZt \Qt}^2  \right) +   \delta \left( \nucnorm{ \Zt\Qtp \QtpT \ZtT } + \nucnorm{ \tilZt\Qtp \QtpT \tilZtT }   \right) \nonumber \\
\stackrel{(b)}{=}& \delta \sqrt{r} \left(  \norm{\symA} + 2\norm{\Zt \Qt}^2  \right) + 2\delta \nucnorm{ \Zt\Qtp \QtpT \ZtT }  \nonumber\\
\stackrel{(c)}{\le} & \delta \sqrt{r}  \left(  9 \norm{X} + 2\nucnorm{ \Zt\Qtp \QtpT \ZtT }   \right), \label{ineq:phase2aux2}
\end{align}
where inequality $(a)$ follows from Lemma \ref{lemma:RIPlemma} and the restricted isometry property of the measurement operator $\mathcal{B}$.
In equality $(b)$ we used the symmetry between $\Zt$ and $\tilZt$ (see Lemma \ref{lemma:symmetry}) and in inequality $(c)$ we used the induction hypothesis \eqref{ineq:phase2hyp4}.
Next, we are going to show inequality \eqref{ineq:phase2hyp2} for $t+1$.
For that, we observe that 
\begin{equation}
\nucnorm{ \Zt\Qtp \QtpT \ZtT } \le \left( k -r \right) \norm{\Zt \Qtp}^2. \label{ineq:phase2aux3}
\end{equation}
To estimate this expression further, we observe that due to the induction hypothesis \eqref{ineq:phase2hyp2} it holds
\begin{align}
\norm{\Zt \Qtp} &\le  \left( 1+ \frac{\mu}{1500} \sglmin(X)  \right)^{t-t_1} \norm{\Ztone Q_{t_1,\bot}} \nonumber \\
& \stackrel{(a)}{\le} \exp \left(    \frac{2 \ln \left( \frac{ \sqrt{ \sglmin \left(X\right) } }{\sqrt{8} \sglmin \left( \LAT \Ztone \right) }   \right)}{\ln \left( 1 +\frac{\mu}{8} \sglmin \left(X\right)  \right)} \ln \left(  1+ \frac{\mu}{1500} \sglmin(X)   \right)   \right) \norm{\Ztone Q_{t_1,\bot}} \nonumber \\
& \stackrel{(b)}{\le} \exp \left(  \frac{ \ln \left( \frac{ \sqrt{ \sglmin \left(X\right) } }{\sqrt{8} \sglmin \left( \LAT \Ztone \right) }   \right)   }{5}  \right) \norm{\Ztone Q_{t_1,\bot}} \nonumber \\
& =  \left(  \frac{ \sglmin \left(X\right)  }{ 8 \sglmin^2 \left( \LAT \Ztone \right)  } \right)^{1/10}  \norm{\Ztone Q_{t_1,\bot}} \nonumber \\
& \stackrel{(c)}{\le}  \left(  \frac{ \sglmin \left(X\right)  }{ 8 \sglmin^2 \left( \LAT P_{Z_{t_1} Q_{t_1}} \right) \sglmin^2 \left(\Ztone Q_{t_1}  \right)  } \right)^{1/10}  \norm{\Ztone Q_{t_1,\bot}} \nonumber \\
& \stackrel{(d)}{\le}  \left(  \frac{ \sglmin \left(X\right)  }{ 128 \norm{ \Ztone Q_{t_1,\bot} }^2  } \right)^{1/10}  \norm{\Ztone Q_{t_1,\bot}} \nonumber \\
& = \left( \frac{1}{128} \right)^{1/10}  \left( \sglmin \left(X\right) \right)^{1/10}   \norm{ \Ztone Q_{t_1,\bot}}^{4/5}   \label{ineq:phase2aux5} \\
& \stackrel{(e)}{\le} \frac{ \hat{c}_3^{4/5} \sqrt{ \sglmin \left(X\right) } }{ \kappa^{7/2} k^{4/5}}, \label{ineq:phase2aux4}
\end{align}
where in inequality $(a)$ we have used the upper bound on $ t-t_1 \le t_2 -t_1$ in inequality \eqref{ineq:t2bound}.
In inequality $(b)$ we have used the elementary inequalities $ \frac{x}{1-x} \le \ln \left(1+x\right) \le x $ and the assumption on the step size $\mu \le \frac{ \hat{c}_2 }{\kappa^4 \norm{X}}$. 
Inequality $(c)$ follows from the fact that $ \sglmin \left( \LAT \Ztone \right) \ge \sglmin \left(   \LAT P_{Z_{t_1} Q_{t_1}}  \right) \sglmin \left( \Ztone Q_{t_1} \right)  $.
Note that it follows from assumption \eqref{ineq:phase2assump2} that $\sglmin \left(   \LAT P_{Z_{t_1} Q_{t_1}}  \right)  \ge 1/2 $.
Together with assumption \eqref{ineq:phase2assump1} this implies inequality $(d)$. 
Inequality $(e)$ follows again from assumption \eqref{ineq:phase2assump1}.
We remark that with an analogous computation as in the above inequality chain, we also can show that 
\begin{equation}\label{eqref:phase2aux12}
\begin{split}
\norm{ \tilZtT \Zt \Qtp }
&\le \frac{c_4 c_5 \sqrt{\norm{X}}}{\kappa^3} \left(  1 + \frac{\mu \sglmin \left(X\right)}{1500} \right)^{t-t_1} \norm{ Z_{\tone} Q_{\tone, \bot} }\\
&\le \left( \frac{1}{128} \right)^{1/10} \frac{ \hat{c}_4 \hat{c}_5 \sqrt{\norm{X}}}{\kappa^3} \left( \sglmin \left(X\right) \right)^{1/10}   \norm{ \Ztone Q_{t_1,\bot}}^{4/5}  
\le   \frac{ \hat{c}_3^{4/5} \hat{c}_4 \hat{c}_5 \sqrt{ \norm{X} \sglmin \left(X\right) } }{ \kappa^{7/2} k^{4/5}}.
\end{split}
\end{equation}
Combining \eqref{ineq:phase2aux3} and \eqref{ineq:phase2aux4} and inserting this into \eqref{ineq:phase2aux2} it follows that 
\begin{equation}\label{ineq:phase2aux11}
    \norm{\Deltat} \le 11 \delta \sqrt{r} \norm{X} \le \frac{11 \hat{c}_1}{\kappa^2} \sglmin \left(X\right),
\end{equation}
where in the last inequality we have used our assumption $\delta \le \frac{ \hat{c}_1}{\kappa^{3} \sqrt{r} }$.
Thus, we conclude that all the assumptions for Lemma \ref{lemma:sigmingrowth} are fulfilled.
It follows that 
\begin{equation}\label{ineq:phase2aux1}
    \sglmin(\LAT\Zplus) \ge \sglmin(\LAT\Zplus\Qt) \ge \sglmin(\LAT\Zt)\left(1 + \frac{1}{4}\mu\sglmin(X) - \mu\sglmin^2(\LAT\Zt)\right).
\end{equation}
Since we assumed $t < t_2 $, which implies by the definition of $t_2$ that $ \sglmin \left( \LAT \Zt \right) < \sqrt{\frac{ \sglmin(X) }{8}} $, we obtain that 
\begin{equation*}
  \sglmin \left(\LAT \Zplus\right) \ge \left(1 + \frac{1}{8}\mu\sglmin(X)  \right)  \sglmin(\LAT\Zt).
\end{equation*}
This implies \eqref{ineq:phase2hyp1} for $t+1$.
Note that \eqref{ineq:phase2aux1} also implies that $\LAT\Zplus\Qt$ has full rank.
Hence, we can apply Lemma \ref{lemma:noisetermgrowth} and by choosing the absolute constants $\hat{c}_1$, $\hat{c}_2$, and $\hat{c}_4 $ small enough we obtain that
\begin{equation}
\norm{\Zplus\Qplusp} \le \left (1 - \frac{\mu}{2}\norm{\Zt\Qtp}^2 + \frac{ \mu \sglmin(X)}{3000} \right) \norm{\Zt\Qtp}  + 2\mu\sqrt{\norm{X}}\norm{\tilZtT\Zt \Qtp}.
\end{equation} 
We obtain that
\begin{align*}
 \norm{\Zplus\Qplusp} 
 \stackrel{(a)}{\le} & \left (1 - \frac{\mu}{2}\norm{\Zt\Qtp}^2 + \frac{ \mu   \sglmin(X)}{3000} + \frac{ 2\mu \hat{c}_4 \hat{c}_5 \norm{X} }{\kappa^{3}} \right) \left( 1+ \frac{\mu}{1500} \sglmin(X)  \right)^{t-t_1} \norm{\Ztone Q_{t_1,\bot}}\\
 \stackrel{(b)}{\le} & \left( 1+ \frac{\mu}{1500} \sglmin(X)  \right)^{t+1-t_1} \norm{\Ztone Q_{t_1,\bot}}.
\end{align*}
Inequality $(a)$ is due to induction hypotheses \eqref{ineq:phase2hyp2} and \eqref{ineq:phase2hyp_balanc2}.
Inequality $(b)$ follows from choosing the absolute constants $\hat{c}_4$ and $ \hat{c}_5$ to be small enough.
This implies inequality \eqref{ineq:phase2hyp2} for $t+1$.

Next, we observe that the assumptions of Lemma \ref{lemma:anglecontrol} are satisfied and hence it follows that 
\begin{equation}\label{ineq:intern111}
\begin{split}
    &\norm{\LAPT\P{\Zplus Q_{t+1} }} \le \\
&\left(1 - \frac{1}{4}\mu\sglmin \left(X\right) \right)\norm{\LAPT\P{\Zt\Qt}} 
+  2 \mu  \sqrt{\norm{X}} \norm{ \tilZtT \P{\Zt\Qt} }
+C\mu\frac{\sqrt{\norm{X}}\norm{\tilZtT\Zt \Qtp}}{\sglmin\left(\Zt\Qt\right)} + C \mu \norm{\Deltat} +  C\mu^2\norm{X}^2.
\end{split}
\end{equation}
In order to proceed, we note that 
\begin{align}
    \frac{\norm{\tilZtT\Zt \Qtp}}{\sglmin\left(\Zt\Qt\right)} 
    \le & \frac{ \norm{\tilZtT\Zt \Qtp}}{\sglmin\left( \LAT \Zt\right)} \nonumber \\
    \stackrel{(a)}{\le} & \frac{ \hat{c}_4 \hat{c}_5 \sqrt{ \norm{X} } \left( 1+ \frac{\mu}{1500} \sglmin(X)  \right)^{t-t_1} \norm{\Ztone Q_{t_1,\bot}}}{ \kappa^{3} \left(1 + \frac{1}{8}\mu\sglmin(X)  \right)^{t-t_1} \sglmin \left( \LAT \Ztone \right)} \nonumber\\
    \le & \frac{ \hat{c}_4 \hat{c}_5 \sqrt{ \norm{X}  }   \norm{\Ztone Q_{t_1,\bot}}}{ \kappa^{3} \sglmin \left( \LAT \Ztone \right)} \nonumber \\
    \le & \frac{ \hat{c}_4 \hat{c}_5 \sqrt{ \norm{X}  }   \norm{\Ztone Q_{t_1,\bot}}}{ \kappa^{3} \sglmin \left( \LAT P_{\Ztone Q_{t_1}} \right) \sglmin \left( \Ztone Q_{t_1} \right) } \nonumber \\
    \stackrel{(b)}{\le} & \frac{4 \hat{c}_4 \hat{c}_5 \sqrt{ \norm{X}  } }{ \kappa^{3}}, \label{ineq:aux23}
\end{align}
where in inequality $(a)$ we have used the induction hypotheses \eqref{ineq:phase2hyp1} and \eqref{ineq:phase2hyp_balanc2}.
Inequality $(b)$ is due to the assumption \eqref{ineq:phase2assump1} and the induction hypothesis \eqref{ineq:phase2hyp3}.
Combining this inequality chain with inequality \eqref{ineq:intern111} we obtain that
\begin{align*}
    &\norm{\LAPT\P{\Zplus Q_{t+1}}}  \\
    \le &\left(1 - \frac{1}{4}\mu\sglmin \left(X\right) \right)\norm{\LAPT\P{\Zt\Qt}} 
    +  2 \mu  \sqrt{\norm{X}} \norm{ \tilZtT \P{\Zt\Qt} }
    +\frac{4C \hat{c}_4 \hat{c}_5 \mu  \norm{X} }{ \kappa^{3}}+ C \mu \norm{\Deltat} +  C\mu^2\norm{X}^2\\
    \stackrel{(a)}{\le} &\left(1 - \frac{1}{4}\mu\sglmin \left(X\right) \right) \frac{c_4}{\kappa^2} 
    +  \frac{ 2 \mu \hat{c}_4  \hat{c}_6 \norm{X} }{\kappa^3}
    +\frac{4 \mu C \hat{c}_4 \hat{c}_5    \norm{X} }{ \kappa^{3}}+ \frac{11 C \hat{c}_1 \mu \sglmin \left(X\right) }{\kappa^2}   + \frac{ \mu C \hat{c}_2  \sglmin (X)}{\kappa^3} \\
    = & \left( \hat{c}_4 - \mu \left( \frac{\hat{c}_4}{4} - 2 \hat{c}_4  \hat{c}_6 - 4C \hat{c}_4 \hat{c}_5 - 11 C \hat{c}_1 - \frac{ C \hat{c}_2}{\kappa}  \right) \sglmin \left(X\right)  \right) \frac{1}{ \kappa^2 } \\
    \stackrel{(b)}{\le} & \frac{ \hat{c}_4 }{\kappa^2}.
\end{align*}
In inequality $(a)$ we have used the induction hypotheses \eqref{ineq:phase2hyp3} and \eqref{ineq:phase2hyp_balanc3}, inequality \eqref{ineq:phase2aux11}, and the assumption that $\mu \le \frac{ \hat{c}_2 }{\kappa^4 \norm{X}}$.
In inequality $(b)$ we used that the constants $\hat{c}_1$ and $\hat{c}_2$ are chosen small enough compared to $\hat{c}_4$ and, moreover, the constants $\hat{c}_5$ and $\hat{c}_6$ are chosen small enough (compared to $1$).
This shows the \eqref{ineq:phase2hyp3} for $t+1$.

Next, recall from \eqref{ineq:phase2aux4} that $\norm{\Zt \Qtp} \le  \frac{ \hat{c}_3^{4/5}  \sqrt{ \sglmin \left(X\right) } }{ \kappa^{7/2} k^{4/5}} $, which allows us to apply Lemma \ref{lemma:normcontrolled}, which yields $\norm{\Zt}\le 2 \sqrt{\norm{X}} $.
This verifies \eqref{ineq:phase2hyp4} for $t+1$.

Moreover, we note that it follows from Lemma \ref{lemma:balancedbase} that 
\begin{equation*}
    \norm{\tilZplusT \Zplus} \le \norm{\tilZtT \Zt} + 400 \mu^2 \norm{X}^3.
\end{equation*}
Inserting this into the induction hypothesis \eqref{ineq:phase2hyp_balanc1} and using the assumption \eqref{ineq:phase2assump_balanc1}, we obtain that 
\begin{equation*}
    \norm{ \tilZplusT \Zplus} \le  \norm{ \widetilde{Z}_{t_1}^T Z_{t_1} } + 400\mu^2 \left( t+1-t_1 \right) \norm{X}^3    \le  \frac{ \hat{c}_7 \norm{X} }{\kappa^{4}},
\end{equation*}
which proves inequality \eqref{ineq:phase2hyp_balanc1} for $t+1$.
We obtain from Lemma \ref{lemma:balancednessperp} that
     \begin{align*}
        &\norm{\tilZplus^T \Zplus \Qplusp} \le  \norm{ \tilZtT \Zt\Qtp } \\
        &+ C\mu \left(   \left( \norm{\LAPT\P{\Zt\Qt}} + \mu\norm{X} \right)\beta  +  \mu \norm{X}^2 \right)  \sqrt{\norm{ X }} \norm{ \Zt\Qtp }+ 8 \mu \beta \norm{\Zt \Qtp}^2,
    \end{align*}
 where we have set $\beta := \norm{\LAPT \PZQ}  \norm{X}  + \norm{ \Zt \Qtp  }^2 + \norm{\Deltat}$. 
It follows from induction hypothesis \eqref{ineq:phase2hyp3}, inequality \eqref{ineq:phase2aux4}, and inequality \eqref{ineq:phase2aux11} that 
\begin{align*}
\beta
\le \frac{ \left( \hat{c}_4 +\hat{c}_3^{8/5} + 11 \hat{c}_1 \right) \sglmin (X)}{\kappa} .
\end{align*}
We obtain that
\begin{align*}
    &\norm{\tilZplus^T \Zplus \Qplusp} \le  \norm{ \tilZtT \Zt\Qtp } \\
    &+ C\mu \left( \left(  \hat{c}_4 +\hat{c}_3^{8/5} + 11 \hat{c}_1 \right)  \left( \norm{\LAPT\P{\Zt\Qt}} + \mu\norm{X} \right) \frac{\sglmin (X) }{\kappa}+  \mu \norm{X}^2 \right)  \sqrt{\norm{ X }} \norm{ \Zt\Qtp } \\ 
    &+ \frac{ 8 \left( \hat{c}_4 +\hat{c}_3^{8/5} + 11 \hat{c}_1 \right) \mu \sglmin (X)  \norm{\Zt \Qtp}^2}{\kappa} \\
    \stackrel{(a)}{\le} & \norm{ \tilZtT \Zt\Qtp } + C\mu \left( \left( \hat{c}_4 +\hat{c}_3^{8/5} +11 \hat{c}_1 \right)  \left( \frac{ \hat{c}_4}{\kappa^2} + \frac{ \hat{c}_2}{\kappa^4}\right) \frac{\sglmin (X) }{\kappa}+  \frac{ \hat{c}_2 \norm{X}}{\kappa^4} \right)  \sqrt{\norm{ X }} \norm{ \Zt\Qtp }\\ 
    & + 8 \hat{c}_3^{4/5} \left( \hat{c}_4 + \hat{c}_3^{8/5} + 11 \hat{c}_1 \right) \mu \frac{ \sglmin (X) \sqrt{\norm{X}}  \norm{\Zt \Qtp}}{\kappa^{9/2} k^{4/5} }\\
    \le & \norm{ \tilZtT \Zt\Qtp }  + \mu \left( C \left( \hat{c}_4+ \hat{c}_3^{8/5} + 11 \hat{c}_1 \right) \left( \hat{c}_4 +\hat{c}_2\right) + C \hat{c}_2 + 8 \hat{c}_3^{4/5} \left( \hat{c}_4 + \hat{c}_3^{8/5} +11 \hat{c}_1  \right)  \right) \\
    & \cdot \sglmin (X)  \frac{\sqrt{\norm{X}}  \norm{\Zt \Qtp}}{\kappa^3}  \\
    \stackrel{(b)}{\le} & \norm{ \tilZtT \Zt\Qtp }  +  \hat{c}_4 \hat{c}_5 \mu  \sglmin (X)  \frac{ \sqrt{\norm{X}}  \norm{ \Zt \Qtp}}{1500 \kappa^3}  \\
    \stackrel{(c)}{\le} &  \frac{ \hat{c}_4 \hat{c}_5 \sqrt{\norm{X}}}{\kappa^3} \left( 1+ \frac{\mu}{1500} \sglmin(X)  \right)^{t+1-t_1} \norm{\Ztone Q_{t_1,\bot}}.
\end{align*}
Inequality $(a)$ follows from inequalities \eqref{ineq:phase2hyp3}, \eqref{ineq:phase2aux4}, and the assumption $ \mu \le \frac{ \hat{c}_2}{\norm{X} \kappa^4}$.
Inequality $(b)$ follows from the fact that the constants $ \hat{c}_1$, $\hat{c}_2$, and  $ \hat{c}_3$ are chosen small enough (compared to $\hat{c}_4 \hat{c}_5$) and that $\hat{c}_4$ is chosen small enough compared to $\hat{c}_5$.
Inequality $(c)$ is due to inequalities \eqref{ineq:phase2hyp2} and \eqref{ineq:phase2hyp_balanc2}.
This shows \eqref{ineq:phase2hyp_balanc2} for $t+1$.

In order to prove \eqref{ineq:phase2hyp_balanc3} for $t+1$ we note that from Lemma \ref{ref:balancednessangle} it follows that 
    \begin{align*}
        &\norm{ \tilZplus^T P_{\Zplus \Qplus}  } \\
        \le & \left( 1-  \frac{\mu}{4} \sglmin \left( X \right) \right) \norm{\tilZt^T \PZQ } +4 \mu \norm{X} \norm{  \Zt \Qtp  }
        + 2\mu \norm{ \tilZt^T \Zt  } \sqrt{\norm{X}}+ \frac{\mu  \norm{\tilZt^T \Zt\Qtp}  \sglmin (X) }{ \sglmin \left( \Zt\Qt \right)  } \\
        &  +800 \mu^2 \norm{X}^{5/2}\\
        \stackrel{(a)}{\le} & \left( 1-  \frac{\mu}{4} \sglmin \left( X \right) \right) \frac{ \hat{c}_4 \hat{c}_6 \sqrt{\norm{X} } }{\kappa^{3}} +  \frac{ 4\mu \hat{c}_3^{4/5} \sqrt{ \sglmin \left(X\right) } \norm{X} }{ \kappa^{7/2} k^{4/5}}
        + 2\mu  \frac{ \hat{c}_7 \norm{X}^{3/2}  }{\kappa^{4}} + \frac{4 \mu  \hat{c}_4 \hat{c}_5 \sqrt{ \norm{X} } \sglmin (X) }{ \kappa^3 } \\
        &  + 800 \mu \hat{c}_2 \frac{\sqrt{\norm{X}} \sglmin (X) }{\kappa^{3}}\\
        = & \left( \hat{c}_4 \hat{c}_6   - \mu \left(  \frac{  \hat{c}_4 \hat{c}_6}{4} -  \frac{ 4  \hat{c}_3^{4/5} }{ k^{4/5}} -  2  \hat{c}_7-  4   \hat{c}_4 \hat{c}_5  - 800 \hat{c}_2 \right)  \sglmin (X)   \right) \frac{\sqrt{\norm{X}}}{\kappa^3}\\
        \stackrel{(b)}{\le} & \frac{ \hat{c}_4 \hat{c}_6 \sqrt{\norm{X}}}{ \kappa^{3}},
    \end{align*}
where in inequality $(a)$ we have used the induction hypotheses \eqref{ineq:phase2hyp_balanc1} and \eqref{ineq:phase2hyp_balanc3}, the assumption $\mu \le \frac{\hat{c}_2}{\kappa^4 \norm{X}}$, and inequalities \eqref{ineq:phase2aux4} and \eqref{ineq:aux23}.
Inequality $(b)$ follows from the fact that the constants $\hat{c}_2$, $\hat{c}_3$, and $\hat{c}_7$ are chosen small enough compared to $\hat{c}_4 \hat{c}_5 \ll \hat{c}_4 \hat{c}_6$ and $ \hat{c}_5$ is chosen small enough compared to $\hat{c}_6$.

Hence, we have verified the inequalities \eqref{ineq:phase2hyp1},  \eqref{ineq:phase2hyp2},  \eqref{ineq:phase2hyp3},  \eqref{ineq:phase2hyp4}, and the three conditions regarding the imbalance matrix (\eqref{ineq:phase2hyp_balanc1}, \eqref{ineq:phase2hyp_balanc2}, and \eqref{ineq:phase2hyp_balanc3})  for $t+1$. Thus, the induction step is completed.\\

In order to complete the proof it remains to show that inequalities \eqref{ineq:phase2final1}--\eqref{ineq:phase2final_balanc3} hold for $t=t_2$.
For that, we set 
\begin{equation*}
    \gamma := \left( 1 +\frac{\mu \sglmin (X)}{1500} \right)^{t-t_1} \norm{\Ztone Q_{t_1,\bot}}.
\end{equation*}
Note that upper bound in line \eqref{ineq:gammabound} follows directly from \eqref{ineq:phase2aux5} (whereas the lower bound is immediate).
Next, we note that for $t=t_2$ inequality \eqref{ineq:phase2final1} follows directly from the definition of $t_2$.
Inequality \eqref{ineq:phase2final2} is due to \eqref{ineq:phase2aux5} and the definition of $\gamma$.
Moreover, inequality \eqref{ineq:phase2final3}, respectively inequality \eqref{ineq:phase2final4}, follow directly from \eqref{ineq:phase2hyp3}, respectively \eqref{ineq:phase2hyp4}, applied to $t=t_2$.
Analogously, inequalities \eqref{ineq:phase2final_balanc1} and \eqref{ineq:phase2final_balanc3} regarding the imbalance matrix follow from \eqref{ineq:phase2hyp_balanc1} and \eqref{ineq:phase2hyp_balanc3} with $t=t_2$.
Inequality \eqref{ineq:phase2final_balanc2} follows from \eqref{eqref:phase2aux12} with $t=t_2$ and the definition of $\gamma$.

\end{proof}

%% file: lemma_phase3combined.tex
\subsection{Proof of main lemma for Phase 3 (Lemma \ref{lemma:phase3combined})}\label{sec:phase3proof}

\begin{proof}[Proof of Lemma \ref{lemma:phase3combined}]
At the beginning of the proof we would like to recall that the constants
$ \hat{c}_1,\hat{c}_2,\hat{c}_3,\hat{c}_4,\hat{c}_5,\hat{c}_6,\hat{c}_7 $
fulfill the relationships
$\hat{c}_1, \hat{c}_2, \hat{c}_3^{4/5} \ll \hat{c}_4 \hat{c}_5 $, $\hat{c}_4 \ll \hat{c}_5 \ll \hat{c}_6 \ll 1$ since they are chosen exactly as in Lemma \ref{lemma:phase2combined}.

To simplify the notation in the following, we again use the notation 
\begin{equation*}
    \Dt :=\symA - \Zt\ZtT + \tilZt\tilZtT.
\end{equation*}
Note that for $t=t_2$ it holds that 
\begin{equation}\label{ineq:auxfinal1}
    \norm{D_{t_2}} \le \norm{Z_{t_2}}^2 + \norm{\widetilde{Z}_{t_2}}^2 + \norm{X} \le 9 \norm{X},
\end{equation}
where we have used that $ \norm{\tilde{Z}_{t_2} } = \norm{\Zttwo} \le 2 \sqrt{\norm{X}} $ by assumption \eqref{ineq:phase3assump4} and Lemma \ref{lemma:symmetry}.
Set 
\begin{align}
    \tilde{t}&:= \min \left\{ t \in \mathbb{N}: t \ge t_2 \text{ and } \nucnorm{ \Zt \Qtp \QtpT \ZtT   } \ge \frac{ \norm{ \LAT \Dt }}{200}  \right\} \nonumber, \\
    t_3&:= \min \left\{ \tilde{t}; \ t_2 + \Bigg\lfloor \frac{300 \ln \left( \frac{9  \sqrt{\norm{X}} }{200 k \gamma} \right) }{\mu \sglmin (X)} \Bigg\rfloor  \right\}.\label{phase3:tbounddefinition}
\end{align}

We are using an induction argument to show that the following inequalities hold for $ t_2 \le t \le t_3 $.
\begin{align}
    \sglmin (\LAT \Zt ) &\ge \sqrt{\frac{\sglmin (X)}{8}}, \label{ineq:phase3hyp1} \\
    \norm{\Zt \Qtp} & \le  \left(1 + \frac{ \mu}{1500} \sglmin(X) \right)^{t-t_2}    \gamma, \label{ineq:phase3hyp2} \\
    \norm{\LAPT \PZQ} & \le \frac{\hat{c}_4}{\kappa^2}, \label{ineq:phase3hyp3}\\
    \norm{\Zt} &\le 2 \sqrt{\norm{X}}, \label{ineq:phase3hyp4}\\
    \norm{\tilZtT \Zt } &\le  \norm{ \widetilde{Z}_{t_2}^T Z_{t_2} } + 400\mu^2 \left( t-t_2 \right) \norm{X}^3    \le  \frac{\hat{c}_7 \norm{X}}{\kappa^{4}},  \label{ineq:phase3hyp_balanc1}\\
    \norm{ \tilZtT \Zt \Qtp } &\le \frac{ \hat{c}_4 \hat{c}_5 \sqrt{ \norm{X}} }{\kappa^{3} } \left( 1+ \frac{\mu}{1500} \sglmin(X)  \right)^{t-t_2} \gamma,  \label{ineq:phase3hyp_balanc2}\\
    \norm{ \tilZtT \PZQ } &\le \frac{ \hat{c}_4 \hat{c}_6 \sqrt{\norm{X}}}{\kappa^3},  \label{ineq:phase3hyp_balanc3}\\
    \norm{\LAT \Dt} &\le \left(1 - \frac{\mu}{300}\sglmin(X) \right)^{t-t_2} \norm{\LAT D_{t_2}}.\label{ineq:phase3hyp6}
\end{align}
We observe that in the base case $t=t_2$ the above inequalities follow directly from our assumptions. 
Next, we want to show the induction step $t\rightarrow t+1$ for $t_2 \le t < t_3 $.
For that, we first note that
\begin{align}
\norm{\Deltat} =& \norm{(\Bcal^* \Bcal - \I)  \left( \symA - \Zt \ZtT +  \tilZt \tilZtT \right) } \nonumber \\
\le &  \norm{(\Bcal^* \Bcal - \I)  \left( \symA - \Zt \Qt \QtT \ZtT +  \tilZt \Qt \QtT \tilZtT \right) } +\norm{(\Bcal^* \Bcal - \I)  \left(  \Zt \Qtp \QtpT \ZtT  \right) } \nonumber\\
& +\norm{(\Bcal^* \Bcal - \I)  \left(  \tilZt \Qtp \QtpT \tilZtT  \right) } \nonumber \\
\stackrel{(a)}{\le} & \delta \sqrt{r} \norm{\symA - \Zt \Qt \QtT \ZtT +  \tilZt \Qt \QtT \tilZtT  } + \delta \nucnorm{ \Zt \Qtp \QtpT \ZtT   } \nonumber \\
& + \delta \nucnorm{ \tilZt \Qtp \QtpT \tilZtT } \nonumber \\
\stackrel{(b)}{=} & \delta  \sqrt{r}  \norm{\symA - \Zt \Qt \QtT \ZtT +  \tilZt \Qt \QtT \tilZtT  }  + 2\delta \nucnorm{ \Zt \Qtp \QtpT \ZtT   } \nonumber \\
\le &  \delta \sqrt{r}   \norm{\symA - \Zt  \ZtT +  \tilZt  \tilZtT  }  + 2\delta \nucnorm{ \Zt \Qtp \QtpT \ZtT   } +  2\delta \sqrt{r}  \norm{ \Zt \Qtp}^2 \nonumber \\
\stackrel{(c)}{\le} & 6 \delta  \sqrt{r}   \norm{\LAT \Dt} + 2 \delta \nucnorm{ \Zt \Qtp \QtpT \ZtT   } + 6 \delta \sqrt{r}  \norm{ \Zt \Qtp}^2 \nonumber \\ 
\stackrel{(d)}{\le } &  7 \delta  \sqrt{r}   \norm{\LAT \Dt}. \label{phase3:aux2}
\end{align}
In inequality $(a)$ we have used the Restricted Isometry Property (see Lemma \ref{lemma:RIPlemma}) and equality $(b)$ follows from the symmetry between $\Zt$ and $\tilZt$.
Inequality $(c)$ is due to Lemma \ref{SpecLossboundedlemma}.
Inequality $(d)$ follows from $t\le t_3 \le \tilde{t}$ and the definition of $\tilde{t}$. 
This implies that
\begin{align}\label{ineq:auxdeltabound}
    \norm{\Deltat} \le 7 \delta \sqrt{r}  \norm{\LAT \Dt } \stackrel{(a)}{\le} 7 \delta \sqrt{r}  \norm{\LAT D_{t_2}} \stackrel{(b)}{\le} \frac{63 \hat{c}_1}{\kappa^2} \sglmin (X),
\end{align}
where inequality $(a)$ follows from the induction hypothesis \eqref{ineq:phase3hyp6} and inequality $(b)$ follows from the estimate \eqref{ineq:auxfinal1} and the assumption on the RIP-constant, $\delta < \frac{ \hat{c}_1}{\kappa^3 \sqrt{r} }$.

In order to apply our lemmas, we also need to check that the condition 
\begin{equation}\label{ineq:aux25}
\norm{\Zt \Qtp}  \le \min \left\{ \frac{c \sqrt{\sglmin (X)} }{\sqrt{\kappa}};  2 \sglmin \left( \Zt \Qt \right)\right\}
\end{equation}
holds.
Due to \eqref{ineq:phase3hyp1} and  $ \sglmin (\LAT \Zt) \le   \sglmin \left( \Zt \Qt \right)  $ it suffices to check $  \norm{\Zt \Qtp}  \le c \sqrt{\frac{\sglmin (X)}{\kappa}} $. 
For that purpose, note that it follows from inequality \eqref{ineq:phase3hyp2} that
\begin{align}
\norm{\Zt \Qtp} &\le  \left(1 + \frac{ \mu \sglmin(X) }{1500} \right)^{t-t_2}    \gamma  \nonumber \\
& \le  \left(1 + \frac{ \mu  \sglmin(X) }{1500} \right)^{t_3-t_2}    \gamma \nonumber\\
&\stackrel{(a)}{\le} \exp \left(  \frac{1}{5} \ln \left( \frac{9 \sqrt{\norm{X}} }{200k \gamma} \right)  \right) \gamma  \nonumber\\
&=  \left( \frac{9}{200k} \right)^{1/5} \norm{X}^{1/10}  \gamma^{4/5} \label{ineq:phase3_113}\\
&\stackrel{(b)}{\le}  \frac{ c^{4/5}_3 \sqrt{ \sglmin (X)}}{ \kappa^{7/2} } . \label{ineq:phase3_112}
\end{align}
In inequality $(a)$ we used definition \eqref{phase3:tbounddefinition}, $t\le t_3$ and the elementary inequality $ \ln \left(1+x\right) \le x $ for $x > -1$. 
Inequality $(b)$ follows from assumption \eqref{phase3:gammabound}.
In particular, this implies inequality \eqref{ineq:aux25}.

We are now going to show that the conditions \eqref{ineq:phase3hyp1}--\eqref{ineq:phase3hyp6} hold simultaneously for $t+1$.
For that, we first note that it follows from Lemma \ref{lemma:sigmingrowth}, the induction hypotheses \eqref{ineq:phase3hyp1} and \eqref{ineq:phase3hyp4}, and the assumption on the step size, $\mu \le \frac{\hat{c}_2}{\kappa^{4} \norm{X}}$,  that 
\begin{align*}
    \sglmin(\LAT \Zplus) \ge  \sglmin(\LAT \Zplus \Qt)  \ge \sglmin (\LAT \Zt) \left( 1 + \mu \frac{\sglmin (X)}{4} - \mu \sglmin^2 (\LAT \Zt)  \right) \ge  \sqrt{\frac{\sglmin (X)}{8}}. 
\end{align*}
This shows condition \eqref{ineq:phase3hyp1} for $t+1$.

We also note that the above inequality chain implies that $ \LAT \Zplus \Qt $ has full rank.
Thus, we can apply Lemma \ref{lemma:noisetermgrowth} since all conditions are satisfied. 
We obtain that 
\begin{equation*}
\norm{\Zplus\Qplusp} \le \left (1 - \frac{\mu}{2}\norm{\Zt\Qtp}^2 + \frac{ \mu  \sglmin(X)}{3000} \right) \norm{\Zt\Qtp} + 2\mu\sqrt{\norm{X}}\norm{\tilZtT\Zt \Qtp}.
\end{equation*}
We obtain that
\begin{align*}
 \norm{\Zplus\Qplusp} 
 \stackrel{(a)}{\le} & \left (1 - \frac{\mu}{2}\norm{\Zt\Qtp}^2 + \frac{ \mu   \sglmin(X)}{3000} + \frac{ 2\mu c_4 c_5 \norm{X} }{\kappa^{3}} \right) \left( 1+ \frac{\mu}{1500} \sglmin(X)  \right)^{t-t_2} \gamma\\
 \stackrel{(b)}{\le} & \left( 1+ \frac{\mu}{1500} \sglmin(X)  \right)^{t+1-t_2} \gamma.
\end{align*}
Inequality $(a)$ follows from induction hypotheses \eqref{ineq:phase3hyp2} and \eqref{ineq:phase3hyp_balanc2}.
Inequality $(b)$ can be obtained by choosing the absolute constants $\hat{c}_4$ and  $\hat{c}_5$ small enough.
This implies inequality \eqref{ineq:phase3hyp2} for $t+1$.

In order to proceed, we note that
\begin{align}
    \frac{ \norm{\tilZt^T \Zt \Qtp} }{ \sglmin \left(\Zt \Qt \right) } 
    \le & \frac{ \norm{\tilZtT\Zt \Qtp}}{\sglmin\left( \LAT \Zt\right)} \nonumber \\
    \stackrel{(a)}{\le} & \frac{ \sqrt{8} \hat{c}_4 \hat{c}_5 \sqrt{ \norm{X} } \left( 1+ \frac{\mu}{1500} \sglmin(X)  \right)^{t-t_2} \gamma }{ \kappa^{3} \sqrt{\sglmin (X)} } \nonumber\\
    \stackrel{(b)}{\le}  & \frac{\sqrt{8} \hat{c}_4 \hat{c}_5 \sqrt{\norm{X}}}{ \kappa^3 } \cdot  \frac{ \hat{c}^{4/5}_3 \sqrt{ \sglmin (X)}}{ \kappa^{7/2} }   \\
    \le & \frac{4 \hat{c}_4 \hat{c}_5 \sqrt{\norm{X}}}{ \kappa^3 }  , \label{ineq:phase3_aux22}
\end{align}
where inequality (a) follows from induction hypotheses \eqref{ineq:phase3hyp1} and \eqref{ineq:phase3hyp_balanc2}.
Inequality $(b)$ can be seen from similar arguments as in the proof of inequality \eqref{ineq:phase3_112}.
From inequality \eqref{ineq:phase3_112} and inequality \eqref{ineq:phase3_aux22} we see that the assumptions of Lemma \ref{lemma:anglecontrol} are fulfilled. 
Thus, we can apply Lemma \ref{lemma:anglecontrol} and obtain that
\begin{equation}\label{ineq:phase3_111}
\begin{split}
    &\norm{\LAPT\P{\Zplus Q_{t+1} }} \le \\
&\left(1 - \frac{1}{4}\mu\sglmin \left(X\right) \right)\norm{\LAPT\P{\Zt\Qt}} 
+  2 \mu  \sqrt{\norm{X}} \norm{ \tilZtT \P{\Zt\Qt} }
+C\mu\frac{\sqrt{\norm{X}}\norm{\tilZtT\Zt \Qtp}}{\sglmin\left(\Zt\Qt\right)} + C \mu \norm{\Deltat} +  C\mu^2\norm{X}^2.
\end{split}
\end{equation}
Inserting inequality \eqref{ineq:phase3_aux22} into inequality \eqref{ineq:phase3_111} we obtain that
\begin{align*}
    &\norm{\LAPT\P{\Zplus Q_{t+1}}}  \\
    \le &\left(1 - \frac{1}{4}\mu\sglmin \left(X\right) \right)\norm{\LAPT\P{\Zt\Qt}} 
    +  2 \mu  \sqrt{\norm{X}} \norm{ \tilZtT \P{\Zt\Qt} }
    +\frac{4C \hat{c}_4 \hat{c}_5 \mu  \norm{X} }{ \kappa^{3}}+ C \mu \norm{\Deltat} +  C\mu^2\norm{X}^2\\
    \stackrel{(a)}{\le} &\left(1 - \frac{1}{4}\mu\sglmin \left(X\right) \right) \frac{\hat{c}_4}{\kappa^2} 
    +  \frac{ 2 \mu \hat{c}_4  \hat{c}_6 \norm{X} }{\kappa^3}
    +\frac{4 \mu C \hat{c}_4 \hat{c}_5    \norm{X} }{ \kappa^{3}}+ \frac{63 C \hat{c}_1 \mu \sglmin \left(X\right) }{\kappa^2}   + \frac{ \mu C \hat{c}_2  \sglmin (X)}{\kappa^3} \\
    = & \left( \hat{c}_4 - \mu \left( \frac{\hat{c}_4}{4} - 2 \hat{c}_4  \hat{c}_6 - 4C \hat{c}_4 \hat{c}_5 - 63 C \hat{c}_1 - \frac{ C \hat{c}_2}{\kappa}  \right) \sglmin \left(X\right)  \right) \frac{1}{ \kappa^2 } \\
    \stackrel{(b)}{\le} & \frac{ \hat{c}_4 }{\kappa^2}.
\end{align*}
In inequality $(a)$ we have used the induction hypotheses \eqref{ineq:phase3hyp3} and \eqref{ineq:phase3hyp_balanc3}, inequality \eqref{ineq:auxdeltabound}, and the assumption that $\mu \le \frac{ \hat{c}_2 }{\kappa^4 \norm{X}}$.
In inequality $(b)$ we used that the constants $ \hat{c}_1$ and $ \hat{c}_2$ are chosen small enough compared to $\hat{c}_4$ and also that the constants $\hat{c}_5$ and $\hat{c}_6$ are chosen small enough compared to $1$.
This shows the \eqref{ineq:phase3hyp3} for $t+1$.
Next, we note we can apply Lemma \ref{lemma:normcontrolled} since we have that $ \norm{ \Zt \Qtp } \le \frac{\sqrt{\norm{X}}}{100} $ due to \eqref{ineq:phase3_112} with an absolute constant $\hat{c}_3$ chosen small enough.
This implies inequality \eqref{ineq:phase3hyp4} for $t+1$.

Moreover, we obtain from Lemma \ref{lemma:balancedbase} that 
\begin{equation*}
    \norm{\tilZplusT \Zplus} \le \norm{\tilZtT \Zt} + 400 \mu^2 \norm{X}^3.
\end{equation*}
Combining this inequality with the induction hypothesis \eqref{ineq:phase3hyp_balanc1} and using the assumption \eqref{ineq:phase3assump_balanc1}, we obtain that 
\begin{equation*}
    \norm{ \tilZplusT \Zplus} \le  \norm{ \widetilde{Z}_{t_2}^T Z_{t_2} } + 400\mu^2 \left( t+1-t_2 \right) \norm{X}^3    \le  \frac{ \hat{c}_7 \norm{X} }{\kappa^{4}},
\end{equation*}
which shows the induction hypothesis \eqref{ineq:phase3hyp_balanc1} for $t+1$.
We obtain from Lemma \ref{lemma:balancednessperp} that
     \begin{align*}
        &\norm{\tilZplus^T \Zplus \Qplusp} \le  \norm{ \tilZtT \Zt\Qtp } \\
        &+ C\mu \left(   \left( \norm{\LAPT\P{\Zt\Qt}} + \mu\norm{X} \right)\beta  +  \mu \norm{X}^2 \right)  \sqrt{\norm{ X }} \norm{ \Zt\Qtp }+ 8 \mu \beta \norm{\Zt \Qtp}^2,
    \end{align*}
 where we recall that $\beta = \norm{\LAPT \PZQ}  \norm{X}  + \norm{ \Zt \Qtp  }^2 + \norm{\Deltat}$. 
It follows from induction hypothesis \eqref{ineq:phase3hyp3}, inequality \eqref{ineq:phase3_112}, and inequality \eqref{ineq:auxdeltabound} that 
\begin{equation*}
\beta \le \frac{ \left( \hat{c}_4 +\hat{c}_3^{8/5} + 63 \hat{c}_1 \right) \sglmin (X)}{\kappa} .
\end{equation*}
Next, we compute that
\begin{align*}
    &\norm{\tilZplus^T \Zplus \Qplusp} \le  \norm{ \tilZtT \Zt\Qtp } \\
    &+ C\mu \left( \left( \hat{c}_4 +\hat{c}_3^{8/5} + 63 \hat{c}_1 \right)  \left( \norm{\LAPT\P{\Zt\Qt}} + \mu\norm{X} \right) \frac{\sglmin (X) }{\kappa}+  \mu \norm{X}^2 \right)  \sqrt{\norm{ X }} \norm{ \Zt\Qtp } \\ 
    &+ \frac{ 8 \left( \hat{c}_4 +\hat{c}_3^{8/5} + 63 \hat{c}_1 \right) \mu \sglmin (X)  \norm{\Zt \Qtp}^2 }{\kappa}\\
    \stackrel{(a)}{\le} & \norm{ \tilZtT \Zt\Qtp } + C\mu \left( \left( \hat{c}_4 +\hat{c}_3^{8/5} + 63 \hat{c}_1 \right)  \left( \frac{\hat{c}_4}{\kappa^2} + \frac{\hat{c}_2}{\kappa^4}  \right) \frac{\sglmin (X) }{\kappa}+  \frac{ \hat{c}_2 \norm{X}}{\kappa^4} \right)  \sqrt{\norm{ X }} \norm{ \Zt\Qtp }\\ 
    & + 8 \hat{c}_3^{4/5} \left( \hat{c}_4 + \hat{c}_3^{8/5} + 63 \hat{c}_1 \right) \mu \frac{ \sglmin (X) \sqrt{\norm{X}}  \norm{\Zt \Qtp}}{\kappa^{4}}\\
    \le & \norm{ \tilZtT \Zt\Qtp } \\
    & + \mu \left( C \left( \hat{c}_4+\hat{c}_3^{8/5} + 63\hat{c}_1 \right) \left( \hat{c}_4 +\hat{c}_2 \right) + C \hat{c}_2 + 8 \hat{c}_3^{4/5} \left( \hat{c}_4 +\hat{c}_3^{8/5} + 63 \hat{c}_1  \right)  \right) \sglmin (X)  \frac{\sqrt{\norm{X}}  \norm{\Zt \Qtp}}{\kappa^3}  \\
    \stackrel{(b)}{\le} & \norm{ \tilZtT \Zt\Qtp }  +  \hat{c}_4 \hat{c}_5 \mu  \sglmin (X)  \frac{ \sqrt{\norm{X}}  \norm{\Zt \Qtp}}{1500 \kappa^3}  \\
    \stackrel{(c)}{\le} &  \frac{ \hat{c}_4 \hat{c}_5 \sqrt{\norm{X}}}{\kappa^3} \left( 1+ \frac{\mu}{1500} \sglmin(X)  \right)^{t+1-t_2} \gamma.
\end{align*}
Inequality $(a)$ is due to inequalities \eqref{ineq:phase3hyp3}, \eqref{ineq:auxdeltabound}, \eqref{ineq:phase3_112}, and the assumption $ \mu \le \frac{\hat{c}_2}{\norm{X} \kappa^4}$.
Inequality $(b)$ can be seen from the fact that the constants $ \hat{c}_1$, $ \hat{c}_2$, and  $ \hat{c}_3$ are chosen small enough (compared to $ \hat{c}_4$) and that $\hat{c}_4$ is chosen small enough compared to $\hat{c}_5$.
Inequality $(c)$ follows from inequalities \eqref{ineq:phase3hyp2} and \eqref{ineq:phase3hyp_balanc2}.
This shows \eqref{ineq:phase3hyp_balanc2} for $t+1$.

Next, we want to prove \eqref{ineq:phase3hyp_balanc3} for $t+1$.
For that, we apply Lemma \ref{ref:balancednessangle} and obtain that 
    \begin{align*}
        &\norm{ \tilZplus^T P_{\Zplus \Qplus}  } \\
        \le & \left( 1-  \frac{\mu}{4} \sglmin \left( A \right) \right) \norm{\tilZt^T \PZQ } +4 \mu \norm{X} \norm{  \Zt \Qtp  }
        + 2\mu \norm{ \tilZt^T \Zt  } \sqrt{\norm{X}} \\
        & + \frac{\mu  \norm{\tilZt^T \Zt\Qtp}  \sglmin (X) }{ \sglmin \left( \Zt\Qt \right)  }   +800 \mu^2 \norm{X}^{5/2}\\
        \stackrel{(a)}{\le} & \left( 1-  \frac{\mu}{4} \sglmin \left( A \right) \right) \frac{ \hat{c}_4 \hat{c}_6 \sqrt{\norm{X} } }{\kappa^{3}} +  \frac{ 4\mu \hat{c}_3^{4/5} \sqrt{ \sglmin \left(X\right) } \norm{X} }{ \kappa^{7/2}  }
        + 2\mu  \frac{ \hat{c}_7 \norm{X}^{3/2}  }{\kappa^{4}} + \frac{4 \mu  \hat{c}_4 \hat{c}_5 \sqrt{ \norm{X} } \sglmin (X) }{ \kappa^3 }\\
        &   + 800 \mu \hat{c}_2 \frac{\sqrt{\norm{X}} \sglmin (X) }{\kappa^{3}}\\
        = & \left( \hat{c}_4 \hat{c}_6   - \mu \left(  \frac{  \hat{c}_4 \hat{c}_6}{4} -  4  \hat{c}_3^{4/5} -  2  \hat{c}_7-  4   \hat{c}_4 \hat{c}_5  -  800 \hat{c}_2 \right)  \sglmin (X)   \right) \frac{\sqrt{\norm{X}}}{\kappa^3}\\
        \stackrel{(b)}{\le} & \frac{ \hat{c}_4 \hat{c}_6 \sqrt{\norm{X}}}{ \kappa^{3}}.
    \end{align*}
In inequality $(a)$ we have used the induction hypotheses \eqref{ineq:phase3hyp_balanc1} and \eqref{ineq:phase3hyp_balanc3}, the assumption $\mu \le \frac{ \hat{c}_2}{\kappa^4 \norm{X}}$, and inequalities \eqref{ineq:phase3_112} and \eqref{ineq:phase3_aux22}.
Inequality $(b)$ follows from the fact that the constants $\hat{c}_2$, $\hat{c}^{4/5}_3$, and $\hat{c}_7$ are chosen small enough compared to $\hat{c}_4 $ and the constant $ \hat{c}_5$ is chosen small enough compared to the constant $\hat{c}_6$.

Due to inequality \eqref{phase3:aux2} and the assumption $ \delta \le \frac{\hat{c}_1}{ \kappa^3 \sqrt{r} } $ it holds that 
\begin{equation*}
\norm{\Deltat} \le 7 \delta \sqrt{r} \norm{\LAT \Dt} \le \frac{ 7 \hat{c}_1}{\kappa^3} \norm{\Dt}.
\end{equation*}
Hence, we can apply Lemma \ref{lemma:localconvergence}.
We obtain that 
\begin{align*}
&\norm{\LAT(\symA - \Zplus\ZplusT + \tilZplus \tilZplusT)} \\
\le &\left(1 - \frac{\mu}{128}\sglmin(X) \right)\norm{\LAT \left( \symA - \Zt\ZtT + \tilZt\tilZtT \right)} +  \frac{\mu}{20}\sglmin(X)\norm{\Zt\Qtp}^2.
\end{align*}
Moreover, note that due to $t\le t_3 \le \tilde{t}$ we have that
\begin{align*}
    \norm{\Zt\Qtp}^2 \le \nucnorm{ \Zt\Qtp \QtpT \ZtT } \le \frac{\norm{\LAT D_t }}{200}.   
\end{align*} 
Combining the last two inequalities we obtain that
\begin{equation*}
    \norm{\LAT(\symA - \Zplus\ZplusT + \tilZplus \tilZplusT)}   \le \left(1 - \frac{\mu}{300}\sglmin(X) \right)\norm{\LAT \left( \symA - \Zt\ZtT + \tilZt\tilZtT \right)}.   
\end{equation*}
Together with the induction hypothesis the above inequality implies that inequality \eqref{ineq:phase3hyp6} holds for $t+1$.\\

\noindent \textbf{Bounding the final error:}
It remains to prove the final error estimate \eqref{phase3:finalerror}.
For that, we distinguish two cases, namely $t_3=\tilde{t}$ and 
$t_3 = t_2 + \Bigg\lfloor \frac{300 \ln \left( \frac{9  \sqrt{\norm{X}} }{200 k \gamma} \right) }{\mu \sglmin (X)} \Bigg\rfloor$.
We begin with analyzing the case that 
$t_3= t_2+\Bigg\lfloor \frac{300 \ln \left( \frac{9  \sqrt{\norm{X}} }{200 k \gamma } \right) }{\mu \sglmin (X)} \Bigg\rfloor$. 
From Lemma \ref{SpecLossboundedlemma} it follows that 
\begin{align*}
    \norm{\Dthree} \le & 6\norm{\LAT \Dthree} + 4 \norm{Z_{t_3} Q_{t_3, \bot}}^2 \\
    \stackrel{(a)}{\le} & 6 \left( 1 - \frac{\mu}{300}\sglmin(X)  \right)^{t_3-t_2} \norm{\LAT \Dtwo } + 4  \left( \frac{9}{200k} \right)^{2/5} \norm{X}^{1/5} \gamma^{8/5}  \\
    \stackrel{(b)}{\le} & 54 \left( 1 - \frac{\mu}{300}\sglmin(X)  \right)^{t_3-t_2} \norm{X} + 4 \left( \frac{9}{200k} \right)^{2/5} \norm{X}^{1/5} \gamma^{8/5}  , 
\end{align*}
where in $(a)$ we used inequalities \eqref{ineq:phase3hyp6} and \eqref{ineq:phase3_113}. 
Inequality $(b)$ follows from \eqref{ineq:auxfinal1}. 
We observe that
\begin{align*}
\left( 1 - \frac{\mu}{300}\sglmin(X)  \right)^{t_3-t_2} 
& = \exp \left(  \ln \left( 1 - \frac{\mu}{300}\sglmin(X)   \right)   \Bigg\lfloor \frac{300 \ln \left( \frac{9  \sqrt{\norm{X}} }{200 k \gamma } \right) }{\mu \sglmin (X)} \Bigg\rfloor  \right)\\
& \stackrel{(a)}{\le} \exp \left(  - \frac{\mu\sglmin (X)}{300}   \Bigg\lfloor \frac{300 \ln \left( \frac{9  \sqrt{\norm{X}} }{200 k \gamma } \right) }{\mu \sglmin (X)} \Bigg\rfloor  \right)\\
& \stackrel{(b)}{\le} \exp \left(-  \ln \left( \frac{9  \sqrt{\norm{X}}}{200 k \gamma } \right)  + \frac{\mu \sglmin (X)}{300}  \right)\\
&\stackrel{(c)}{\le}  \frac{200 e k   \gamma }{9  \sqrt{\norm{X}}},  
\end{align*}
where $(a)$ follows the elementary inequality $  \ln (1+x) \le x $.  
Inequality $(b)$ follows from $ \lfloor x \rfloor \ge x-1 $.
In inequality $(c)$ we have used our assumption on the step size $\mu$.
Combining the last two estimates it follows that 
\begin{align*}
    \norm{D_{t_3}} &\lesssim  k \gamma \sqrt{\norm{X}}  + \frac{ \gamma  \sqrt{\norm{X} }}{\sqrt{k}} 
    \lesssim k \gamma  \sqrt{\norm{X} }.
\end{align*}
This shows the claim \eqref{phase3:finalerror} in the scenario that $t_3= t_2+\Bigg\lfloor \frac{300 \ln \left( \frac{9  \sqrt{\norm{X}} }{200 k \gamma} \right) }{\mu \sglmin (X)} \Bigg\rfloor$.

Now we consider the scenario that $ t_3 =\tilde{t} $.
In this case, we have that 
\begin{equation}\label{equ:aux223}
  \frac{ \norm{ \LAT D_{t_3} }}{200} \le   \nucnorm{ Z_{t_3} Q_{t_3, \bot}  Q_{t_3, \bot}^T Z_{t_3}^T   }.
\end{equation}
Hence, we can use Lemma \ref{SpecLossboundedlemma} and obtain that
\begin{align*}
    \norm{\Dthree} \le & 6\norm{\LAT \Dthree} + 4 \norm{Z_{t_3} Q_{t_3, \bot}}^2\\
    \le & 6\norm{\LAT \Dthree} + 4 \nucnorm{Z_{t_3} Q_{t_3, \bot}}^2\\
    \stackrel{\eqref{equ:aux223}}{\le} & 5\nucnorm{ Z_{t_3} Q_{t_3, \bot}  Q_{t_3, \bot }^T Z_{t_3}^T   } \\
    \le &  5k \norm{Z_{t_3} Q_{t_3, \bot}}^2 \\
    \stackrel{\eqref{ineq:phase3_113}}{\lesssim} & k^{3/5} \norm{X}^{1/5} \gamma^{8/5} \\
    \stackrel{\eqref{ineq:phase3assump2}}{\le} & \gamma \sqrt{\norm{X} }.
\end{align*}
This implies the claim \eqref{phase3:finalerror} in the case that $t_3 = \tilde{t}$.
Thus, the the proof of Lemma \ref{lemma:phase3combined} is complete.
\end{proof}

%% file: proof_mainresult.tex
\begin{proof}[Proof of Theorem \ref{theorem:main}]
We first note that 
\begin{align*}
    \norm{ \left(\mathcal{A}^* \mathcal{A}\right) \left(X\right) -X } 
    \le \delta \sqrt{r} \norm{X}
    \le \frac{c}{\kappa^2} \sglmin \left(X\right),
\end{align*}
where the first inequality follows from the restricted isometry property (see, e.g., \cite[Lemma 7.3]{stoger2021small}) and the fact that $X$ has rank $r$.
The second inequality follows from our assumption on $\delta$.
Thus, we can appy Lemma \ref{lemma:spectralmain} and we obtain that after $t_1$ iterations,
where $t_1$ satisfies
\begin{equation}\label{main:tonebound}
    t_1  
    \le \frac{17 \ln\left(\frac{6\kappa^{2} \sqrt{ \max \left\{ n_1+n_2; k \right\} } }{c\varepsilon \left( \sqrt{k} - \sqrt{r-1} \right) }\right)}{ \mu \sglmin \left(X\right) },
\end{equation} 
and by choosing the constant $c>0$ in Lemma \ref{lemma:spectralmain} small enough the following inequalities hold with probability at least $1 - C_3 \exp \left( -c_4 k \right) + \left( C_1 \varepsilon  \right)^{k-r+1}  $:
\begin{align}
\sigma_{\min} \left( Z_{\tone}Q_{\tone} \right) 
&\ge \left(\frac{2\kappa^{2}  }{c}\right)^{8\kappa} \frac{\alpha \varepsilon \left( \sqrt{k} -\sqrt{r-1} \right)}{4} ,\label{ineq:proofmain12} \\
\norm{Z_{\tone}Q_{\tone,\bot}} 
&\le \min \left\{ 2 \sglmin \left( Z_{\tone} Q_{\tone} \right); \  \left( \frac{6\kappa^2  \sqrt{\max \left\{  n_1+n_2; k \right\} }}{ c \varepsilon \left( \sqrt{k} - \sqrt{r -1} \right)} \right)^{ 16 \kappa} \cdot \frac{12c \alpha \sqrt{\max \left\{ n_1 + n_2; k \right\} }}{ \kappa^2 } \right\} , \label{ineq:proofmain1} \\
\norm{L_{X,\bot}^{T} P_{Z_{\tone}Q_{\tone}}}&\le \frac{ \hat{c}_4 }{\kappa^2}, \label{main:intern2} \\
\norm{\Ztone} &\le 2 \sqrt{\norm{X}}, \nonumber \\
\norm{\tilde{Z}_{\tone}^T Z_{\tone}  }
&\le  \frac{ \hat{c}_7\norm{X}}{3 \kappa^4}, \label{main:balancintern1} \\ 
\norm{\tilde{Z}_{\tone}^T Z_{\tone} Q_{\tone, \bot}  }
&\le  \frac{\hat{c}_4 \hat{c}_5 \sqrt{\norm{X}} }{\kappa^3} \norm{Z_{\tone} Q_{\tone,\bot} }, \nonumber \\ 
\norm{\tilde{Z}_{\tone}^T P_{ Z_{\tone} Q_{\tone} }  }
&\le \frac{ \hat{c}_4 \hat{c}_6 \sqrt{\norm{X}} }{\kappa^3}  . \nonumber
\end{align}
Here and in the following $\hat{c}_1, \hat{c}_2, \hat{c}_3, \hat{c}_4, \hat{c}_5, \hat{c}_6, \hat{c}_7$ denote the constants in Lemmas \ref{lemma:phase2combined} and \ref{lemma:phase3combined}.
Next, we want to apply Lemma \ref{lemma:phase2combined}, which describes the second convergence phase.
We assume without loss of generality that $ \sglmin \left( \LAT \Ztone \right) \le \sqrt{\frac{\sglmin \left(X\right)}{8}} $ (otherwise we can skip the second convergence phase and go directly to the third phase).
Thus, all assumptions in Lemma \ref{lemma:phase2combined} except \eqref{ineq:phase2assump1} and \eqref{ineq:phase2assump_balanc1} can be immediately deduced from the above inequalities.
To see that assumption \eqref{ineq:phase2assump1} is also fulfilled we check that
\begin{align*}
    \norm{Z_{\tone} Q_{\tone,\bot} } 
    \stackrel{(a)}{\le} &\left( \frac{6\kappa^2  \sqrt{\max \left\{  n_1+n_2; k \right\} }}{ c \varepsilon \left( \sqrt{k} - \sqrt{r -1} \right)} \right)^{ 16 \kappa} \cdot \frac{12c \alpha \sqrt{\max \left\{ n_1 + n_2; k \right\} }}{ \kappa^2 } \\
    \stackrel{(b)}{\le} & \frac{ \hat{c}_3 \sqrt{\sglmin (X)}}{ k \kappa^{35/8}}. 
\end{align*}
Here, inequality $(a)$ follows from \eqref{ineq:proofmain1}.
Inequality $(b)$ follows from our assumption on the scale of initialization $\alpha$ and by choosing the constant $C_1$ sufficiently large.
This shows that condition \eqref{ineq:phase2assump1} is satisfied.
Next, we check that condition \eqref{ineq:phase2assump_balanc1} is satisfied.
For that, we note first that
\begin{align}
 \Bigg\lceil  \frac{\ln \left( \frac{ \sqrt{ \sglmin \left(X\right) } }{\sqrt{8} \sglmin \left( \LAT \Ztone \right) }   \right)}{\ln \left( 1 +\frac{\mu}{8}  \sglmin \left(X\right)   \right)} \Bigg\rceil 
 \le &\frac{\ln \left( \frac{ \sqrt{ \sglmin \left(X\right) } }{\sqrt{8} \sglmin \left( \LAT \Ztone \right) }   \right)}{\ln \left( 1 +\frac{\mu}{8}  \sglmin \left(X\right)   \right)} +1 \nonumber \\
 \stackrel{(a)}{\le} &\frac{16\ln \left( \frac{ \sqrt{ \sglmin \left(X\right) } }{\sqrt{8} \sglmin \left( \LAT \Ztone \right) }   \right)}{\mu \sglmin \left(X\right)} +1\nonumber\\
 \stackrel{(b)}{\le} &\frac{16\ln \left( \frac{ \sqrt{ \sglmin \left(X\right) } }{\sqrt{2} \sglmin \left( \Ztone Q_{\tone} \right) }   \right)}{\mu \sglmin \left(X\right)} +1\nonumber\\
 \stackrel{(c)}{\le} &\frac{16\ln \left( \frac{ 2 \sqrt{2 \sglmin \left(X\right) }  }{ \alpha \varepsilon \left( \sqrt{k} -\sqrt{r-1}  \right)} \cdot \left( \frac{c}{2k^2} \right)^{8 \kappa}      \right)}{\mu \sglmin \left(X\right)} +1\nonumber\\
 \le &\frac{17\ln \left( \frac{ 2  \sqrt{ 2 \sglmin \left(X\right) }  }{ \alpha \varepsilon \left( \sqrt{k} -\sqrt{r-1}  \right)}  \right)}{\mu \sglmin \left(X\right)}.\label{ineq:mainresultsintern1}
\end{align}
Inequality $(a)$ follows from $\frac{x}{1-x} \le \ln (1+x) $ for $ x > -1 $ and from our assumption on the step size $\mu$.
Inequality $(b)$ is due to $\sglmin \left( \LAT \Ztone \right) \ge \frac{1}{2} \sglmin \left( \Ztone Q_{\tone} \right) $, which follows from \eqref{main:intern2}.
Inequality $(c)$ is due to \eqref{ineq:proofmain12}.
Thus, it follows that 
\begin{align}
    \norm{\tilde{Z}_{\tone}^T Z_{\tone}  }+
    400 \mu^2 \Bigg\lceil  \frac{\ln \left( \frac{ \sqrt{ \sglmin \left(X\right) } }{\sqrt{8} \sglmin \left( \LAT \Ztone \right) }   \right)}{\ln \left( 1 +\frac{\mu}{8}  \sglmin \left(X\right)   \right)} \Bigg\rceil \norm{X}^3
     \stackrel{\eqref{main:balancintern1}}{\le}
     &
    \frac{\hat{c}_7\norm{X}}{3 \kappa^4}+
    400 \mu^2 \Bigg\lceil  \frac{\ln \left( \frac{ \sqrt{ \sglmin \left(X\right) } }{\sqrt{8} \sglmin \left( \LAT \Ztone \right) }   \right)}{\ln \left( 1 +\frac{\mu}{8}  \sglmin \left(X\right)   \right)} \Bigg\rceil \norm{X}^3 \nonumber \\
    \stackrel{\eqref{ineq:mainresultsintern1}}{\le} &
    \frac{\hat{c}_7 \norm{X}}{3 \kappa^4}+
    400 \cdot 17 \mu \kappa \ln \left( \frac{ 2  \sqrt{ 2\sglmin \left(X\right) }  }{ \alpha \varepsilon \left( \sqrt{k} -\sqrt{r-1}  \right)}  \right) \norm{X}^2 \nonumber \\
    \stackrel{(b)}{\le} &
    \frac{2 \hat{c}_7 \norm{X}}{3 \kappa^4}. \label{ineq:mainresultsintern2}
\end{align}
In inequality $(b)$ we used the assumption on our step size, see \eqref{main:stepsizeassumption}.
Thus, condition \eqref{ineq:phase2assump_balanc1} is satisfied and we can apply Lemma \ref{lemma:phase2combined}.
Thus, we obtain that after $t_2$ iterations, where 
\begin{equation}\label{main:ttwobound}
    t_2 - t_1 
    \le \Bigg\lceil \frac{\ln \left( \frac{ \sqrt{ \sglmin \left(X\right) } }{\sqrt{8} \sglmin \left( \LAT \Ztone \right) }   \right)}{\ln \left( 1 +\frac{\mu}{8}  \sglmin \left(X\right)   \right)} \Bigg\rceil
    \stackrel{\eqref{ineq:mainresultsintern1}}{\le}
     \frac{17\ln \left( \frac{ 2  \sqrt{ 2 \sglmin \left(X\right) }  }{ \alpha \varepsilon \left( \sqrt{k} -\sqrt{r-1}  \right)}  \right)}{\mu \sglmin \left(X\right)},
\end{equation}
that our iterates $\Zttwo$ and $\widetilde{Z}_{t_2}$ satisfy 
\begin{align}
    \sglmin \left( \LAT \Zttwo \right) &\ge  \sqrt{\frac{\sglmin (X)}{8}}, \nonumber \\
    \norm{\Zttwo Q_{t_2,\bot}} 
    & \le  \gamma ,   \nonumber  \\
    \norm{\LAPT P_{Z_{t_2}} Q_{t_2} } & \le \frac{ \hat{c}_4}{\kappa^2}, \nonumber  \\
    \norm{\Zttwo} &\le 2 \sqrt{\norm{X}},  \nonumber  \\
    \norm{ \widetilde{Z}_{t_2}^T Z_{t_2}  } &\le \norm{ \widetilde{Z}_{t_1}^T Z_{t_1} } + 400\mu^2 \left( t_2-t_1 \right) \norm{X}^3 \stackrel{\eqref{ineq:mainresultsintern2}}{\le} \frac{2 \hat{c}_7 \norm{X}}{3\kappa^4} , \label{ineq:mainintern17}  \\ 
    \norm{ \widetilde{Z}_{t_2}^T Z_{t_2} Q_{t_2,\bot} }  
    &\le  \frac{ \hat{c}_4 \hat{c}_5 \sqrt{\norm{X}} \gamma }{ \kappa^3} , \nonumber \\
    \norm{\widetilde{Z}_{t_2}^T P_{Z_{t_2} Q_{t_2 } }  }  &\le  \frac{ \hat{c}_4 \hat{c}_6 \sqrt{\norm{X}}}{\kappa^3},  \nonumber
\end{align}
where
\begin{equation}\label{main:gammadefinition}
  \gamma
   = \max \left\{  \left( \frac{1}{128} \right)^{1/10}  \left( \sglmin \left(X\right) \right)^{1/10}   \norm{ \Ztone Q_{t_1,\bot}}^{4/5}; \alpha \right\} .
\end{equation}
(Note that Lemma \ref{lemma:phase2combined} yields a somewhat stronger bound on $\gamma$, see inequality \eqref{ineq:gammabound}.
However, additionally requiring that $\alpha \le \gamma$ will make the proof below slightly more convenient.)
In the next step, we want to apply Lemma \ref{lemma:phase3combined}, which describes the third phase, i.e., the local convergence phase.
Again, we need to check that the conditions are satisfied.
Note that if $\gamma = \alpha$, then condition \eqref{phase3:gammabound} holds due to our assumption on $\alpha$, see \eqref{main:alphabound}.
In the other case, we note that
\begin{align*}
    \gamma 
    &=
   \left( \frac{1}{128} \right)^{1/10}  \left( \sglmin \left(X\right) \right)^{1/10}   \norm{ \Ztone Q_{t_1,\bot}}^{4/5}\\ 
   &\stackrel{\eqref{ineq:proofmain1}}{\le} \left( \frac{1}{128} \right)^{1/10} \left( \sglmin \left(X\right) \right)^{1/10} \left( \frac{6\kappa^2  \sqrt{\max \left\{  n_1+n_2; k \right\} }}{ c \varepsilon \left( \sqrt{k} - \sqrt{r -1} \right)} \right)^{ \frac{ 64 \kappa}{5}} \cdot \frac{ \left( 12c \sqrt{\max \left\{ n_1 + n_2; k \right\} }\right)^{4/5} \alpha^{4/5} }{ \kappa^{8/5} }\\
   & \le \hat{c}_3 \min \left\{  \frac{\sqrt{\sglmin(X)}}{\kappa^{9/2}}   ;\frac{\sqrt{\norm{X}}}{k^{4/3}}  \right\},
\end{align*}
where for the last line we have used our assumption on $\alpha$, see \eqref{main:alphabound},
which shows condition \eqref{phase3:gammabound}.
It remains to verify condition \eqref{ineq:phase3assump_balanc1}.
For that, we note that
\begin{align*}
    120000 \mu \frac{ \ln \left( \frac{9 \sqrt{\norm{X}}}{200 k \gamma} \right) \norm{X}^3 }{ \sglmin \left(X\right) }
    &\stackrel{\alpha \le \gamma }{\le} 120000 \mu  \kappa \ln \left( \frac{9 \sqrt{\norm{X}}}{200 k \alpha} \right)  \norm{X}^2 
    \le \frac{\hat{c}_7 \norm{X}}{3 \kappa^4},
\end{align*}
where we have used our assumption on the step size $\mu$.
Combining this estimate with inequality \eqref{ineq:mainintern17} we see that condition \eqref{ineq:phase2assump_balanc1} is fulfilled.
Thus, from Lemma \ref{lemma:phase3combined} we obtain that after $t_3$ iterations, where
\begin{align}\label{main:tthreebound}
    t_3-t_2 \le \frac{300 \ln \left( \frac{9  \sqrt{\norm{X}} }{200 k \gamma} \right) }{\mu \sglmin (X)},
\end{align}
the following bound for the reconstruction error holds:
\begin{align*}
    \norm{V_{t_3} W_{t_3}^T - X}
    =&\norm{\symA - Z_{t_3} Z^T_{t_3}+  \tilde{Z}_{t_3} \tilde{Z}^T_{t_3} }\\
    \lesssim &  k \gamma \sqrt{\norm{X} } 
\end{align*}
If $\gamma = \alpha$ holds, then we note that 
\begin{equation*}
    \frac{\norm{V_{t_3} W_{t_3}^T - X}}{\norm{X}}
    \lesssim 
    \frac{k \alpha}{ \sqrt{ \norm{X} } } 
    \lesssim  
    \frac{ \alpha^{3/5} }{  \norm{X}^{3/10}  },
\end{equation*}
where in the last line we used assumption our assumption on $\alpha$, which is inequality \eqref{main:alphabound}.
Thus, in the case that $\alpha = \gamma$ inequality \eqref{equ:mainresultfinalbound} follows for $T=t_3$.
Otherwise, we obtain 
that
\begin{align*}
    \norm{V_{t_3} W_{t_3}^T - X} 
    \stackrel{\eqref{main:gammadefinition}}{\lesssim } & k  \sqrt{\norm{X} } \left( \sglmin \left(X\right) \right)^{1/10}   \norm{ \Ztone Q_{t_1,\bot}}^{4/5} \\
    \stackrel{\eqref{ineq:proofmain1}}{\lesssim} &   k  \sqrt{\norm{X} } \left( \sglmin \left(X\right) \right)^{1/10}   
    \cdot \left( \frac{6\kappa^2  \sqrt{\max \left\{  n_1+n_2; k \right\} }}{ c \varepsilon \left( \sqrt{k} - \sqrt{r -1} \right)} \right)^{ \frac{64 \kappa}{5} } \cdot \frac{ \alpha^{4/5} \left( \max \left\{ n_1 + n_2; k \right\} \right)^{2/5} }{ \kappa^{8/10} } \\
    \stackrel{}{\lesssim} &
    \norm{X}^{7/10} \alpha^{3/5},
\end{align*}
where the last line follows from our assumption on $\alpha$, which is inequality \eqref{main:alphabound}, and by choosing the constant $C_2 $ therein large enough.
Thus, by rearranging terms it follows that 
\begin{align*}
    \frac{\norm{V_{t_3} W_{t_3}^T - X}}{\norm{X}}
    \lesssim
    \frac{\alpha^{3/5}}{\norm{X}^{3/10} }.
\end{align*}
This shows inequality \eqref{equ:mainresultfinalbound} with $T=t_3$.
It remains to show that $T$ fulfills inequality \eqref{lemma:tildetbound}. 
For that we note that 
\begin{align*}
    T
    &=t_1 + \left( t_2 - t_1 \right) + \left(t_3 -t_2\right)\\
    &\overset{\eqref{main:tonebound},\eqref{main:ttwobound},\eqref{main:tthreebound}}{\lesssim} \frac{ \ln\left(\frac{6\kappa^{2} \sqrt{ \max \left\{ n_1+n_2; k \right\} } }{c\varepsilon \left( \sqrt{k} - \sqrt{r-1} \right) }\right)}{ \mu \sglmin \left(X\right) }
    +\frac{\ln \left( \frac{ 2  \sqrt{2 \sglmin \left(X\right) }  }{ \alpha \varepsilon \left( \sqrt{k} -\sqrt{r-1}  \right)}  \right)}{\mu \sglmin \left(X\right)}
    +\frac{ \ln \left( \frac{9  \sqrt{\norm{X}} }{200 k \gamma} \right) }{\mu \sglmin (X)}\\
    &\stackrel{(a)}{\le} \frac{ \ln\left(\frac{6\kappa^{2} \sqrt{ \max \left\{ n_1+n_2; k \right\} } }{c\varepsilon \left( \sqrt{k} - \sqrt{r-1} \right) }\right)}{ \mu \sglmin \left(X\right) }
    +\frac{\ln \left( \frac{ 2  \sqrt{ 2 \sglmin \left(X\right) }  }{ \alpha \varepsilon \left( \sqrt{k} -\sqrt{r-1}  \right)}  \right)}{\mu \sglmin \left(X\right)}
    +\frac{ \ln \left( \frac{9  \sqrt{\norm{X}} }{200 k \alpha} \right) }{\mu \sglmin (X)}\\
    &\overset{(b)}{\lesssim} \frac{ \ln \left( \frac{ 2 \sqrt{2\norm{X}} }{  \varepsilon \alpha \left( \sqrt{k} - \sqrt{r-1} \right) } \right) }{\mu \sglmin (X)}.
\end{align*}
where in inequality $(a)$ we used that $\alpha \le \gamma$, see inequality \eqref{main:gammadefinition}.
Inequality $(b)$ follows from the assumption on $\alpha$, see \eqref{main:alphabound}.
This shows inequality \eqref{lemma:tildetbound} and the proof of Theorem \ref{theorem:main} is complete.
\end{proof}